\documentclass[10pt,english]{article}
\usepackage[T1]{fontenc}
\usepackage{upgreek}
\usepackage{geometry}
  \usepackage{comment}
\usepackage{amsmath}
\usepackage{titlesec}
\usepackage{amsmath}
 \titleclass{\subsubsubsection}{straight}[\subsubsection]

\newcounter{subsubsubsection}[subsubsection]
\renewcommand\thesubsubsubsection{\thesubsubsection.\arabic{subsubsubsection}}

\titleformat{\subsubsubsection}[runin]
  {\normalfont\normalsize\bfseries}{\thesubsubsubsection}{1em}{}

\titlespacing*{\subsubsubsection}{0pt}{3.25ex plus 1ex minus .2ex}{1em}

\usepackage[showonlyrefs]{mathtools}
\usepackage{amsthm}
\usepackage{setspace}
\usepackage[ruled,vlined]{algorithm2e}
\usepackage[affil-it]{authblk}
\usepackage[english]{babel}
\usepackage{blindtext}
\usepackage{graphicx,psfrag,epsf}
\usepackage{amssymb}
\usepackage[authoryear]{natbib}
\usepackage[unicode=true,
 bookmarks=false,
 breaklinks=false,pdfborder={0 0 1},colorlinks=true,citecolor = blue]{hyperref}
\usepackage{caption}
\usepackage{subcaption}
\usepackage[table]{xcolor}
\usepackage{bm}
\usepackage{color}
\usepackage{tikz}
\usepackage{stmaryrd}
\usepackage{url}
\usepackage{mathrsfs}  
\usepackage{amssymb}
\usepackage{subcaption}
\usepackage{etoolbox}
\usepackage{cellspace}
\usepackage{amsmath}
\usepackage{kantlipsum}
\newcommand{\ER}{Erd\H{o}s--R\'enyi}

\newcommand{\normal}{\text{$\mathrm{N}$}}
\newcommand\eqi{\mathrel{\stackrel{\makebox[0pt]{\mbox{\normalfont\tiny (i)}}}{=}}}

\newcommand\eqii{\mathrel{\stackrel{\makebox[0pt]{\mbox{\normalfont\tiny (ii)}}}{=}}}

\newcommand\eqiii{\mathrel{\stackrel{\makebox[0pt]{\mbox{\normalfont\tiny (iii)}}}{=}}}

\newcommand{\lp}{\lambda^{-1}_{\min,\perp}(\M L(\M A))}

\newcommand{\PP}{\mathbb{P}}
\newcommand{\EE}{\mathbb{E}}

\newcommand{\itt}{{\mathbf i}}
\newcommand{\jtt}{{\mathbf j}}
\newcommand{\M}{\mathbf}

\newcommand{\fav}{\mathcal{Q}}

\newcommand{\bee}{\begin{equation}\begin{aligned}}
	\newcommand{\ee}{\end{aligned}\end{equation}}
\newcommand{\bea}{\begin{eqnarray*}}

\newcommand{\RR}{\mathbb{R}}
\newcommand{\pr}{\mathbb{P}}
\newcommand{\bx}{{{\M x}}}

\newcommand{\lap}{\pmb{\mathscr{L}}}
\newcommand{\midplus}{\;\middle|\;}
\newcommand{\eea}{\end{eqnarray*}}
\newcommand{\iz}{i_0}
\newcommand{\jz}{j_0}
\newcommand{\be}{\begin{eqnarray}}
\newcommand{\bay}{\begin{array}}
\newcommand{\eay}{\end{array}}
\newcommand{\bi}{\begin{itemize}}
\newcommand{\ei}{\end{itemize}}
\newcommand{\ben}{\begin{enumerate}}
\newcommand{\een}{\end{enumerate}}
\newcommand{\bcen}{\begin{center}}
\newcommand{\ecen}{\end{center}}

\newcommand{\bds}{\boldsymbol}

\def\E{\mathbb{E}}

\def\R{ \mathbb{R} }

\newcommand{\T}{\mathrm{\scriptscriptstyle T}}

\newtheorem{theorem}{Theorem}%
\newtheorem{lemma}{Lemma}%
\newtheorem{proposition}{Proposition}%
\newtheorem{corollary}{Corollary}%
\newtheorem{assumption}{Assumption}%
\newtheorem{remark}{Remark}
\newtheorem{definition}{Definition}

\DeclareMathOperator*{\argmin}{arg\,min}

\let\hat\widehat
\let\tilde\widetilde

\begin{document}
\title{\Large Fisher Random Walk: Automatic Debiasing Contextual Preference Inference for Large Language Model Evaluation}
\date{}
\author{
  Yichi Zhang$^{1}$, Alexander Belloni$^{2}$, Ethan X. Fang$^{3}$,
  Junwei Lu$^{4}$\footnote{Corresponding author. Email: {junweilu@hsph.harvard.edu}}, Xiaoan Xu$^{3}$\bigskip \\
\small 
\vspace{-15pt}
$^1${Department of Statistics, Indiana University Bloomington} \\
\small 
$^2${Fuqua School of Business, Duke University} \\
\small 
$^3${Department of Biostatistics \& Bioinformatics, Duke University}\\
\small
$^4${Department of Biostatistics, Harvard T.H. Chan School of Public Health}}
\maketitle
\vspace{-20pt}
\begin{abstract}
Motivated by the need for rigorous and scalable evaluation of large language models, we study contextual preference inference for pairwise comparison functionals of context-dependent preference score functions across domains. Focusing on the contextual Bradley-Terry-Luce model, we develop a semiparametric efficient estimator that automates the debiased estimation through aggregating weighted residual balancing terms across the comparison graph. We show that the efficiency is achieved when the weights are derived from a novel strategy called Fisher random walk. We also propose a computationally feasible method to compute the weights by a potential representation of nuisance weight functions. We show our inference procedure is valid for general score function estimators accommodating the practitioners' need to implement flexible deep learning methods. We extend the procedure to multiple hypothesis testing using a Gaussian multiplier bootstrap that controls familywise error and to distributional shift via a cross-fitted importance-sampling adjustment for target-domain inference. Numerical studies, including language model evaluations under diverse contexts, corroborate the accuracy, efficiency, and practical utility of our method.
\end{abstract}

\allowdisplaybreaks
 \section{Introduction}

Various large language models (LLMs), such as ChatGPT \citep{openai2023gpt}, Claude \citep{anthropic2024claude}, Llama \citep{touvron2023llama}, and DeepSeek \citep{bi2024deepseek}, excel across tasks including translation, data analysis, code generation, medical assistance, and reasoning \citep{ma2023insightpilot,ni2023lever,zhang2024llm,buhr2023chatgpt}. This rapid proliferation demands rigorous and scalable evaluation of the performance of different LLMs. Numerous benchmarks have been proposed to assess LLMs' specific capabilities (e.g., sentiment, classification, math, robustness, fairness) \citep{bang2023multitask,collins2024evaluating,dao2023investigating,didolkar2024metacognitive,pena2023leveraging,wang2023robustness,zhuo2023robustness,yang2023large,zemel2013learning,hardt2016equality} or provide holistic LLM rankings \citep{bommasani2023holistic,zheng2024judging,li2023alpacaeval,chiang2024chatbot}. However, the benchmark-based evaluation methods cannot be scaled up as high-quality expert annotations are costly and scarce. In practice, pairwise human preference evaluation, which asks annotators to choose the preferred answer, offers a cheaper, scalable alternative and is widely used both for tuning via RLHF \citep{christiano2017deep,stiennon2020learning,ouyang2022training} and for model evaluation \citep{li2023alpacaeval,chiang2024chatbot}. For example,
Chatbot Arena \citep{chiang2024chatbot} employs the parametric Bradley-Terry-Luce (BTL) model \citep{debreu1960individual,bradley1952rank} to estimate a global latent preference score for each model and ranks all LLMs accordingly. However, relying solely on such global scores can obscure the heterogeneity of LLM performance across different contexts or domains (e.g., varying tasks or question topics). Some LLMs may excel in medical domains, while others perform better in quantitative fields. A single global score cannot fully capture the specialized capabilities of each LLM. Furthermore, it is also of practical interest to construct an uncertainty assessment for the quality of LLM evaluations.

This paper aims to develop an inference method to more accurately and reproducibly compare the abilities of various LLMs across different context domains. We formalize the problem as {\it contextual preference inference}.  Given a random context \(\M X\), there are $n$ LLMs with unknown context-dependent preference scores \(\theta_1^*(\M X), \dots, \theta_n^*(\M X)\). When comparing the outputs from items $i$ and $j$, the annotator will label the preference $Y_{ij}(\M X)$ such that $Y_{ij}(\M X) = 1$ if $\text{output from $i$ is preferred over $j$}$ and $0$ otherwise. We assume that for each $i,j$, $Y_{ij}(\M X)$ is an independent Bernoulli random variable with probability only depending on the score difference, i.e., $\mathbb{E}[Y_{ij}(\M X)|\M X = \M x] = \psi(\theta^*_i(\M x) - \theta^*_j(\M x))$ where $\psi$ could be any function mapping to $[0,1]$. In this paper, we will focus on the BTL model, which considers $\psi(t) = {1}/(1 + \exp(-t))$ for the simplicity of presentation. Motivated by annotation cost constraints, we will allow for settings where we do not observe $Y_{ij}(\M X)$ for all pairs $(i,j) \in [n]^2$ but only a subset of pairs on the comparison graph encoded by the adjacency matrix $\M A = [A_{ij}]_{n \times n}$. Namely, we will only compare items $i$ and $j$ and observe $Y_{ij}(\M X)$ when $A_{ij} = 1$. The primary goal of the contextual preference inference is to compare two items of interest, $\iz$ and $\jz$, within a contextual domain $\Omega$. Specifically, denote the indicator function as $\mathbb{I}(\cdot)$ and we consider the parameter of interest
\begin{equation}\label{def:qijomega}
  \fav_{\iz\jz}(\Omega) 
  = \mathbb{E}\left[\mathbb{I}(\mathbf{X} \in \Omega)\left(\theta^*_{\iz}(\mathbf{X}) - \theta^*_{\jz}(\mathbf{X})\right)\right],
  \end{equation}
where $\mathbf{X}$ is a random prompt input, and the preference hypothesis 
\begin{equation} \label{test:functional:intro}
  \mathrm{H}_{0}: \fav_{\iz\jz}(\Omega) \le 0, \text{i.e., $\jz$ is preferred over $\iz$}
  \quad \text{vs.} \quad 
  \mathrm{H}_{1}: \fav_{\iz\jz}(\Omega) > 0, \text{i.e., $\iz$ is preferred over $\jz$}
  \end{equation}

The contextual preference inference is essentially a statistical inference problem for $\fav_{\iz\jz}(\Omega)$ as a functional of the nonparametric nuisances ${\boldsymbol{\theta}}^*(\cdot) = (\theta^*_1(\cdot), \dots, \theta^*_n(\cdot))^\T$. \citet{chernozhukov2021automatic} have studied the efficient influence function of the functional of the regression function $\theta_0(x) = \mathbb{E}[Y|X=x]$. They proposed to debias the plug-in estimator by a weighted regression residual term where the balancing weight function is the Riesz representer of the functional's linearization. This method is called automatic debiased machine learning (autoDML) and has been generalized to general functionals of nonparametric models \citep{Chernozhukov2022b,Ichimura2022,chernozhukov2023automatic,vanderLaan2025}. For finite-dimensional regression problems, autoDML becomes the weighted residual balancing method \citep{zubizarreta2015stable,imai2014covariate,athey2018approximate} which can be viewed as a generalization of both the one-step estimator using the Newton-Raphson method \citep{van2000asymptotic} and the augmented inverse probability weighted estimator \citep{robins1995analysis,robins2000marginal}. Comparing with these debiasing methods for a single prediction problem, the contextual preference inference considers multiple regression problems: each one $\mathbb{E}[Y_{ij}(\M X)|\M X = \M x]$ is defined on an edge $(i,j)$ of the comparison graph.
Such difference makes the existing weighted residual balancing methods not directly applicable due to the following two major challenges. First, items \(\iz\) and \(\jz\) may not be directly compared and $Y_{\iz\jz}$ may not be observed, in which case neither the regression residual nor the balancing weight function proposed in \citet{chernozhukov2021automatic,Chernozhukov2022b} could be computed directly. Second, even when a direct comparison exists between items \(\iz\) and \(\jz\), it will be statistically inefficient to only utilize samples comparing between \(\iz\) and \(\jz\), ignoring all other comparisons which also involve information about $\theta^*_{\iz}$ and $\theta^*_{\jz}$.

To address these issues, we present a debiasing estimator for contextual preference inference based on a novel strategy called {\it Fisher random walk} to aggregate the estimators across the edges on the comparison graph. On each edge $(i,j)$ of the comparison graph, as $Y_{ij}$ is observed, we can have a debiasing term as the weighted residuals for $Y_{ij}$ regression. Our idea for inferring the functional of interest $\fav_{\iz\jz}(\Omega)$ is to average over all these residual terms across the comparison graph with the average weights derived from a prior induced by a random walk on the comparison graph. We show that if the random walk has the transition probability proportional to the Fisher information of the regression model on each edge, the debiasing procedure for the preference inference will be semiparametric efficient. We further propose a computationally feasible method to compute the Fisher random walk induced prior by identifying a novel potential representation of the residual balancing weights. Such representation shows that the nuisance weight functions defined on the comparison graph edges can be reformulated as the differences of some potential functions defined on the nodes of the comparison graph. Therefore, we only need to estimate $O(n)$ nuisance potential functions instead of $O(n^2)$ nuisance weight functions, which significantly eases the computational complexity of our inference method.
We prove that the proposed Fisher random walk debiased estimator asymptotically attains the semiparametric efficiency lower bound. Our proposed method is efficient as long as the nuisance preference score functions achieve a certain statistical rate, which is applicable to a wide range of nonparametric estimators in modern artificial intelligence, such as rectified linear unit deep neural networks.

To further evaluate if certain LLM is the best across different domains, or whether we can evaluate the performance LLMs on some target domain by transferring the observations of LLMs' performance from the source domain, we generalize the inference of $\fav_{\iz\jz}(\Omega)$ to two general settings: the multiple domain hypotheses and the domain shift hypothesis. For the first setting, we generalize \eqref{test:functional:intro} to multiple hypotheses $\mathrm{H}_{0t}: {\fav}_{i_tj_t}(\Omega_t)$ for mutiple pairs $(i_t,j_t)$ and multiple domains $\Omega_t$ for $t = 1, \ldots, T$. 
We extend our inference procedure to the multiple testing setting using the Gaussian multiplier bootstrap~\citep{cck2013gaussian} and show its validity by controlling the familywise error rate. For the second setting, we aim to accommodate distributional shifts in the contextual variable distribution. Namely, given the observed context $\mathbf{X}$ following the source distribution $\mathcal{X}_{\text{source}}$, we aim to infer the preference functional on the target domain $\mathbb{E}_{\mathbf{X} \in \mathcal{X}_{\text{target}}}\left[\mathbb{I}(\mathbf{X} \in \Omega)\left(\theta^*_{\iz}(\mathbf{X}) - \theta^*_{\jz}(\mathbf{X})\right)\right]$ where $\mathbf{X}$ follows some target distribution $\mathcal{X}_{\text{target}}$. For example, suppose the preference dataset is collected over the domain of questions related to general sports. Users may focus on different types of sports, leading to substantially different distributions of their prompting questions. We propose a cross-fitted importance sampling statistic to adjust the domain shift and show its validity for inferring the preference functional in the target domain.

\subsection{Contributions of this work}\label{sec:contribute} 

We develop a semiparametric efficient estimator for contextual preference inference via a (neW) Fisher random walk debiasing strategy. Our contributions are mainly four fold.

\begin{itemize}
    \item \textbf{Fisher random walk debiasing for contextual preference inference.}  Existing inference methods for functionals like autoDML \citep{chernozhukov2021automatic,Chernozhukov2022b} could not be directly applied to the contextual preference inference due to the missingness of direct comparison data. We propose a novel debiasing strategy using a Fisher-information-weighted random walk to aggregate regression residuals across the entire comparison graph to automatically debiasing the general plug-in estimator.

    \item \textbf{Potential representation for scalable debiasing computation.} Existing balancing weights methods  \citep{imai2014covariate,athey2018approximate,zubizarreta2015stable} require to solve a computationally intensive optimization problem. Proved by a physics analogue of the electric potential in electrical networks, we show edge weights are differences of node potentials, reducing the computation from estimating $O(n^2)$ weight nuisance functions to \(O(n)\) node-wise functions estimated by graph Laplacian solvers.

    \item \textbf{Semiparametric efficiency via a new graph Laplacian perturbation bound.} Existing asymptotic normality results for preference score estimators are for parametric BTL model \citep{han2020asymptotic,gao2023uncertainty,fan2024uncertainty} or based on specific kernel estimators \citep{wang2024ranking}. We show the semiparametric efficiency of the proposed Fisher random walk debiased estimator for general machine learning methods. To prove the efficiency, we establish a novel entrywise perturbation inequality for graph Laplacian pseudoinverses, which has its own theoretical interest.

    \item \textbf{Oracle inequality for contextual BTL model.}  Existing rates for contextual BTL focus on linear models or kernel-based methods \citep{fan2024uncertainty,wang2024ranking}, misaligned with the real practice using deep learning models. We prove an oracle inequality for the statistical rate of the preference score function that decomposes error into uniform approximation and covering entropy applies to flexible learners, including ReLU networks, under standard nonparametric rates \citep{schmidt2020nonparametric,kohler2021rate,bauer2019deep,bos2022convergence,kim2021fast,shen2022approximation,zhou2024classification}.
\end{itemize}

\subsection{Related literature}
Our work is related to several streams of literature, including ranking models,  double/debiased machine learning, and the nonparametric estimation using ReLU neural networks (ReLU-DNN). We discuss the most relevant works in each area as follows. In this paper, we focus on the inference problem under the contextual BTL model in this work. For the traditional BTL model, the preference score functions are constants, referred to as preference scores. \citet{han2020asymptotic,gao2023uncertainty} established the asymptotic normality of the maximum likelihood estimator when the comparison graph follows an {\ER} model. %
 \citet{liu2023lagrangian} developed a Lagrangian debiasing method to establish asymptotic normality and conduct inference for preference scores, and also proposed a combinatorial inferential framework for testing general ranking properties under the BTL model.
Under the contextual BTL model, recent work has investigated inference with linear preference score functions \citep{fan2024uncertainty}.  While this work was being completed, a recent paper \citep{wang2024ranking} also studied inference under the contextual BTL model with nonlinear and twice-differentiable preference score functions. %
This paper contributes to the literature of ranking inference by developing a double machine learning framework for the inference of a contextual BTL model with very general nonparametric preference functions. %
It does not require linearity or smoothness conditions on the preference functions, and enables valid inference as long as the score preference functions can be reasonably estimated by flexible machine learning algorithms, offering broad applicability and flexibility in practice.

Our proposed inference procedure is built upon some essential ideas in the debiased inference framework \citep{chernozhukov2018double} and semiparametric statistics \citep{bickel1993efficient}. %
Some recent advances in debiased inference include (i) estimating nuisance functions and conducting the theoretical analysis of debiased inference with general machine learners like random forests and neural networks  \citep{farrell2021deep,foster2023orthogonal}, (ii)   automatic debiasing inference     \citep{Chernozhukov2022b,singh2024kernel,chernozhukov2021automatic}, which allows the analytic form of the Riesz regression function to be unknown and can be directly estimated by any general machine learners, (iii) debiased inference under covariate shift \citep{chernozhukov2023automatic}, (iv) debiased inference under high-dimensional linear models \citep{wang2020debiased,athey2018approximate}.  
We depart from and contribute to these literature streams by developing a debiased inference methodology under the contextual BTL model setting, which involves an intricate double-indexed data structure and a growing number of nuisance functions.  

We also contribute to the  literature of nonparametric estimation using ReLU-DNN, by studying the nonparametric estimation of preference score functions under the contextual BTL model through  ReLU-DNN approximation. %
 For example, the  \citet{bauer2019deep,kohler2021rate,schmidt2020nonparametric} study the theoretical properties of ReLU-DNNs in nonparametric regression.  {\color{black}
\citet{fan2024noise} further explores robust ReLU-DNN nonparametric regression with heavy-tailed noise. \citet{kim2021fast,shen2022approximation,zhou2024classification} study the theoretical properties of ReLU-DNNs for binary classification, while \citet{bos2022convergence} derive the convergence rate of ReLU-DNNs in multiclass classification. For comprehensive overviews, see surveys in \citet{fan2020selective,  bartlett2021deep, suh2024survey}.}

\noindent{\bf Paper Organization.} The rest of the paper is organized as follows. Section~\ref{sec:method} introduces the contextual BTL setup and notation. Section~\ref{sec:random-walk} presents the Fisher random walk debiasing method. Section~\ref{sec:gaussian} develops theoretical properties, including the oracle inequality of estimation rate and semiparametric efficiency. Section~\ref{sec:extend} discusses extensions to multiple-domain hypotheses and domain shift. Section~\ref{sec:exp} reports simulations and real-data analyses.

\section{Contextual Preference Inference}\label{sec:method}

In this section, we will formally set up the models and observations of contextual preference inference. 

\subsection{Contextual BTL model}%

We first introduce the contextual BTL ranking model \citep{fan2024uncertainty}.    
 The contextual BTL model generalizes the classical BTL model \citep{bradley1952rank,debreu1960individual} by allowing item comparisons to depend on contextual information.
While contexts in LLM evaluation are text-based, treating each unique text as a discrete variable is infeasible. The sheer volume of possible inputs would create a prohibitive number of discrete instances. We will adopt a more effective strategy by representing these text inputs as embeddings in a Euclidean space using text embedding models \citep{reimers2019sentencebert,behnamghader2024llm2vec}. This approach can both manage dimensionality but also capture semantic relationships via the directions of embeddings. Therefore, we represent the context as \(\M x \in \mathbb{X} \subseteq \mathbb{R}^d\), for some compact set $\mathbb{X}$.
Suppose that there are \(n\) items for comparison. For any $i,j \in [n]$, we denote by $Y_{ij}(\M x)$ the comparison result between item $i$ and item $j$ under context $\M x$ as
$
Y_{ij}(\M x) = 
1  \text{ if item $i$ is preferred over item $j$ under context $\M x$, and }
0 \text{ otherwise,}
$
and symmetrically, $Y_{ji}(\M x) = 1 - Y_{ij}(\M x)$. The contextual BTL model assumes  
that the preference score of item~\(i\) is captured by a context-dependent preference score function \(\theta_i^*(\M x)\). %
Under the contextual BTL model, we assume that the distribution of $Y_{ij}(\M x)$ follows the logistic model 
$
\mathbb{P}\big( Y_{ij} (\M x) = 1\big)  =  \psi\big(\theta_i^*(\M x) - \theta_j^*(\M x)\big), 
$
recalling that $\psi(t) = 1/(1+e^{-t})$.
For ease of presentation, we assume that there are no ties among items. Nevertheless, our theoretical and methodological framework can be extended to accommodate ties.

As the contextual BTL model only depends on the differences of the preference scores, to ensure the identifiability of $\bds\theta^*(\M x)$, we impose the mean-zero normalization constraint
\bee\label{zero-constraint}
\M 1^\T\bds\theta^*(\M x) = 0,\text{ for any $\M x\in\mathbb{X}$,}
\ee
  which is widely used in the literature \citep{han2020asymptotic,gao2023uncertainty}. We further assume that the true preference score functions are uniformly bounded by some constant $C > 0$. In summary, we consider the following function class where the truth $\bds\theta^*$ belongs to  
\bee\label{def:truthclass}
 \Theta  = \left\{ {\bds\theta}(\cdot)\midplus  \theta_i(\M x)\in[-C,C],  \bds1^\T{\bds\theta}(\M x) = 0\text{ for any }i\in[n]\text{ and any }\M x\in \mathbb{X}\right\}.
\ee  

We allow that only a subset of item pairs are compared in the observational data. Accordingly, we define a comparison graph over \(n\) nodes, where nodes \(i\) and \(j\) are connected if and only if items \(i\) and \(j\) are compared. In particular, let the adjacency matrix of the comparison graph be $\M A = [A_{ij}]_{n \times n} \in \{0,1\}^{n \times n}$, where $A_{ij} = 1$ indicates that items $i$ and $j$ are compared in the data, and $A_{ij} = 0$ otherwise. We refer $\M A$ as the comparison graph, and the edge set of the comparison graph is $\mathcal{E}(\M A) = \{(i,j) \mid A_{ij} = 1,\ i > j\}$. In addition, we let $|\M A|$ denote the total number of edges in the comparison graph. 

The comparison graph could be either deterministic or random. The results throughout the paper are valid as long as $\M A$ satisfies some good properties defined in \eqref{upperbound:A}. However, for the simplicity of the presentation, in the rest of the paper, we assume $\M A$ is an Erd\H{o}s--R\'enyi graph with connection probability $p > 0$, under which $\{A_{ij} \mid (i,j) \in \mathcal{E}(\M A)\}$ are i.i.d. Bernoulli random variables with success probability $p$. The Erd\H{o}s--R\'enyi graph will satisfy the good properties with high probability when $p$ is large enough. We refer to Remark~\ref{rm:good-A} for the detailed discussion.

\subsection{Generic preference score function estimator}\label{sec:scoreest:relu}

Let $\M X$ be the random context sampled from the context distribution $\mathcal X$ supported on $\mathbb{X}$.  To have a meaningful inference for $\fav_{\iz\jz}(\Omega)$ in \eqref{def:qijomega}, we assume that the considered contextual domain $\Omega$ has a positive measure, i.e., $\mathbb{P}(\M X\in\Omega) \ge c_\Omega$ for some constant $c_\Omega > 0$.
We observe i.i.d. context samples \(\M X_{ij\ell} \sim \mathcal X\) representing the \(\ell\)th comparison between item \(i\) and item~$j$, and let the corresponding comparison outcome be \(Y_{ij\ell} = Y_{ij}(\M X_{ij\ell})\). We assume that each contextual variable \(\M X_{ij\ell}\) is drawn independently from a reference distribution \(\mathcal{X}\), and there are \(L\) comparisons between items~$i$ and \(j\) that \((\M X_{ij1}, Y_{ij1}), \ldots, (\M X_{ij \ell}, Y_{ijL})\) are  i.i.d. samples. For ease of presentation, we assume that each item pair is compared  \(L\) times in the dataset.
We denote the full dataset as
$
\mathcal{D}_n = \big\{(\M X_{ij\ell},\ Y_{ij\ell}) \mid (i,j,\ell) \in \mathcal{E}(\M A) \times [L]\big\}$.

For any \({\boldsymbol{\theta}}(\mathbf{x}) = (\theta_1(\mathbf{x}), \ldots, \theta_n(\mathbf{x}))^\top\), the empirical negative log-likelihood loss is 
\begin{equation}\label{equ:likelihood_func}
\begin{aligned}
\mathcal{L}_n({\boldsymbol{\theta}}) = -\frac{1}{n^2pL} \sum_{(i,j)\in\mathcal{E}(\mathbf{A}),\,\ell\in[L]} &Y_{ij\ell} \log \Big(\psi\big(\theta_i(\mathbf{X}_{ij\ell}) - \theta_j(\mathbf{X}_{ij\ell})\big)\Big) \\
&+ (1 - Y_{ij\ell}) \log \Big(1 - \psi\big(\theta_i(\mathbf{X}_{ij\ell}) - \theta_j(\mathbf{X}_{ij\ell})\big)\Big).
\end{aligned}
\end{equation}

We consider general classes of the estimator $\hat{\boldsymbol{\theta}}(\M x)$  to be obtained by minimizing the loss \eqref{equ:likelihood_func} over all ${\bds \theta}$ with its entries $\theta_i \in \mathcal{F}$ for all $i \in [n]$.
One example of the function class $\mathcal{F}$ is the sparse ReLU-DNN function class   \citep{schmidt2020nonparametric} which we will specify later in Definition~\ref{def:commodel}.  To satisfy the normalization condition~\eqref{zero-constraint}, for any function class $\mathcal{F}$, we define the normalized estimator class and our generic estimator as
\bee\label{proposed:estimator}
\mathcal{G}(\mathcal{F} ^n) = \Big\{ \boldsymbol{\theta}(\cdot) = \boldsymbol{\vartheta}(\cdot) - \frac{\mathbf{1} \mathbf{1}^\top}{n}  \boldsymbol{\vartheta}(\cdot) \mid  \boldsymbol{\vartheta} = (\vartheta_1, \ldots, \vartheta_n)\in \mathcal{F}^n \cap \Theta \Big\} \text{ and } \hat{\boldsymbol{\theta}}  = \argmin_{{\boldsymbol{\theta}} \in \mathcal{G}(\mathcal{F} ^n)} \mathcal{L}_n(\boldsymbol{\theta}).
\ee

The estimator above is generic as we do not specify $\mathcal{F}$. We will establish an oracle inequality for the estimation rate of the generic $\hat{\boldsymbol{\theta}}$ in Section~\ref{sec:theory} and specify the general sufficient conditions on $\mathcal{F}$ needed for valid inference.

\section{Automatic Debiasing for Prefence Inference}\label{sec:random-walk}

In this section, we will introduce a new debiasing strategy called Fisher random walk to infer $\fav_{i_0j_0}(\Omega)$ in \eqref{def:qijomega}.  We will begin with reviewing the one-step estimator when the edges $(\iz, \jz)$ exists on the comparison graph. The problem then can be simplified as a single logistic regression $\mathbb{E}[Y_{\iz\jz}| \M X = \M x] = \psi(\theta_{i_0}(\M x ) - \theta_{j_0}(\M x ))$. Denote $m_{i_0j_0}(\M x) = \mathbb{I}(\M x\in\Omega)\{\theta_{i_0}(\M x ) - \theta_{j_0}(\M x )\}$ such that $\fav_{i_0j_0}(\Omega) = \mathbb{E}[m_{i_0j_0}(\M X)]$. The plug-in estimator for $m_{i_0j_0}(\M x)$ is $\hat{m}_{i_0j_0}(\M x) = \mathbb{I}(\M x\in\Omega)\{\hat\theta_{i_0}(\M x ) - \hat\theta_{j_0}(\M x )\}$, which usually fails to achieve asymptotic normality owing to non-negligible bias induced by the nuisance estimation rate. The problem of debiasing the plug-in estimator was first considered for the finite-dimensional case.
Define $\beta^* = \argmin_{\beta \in \mathbb{R}^d}  \mathbb{E}[\ell(Z,\beta)]$ for $\ell$ being some negative log-likelihood function of $Z$ and $\hat\beta$ is some initial estimator using $Z_1, \ldots, Z_n$ as i.i.d. copies of $Z$.  The classical one-step estimator \citep{van2000asymptotic} debiases $\hat\beta$ by the Newton-Raphson method
\bee\label{eq:one-step}
\hat\beta^d = \hat\beta - \Big(\frac{1}{n} \sum_{i=1}^n \nabla_{\beta}^2\ell(Z_i,\hat\beta)\Big)^{-1}\frac{1}{n} \sum_{i=1}^n \nabla_{\beta}\ell(Z_i,\hat\beta)
\ee
Generalizing the above estimator to the functional case, we can apply autoDML \citep{chernozhukov2021automatic,Chernozhukov2022b,vanderLaan2025} to debias $\hat{m}_{i_0j_0}(\M x)$ as 
\bee\label{debiased:fav}
 \tilde{m}_{\iz\jz}(\M x) = \hat{m}_{\iz\jz}(\M x) + \hat\eta_{\iz\jz}(\M x) \hat\rho_{\iz\jz}(\M x,Y_{\iz\jz}),  \text{where}
\ee
$$\hat\eta_{\iz\jz}(\M x) ={\mathbb{I}(\M x\in\Omega)}\big(\psi'(\hat\theta_{\iz}(\M x ) - \hat\theta_{\jz}(\M x ))\big)^{-1},\quad \hat\rho_{\iz\jz}(\M x, y) = Y - \psi(\hat\theta_{\iz}(\M x ) - \hat\theta_{\jz}(\M x )).$$
Here \eqref{debiased:fav} generalizes \eqref{eq:one-step} as  the residual $\hat\rho_{\iz\jz}$ happens to be the score function of the logistic model, and the residual balancing weight $\hat\eta_{\iz\jz} $ involves $\psi'$ as the Fisher information of the logistic model. However, we cannot directly apply it to the contextual preference inference and need to introduce a new debiasing strategy in the next part.

\subsection{Fisher random walk debiased estimator}
\label{sec:principal}

Two major challenges to apply \eqref{debiased:fav} to the contextual preference inference are (1) it is not applicable when $A_{\iz\jz}=0$ and $Y_{\iz\jz}$ is not observed, and (2) how to aggregate the observations $\{(\M X_{ij\ell}, Y_{ij\ell}): \ell \in [L]\}$ for all $(i,j)$'s on the comparison graph to achieve the semiparametric efficiency. 
To address these two issues, we propose to debias the initial estimator $\hat{m}_{\iz\jz}$ by aggregating the one-step debiasing term $\hat\eta_{ij}(\M x) \hat\rho_{ij}(\M x,y)$ for edges $(i,j)$'s enumerating over all paths connecting $\iz$ and $\jz$ in graph $\M A$. 
In particular, let $\mathrm{R} = \{(i_0,i_1),(i_1,i_2),\ldots,(i_{S-1},j_0)\}$ be some directed path connecting $\iz$ and $\jz$ on $\M A$ with length $S$, illustrated in Figure~\ref{fig:randomwalk}(a). We can generalize \eqref{debiased:fav} by summing the debiasing terms over all the edges on the path $\mathrm{R}$ and have the debiased estimator 

\bee\label{debiased:fav:path:single}
\hat{m}^{d}_{\iz\jz; \mathrm{R}}(\M x) &= \hat{m}_{\iz\jz}(\M x) + \sum_{(i,j) \in \mathrm{R}} \hat\eta_{ij}(\M x) \hat\rho_{ij}(\M x,Y_{ij}) 
\\
&= \hat{m}_{\iz\jz}(\M x) + \sum_{(i,j) \in \mathcal{E}(\M A)} \mathrm{R}(i,j)\hat\eta_{ij}(\M x) \hat\rho_{ij}(\M x,Y_{ij}),
\ee
where $\mathrm{R}(i,j) = |\{(i,j) \in \mathrm{R}\}| - |\{(j,i) \in \mathrm{R}\}|$ is the signed number of the directed edge  $(i,j)$  in $\mathrm{R}$. 

In the next step, instead of choosing a specific path $\mathrm{R}$, we compute the debiasing term by averaging over all possible paths on the comparison graph $\M A$. Specifically, for every $\M x\in\mathbb{X}$, we define a prior distribution $\mathcal{R}$ over all paths connecting $\iz$ and $\jz$ in $\M A$, and then average the debiasing term in \eqref{debiased:fav:path:single} over these paths weighted by the prior $\mathcal{R}$. 
 In particular, our random-walk debiased estimator is
\bee\label{debiased:fav:path}
\EE_{\mathrm{R} \sim \mathcal{R}}[ \hat{m}^{d}_{\iz\jz; \mathrm{R}}(\M x)] = \hat{m}_{\iz\jz}(\M x) + \sum_{(i,j) \in \mathcal{E}(\M A)} \EE_{\mathrm{R} \sim \mathcal{R}}[\mathrm{R}(i,j)] \hat\eta_{ij}(\M x) \hat\rho_{ij}(\M x,Y_{ij}).
\ee
See Figure~\ref{fig:randomwalk}(b) for an illustration.

Lastly, we aim to design a prior distribution $\mathcal{R}$ derived from some random walk from $\iz$ to $\jz$. Motivated by the one-step estimator \eqref{eq:one-step}, we consider using the Fisher information as the transition probability of the random walk.
\begin{definition}[Fisher Random Walk] A random walk over the comparison graph $\M A$ is called a Fisher random walk when its transition probability is proportional to some Fisher information. For the logistic model, the transition probability of the Fisher random walk for any given score functions $\bds \theta$ becomes
  \bee\label{trans:propose}
  {P}_{ij}({\bds\theta})  = \frac{ A_{ij}\psi'\big(\theta_i(\M x)  - \theta_j(\M x)\big)}{\sum_{k = 1}^n  A_{ik}\psi'\big(\theta_i(\M x)  - \theta_k(\M x)\big)}, \text{ for all $i,j \in [n]$.}
  \ee  
\end{definition}

Let \(\mathcal{R}_{\iz\jz}(\M x, \bds \theta)\) be the distribution of the random path that starts at \(\iz\) and stops upon its first visit to  \(\jz\) following the transition probability defined in \eqref{trans:propose}. 
We then choose the prior distribution $\mathcal{R}$ in \eqref{debiased:fav:path} to be \(\mathcal{R}_{\iz\jz}(\M x, \hat{\bds\theta})\) using the generic estimator $ \hat{\bds\theta}$ in \eqref{proposed:estimator} and propose the  {Fisher random walk debiasing function}
\bee \label{debiased:R}
\hat{m}^{d}_{\iz\jz}(\M x) =  \hat{m}_{\iz\jz}(\M x) + \sum_{(i,j) \in \mathcal{E}(\M A)} \EE_{\mathrm{R} \sim \mathcal{R}_{\iz\jz}(\M x, \hat{\bds\theta})}[\mathrm{R}(i,j)] \hat\eta_{ij}(\M x) \hat\rho_{ij}(\M x,Y_{ij}).
\ee
Therefore, for $(i,j)\in\mathcal{E}(\M A)$,  define the residual balancing weight term for generic score functions $\bds \theta$ as 
\bee\label{def:W}
{W}_{ij}(\M x\mid {\bds \theta}, \iz,\jz) = |\M A|\EE_{\mathrm{R} \sim \mathcal{R}_{\iz\jz}(\M x, {\bds \theta})}[\mathrm{R}(i,j)] \cdot{\mathbb{I}(\M x\in\Omega)}\big(\psi'(\theta_{i}(\M x ) - \theta_{i}(\M x ))\big)^{-1},
\ee 
and we can simplify \eqref{debiased:R}  and propose the {\it Fisher random walk debiased estimator} as
\bee \label{debiased:Q}
  \hat{\mathcal{Q}}_{\iz\jz}(\Omega) 
  =  \frac{1}{|\M A|L}\sum_{(i,j)\in\mathcal{E}(\M A),\, \ell\in[L]} \Big\{ \hat{m}_{\iz\jz}(\M X_{ij \ell}) + {W}_{ij}(\M X_{ij \ell}\mid \hat{\bds \theta}, \iz,\jz) \hat\rho_{ij}(\M X_{ij \ell},Y_{ij\ell})\Big\}.
\ee
We will show in Theorem~\ref{thm:oracle} that the above estimator is semiparametric efficient.
Notice that the first term in the above estimator is a plug-in estimator $\hat m_{i_0j_0}(\M x) = \mathbb{I}(\M x\in\Omega)\{\hat\theta_{i_0}(\M x ) - \hat\theta_{j_0}(\M x )\}$ and the second debiasing term is a weighted residual balancing term with the residuals $\hat\rho_{ij}$ defined in \eqref{debiased:fav} weighted by the residual balancing weights $W_{ij}$. In this sense, the Fisher random walk debiased estimator is a generalization of the autoDML \citep{chernozhukov2021automatic,Chernozhukov2022b,vanderLaan2025}   to the comparison graph setting.

Meanwhile, we note that, despite its efficiency, computing these balancing weights ${W}_{ij}(\M x\mid \hat{\bds \theta}, \iz,\jz)$ for all $(i,j)\in\mathcal{E}(\M A)$ is not straightforward. In particular, the weights do not admit closed-forms, and it is computationally expensive to compute all weights over the entire comparison graph, which are of order $O(n^2L)$. To overcome the computational burden, we provide a computationally efficient reformulation of the weights \eqref{def:W} in the following theorem.

\begin{figure}\centering
  \begin{subfigure}[b]{0.25\textwidth}
   \includegraphics[width=1\textwidth]{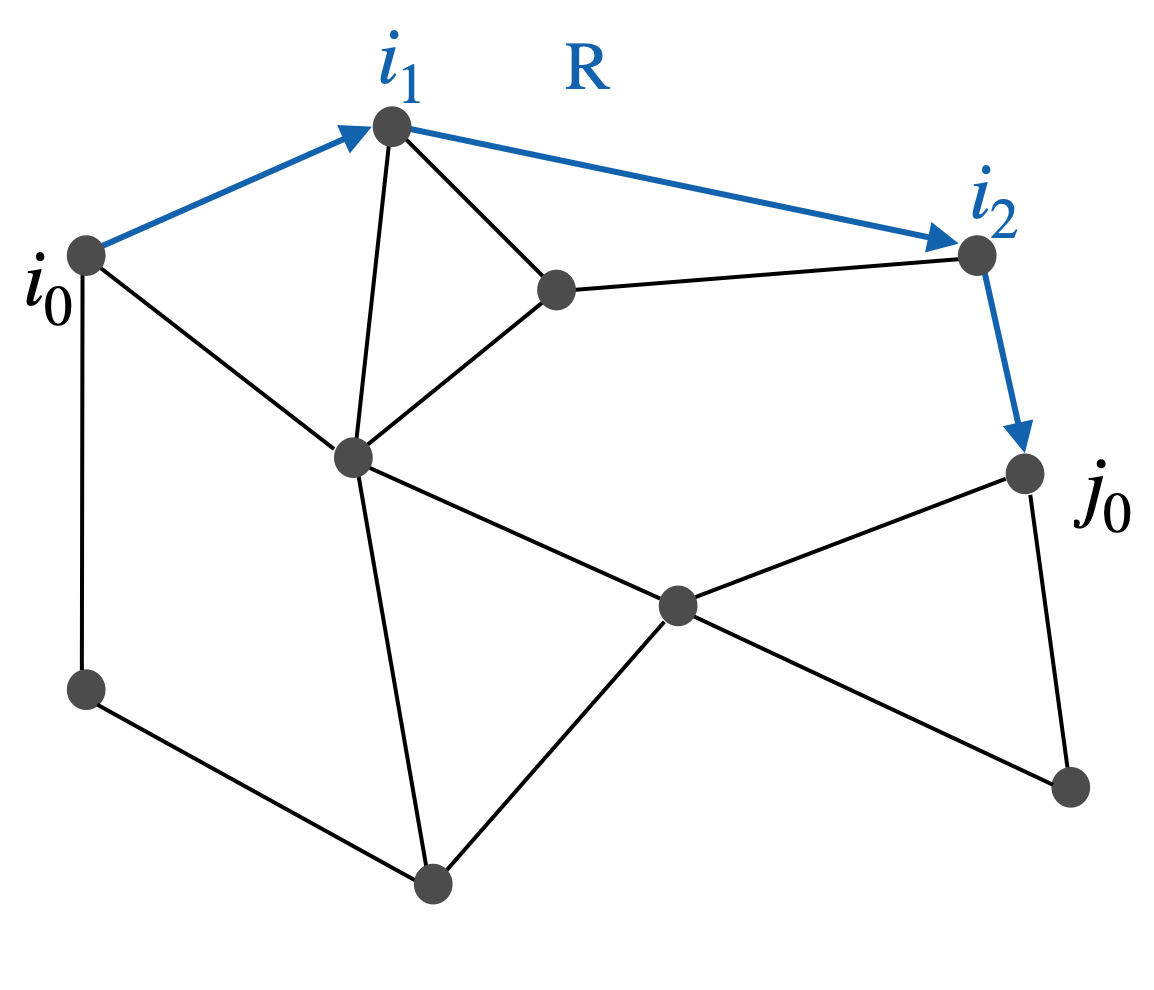}
      \caption{Graph path}\end{subfigure}\quad\quad\quad\quad\quad\quad\quad\quad 
\begin{subfigure}[b]{0.25\textwidth}
\includegraphics[width=1\textwidth]{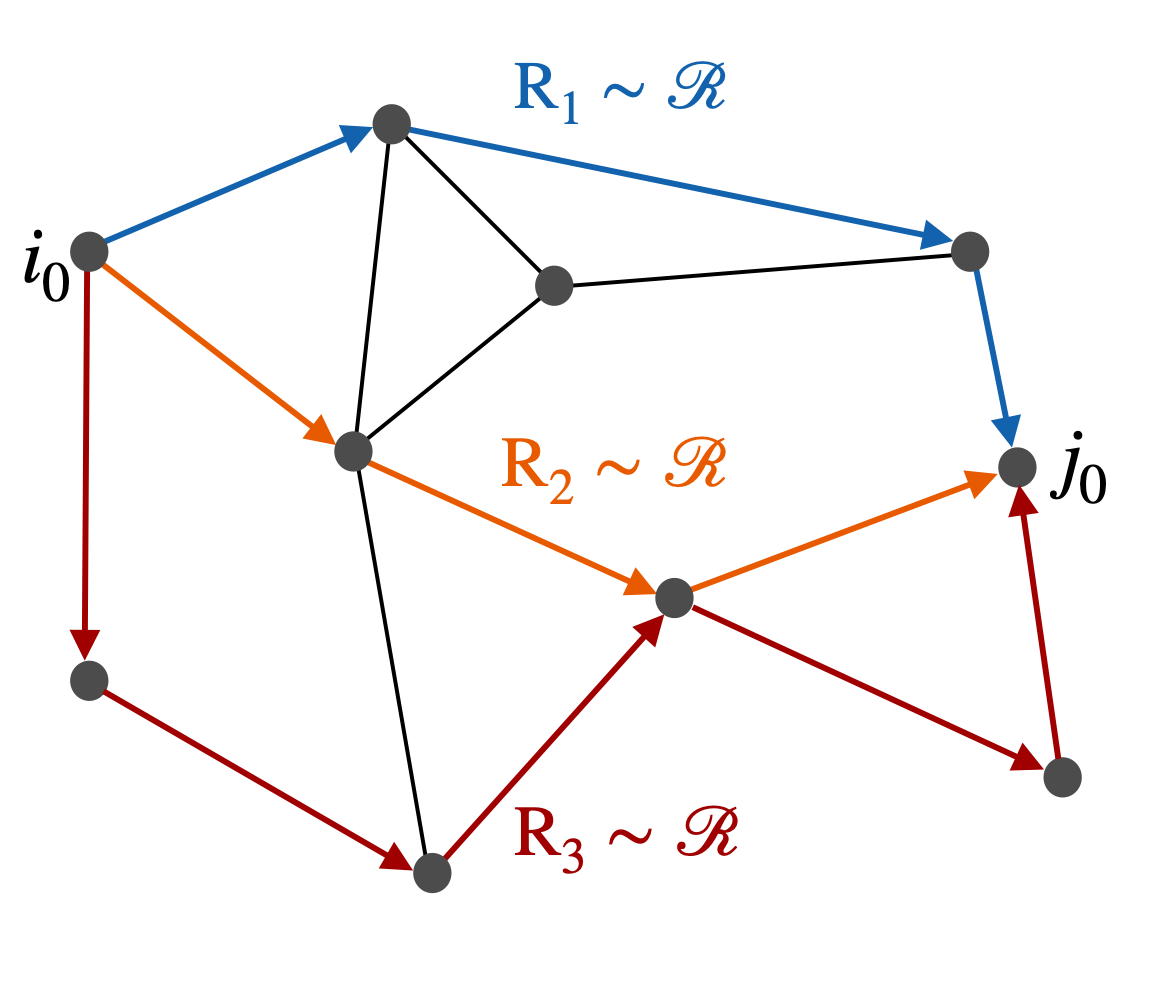}  
\caption{Fisher Random walk}
 \end{subfigure} 
\caption{%
In panel~(a), one graph path $\mathrm{R} = \{(i_0,i_1),(i_1,i_2),\ldots,(i_{S-1},j_0)\}$ connects node $i_0$ and $j_0$.  
In panel~(b), paths in different colors represent distinct random walks sampled from any prior distribution ${\mathcal{R}} $. %
}\label{fig:randomwalk} 
\end{figure}

\begin{theorem}[Potential representation of residual balancing weight] \label{po:solution}
Suppose that $\M A$ is a connected graph. Given any score functions $\bds \theta \in \Theta$ in \eqref{def:truthclass}, we have the  potential representation of residual balancing weight
\bee\label{def:alpha:main}
{W}_{ij}(\M x\mid {\bds \theta}, \iz,\jz)=\pi_i(\M x \mid {\bds \theta}, \iz,\jz) - \pi _j(\M x \mid {\bds \theta}, \iz,\jz), \text{ for all $i,j \in [n]$}, 
\ee 
where we have the closed form of ${\bds\pi}$ as 
\begin{align}\label{def:H}
&{\bds \Xi}(\M x \mid {\bds \theta})  = \Big[A_{ij} \psi'\big(\theta_i(\M x) - \theta_j(\M x)\big)  \Big]_{(i,j)\in[n]^2},\quad
{\pmb{\mathscr{L}}}(\M x\mid {\bds \theta} ) = \mathrm{diag}\big({\bds \Xi}(\M x \mid {\bds \theta})\M 1\big) - {\bds \Xi}(\M x\mid {\bds \theta} ),
\\
\label{iota0}
& {\bds\pi}(\M x \mid {\bds \theta}, \iz,\jz) = \mathbb{I}(\M x\in\Omega){ |\M A|}{\pmb{\mathscr{L}}}^{\dagger}(\M x \mid {\bds \theta}) (\M e_{\iz} - \M e_{\jz}) .
\end{align}
Here ${\pmb{\mathscr{L}}}^{\dagger}(\M x \mid {\bds \theta})$ is the Moore--Penrose inverse of the graph Laplacian ${\pmb{\mathscr{L}}}(\M x \mid {\bds \theta})$, and $\M e_i$ is a unit vector with all entries being zero except the $i$-th entry being one. 
\end{theorem}

The proof of Theorem~\ref{po:solution} is in Supplementary Material~\ref{sec:pf:thm2}. Denote $\hat{\bds\pi}(\M x\mid \iz,\jz) = {\bds\pi}(\M x \mid \hat{\bds \theta}, \iz,\jz)$. In the rest of the paper, when it is clear from the context, we will omit the indices $\iz$, $\jz$ for notation simplicity and denote $\hat{\bds\pi}(\M x\mid \iz,\jz)$ as $\hat{\bds\pi}(\M x)$.  Applying the potential representation above, we can write the Fisher random walk debiased estimator in \eqref{debiased:Q} as the following computationally feasible form 
\[
  \hat{\mathcal{Q}}_{\iz\jz}(\Omega) 
  =  \frac{1}{|\M A|L}\sum_{(i,j)\in\mathcal{E}(\M A),\, \ell\in[L]} \Big\{ \hat{m}_{\iz\jz}(\M X_{ij \ell}) + \big(\hat{\pi}_i(\M X_{ij}) -  \hat{\pi}_j(\M X_{ij})\big)\hat\rho_{ij}(\M X_{ij \ell},Y_{ij\ell})\Big\}.
\]

We call \eqref{def:alpha:main} the potential representation of the balancing weight $W$ following an analogue in physics \citep{halliday2013fundamentals}. Viewing the weight matrix $W$ as a discrete vector field,   \eqref{def:alpha:main} implies that the vector field is ``conservative'' or ``path independent'', and can be represented as the difference of the  ``potential energy'' ${\bds\pi}$. In fact, we can map the preference inference to an equivalent electrical network. Under such mapping, besides the formal proof, we will also provide a heuristic physics proof of Theorem~\ref{po:solution} in the Supplementary Material~\ref{sec:physics-proof} showing ${\bds\pi}$ matches the electric potential in some electrical network and thus \eqref{def:alpha:main} is naturally true by the Ohm's law \citep{griffiths2023introduction}. 
The potential representation eases the computation of the debiasing procedure from computing the $W$ matrix to computing the ${\bds\pi}$ vector, and reveals the essential degree of freedom for  $n-2$ nuisance preference score functions. Besides, as the potential representation in \eqref{def:alpha:main} is invariant for any constant shift of  ${\bds\pi}$, we centerize ${\bds\pi}$ to have a mean of zero for identifiability.
If we take ${\bds \Xi}$ as the weights for the comparison graph $\M A$, the transition probability matrix in \eqref{trans:propose} is normalized by these weights. Thus, ${\pmb{\mathscr{L}}}$ becomes the  corresponding    graph Laplacian of the Fisher random walk.  As   graph Laplacian,     we have ${\pmb{\mathscr{L}}}^\T\M 1 = \M 0$, and thus
$
\M 1^\T{\bds\pi} = 0.
$ 
This implies that the potentials ${\bds\pi}$ are centered at zero.

Utilizing the potential representation in Theorem~\ref{po:solution}, we summarize our Fisher random walk debiased estimator in Algorithm~\ref{alg:pde}. In the algorithm, we apply the cross-fitting \citep{chernozhukov2018double} by splitting $L$ comparisons into $S$ splits that each time we use one split to compute the debiased estimator $\hat{\mathcal{Q}}^{(s)}_{\iz\jz}(\Omega)$ in \eqref{debiased:Q} by plugging in $\hat{\bds\theta}$ and $\hat{\bds\pi}$  estimated using the remaining samples.

\begin{algorithm}[t]
\caption{Cross-fitting of Fisher random walk estimator}\label{alg:pde}
\KwIn{Comparison graph $\M A$, sample $\mathcal{D}_n$, number of cross-fittings $S$,   
    contextual domain and items  $(\Omega,\iz,\jz)$.}
\begin{itemize}
\item[1.] Split $[L]$ equally into $S$ subsets, namely, $\mathcal{J}^{(1)}_n,\ldots,\mathcal{J}^{(S)}_n$, such that $\mathcal{J}_n^{(1)}\cup\ldots\cup\mathcal{J}^{(S)}_n = [L]$. We then split $\mathcal{D}_n$ into $\mathcal{D}^{(1)}_n,\ldots,\mathcal{D}^{(S)}_n$ as $\mathcal{D}^{(s)}_n = \{(\M X_{ij \ell},Y_{ij \ell})\mid (i,j) \in \mathcal{E}(\M A), \ell\in\mathcal{J}^{(s)}_n\}$\;
\item[2.] For each $s\in[S]$, use $\mathcal{D}^{(-s)}_n = \mathcal{D}_n\backslash \mathcal{D}_n^{(s)}$ to estimate $\hat{\bds\theta}^{(s)}(\M x)$ via \eqref{proposed:estimator}, and estimate $\hat{\bds\pi}^{(s)}(\M x)$ via \eqref{iota0}  by plugging in  ${\bds\theta} = \hat{\bds\theta}^{(s)}$\;
\item[3.] For each $s\in[S]$, obtain the test statistics  as
\bee\label{principal_est_formula}
\hat{\fav}_{\iz\jz}^{(s)}(\Omega) = \frac{1}{|\M A|L/S}\sum_{(i,j)\in\mathcal{E}(\M A),\, \ell\in\mathcal{J}^{(s)}_n} & \mathbb{I}(\M X_{ij \ell} \in \Omega) \left(\hat\theta^{(s)}_{\iz}(\M X_{ij \ell})  - \hat\theta^{(s)}_{\jz}(\M X_{ij \ell})\right) \\&\!\!\!\!\!\!\!\!\!\!\!\!\!\!\!\!\!\!
+ \left(\hat{\pi}^{(s)}_{i}(\M X_{ij \ell}) - \hat{\pi}^{(s)}_{j}(\M X_{ij \ell})\right)\left(Y_{ij \ell} - \psi\left(\hat\theta^{(s)}_{i}(\M X_{ij \ell})  - \hat\theta^{(s)}_{j}(\M X_{ij \ell})\right)\right) .
\ee 
\end{itemize}
\KwOut{$\hat{\fav}_{\iz\jz}(\Omega) = \sum_{s = 1}^S \hat{\fav}_{\iz\jz}^{(s)}(\Omega) /S$.}
\end{algorithm}

\subsection{Uniform statistical rates of the potential representations}\label{sec:proposedestimator}
Here we derive the statistical rate of the potential vector estimator $\hat{\bds\pi}$ in \eqref{iota0}. We will achieve an oracle inequality of $\hat{\bds\pi}$ which only depends on the rate of the nuisance functions $\hat{\bds\theta}$. Therefore, our debiasing procedure in Algorithm~\ref{alg:pde} is automatic in the sense that it does not depend on how we estimate the nuisance functions. Let  
$
{\bds\pi}^*(\M x \mid \iz,\jz) = {\bds\pi}(\M x \mid {\bds \theta}^*, \iz,\jz)$ be the true potential.  We define the uniform $\mathcal{L}_2$ error bound of $\hat{\bds\pi}$ as follows.

\begin{definition}[Uniform $\mathcal{L}_2$ rates]\label{L2:def:theta}
  Let the conditional  $\mathcal{L}_2$ errors of $\hat{\bds \pi}$  be
  $$
  \mathcal{E}_2(\hat{\bds\pi},\bds\pi^*\mid \M A,\iz,\jz) = n^{-1} \E \big({\|\hat{\bds\pi}(\M X\mid \iz,\jz)-\bds\pi^*(\M X\mid \iz,\jz)\|^2_2} \mid \M A\big), $$
  where $\M X$ is independent to $\mathcal{D}_n$ and $\M A$. The same notation applies to other estimators. The uniform $\mathcal{L}_2$ errors of $\hat{\bds \pi}$ over all the possible choices of $(\iz,\jz)\in[n]^2$ and $\Omega\subseteq \mathbb{X}$ is 
  \bee\nonumber
\mathcal{E}_{2,\mathrm{unif}}\big(\hat{\bds\pi} ,{\bds\pi}^* \mid \M A \big) &= \sup_{(\iz,\jz)\in[n]^2,\, \Omega\subseteq\mathbb{X}} \mathcal{E}_2(\hat{\bds\pi},\bds\pi^*\mid \M A,\iz,\jz).
\ee 
  \end{definition} 
Here, the norms above are indexed by $\M A$ to emphasize that we condition on the comparison graph $\M A$ as fixed. This will help to clarify how our theoretical results can be generalized to the setting of a deterministic comparison graph. See Remark~\ref{rm:good-A} for detailed discussion. 
The following proposition shows that the rate of convergence of $ \hat{\bds\pi}(\M x )$ can be bounded by the rate of $\hat{\bds\theta}(\M x)$.
\begin{proposition}[Oracle inequality of potential vector]\label{po:pierror} 
For the potential vector estimator $\hat{\bds\pi}(\M x)$ in \eqref{iota0} based on the estimator $\hat{\bds\theta}(\M x)$ in \eqref{proposed:estimator}, if $np \geq 40 \log n$, then we have, with probability at least $1- 4/n$,
\bee\label{bound:pi}
\mathcal{E}_{2,\mathrm{unif}}\big(\hat{\bds\pi} ,{\bds\pi}^* \mid \M A \big)   \leq C\cdot  n  \cdot \mathcal{E}_\infty(\hat{\bds\theta},{\bds\theta}^*\mid \M A),
\ee
 where the $\mathcal{L}_\infty$-error of $\hat{\bds\theta}(\M x)$ is  $\mathcal{E}_\infty(\hat{\bds\theta},{\bds\theta}^*\mid \M A) = \E(\|\hat{\bds\theta}(\M X) - {\bds\theta}^*(\M X)\|^2_{\infty}\mid \M A)$, and $\|\cdot\|_\infty$ is the vector max norm.
\end{proposition}
The proof of Proposition~\ref{po:pierror} is in Supplementary Material~\ref{sec:pfpo:pierror}. Proposition~\ref{po:pierror} essentially shows the potential representation vector ${\bds \pi}(\M x \mid {\bds \theta}, \iz, \jz)$ is Lipschitz  in terms of ${\bds \theta}$. We remark that  to achieve the uniform bound in \eqref{bound:pi},   we derive   the following uniform inequality  for all possible $\Omega\subseteq\mathbb{X}$ and $\iz,\jz\in[n]$: 
$$
  \mathcal{E}_2(\hat{\bds\pi},\bds\pi^*\mid \M A,\iz,\jz) \lesssim n\cdot\E \left(\mathbb{I}(\M X\in\Omega)\| \hat{\bds\theta}(\M X) - {\bds\theta}^*(\M X)   \|^2_{\infty}\mid \M A\right)\lesssim n\cdot\E \left( \| \hat{\bds\theta}(\M X) - {\bds\theta}^*(\M X)   \|^2_{\infty}\mid \M A\right),
$$
 which achieves the uniformity. 
In Theorem~\ref{thm:L2error:main}, we will show an oracle inequality for the $\mathcal{L}_2$ rate $\mathcal{E}_2(\hat{\bds\theta},{\bds\theta}^*\mid \M A) = n^{-1}\E(\|\hat{\bds\theta}(\M X) - {\bds\theta}^*(\M X)\|^2_{2}\mid \M A)$.  Combining Theorem~\ref{thm:L2error:main} with with \eqref{bound:pi}, we can then can have a concrete rate for $\hat{\bds\pi}$ using $\mathcal{E}_{2,\mathrm{unif}}\big(\hat{\bds\pi} ,{\bds\pi}^* \mid \M A \big)   \lesssim n  \cdot \mathcal{E}_\infty(\hat{\bds\theta},{\bds\theta}^*\mid \M A) \lesssim n^2  \cdot \mathcal{E}_2(\hat{\bds\theta},{\bds\theta}^*\mid \M A)$. 
Also, in Algorithm~\ref{alg:pde}, we use cross-fitting to split $L$ samples  into  $S$ subsets and estimate $\hat{\bds\theta}^{(s)}$ for $s  \in [S]$. As we estimate $\hat{\bds\pi}^{(s)}$ by  plugging  $\hat{\bds\theta} = \hat{\bds\theta}^{(s)}$ into \eqref{iota0}, Proposition~\ref{po:pierror}  applies to the cross-fitting estimator $\hat{\bds\theta}^{(s)}$ for all $s  \in [S]$ with high probability.

\section{Theoretical Properties}\label{sec:gaussian}

In this section, we show the asymptotic normality of the Fisher random walk debiased estimator $\hat{\mathcal{Q}}_{\iz\jz}(\Omega)$ obtained from Algorithm~\ref{alg:pde}. We denote $\Phi(\cdot)$ as the cumulative distribution function of the standard normal distribution. 
Note that the algorithm depends on a generic score estimator $\hat{\bds\theta}(\M x)$. We assume that $\hat{\bds\theta}(\M x)$ achieves the following rate of convergence.

\begin{assumption}
\label{ass:theta} Let $\hat{\bds\theta}$ be obtained through \eqref{proposed:estimator} through $\mathcal{D}_n$. We assume that the $\mathcal{L}_2$ error of $\hat{\bds\theta}$ satisfies that, with probability at least $1-1/n$,
\bee\label{rate:theta:inference}
\mathcal{E}_2(\hat{\bds\theta} ,{\bds\theta}^*\mid \M A) %
 = o\left(\frac{1}{n^3p}  \wedge  \frac{1}{n\sqrt{npL}}\right),
\ee
where $\mathcal{E}_2(\hat{\bds\theta},{\bds\theta}^*\mid \M A) = n^{-1}\E(\|\hat{\bds\theta}(\M X) - {\bds\theta}^*(\M X)\|_2^2\mid \M A)$ is the mean squared error for $\hat{\bds\theta}$ given fixed $\M A$.
\end{assumption}
This assumption imposes a generic upper bound sufficient for the validity of the Fisher random walk debiased estimator.
In Theorem~\ref{thm:L2error:main}, we will establish an oracle inequality for $\mathcal{E}_2(\hat{\bds\theta},{\bds\theta}^*\mid \M A)$ and provide the sufficient condition for \eqref{rate:theta:inference} using the properties of the function class $\mathcal{F}$.

In the next theorem, we show that the asymptotic variance of $\hat{\mathcal{Q}}_{\iz\jz}(\Omega)$ is
\begin{align}\label{asymptotic:variance}
V_{\iz\jz}(\Omega)  & =\frac{1}{L}\E\Big(\frac{1}{{|\M A|}}\Big\{\mathbb{I}(\M X  \in \Omega)  (\theta^*_{\iz}(\M X) - \theta^*_{\jz}(\M X))  - \fav_{\iz\jz}(\Omega)\Big\}^2  \Big) + \frac{\sigma({\M A})}{L}, \text{ where }\\
\sigma({\M A})  & = \E\Big(\mathbb{I}(\M X\in\Omega)(\M e_{\iz} - \M e_{\jz})^\T{\pmb{\mathscr{L}}}^\dagger(\M X| {\bds \theta}^*)(\M e_{\iz} - \M e_{\jz}) \Big) = \E\Big[\pi_{\iz}(\M X\mid {\bds \theta}^*, \iz, \jz) - \pi_{\jz}(\M X\mid {\bds \theta}^*, \iz, \jz)\Big]. \label{eq:asymptotic:variance2}
\end{align} 
The last equality in \eqref{eq:asymptotic:variance2} follows the definition of potential representation in \eqref{iota0}. We note that we estimate the variance $V_{\iz\jz}(\Omega)$ by cross-fitting using Algorithm~\ref{alg:ci:addition}. 
There are two terms in \eqref{asymptotic:variance}: the first term is the variance contributed by the randomness of the context $\M X$, and the second term is the variance contributed by the comparisons across the comparison graph. In graph theory, $(\M e_{\iz} - \M e_{\jz})^\T{\pmb{\mathscr{L}}}^\dagger(\M X| {\bds \theta}^*)(\M e_{\iz} - \M e_{\jz})$ is called the resistance distance between $\iz$ and $\jz$ of the graph with edge weights $\M \Xi(\M X| {\bds \theta}^*)$ \citep{Klein1993}. Therefore, $\sigma(\M A)$ can be interpreted as a measure of how ``far apart'' \(\iz\) and \(\jz\) are within the weighted comparison graph.  Consequently, Equation \eqref{asymptotic:variance} provides an intuitive and formal interpretation: preference inference becomes more difficult when the resistance distance between the compared items is larger.

We then present the theorem on the asymptotic normality of the Fisher rank walk debiased estimator. 

\begin{algorithm}[t]
  \caption{Variance estimation and confidence interval for $\fav_{\iz\jz}(\Omega)$}\label{alg:ci:addition}
  \KwIn{Comparison graph $\M A$, samples $\mathcal{D}_n$, number of cross-fittings $S$, nuisance estimator  class $\mathcal{F}_{\theta}$, 
      the covariates domain of interests $\Omega$, confidence level $1 - \alpha$.}
  \KwOut{$(1 - \alpha)$-CI: $\hat{\mathcal{C}}_{\iz\jz,1-\alpha}(\Omega)$. }
  \begin{itemize}
  \item[1.] Run   Algorithm~\ref{alg:pde}.
  \item[2.] For each $s\in[S]$, obtain $\hat{\sigma}^{(s)}(\M A)$ as $
  S|\M A|^{-2}L^{-1} \Big( \sum_{(i,j)\in\mathcal{E}(\M A)\atop \ell\in\mathcal{J}^{(s)}_n}\hat{\pi}_{\iz}^{(s)}(\M X_{ij \ell}) - \hat{\pi}_{\jz}^{(s)}(\M X_{ij \ell})\Big) .
  $
  \item[3.] For each $s\in[S]$, obtain $ \hat{V}^{(s)}_{\iz\jz}(\Omega)$ as $$
  \frac{1}{{|\M A|^2L^2/S^2}}\sum_{(i,j)\in\mathcal{E}(\M A),\,\ell\in\mathcal{J}^{(s)}_n} \Big\{\mathbb{I}(\M X_{ij \ell}  \in \Omega)  (\hat{\theta}^{(s)}_{\iz}(\M X_{ij \ell} ) - \hat{\theta}^{(s)}_{\jz}(\M X_{ij \ell} )) - \hat{\fav}_{\iz\jz}(\Omega)\Big\}^2    + \frac{1}{L}\hat{\sigma}^{(s)}(\M A).
  $$
  \item[4] Obtain 
  $
  \hat{V}_{\iz\jz}(\Omega ) =S^{-1} \sum_{s = 1}^S \hat{V}_{\iz\jz}^{(s)}(\Omega ),
  $ and build $(1-\alpha)$--CI:
  \bee\label{CI:closedform:type1}
  \hat{\mathcal{C}}_{\iz\jz,1-\alpha}(\Omega) = \left(\hat{\fav}_{\iz\jz} (\Omega)  - z_{1 - \alpha/2}\sqrt{\hat{V}_{\iz\jz}(\Omega  )}, \quad \hat{\fav}_{\iz\jz} (\Omega)  + z_{1 - \alpha/2}\sqrt{\hat{V}_{\iz\jz}(\Omega )}\right),
  \ee 
  where $z_{1 - \alpha/2} = \Phi^{-1}(1 - \alpha/2)$.
  \end{itemize}
  \end{algorithm}

\begin{theorem}[Asymptotic normality]\label{theorem:LD:new}
Under Assumption~\ref{ass:theta}, with $np \geq 40 \log n$ and fixed number of cross-fittings $S > 0$, as $\color{black}n\rightarrow \infty$,
we have
 \bee\label{Q:concentration:all} 
 \frac{\hat{\fav}_{\iz\jz} (\Omega)  - \fav_{\iz\jz}(\Omega)}{\sqrt{V_{\iz\jz}(\Omega )}} \rightsquigarrow \normal(0,1),\quad \frac{ {\hat{\fav}_{\iz\jz} (\Omega)  - \fav_{\iz\jz}(\Omega)} }{ \sqrt{\hat{V}_{\iz\jz}(\Omega ) }}  \rightsquigarrow \normal(0,1).
 \ee

\end{theorem} 

We defer the proof of  Theorem~\ref{theorem:LD:new} to Supplementary Material~\ref{sec:pf:theorem:LD:new}. By this theorem, we construct a two-sided confidence interval for $\mathcal{Q}_{\iz\jz}(\Omega)$ with confidence level $1 - \alpha$, denoted as $\hat{\mathcal{C}}_{\iz\jz,1 - \alpha}(  \Omega)$. Also, for the one-sided hypothesis testing in \eqref{test:functional:intro}, we get the p-value  $
p_{\iz\jz}(\Omega) = 1 - \Phi\Big(\hat{\fav}_{\iz\jz} (\Omega) /{\hat{V}^{1/2}_{\iz\jz}(\Omega )}\Big).
$ 
We then compare $p_{\iz\jz}(\Omega)$ with the pre-specified significance level $\alpha$, e.g., $\alpha = 0.05$, and we reject $\mathrm{H}_0$ in \eqref{test:functional:intro} to conclude that item~$\iz$ performs better than item~$\jz$ over $\Omega$, if and only if, $p_{\iz\jz}(\Omega) < \alpha$.

\begin{remark}\label{rm:good-A}
Although  we assume the comparison graph $\M A$ is an  Erd\H{o}s--R\'enyi graph with sampling probability $p$, the theoretical results in the paper can be generalized to the deterministic $\M A$.  In fact, notice that in Section~\ref{sec:gaussian}, except Theorem~\ref{theorem:LD:new} and Theorem~\ref{thm:oracle} introduced later, all the theoretical results are true for any deterministic connected comparison graph. Theorems~\ref{theorem:LD:new} and \ref{thm:oracle} are also true for deterministic $\M A$ as long as it satisfies some good properties. In specific, in our proofs (see Lemma~\ref{lm:good-ER} in the Supplementary Material),  we can prove that Theorems~\ref{theorem:LD:new} and \ref{thm:oracle} are true as long as $\M A$ belongs to the following good property set for some $q \ge C \log n/n$:
\begin{align}\label{upperbound:A}
  \nonumber \mathcal{E}_{\mathrm{good}} = \Big\{\M A \in \{0,1\}^{n\times n} ~\Big|~ & 0.5nq \leq \min_{i\in[n]}\sum_{j\in[n]\backslash\{i\}} A_{ij} \leq \max_{i\in[n]}\sum_{j\in[n]\backslash\{i\}}A_{ij}\leq 2nq;\\&\qquad
 \frac{np}{2} \le \lambda_{\min,\perp}(\M L(\M A)) \le  \lambda_{\max}(\M L(\M A)) \le 2{np}\Big\},
 \end{align} 
 where $\M L(\M A)= \mathrm{diag}(\M A\bds 1) - \M A$ is the Laplacian matrix of $\M A$ and $\lambda_{\min,\perp}(\M L(\M A)) = \inf_{\|\M u\| = 1, \bds 1 ^\T\M u = 0} {\M u^\T\M L(\M A)\M u}$.
 We then show that for Erd\H{o}s--R\'enyi graph $\M A \in \mathcal{E}_{\mathrm{good}}$ with high probability  when the sampling probability $p \gg  \log n/n$. Therefore, our results can be generalize to other settings of comparison graph once  $\mathcal{E}_{\mathrm{good}}$ is verified.
\end{remark}

\subsection{Convergence rate of preference score function estimation}\label{sec:theory}

As we discussed above, we require the plug-in score function estimator $\hat{\boldsymbol{\theta}}$ satisfies Assumption~\ref{ass:theta}. In the previous section, we let $\hat{\boldsymbol{\theta}}  = \argmin_{{\boldsymbol{\theta}} \in \mathcal{G}(\mathcal{F}^n)} \mathcal{L}_n(\boldsymbol{\theta})$ as in \eqref{proposed:estimator}. In what follows, we provide a sufficient condition on the function class $\mathcal{F}$ such that Assumption~\ref{ass:theta} is satisfied, and we give a concrete example.

We first bound the $\mathcal{L}_2$ error $\mathcal{E}_2(\hat{\bds\theta},\bds\theta^* \mid \M A) $ in Definition~\ref{L2:def:theta} by ${{\mathcal{N}_{\delta}(\mathcal{F})}}$,  the minimal $\delta$-covering number of $\mathcal{F}$ with respect to the $\mathcal{L}_\infty$ metric for any $\delta > 0$, and the function class' uniform approximation error that
$$
  \mathrm{UAE}(\mathcal{F},\bds\theta^*) = \inf_{\bds\theta\in \mathcal{G}(\mathcal{F}^n_{\mathcal{D}})} \max_{i\in[n]} \sup_{\M x \in \mathbb{X}}|\theta_i(\M x)  - \theta^*_i(\M x)|.
$$

\begin{theorem}\label{thm:L2error:main}
  Let $\hat{\bds\theta}$ in \eqref{proposed:estimator}  be obtained  through $\mathcal{D}_n$ within function class $\mathcal{F}$. 
  \begin{itemize}
  \item {\bf(Oracle inequality)}  %
  When $np \geq 52\log n$, for any $\delta$ such that ${\mathcal{N}_{\delta}(\mathcal{F})} \geq 2$,  
we have, with probability at least $1 - 3/ n^{2}$, 
\bee
\label{oracle:E2:2}
 \mathcal{E}_2(\hat{\bds\theta},\bds\theta^*{\color{black}\mid \M A}) 
 \leq C \left(\mathrm{UAE}(\mathcal{F},\bds\theta^*) +  \delta + \frac{{n\log ({\mathcal{N}_{\delta}(\mathcal{F})})}   + \sqrt{n\delta L\log({\mathcal{N}_{\delta}(\mathcal{F})})}}{n^2pL}\right). %
  \ee

\item {\bf(Asymptotic normality)}  Given a $\delta = o\left(\frac{1}{n^3p}  \wedge \frac{1}{n\sqrt{npL}}\right) $, suppose that
\bee\label{asymptotic:c}
\mathrm{UAE}(\mathcal{F},\bds\theta^*) &= o\left(\frac{1}{n^3p}  \wedge \frac{1}{n\sqrt{npL}}\right),\quad \log({\mathcal{N}_{\delta}(\mathcal{F})})   = o\left(\frac{L}{n^2} \wedge \frac{\sqrt{npL}}{n}  \wedge \frac{L}{n^3\delta}  \wedge \frac{p}{\delta}\right).
\ee
 We have that Assumption~\ref{ass:theta} is satisfied, and the asymptotic normality in \eqref{Q:concentration:all} holds.

  \end{itemize}
  \end{theorem}
The proof of Theorem~\ref{thm:L2error:main} is provided in  Supplementary Material~\ref{proofthm:L2error:main}. The \(\mathcal{L}_2\) error in \eqref{oracle:E2:2} is called an oracle inequality because it does not rely on any specific model assumption on \(\bds\theta^*(\M x)\) or the function class \(\mathcal{F}\). On the right-hand side of \eqref{oracle:E2:2}, the first term is the uniform approximation error of the estimator class \(\mathcal{F}\), and the second term corresponds to a specified positive radius \(\delta\) for the covering number, and the third term is the covering entropy bound. A smaller radius \(\delta\) leads to a larger covering number bound, and thus an optimal radius minimizes the sum of the second and third terms. There is also an inherent trade-off between the approximation error and the entropy bound in \eqref{oracle:E2:2}. Specifically, a broader estimator class \(\mathcal{F}\) yields a smaller approximation error in the first term but typically incurs a larger covering entropy, resulting in a larger third term in \eqref{oracle:E2:2}.

Next, we give a concrete example of $\mathcal{F}$ named the composition model. This model is widely used in the nonparametric estimation  \citep{horowitz2007rate,juditsky2009nonparametric}. As special cases, it encompasses many classical function classes, such as the additive class and  H\"older class.  We provide the definition based on the compositional function class satisfying the H\"older smoothness, following   \citet{baraud2014estimating,schmidt2020nonparametric}.

 \begin{definition}[Composition model]\label{def:commodel} Let $\beta = r + s$ where $r\in\mathbb{N}$ and $s\in(0,1]$, and constant $\bar C > 0$. %
 We say a $d_0$-variate 
 function $f(\M x)$ is $(\beta,\bar C)$-H\"older if
\[
{ \sum_{\|\bds\alpha\|_1 <r}\sup_{\M x }|\partial^{\bds\alpha}f(\M x)| +\sum_{\|\bds\alpha\|_1 = r}\sup_{\M x_1,\M x_2 }\frac{\left|\partial^{\bds\alpha}f(\M x_1) - \partial^{\bds\alpha}f(\M x_2)\right|}{\|\M x_1 - \M x_2\|^s}}       \leq\bar C, \text{for all $\M x\in\bar{\mathbb{X}}\subseteq \R^{d_0}$.}
\]
 Denote the class of all $d_0$-variate and $(\beta,\bar C)$-H\"older functions over $\bar{\mathbb{X}}$ as $\mathrm{C}_{d_0}^{\beta}(\bar{\mathbb{X}},\bar C)$. Let $q\in\mathbb{N}_+, \M d = (d_0 = d,d_1,\ldots,d_q = 1)\in \mathbb{N}_+^{q + 1},\M t = (t_0,t_1,\ldots,t_q)\in \mathbb{N}_+^{q + 1}$, $\bds \beta = (\beta_0,\beta_1,\ldots,\beta_q) \in \R_+^{q + 1}$. The composition model is
 \bee\nonumber
 \mathscr{C}(q,\M d,\M t,\bds \beta,\bar C) 
  =\Big\{f = g_q\circ\bds g_{q - 1}\circ\cdots\circ \bds g_0\mid \,&  \bds g_{u}= (g_{u}^{(1)},\ldots,g_{u}^{(d_{u})}), \text{ such that } g_{u}^{(v)}\in\mathrm{C}_{t_{u}}^{\beta_{u}}([a_{u},b_{u}]^{t_{u}}, \bar C),
 \\
 &\text{ with }|a_{u}|,|b_{u}|\leq \bar C, \text{ for all }u = 0,\ldots,q\text{ and }v =1,\ldots d_u\Big\}.
 \ee
 \end{definition}
 
 If we assume the truth $\color{black}\bds\theta^*(\M x)\in\mathscr{C}^n(q,\M d,\M t,\bds \beta,\bar{C})$ where all parameters for $\mathscr{C}$ are fixed. Define the effective smoothness $\beta^*$ and dimension $t^*$ as
  \[
  \tilde{\beta}_u = \beta_u\prod_{\ell = u + 1}^q \min\{\beta_{\ell},1\} \text{ and }
  (\beta^*,t^*) = \argmin_{(\beta_u,t_u)_{u=0}^q}\tilde{\beta}_u / t_u.
  \]
  Consider $\mathcal{F} = \mathcal{F}_{\mathcal{D}}$ where all parameters for the ReLU-DNN class $\mathcal{F}_{\mathcal{D}}$ satisfy the same conditions of \citet[Theorem 1]{schmidt2020nonparametric}, and the sample size therein equals to $npL$;~see Section~\ref{sec:DNN} in the Supplementary File for a detailed construction of the ReLU-DNN function class $\mathcal{F}_{\mathcal{D}}$. %
Then by the approximation power and covering number of $\mathcal{F}_{\mathcal{D}}$ \citep[Remark~5]{schmidt2020nonparametric}, we have 
\bee\label{uaelog}
{\color{black}\mathrm{UAE}(\mathcal{F}_{\mathcal{D}},\bds\theta^*)  
  \lesssim %
(npL)^{-\frac{2\beta^*}{2\beta^* + t^*}}},   \quad 
 \log(\mathcal{N}_{\delta}(\mathcal{F}_{\mathcal{D}})) \lesssim %
\log(npL)\Big\{ \log^2(npL) + \log(\delta^{-1}) \Big\}(npL)^{ \frac{t^*}{2\beta^* +  t^*}}.
\ee
By Theorem~\ref{thm:L2error:main} and choosing $\delta = (npL)^{-1}$, we have, when $np \geq 52\log n$,
  \bee\label{oracle:E2:3}
  \mathcal{E}_2(\hat{\bds\theta},\bds\theta^*\mid \M A) \lesssim  C  \log^3(npL) \cdot (npL)^{-\frac{2\beta^*}{2\beta^* + t^*}},
  \ee
with probability at least  $1 - 3/ n^{2}$.  For completeness, we include the detailed derivations of \eqref{uaelog} and \eqref{oracle:E2:3} in Section~\ref{sec:uaelog} %
in the Supplementary Materials. 
The rate  in \eqref{oracle:E2:3} is a nonparametric generalization of the optimal $\mathcal{L}_2$ convergence rate ${O}\big((npL)^{-1}\big)$   in the parametric BTL model where $\bds\theta^*(\M x)$ are constant functions.  The parameters $\beta^*$ and $t^*$ can be viewed as the effective smoothness and dimension of the preference score functions in the compositional model, respectively. They generally measure the complexity of the underlying functions \citep{juditsky2009nonparametric}. When $q = 0$,  the compositional model induces the H\"older function class, and $\beta^*$ and $t^*$ become the smoothness and dimension of H\"older continuous functions, respectively. 
{\color{black}Further suppose $L \asymp n^{\xi_L}$ for some $\xi_L > 0$ and  $p\asymp n^{-\xi_p}\log n$ for some $\xi_p\in(0,1]$, then if $ {2\beta^*} > t^*$ and 
$$
\xi_L - \xi_p > 1 + \frac{ 4t^*}{  2{\beta^*} {\color{black}-} t^* }\text{ and } \xi_L +\frac{t^*}{2\beta^*}\xi_p > 2 + \frac{3t^*}{2\beta^*},
$$}%
chooinsg $\delta = (npL)^{-1}$, we can verify that the condition in \eqref{asymptotic:c} is satisfied,
which thus implies that the asymptotic normality follows as in \eqref{Q:concentration:all}. We illusrate the regions of \(\xi_L, \xi_p\) that guarantee asymptotic normality under different regimes of \(\beta^*\) and \(t^*\) in Figure~\ref{fig:rate}. 

 \begin{figure}
  \begin{subfigure}[b]{0.33\textwidth}
   \includegraphics[width=1\textwidth]{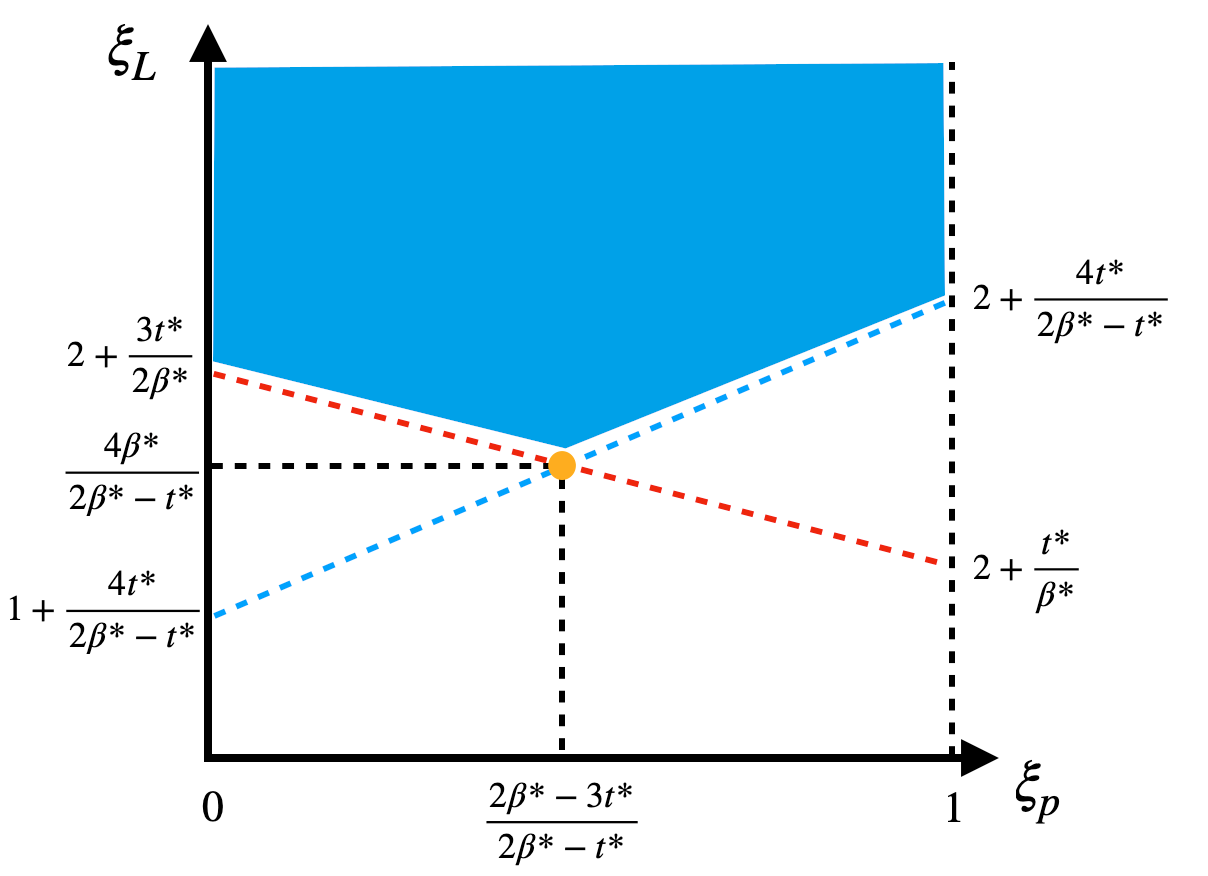}
      \caption{$\beta^*\in(0.5t^*,1.5t^*)$}\label{pa:rate}\end{subfigure}
\begin{subfigure}[b]{0.33\textwidth}
\includegraphics[width=1\textwidth]{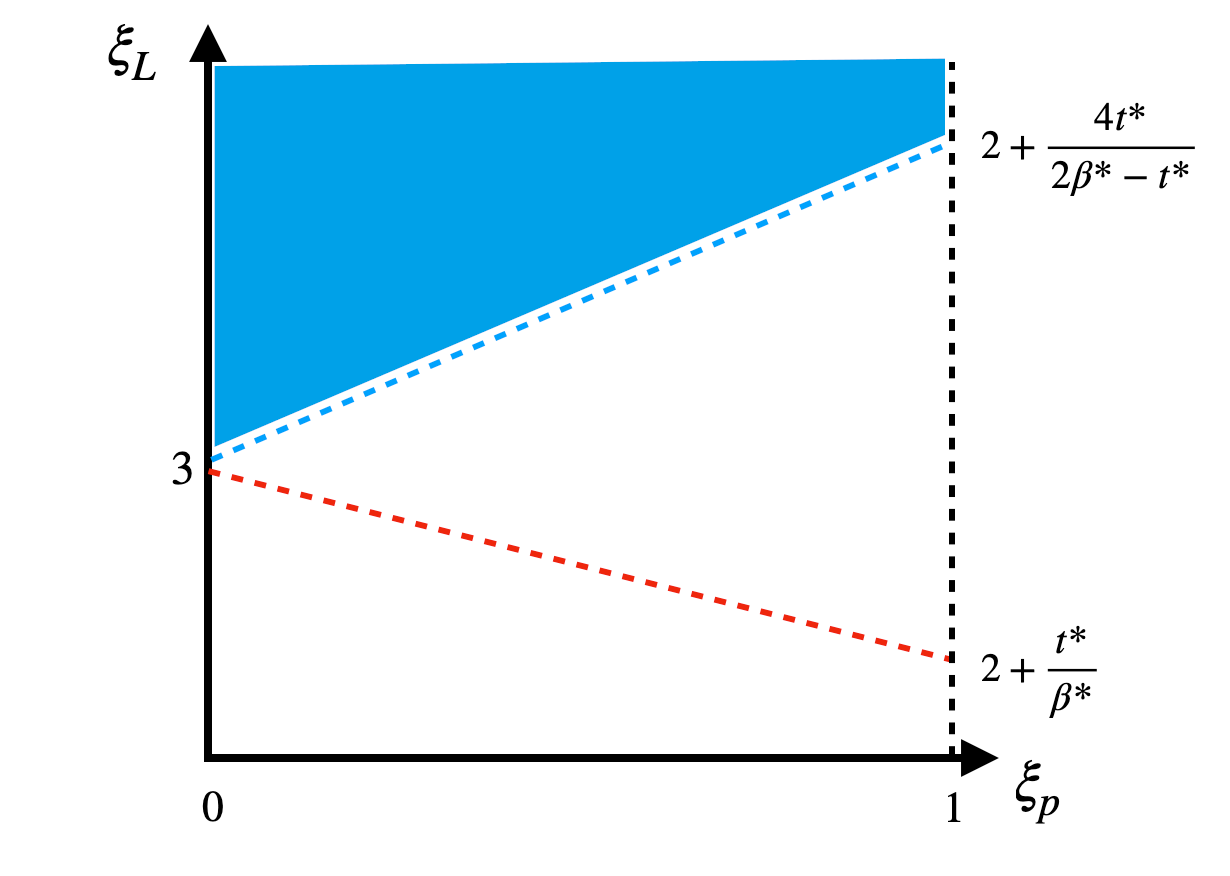}  
\caption{$\beta^* = 1.5t ^*$}\label{pb:rate}
 \end{subfigure} 
\begin{subfigure}[b]{0.33\textwidth}
   \includegraphics[width=1\textwidth]{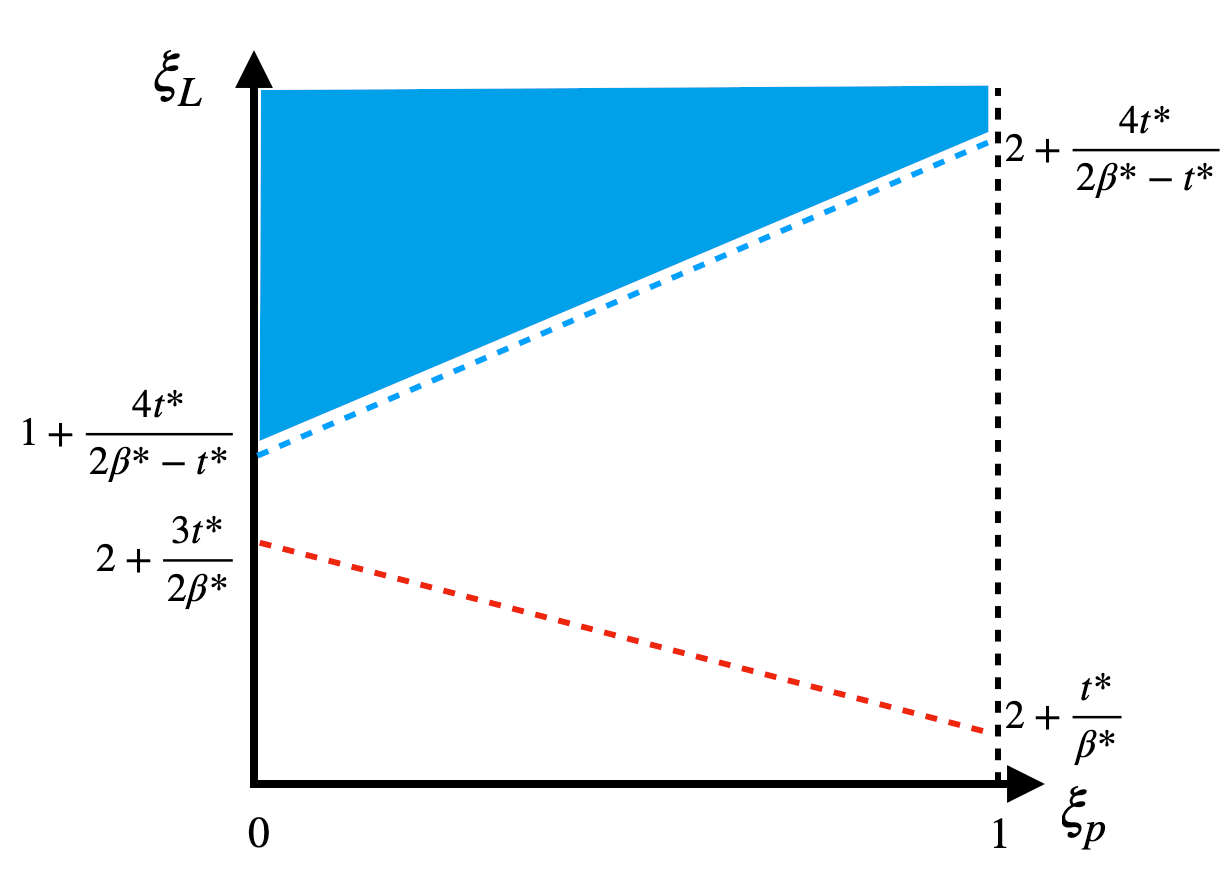}  
    \caption{$\beta^* \in( 1.5t ^*,\infty)$}
\label{pc:rate}
 \end{subfigure} 
\caption{In panel~\ref{pa:rate}, the asymptotic normality region (blue) lies above both the red and blue dashed lines. In panels~\ref{pb:rate} and~\ref{pc:rate}, the corresponding regions (blue) lie above the blue dashed lines.}\label{fig:rate} 
\end{figure}  

\subsection{Semiparametric efficiency lower bound  and asymptotic optimality}\label{sec:eff}
 We now study the statistical optimality of our proposed estimator under the semiparametric statistics framework. With any given comparison graph $\M A$ and a growing number of comparisons $L$, in the next theorem, we justify that the semiparametric efficiency  bound for the estimation of $\mathcal{Q}_{\iz\jz}(\Omega)$ is no smaller than 
 \begin{align}\label{def:tildesigma:main}
\tilde{V}_{\iz\jz}(\Omega)   &= \frac{1}{L|\M A|}\E  \Big(\mathbb{I}(\M X  \in \Omega)(\theta^*_{\iz}(\M X) - \theta^*_{\jz}(\M X))  - \fav_{\iz\jz}(\Omega)\Big)^2   + \frac{1}{L}\tilde{\sigma}(\M A), \text{ where }\\
\nonumber
\tilde{\sigma}(\M A) & =   \E \left(\frac{\mathbb{I}(\M X\in\Omega)(1 + \Delta_{\iz}(\M X))^2}{\sum_{j\in[n]} A_{\iz j}\psi'(\theta^*_{\iz}(\M X) - \theta^*_{j}(\M X))} + \frac{\mathbb{I}(\M X\in\Omega)(1 + \Delta_{\jz}(\M X))^2}{\sum_{i\in[n]}A_{\jz i}\psi'(\theta^*_{\jz}(\M X) - \theta^*_{i}(\M X))} \Big| \M A\right) \\
\nonumber
&\qquad      + 2    \E\left( \frac{A_{\iz\jz}\mathbb{I}(\M X\in\Omega) \psi'(\theta_{\iz}^*(\M X) - \theta^*_{\jz}(\M X))}{\sum_{ j\neq \jz} A_{\iz j}\psi'(\theta^*_{\iz}(\M X) - \theta^*_{j}(\M X)) \cdot \sum_{ i\neq \iz} A_{i\jz}\psi'(\theta^*_{\jz}(\M X) - \theta^*_{i}(\M X))} \Big|\M A\right)  ,
\\ \nonumber
\Delta_{\iz}(\M X)  & =  \frac{  A_{\iz\jz}  \psi'(\theta_{\iz}^*(\M X ) - \theta_{\jz}^*(\M X ))}{\sum_{ j\neq \jz} A_{i_0 j}\psi'(\theta_{\iz}^*(\M X ) - \theta_{j}^*(\M X ))},\quad  \Delta_{\jz}(\M X)   =  \frac{  A_{\iz\jz}  \psi'(\theta_{\iz}^*(\M X ) - \theta_{\jz}^*(\M X ))}{\sum_{ i\neq \iz} A_{i\jz}\psi'(\theta_{\jz}^*(\M X ) - \theta_{i}^*(\M X ))}.
\end{align} %

\begin{theorem}[Efficiency lower bound]\label{thm:efficiency} Suppose $\M A$ is given and connected. As $L\rightarrow\infty$, %
the semiparametric efficiency  lower bound of the asymptotic variance for the estimation of   ${\fav}_{\iz\jz}(\Omega)$, namely ${V}^*_{\iz\jz}(\Omega)$, satisfies
\bee\nonumber
{V}^*_{\iz\jz}(\Omega) \geq \tilde{V}_{\iz\jz}(\Omega).
\ee
\end{theorem} 
We defer the proof of Theorem~\ref{thm:efficiency} to Supplementary Material~\ref{pf:thm:efficiency}. This theorem shows that there is no regular estimator of ${\fav}_{\iz\jz}(\Omega)$ that can have an asymptotic variance smaller than $\tilde{V}_{\iz\jz}(\Omega)$; see e.g., \citet{bickel1993efficient}  for more related discussions on semiparametric statistics.

The next theorem  then shows that the proposed Fisher random walk debiased estimator obtained from Algorithm~\ref{alg:pde}  is asymptotically semiparametric efficient. 

\begin{theorem}[Semiparametric efficiency]\label{thm:oracle} %
Suppose $\bds\theta^*$ satisfies the   H\"older smooth condition, i.e., 
$
|\theta^*_i(\M x) - \theta^*_i(\M y)| \leq c_{\beta}\|\M x - \M y\|^{\beta}
$
for any $\M x,\M y \in\mathbb{X}\text{ and }i\in[n]$, where $c_{\beta} > 0,\beta \in (0,1]$ are  some fixed constants.
If $np  \geq   (\log n)^{\xi}  $ for some fixed $\xi > 3$, then as $n\rightarrow \infty$, we have
\bee\label{rate:oracle}
  \tilde{V}_{\iz\jz}(\Omega )    = \left\{1 + O_P\left(\frac{1}{n^{1/3} } + \sqrt{\frac{\log n}{np}}\right)\right\} V_{\iz\jz}(\Omega).
\ee
\end{theorem}
The proof of Theorem~\ref{thm:oracle} is deferred to Supplementary Material~\ref{pfthm:oracle}.

In comparison with the regular sparse assumption $np \geq \log n$,  we need a mildly denser condition $np \geq (\log n)^\xi$ with $\xi > 3$ to leverage the high-probability spectral concentration of {\ER} graph. The rate in \eqref{rate:oracle} also indicates that the asymptotic variance of our proposed estimator achieves the efficiency lower bound with a faster rate as $n$ grows, when the comparison graph is denser. 

With the asymptotic equivalence of $V_{\iz\jz}(\Omega )$ and $\tilde{V}_{\iz\jz}(\Omega )$ as shown in Theorem~\ref{thm:oracle},  we have another estimator for the variance of $\hat{\fav}_{\iz\jz}(\Omega)$. We can estimate ${\tilde{V}}_{\iz\jz}(\Omega )$ by replacing ${\bds \theta}^*$ with $\hat{\bds \theta}$ in \eqref{def:tildesigma:main}. See Algorithm~\ref{alg:ci:2} for the cross-fitting estimate of ${\tilde{V}}_{\iz\jz}(\Omega)$ in the Supplementary Material for details. In practice, we suggest to compute both estimators and compare them to assess how closely the efficiency of our estimator approaches the efficiency lower bound. %

\section{Applications to Multiple Domain Hypotheses and Domain Shift}\label{sec:extend}

In this section, we generalize our pairwise hypothesis testing in \eqref{test:functional:intro} to two other more complicated hypotheses: the multiple domain hypotheses and the domain shift hypothesis.

\subsection{Multiple domain hypotheses}\label{sec:uniform}

Given a series of pairs $(i_t,j_t)$ and domains $\Omega_t$ for $t = 1, \ldots, T$, we aim to test whether the language model $i_t$ is better than $j_t$ in a specific domain of interest $\Omega_t$, i.e., to test ${\fav}_{i_tj_t}(\Omega_t)> 0$  for all $t = 1, \ldots, T$. In particular, we consider the  multiple hypotheses testing problem that
\bee\label{test:functional:multi}
\mathrm{H}_{0t}: {\fav}_{i_tj_t}(\Omega_t)\leq 0 \quad\text{v.s.}\quad \mathrm{H}_{1t}:  {\fav}_{i_tj_t}(\Omega_t)> 0\text{ for all $t\in[T]$}.
\ee

Denote the index set as $\mathcal{W}_T = \{(i_1,j_1,\Omega_1),\ldots,(i_T,j_T,\Omega_T)\}$.
For a particular example, when we infer if item $\iz$ is the best among all $n$ items over domain $\Omega$, we test \eqref{test:functional:multi} with $\mathcal{W}_T = \{(i_0,1,\Omega),\ldots,(i_0,n,\Omega)\}$ and $T = n$. As another example, when we test if item $\iz$ is better than $\jz$ over multiple domains $\Omega_1,\ldots,\Omega_T$, we test \eqref{test:functional:multi} with $\mathcal{W}_T = \{(i_0,j_0,\Omega_1),\ldots,(i_0,j_0,\Omega_T)\}$.

 To test \eqref{test:functional:multi}, we consider the following maximum statistic
\bee\label{eq:TWT}
T_{\mathcal{W}_T} = \max_{t\in[T]} \sqrt{npL}\left\{\hat{\fav}_{i_tj_t}(\Omega_t) - {\fav}_{i_tj_t}(\Omega_t)\right\},
\ee 
where  $\hat{\mathcal{Q}}_{\mathcal{W}_T} = (\hat{\fav}_{i_1j_1}(\Omega_1),\ldots,\hat{\mathcal{Q}}_{i_Tj_T}(\Omega_T))$ are computed using Algorithm~\ref{alg:pde}.  Following the high-dimensional inferential framework of {\it many approximate means} (MAM) (see e.g., \citet[$\mathsection$2]{belloni2018high}), we approximate $T_{\mathcal{W}_T}$ through the elementwise maximum of  Gaussian multiplier bootstrap \citep{cck2013gaussian} that
\bee\label{form:Tstar}
T_{\mathcal{W}_T}^{\star} =  \max_{t\in[T]} \frac{1}{\sqrt{L}}\sum_{\ell = 1}^L\left\{\sqrt{np}\cdot\hat{\fav}_{i_tj_t\ell}(\Omega_t)- \sqrt{np}\cdot\hat{\fav}_{i_tj_t}(\Omega_t)\right\}\xi_\ell,
\ee
where the samples $\bds\xi  = (\xi_1,\ldots,\xi_{L})$ are $L$  i.i.d. standard Gaussian samples from $\normal(0,1)$. Let
\bee\label{eq:ql}
 \hat{\fav}_{i_tj_t}(\Omega_t)
&=\frac{1}{|\M A|L }\sum_{(i,j)\in\mathcal{E}(\M A)}  \sum_{\ell = 1}^L \Bigg[\mathbb{I}(\M X_{ij \ell} \in \Omega_t) \left\{\hat\theta_{i_t}(\M X_{ij \ell})  - \hat\theta_{j_t}(\M X_{ij \ell})\right\}  \\
& \quad  + \Big\{{\pi}_{i}(\M X_{ij \ell}\mid \hat{\bds \theta},i_t,j_t, \Omega_t) - {\pi}_{j}(\M X_{ij \ell}\mid \hat{\bds \theta},i_t,j_t, \Omega_t)\Big\}\left\{Y_{ij \ell} - \psi\left(\hat\theta_{i}(\M X_{ij \ell})  - \hat\theta_{j}(\M X_{ij \ell})\right)\right\}\Bigg],
\ee
where ${\bds\pi}(\M X_{ij \ell}\mid \hat{\bds \theta},i_t,j_t, \Omega_t)$ is to replace the domain $\Omega$ with $\Omega_t$ in \eqref{iota0}. Our method and theoretical results remain valid if we use the cross-fitting method as in Algorithm~\ref{alg:ci:addition}.
We then obtain the p-value by
$$ 
p_{\mathcal{W}_T}=  \pr_{\bds\xi} \left(T_{\mathcal{W}_T}^{\star} >   \sqrt{npL}  \max_{t\in[T]} \hat{\fav}_{i_tj_t}(\Omega_t)   \mid \hat{\mathcal{Q}}_{\mathcal{W}_T}\right),
$$
where \(\mathbb{P}_{\bds\xi}(\cdot)\) indicates that the randomness arises solely from~\(\bds\xi\), the independent Gaussian samples. For each $t = 1, \ldots, T$,
we reject $\mathrm{H}_{0t}$ in \eqref{test:functional:multi} if and only if $p_{\mathcal{W}_T} < \alpha$. The next theorem justifies our inference procedure.

\begin{theorem}\label{thm:bootstrap}
 Suppose that the conditions in Theorem~\ref{alg:ci:addition} hold. Furthermore, there exists constants $0<c<C$ and $\xi > 3$ such that (i) $\max\{T,L\} \leq  n^{C}$; (ii) $np \geq (\log n)^{\xi} $; (iii) $\inf_{t \in [T]}\pr(\M X\in \Omega_t) \geq c$; (iv) $(\log n)^{10}/(p^3 L ) = o(1)$; (vi) $L \geq c n^{1/3}(\log n)^3$; (v) with probability at least $1 - 1/n$,
\bee\label{new:l2rate}
\mathcal{E}_2(\hat{\bds\theta} ,{\bds\theta}^*\mid \M A) %
 = o\left(\frac{(\log n)^{-1}}{n^3p}  \wedge  \frac{(\log n)^{-1}}{n\sqrt{npL}}\right).
\ee
 Then as $n\rightarrow \infty$, we have
\bee\nonumber
\sup_{x\in \R}\left|\mathbb{P}\Big(T_{\mathcal{W}_T} \leq x\Big| \M A\Big) - \pr_{\bds\xi} \left(T_{\mathcal{W}_T}^{\star} \leq x \Big| \hat{\mathcal{Q}}_{\mathcal{W}_T} \right) \right|  = o_p(1).
\ee
\end{theorem}

Proof of Theorem~\ref{thm:bootstrap} is deferred to Supplementary Material~\ref{pf:thm:bootstrap}. Theorem~\ref{thm:bootstrap} allows both $T$ and $L$ to grow polynomially with respect to $n$ which is consistent with \citet{belloni2018high}.

Next, we define the family-wise error rate (FWER), which is the probability of making at least one Type I error among all $T$ hypotheses that
$$
\mathrm{FWER} = \mathbb{P}\left(\text{Reject $\mathrm{H}_{0t}$ when $\mathrm{H}_{0t}$ is true for at least one $t\in[T]$}\right).
$$
Then, Theorem~\ref{thm:bootstrap} immediately implies the following corollary, which shows that the FWER is controlled.
\begin{corollary}
Under the condition of Theorem~\ref{thm:bootstrap}, we have 
$
\mathrm{FWER} \leq  \alpha + o_{p}(1)
$
as $n\rightarrow \infty$.
\end{corollary}

\subsection{Domain shift}
 \label{sec:distributionalshift}
With contextual variables following some source distribution $\mathcal{X}_{s}$ in data, one might aim to test if item $\iz$ is better than item $\jz$ over $\Omega$ with   $\M X$ alternatively following the target distribution  $\mathcal{X}_t$. In this case, the new   inference target  becomes 
    \bee\nonumber
    \fav_{\iz\jz}(\Omega \mid \kappa) = \mathbb{E}_{\mathbf{X} \sim \mathcal{X}'}\Big(\mathbb{I}(\mathbf{X} \in \Omega)  \left(\theta^*_{\iz}(\mathbf{X}) - \theta^*_{\jz}(\mathbf{X})\right)\Big) = \E_{\M X\sim\mathcal{X}}\Big(\mathbb{I}(\M X \in \Omega) \left(\theta^*_{\iz}(\mathbf{X}) - \theta^*_{\jz}(\mathbf{X})\right)\kappa (\M X)\Big),
\ee
where $\kappa(\M X) = p_t(\M X)/p_s(\M X)$ is the density ratio with $p_t(\M X)$ and $p_s(\M X)$ being   density functions of $\mathcal{X}_6$ and $\mathcal{X}_s$ respectively. Then, our domain shift hypothesis is 
\bee\label{testing:kappa}
    \mathrm{H}_{0}: \fav_{\iz\jz}(\Omega \mid \kappa) \le 0 \quad \text{v.s.} \quad \mathrm{H}_{1}: \fav_{\iz\jz}(\Omega \mid \kappa) > 0.
\ee
For ease of presentation, we consider the case that $\kappa$ is known. We note that our analysis can be generalized to the case where the sample size is large, and we can estimate $\kappa(\M x)$ accurately.
We note that here, we only need external samples of $\M X$ to estimate $\kappa$, which is much easier than estimating $\bds \theta$. Following  Algorithms~\ref{alg:pde} and \ref{alg:ci:addition}, we make some tailored changes to have new  estimator $\hat{\mathcal{Q}}_{\iz\jz}(\Omega \mid \kappa)$ and the its CI $\hat{\mathcal{C}}_{\iz\jz,1-\alpha}(\Omega \mid \kappa)$ under distributional shift. In particular, to have the new estimator $\hat{\mathcal{Q}}_{\iz\jz}(\Omega \mid \kappa)$, we follow Algorithms~\ref{alg:pde}, and replace \eqref{principal_est_formula} by 
\bee\nonumber
\hat{\fav}_{\iz\jz}^{(s)}(\Omega\mid \kappa) = \frac{1}{|\M A|L/S}\sum_{(i,j)\in\mathcal{E}(\M A),\, \ell\in\mathcal{J}^{(s)}_n} & \mathbb{I}(\M X_{ij \ell} \in \Omega) \left(\hat\theta^{(s)}_{\iz}(\M X_{ij \ell})  - \hat\theta^{(s)}_{\jz}(\M X_{ij \ell})\right) \\&
+ \left(\hat{\pi}^{(s)}_{i}(\M X_{ij \ell}) - \hat{\pi}^{(s)}_{j}(\M X_{ij \ell})\right)\left(Y_{ij \ell} - \psi\left(\hat\theta^{(s)}_{i}(\M X_{ij \ell})  - \hat\theta^{(s)}_{j}(\M X_{ij \ell})\right)\right) \kappa(\M X_{ij \ell}),
\ee
and for its CI $\hat{\mathcal{C}}_{\iz\jz,1-\alpha}(\Omega \mid \kappa)$, we follow Algorithm \ref{alg:ci:addition}, and for Step 2, we replace it by  
\bee\nonumber
\hat{\sigma}^{(s)}(\M A\mid\kappa)=\frac{1}{|\M A|^2L/S}\sum_{(i,j)\in\mathcal{E}(\M A)\atop \ell\in\mathcal{J}^{(s)}_n}\Big( \hat{\pi}_{\iz}^{(s)}(\M X_{ij \ell}) - \hat{\pi}_{\jz}^{(s)}(\M X_{ij \ell}) \Big)\kappa(\M X_{ij \ell}).
\ee
We summarize the detailed algrithms in Algorithms~\ref{alg:pde:ds} and \ref{alg:pde:ds:ci} in Supplementary Material~\ref{sec:algorithm}.

The next theorem shows the asymptotic properties of the domain shift estimator. 

\begin{theorem}\label{thm:disshift}
Assume that the conditions in Theorem~\ref{theorem:LD:new} hold, and $\sup_{\M x\in\mathbb{X}}\kappa(\M x) \leq C $ for some $C  > 0$. Then as $n\rightarrow \infty$, we have
\bee\nonumber
\frac{\hat{\fav}_{\iz\jz} (\Omega\mid \kappa)  - \fav_{\iz\jz}(\Omega\mid \kappa )}{\sqrt{V_{\iz\jz}(\Omega\mid  \kappa)}} \rightsquigarrow \normal(0,1),\quad \frac{\hat{\fav}_{\iz\jz} (\Omega\mid \kappa)  - \fav_{\iz\jz}(\Omega\mid \kappa )}{\sqrt{\hat{V}_{\iz\jz}(\Omega\mid \kappa)}}\rightsquigarrow \normal(0,1), \text{ where }
\ee
\bee\nonumber
V_{\iz\jz}(\Omega\mid  \kappa)  &=\frac{1}{L}\E \Big(\frac{1}{{|\M A|}}\Big\{\mathbb{I}(\M X  \in \Omega)\kappa(\M X)(\theta_{\iz}^*(\M X) - \theta_{\jz}^*(\M X))   - \fav_{\iz\jz}(\Omega\mid \kappa)\Big\}^2  \Big) + \frac{\sigma({\M A}\mid \kappa)}{L},
\\
 \sigma({\M A}\mid \kappa)  &= \E \Big(\mathbb{I}(\M X\in\Omega)\kappa(\M X)(\M e_{\iz} - \M e_{\jz})^\T{\pmb{\mathscr{L}}}^\dagger(\M X|{\bds \theta}^*)(\M e_{\iz} - \M e_{\jz}) \Big).
\ee
\end{theorem}

The proof of Theorem~\ref{thm:disshift} is deferred in Supplementary Material~\ref{pf:thm:disshift}.

\section{Numerical Experiments}\label{sec:exp}
\subsection{Simulation}\label{sec:simu}

Using synthetic data, we illustrate the estimation and inference performance of our proposed estimator on the primary testing problem. We consider two settings of data generation. In Setting I, we consider a relatively simple scenario, where $\M X$ is a one-dimensional random variable and preference score functions are linear. In Setting II, we consider a more involved scenario, where $\M X$ is a $50$-dimensional random vector, and preference score functions are non-linear, adapted from the numerical experiment  in \citet{ghorbani2020discussion}.  For both settings,  the comparison graph  is generated from the Erd\H{o}s-R\'enyi random graph model with various $p$ before generating $\mathcal{D}_n$.  We approximate the true value of $\fav_{i_0j_0}(\Omega)$ iby Monte Carlo simulation that $$\fav_{i_0j_0}(\Omega)\approx \sum_{i = 1}^{10^6}\mathbb{I}(\M  X_i\in\Omega) (\theta^*_{i_0}(\M  X_i) - \theta^*_{j_0}(\M  X_i)),$$ where $\M X_1,\ldots,\M X_{10^6}$ are i.i.d. generated from the corresponding reference distribution.  
\paragraph{Setting I.} We generate all   $\{\M X_{ij \ell} \mid (i,j)\in\mathcal{E}(\M A),\ell\in[L]\}$    from the uniform distribution over $ (0,1)$ independently, and we generate all comparison outcomes with true preference score functions   $\theta_i^*(\M x) = \sin(i \pi/8)  \M x$ for all $i\in[n]$. Then, our target estimand is  $\fav_{i_0j_0}(\Omega)$ with $\iz = 1,\jz = 4$ and $\Omega = (0.3,0.8)$, whose true value is approximately  $-0.170$  by Monte Carlo experiments. 
\paragraph{Setting II.} We generate all $\{\M X_{ij \ell}\mid (i,j)\in\mathcal{E}(\M A),\ell\in[L]\}$  from the uniform distribution over $ (-\sqrt{3},\sqrt{3})^{50}$ independently, and we generate all comparison outcomes  with true preference score functions $\theta_i^*(\M x) =  (0.628)^{-1}\sin(i \pi/8){\tanh(\bds\beta^\T\M x)}$ and $\bds\beta = (1/\sqrt{50},\ldots,1/\sqrt{50})^\T$, for all $i\in[n]$. Then, our target estimand is   $\fav_{i_0j_0}(\Omega)$ with $\iz = 1,\jz = 4$, and $\Omega = \{\M x\mid {\bds\beta}^\T\M x > -0.5\}$, whose true value is approximately $-0.229$ by   Monte Carlo experiments.
\par
\
\par
\begin{figure}[t]
     \centering
     \begin{subfigure}[b]{0.32\textwidth}
         \centering
         \includegraphics[width=\textwidth]{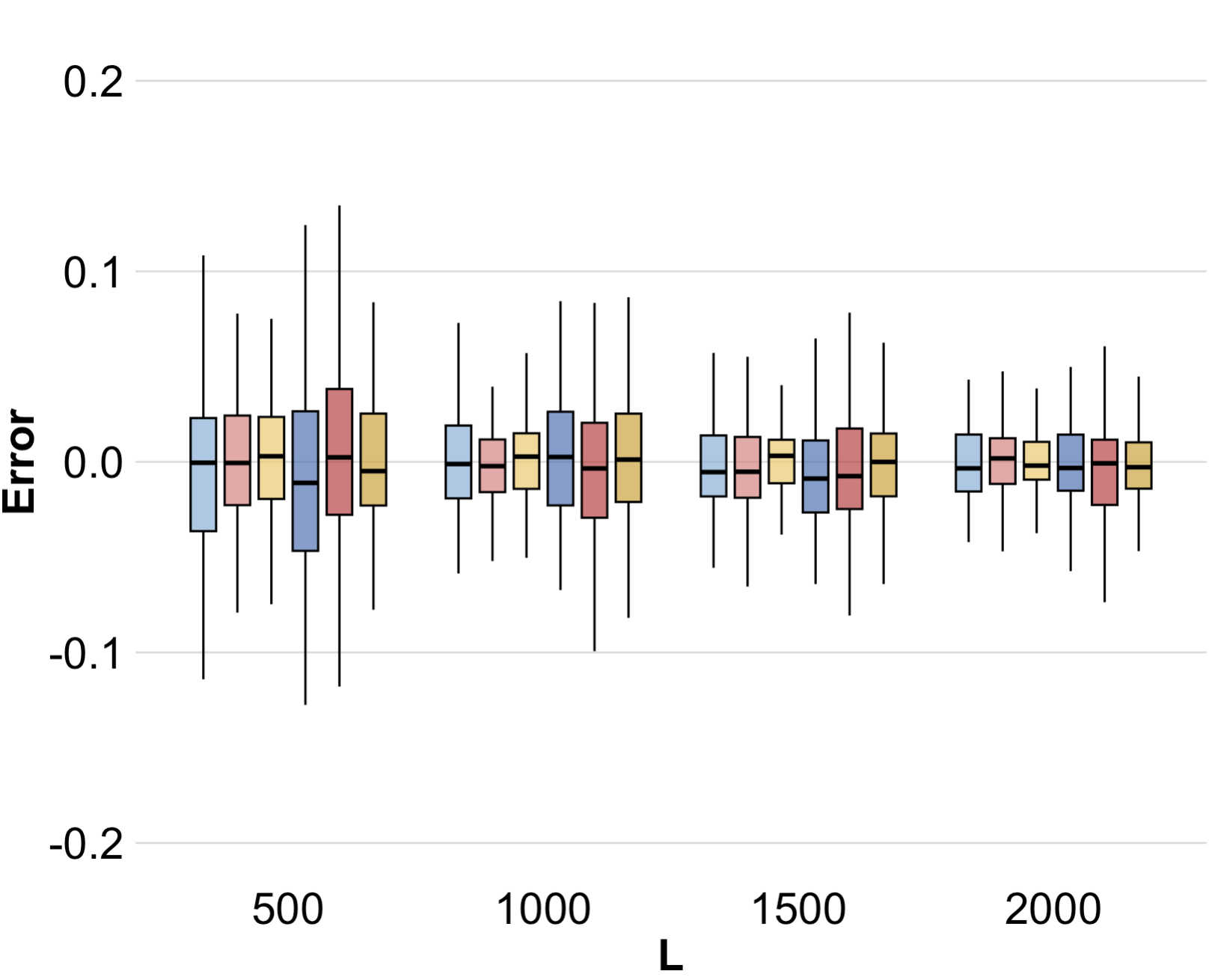}\caption{Box-plots of Estimation Errors}
     \end{subfigure}
     \hfill
     \begin{subfigure}[b]{0.32\textwidth}
         \centering   \includegraphics[width=\textwidth]{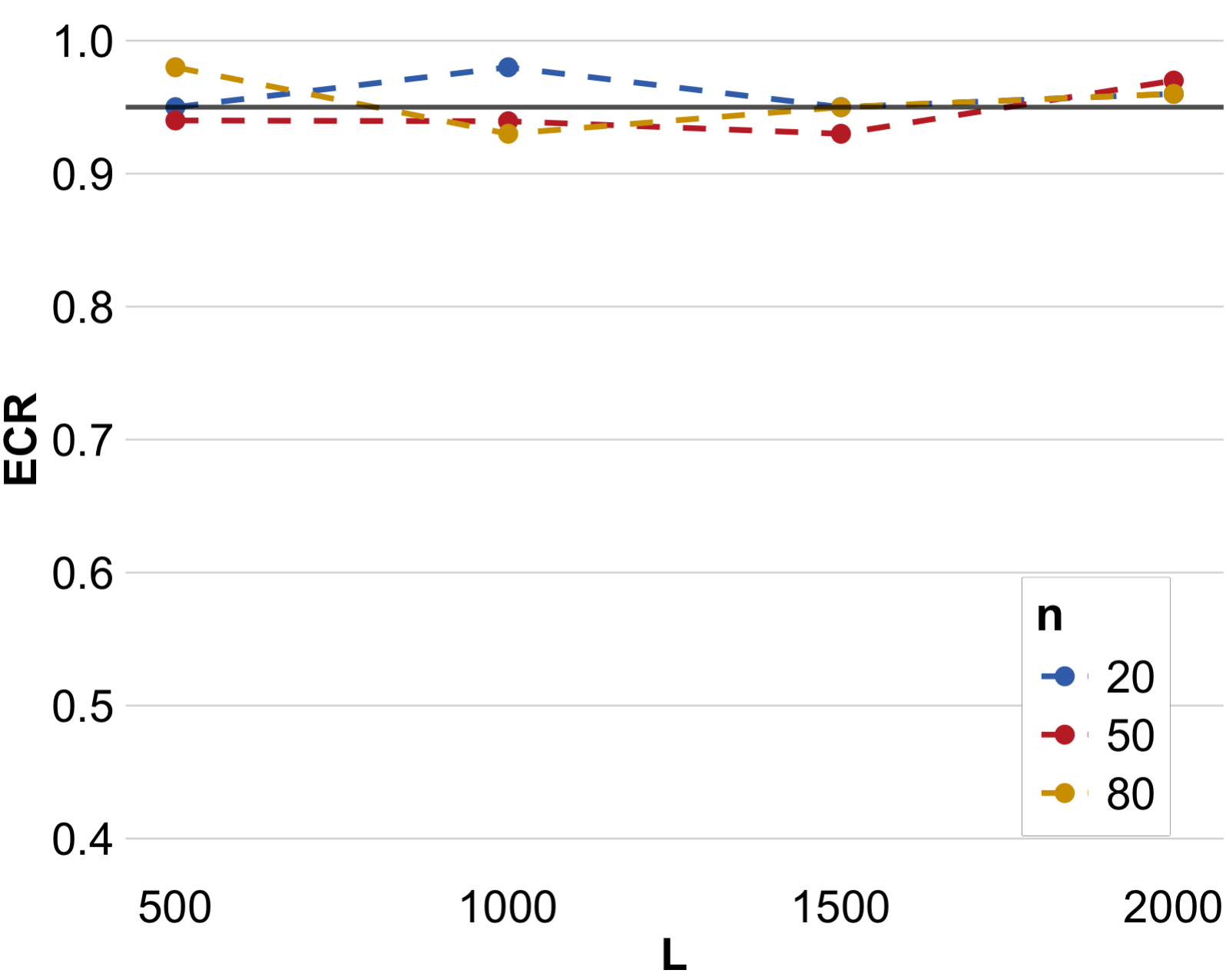}
         \caption{Empirical Coverage Rate}
     \end{subfigure}
     \hfill
     \begin{subfigure}[b]{0.32\textwidth}
         \centering
\includegraphics[width=\textwidth]{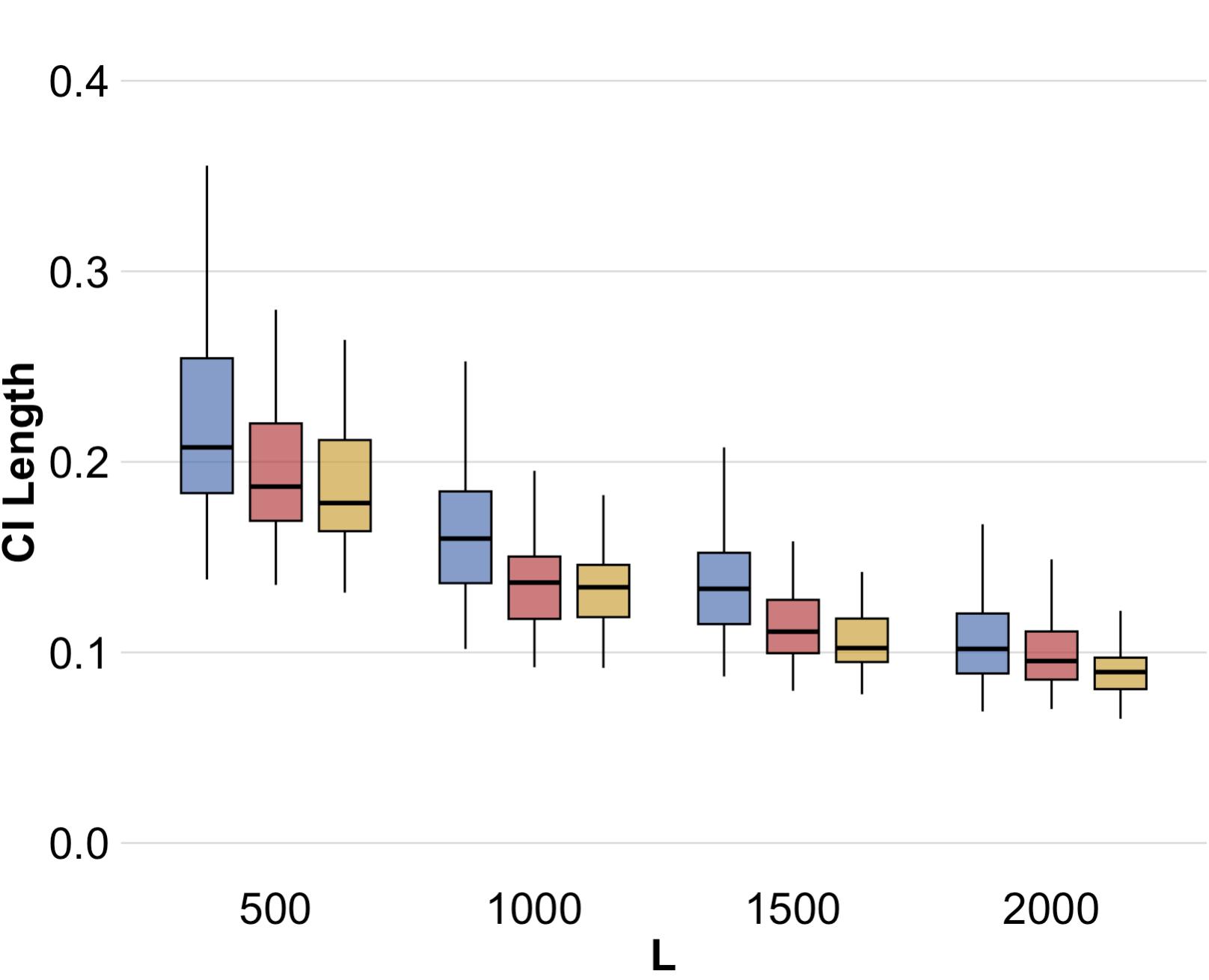}
         \caption{Box-plots of CI lengths}
     \end{subfigure}
\caption{Simulation results under Setting I. In all panels, the colors of blue, red, and orange represent the choices of $n = 20, 50$, and $80$, respectively, and the $X$-axis represents different choices of $L$ for the simulation. Different box-plots in Panel(a) represent the estimation errors of our proposed estimator (box-plots with dark color), and the estimation errors of the plug-in estimator (box-plots with light color) over 100 rounds. Panel(b) represents the ECRs of $\hat{\mathcal{C}}_{\iz\jz,0.95}(\Omega)$ over $100$ rounds  under different simulation settings. Panel(c) contains the box-plots of the lengths of $\hat{\mathcal{C}}_{\iz\jz,0.95}(\Omega)$ over $100$ rounds  under different simulation settings.}\label{simu:S1} 
\end{figure}
For both settings, we test our proposed method using different $(n,p) \in\{  (20,0.2), (50,0.1), (80,0.07)\}$, $\ell\in\{ 500,1000, 1500,2000\}$, and     $S = 3$. For each parameter group   $(n,p,\ell,S)$, we conduct the simulation in $100$ independent rounds under both Settings I and II. With simulated data in each round of the simulation,  we train $\hat{\bds\theta}^{(s)}(\M x)$ via the proposed maximum likelihood estimator with ten-hidden-layer ReLU-DNNs with batch size $16$ for optimization. We then calcuate the estimator and CI, $\hat{\mathcal{Q}}_{\iz\jz}(\Omega)$ and $\hat{\mathcal{C}}_{\iz\jz,0.95}(\Omega)$,  following Algorithms~\ref{alg:pde} and~\ref{alg:ci:addition}, respectively. We summarize the simulation results under Settings~I and~II in Figures~\ref{simu:S1} and  \ref{simu:S2}, and we report the estimation errors of the proposed estimator and the  plug-in estimator 
$$
\frac{1}{S}\sum_{s\in[S]}\frac{\sum_{(i,j)\in\mathcal{E}(\M A),\, \ell\in\mathcal{J}^{(s)}_n}   \mathbb{I}(\M X_{ij \ell} \in \Omega) \left\{\hat\theta^{(s)}_{\iz}(\M X_{ij \ell})  - \hat\theta^{(s)}_{\jz}(\M X_{ij \ell})\right\}}{|\M A|L/S},
$$
in Figures~\ref{simu:S1}(a) and \ref{simu:S2}(a). Under Setting I, the estimation errors of the proposed estimator and plug-in estimator are similar. This is because when the true preference score functions follow linear parametric models and the contextual variable is low-dimensional, the ReLU-DNN-based maximum likelihood estimator can attain a nearly optimal estimation error rate. Intuitively, the class of linear models is nested within the ReLU-DNN function class, ensuring such linear dependencies can be captured efficiently by the neural networks.  Thus, the plug-in estimator is also efficient enough. In contrast, under Setting~II, while the estimation biases of the proposed estimator and plug-in estimator are similar, the estimation variances of our proposed estimator are significantly smaller. This demonstrates the robustness and efficiency of the proposed estimator in estimating non-linear preference score functions, showcasing the statistical efficiency in Theorem~\ref{thm:oracle}.  Meanwhile, we let the empirical coverage rate (ECR) of proposed  CIs with nominal $95\%$ coverage  in Algorithm~\ref{alg:ci:addition} be
$$
\mathrm{ECR} = \#\{\text{the rounds of simulation where $\hat{\mathcal{C}}_{\iz\jz,0.95}(\Omega)$ contains $\mathcal{Q}_{\iz\jz}(\Omega)$}\} / 100.
$$
We report the $\mathrm{ECR}$s  under different parameter settings in Figures~\ref{simu:S1}(b) and \ref{simu:S2}(b), which are all close to $0.95$. This validates the effectiveness of our proposed $95\%$ CI. In addition, we report the lengths of proposed CIs in 100 rounds of the simulation under different settings in Figures~\ref{simu:S1}(c) and \ref{simu:S2}(c). With growing $n$ and $L$, the lengths of proposed CIs decrease, mirroring the asymptotic normality derived in Theorem~\ref{theorem:LD:new}.  
\begin{figure}[t]
     \centering
     \begin{subfigure}[b]{0.32\textwidth}
         \centering
 \includegraphics[width=\textwidth]{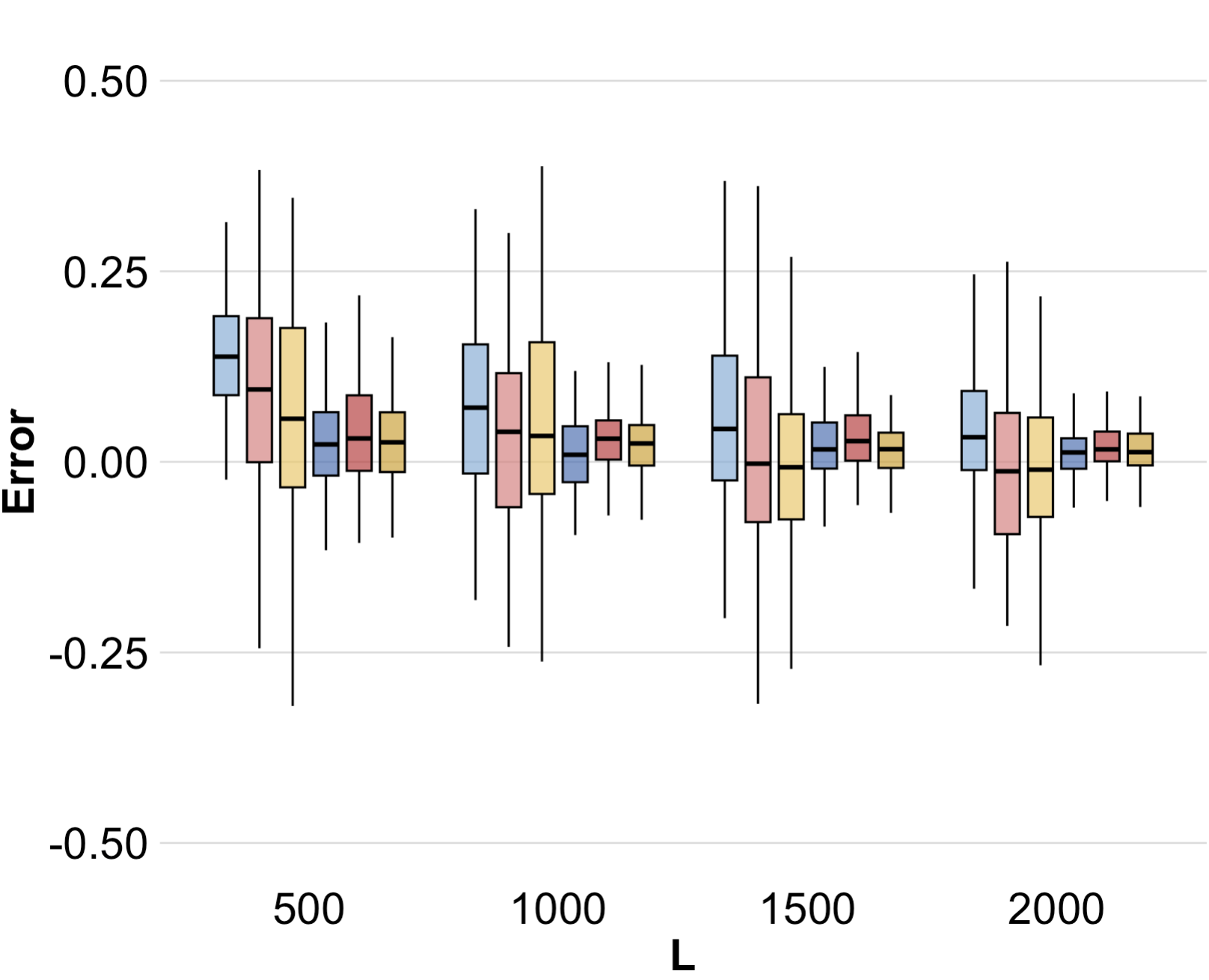}\caption{Box-plots of Estimation Errors}
       
     \end{subfigure}
     \hfill
     \begin{subfigure}[b]{0.32\textwidth}
         \centering   \includegraphics[width=\textwidth]{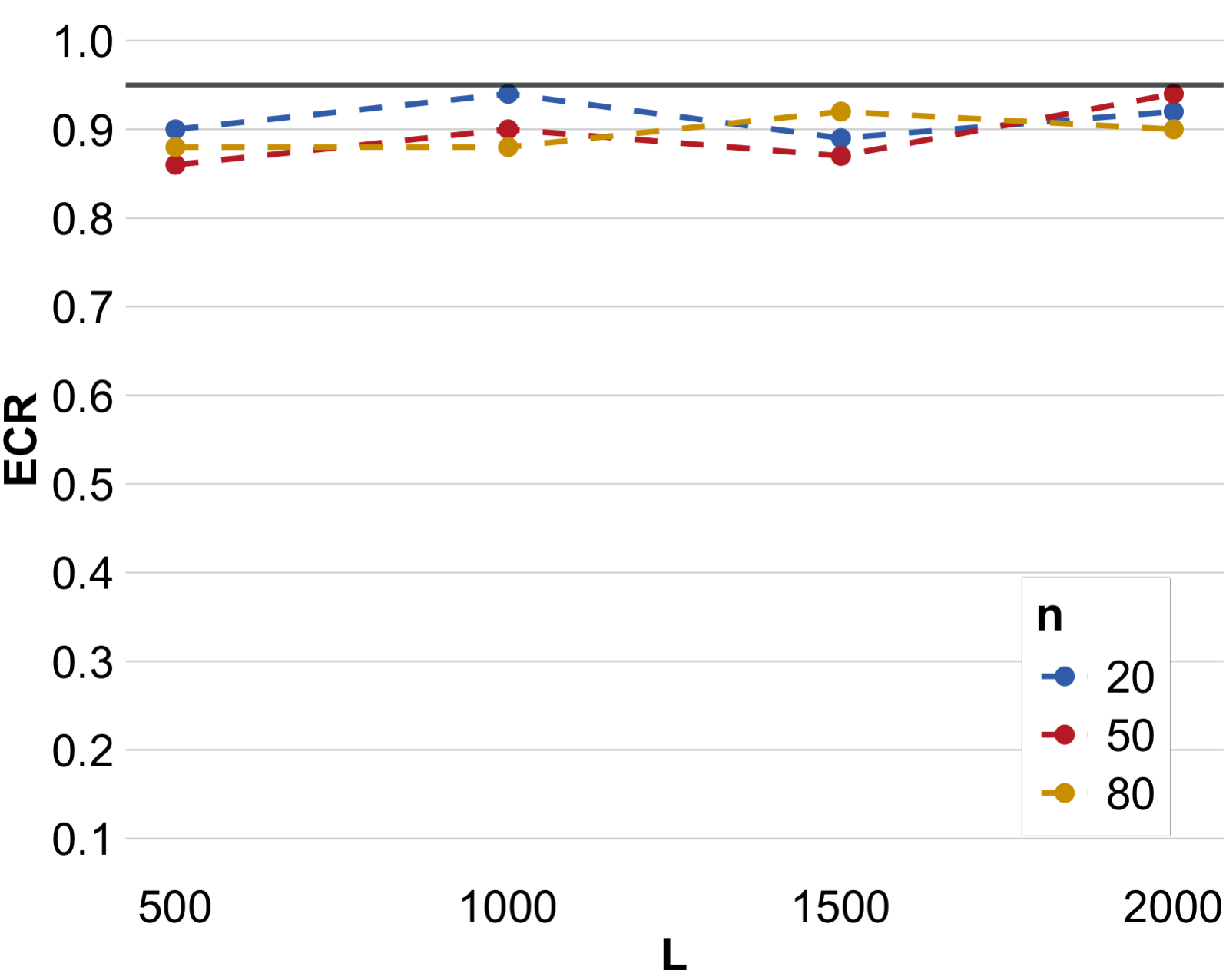}
         \caption{Empirical Coverage Rate }
     \end{subfigure}
\hfill
     \begin{subfigure}[b]{0.32\textwidth}
         \centering
         \includegraphics[width=\textwidth]{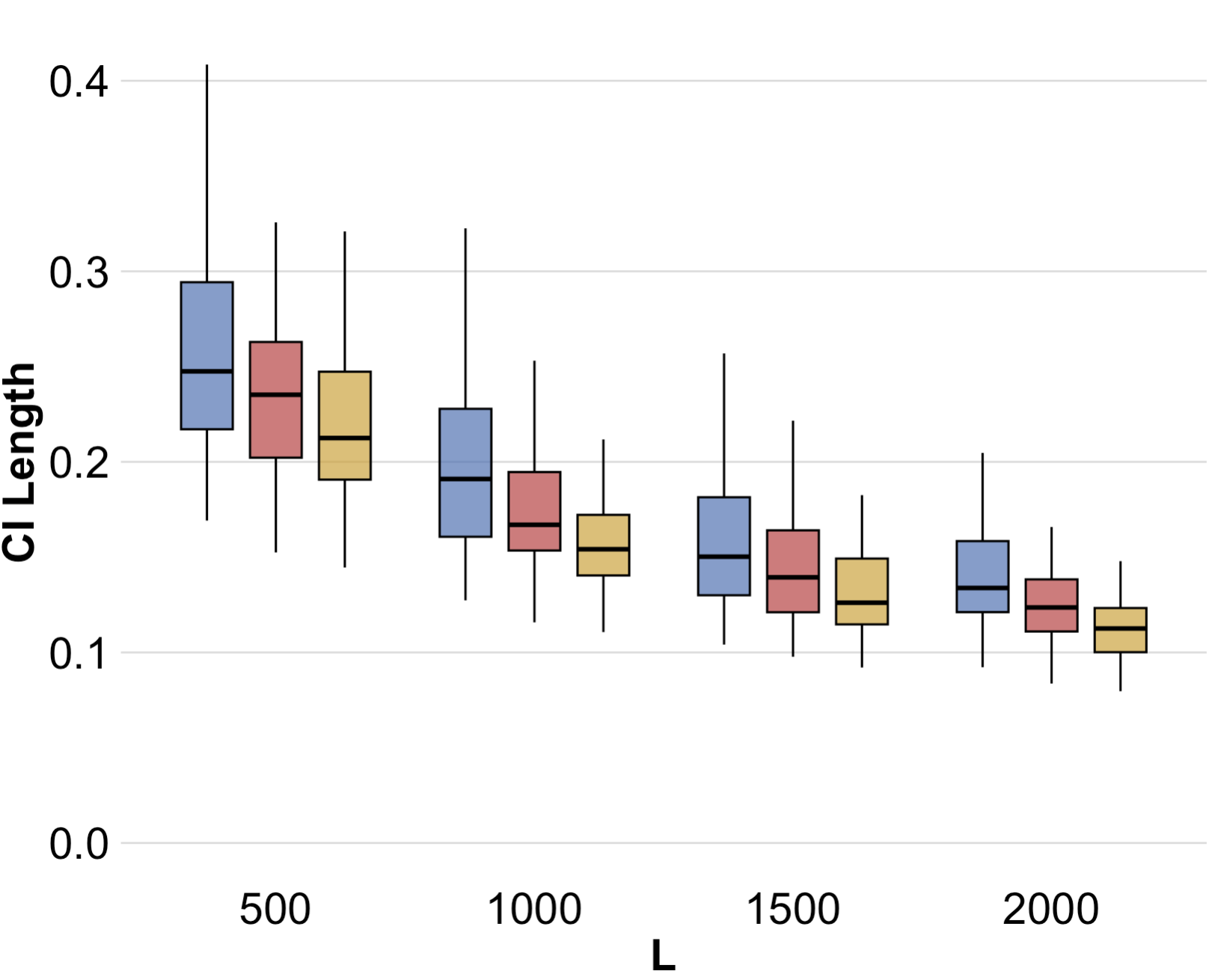}
         \caption{Box-plots of CI lengths}
     \end{subfigure}
\caption{Simulation results under Setting II. The explanations of all panels are similar with Figure~\ref{simu:S1}.}\label{simu:S2}
\end{figure}

\subsection{LLM comparison}\label{sec:llm:real}

We apply our proposed method to evaluate and compare the performances of different large language models (LLMs) on questions from various contexts using the Massive Multitask Language Understanding (MMLU) dataset \citep{hendrycks2020measuring}. MMLU  is a benchmark dataset to evaluate the knowledge and problem-solving abilities of LLMs. It comprises multiple-choice questions from 57 tasks spanning diverse academic and professional subjects. Specifically, we compare five LLMs: Alpaca (2023), LLaMA-1 (2023), LLaMA-2 (2023), GPT-3 (2020), and GPT-4o (2024), indexed from $1$ to $5$, respectively; hence, $n=5$. We focus on five topics in the MMLU dataset relevant to medical and biological reasoning: clinical knowledge, college biology, college medicine, medical genetics, and anatomy.  In short, we denote these topics by $\mathrm{ck},\mathrm{cb},\mathrm{cm},\mathrm{mg},\mathrm{a}$. For each topic, we consider $100$ questions, and we let the $k$-th question of topic $\star\in\{\mathrm{ck},\mathrm{cb},\mathrm{cm},\mathrm{mg},\mathrm{a}\}$ be~$\M Q_{\star,k}$. We further map  $\{\M Q_{\star,k}\mid k\in[100],\star\in\{\mathrm{ck},\mathrm{cb},\mathrm{cm},\mathrm{mg},\mathrm{a}\}\}$ to $256$-dimensional text embeddings  $\{{\M E}_{\star,k}\mid k\in[100],\star\in\{\mathrm{ck},\mathrm{cb},\mathrm{cm},\mathrm{mg},\mathrm{a}\}\} $ using the $\mathrm{text}$-$\mathrm{embedding}$-$3$-$\mathrm{small}$ from OpenAI platform\footnote{\url{https://platform.openai.com/docs/guides/embeddings/}}. For the convenience of the notation, we reindex $\M E_{\star,k}$ to $\M E_{m}$ and $\M Q_{\star,k}$ to $\M Q_{m}$, such that $\M E_m = \M E_{\mathrm{ck},m}$ and $\M Q_m = \M Q_{\mathrm{ck},m}$ for $m\in[100]$, $\M E_m = \M E_{\mathrm{cb},m-100}$ and $\M Q_m = \M Q_{\mathrm{cb},m-100}$ for $m\in[101,200]$, etc.  We define the question domains of text-embeddings for different topics by $\Omega_{\mathrm{ck}},\ldots,\Omega_{\mathrm{a}}$, and thus we have $\M E_m\in\Omega_{\mathrm{ck}}$ for $m\in[100]$, $\M E_m\in\Omega_{\mathrm{cb}}$ for $m\in[101,200]$, etc. 
For each question, we ask each pair of LLMs $100$ times, and each time, we use GPT-4 to compare the answers by the two LLMs and choose the one with a better answer. In particular, we let $(\M X_{ij\{100(m-1)+q\}} = \M E_{m}, Y_{ij\{100(m-1)+q\}})$  represent the contextual variable and comparison result for the $q$th-time comparison on the answers by LLM $i$ and LLM $j$ for   question $\M Q_m$. All generated  observations in our experiment are  $\{(\M X_{ij \ell},Y_{ij \ell})\mid 1\leq i<j\leq 5,\ell=1,\ldots,5\times 10^4\}$. Then we form two   datasets, namely, $\mathcal{}\mathcal{D}_{n,f}$ and $\mathcal{D}_{n,p}$  with    comparison graphs $\M A_{f}$ and  $\M A_{p}$:
\bee\nonumber
\M A_{f} = \big[A_{ij}^f\big]_{n\times n} = \begin{pmatrix}0 & 1  & 1 & 1 & 1
\\
1 & 0  & 1 & 1 & 1
\\
1 & 1  & 0 & 1 & 1
\\
1 & 1  & 1 & 0 & 1
\\
1 & 1  & 1 & 1 & 0
\end{pmatrix},\quad\M A_{\mathrm{p}} =  \left[A_{ij}^p\right]_{n\times n} = \begin{pmatrix}0 & 1 & 0 & 0 & 1
\\
1 & 0  & 1 & 0 & 0
\\
0 & 1  & 0 & 1 & 0
\\
0 & 0  & 1 & 0 & 1
\\
1 & 0  & 0 & 1 & 0
\end{pmatrix},
\ee 
respectively.  In particular, $ \mathcal{D}_{n,f}$ contains all observations generated in our experiment, while $ \mathcal{D}_{n,p}$ contains    the observations $\{(\M X_{ij \ell},Y_{ij \ell})\mid  l = 1,\ldots,5\times 10^4\}$ for all $(i,j)$ such that $A^{p}_{ij} = 1$. 
\par
We apply our inference procedure following the similar setting as our simulations in Section~\ref{sec:simu} for all $\mathcal{Q}_{i,j}(\Omega_\star)$ with $i\neq j\in[5] $ and $\star \in\{\mathrm{ck},\mathrm{cb},\mathrm{cm},\mathrm{mg},\mathrm{a}\}$,  using both the full dataset $\mathcal{D}_{n,f}$ and the dataset with partial observations $\mathcal{D}_{n,p}$. We report the p-values of the pairwise inference problem \eqref{test:functional:intro} for each $\mathcal{Q}_{i,j}(\Omega_\star)$, in Figure~\ref{real:S1:f1} and Figure~\ref{real:S1:f2}, using datasets $\mathcal{D}_{n,f}$ and $\mathcal{D}_{n,p}$, respectively. Setting the significance level at $\alpha = 0.05$, we summarize all significant pairwise comparison results across all five topics. For each topic, all significant pairwise comparisons yield a partial order of all LLMs, which we visualize through Hasse diagrams in Figures~\ref{real:S1:f1}~and~\ref{real:S1:f2}.  
Based on the complete dataset $\mathcal{D}_{n,f}$, the inference results in Figure~\ref{real:S1:f1}  show that GPT-4o consistently outperforms all other models across tasks. GPT-3 and LLaMA-2 generally perform better than LLaMA-1 and Alpaca, while the performance differences between GPT-3 and LLaMA-2 are generally not statistically significant, except in the college medicine domain. When using the partially observed dataset~$\mathcal{D}_{n,p}$, as shown in Figure~\ref{real:S1:f2}, we observe that there are fewer comparison results that remain significant. Such a phenomenon also aligns with our theoretical justifications, which predict that higher sampling probabilities lead to higher power of detecting significant differences under finite samples.
\par
We further examine our proposed inference procedure under the distributional shift scenario, as described in Section~\ref{sec:distributionalshift}. Intuitively, the questions on clinical knowledge and college medicine share substantial overlap in content. To illustrate this, we reduce the dimension of the question embeddings $\{\M E_{\mathrm{ck},m}\}_{m = 1}^{100}$ and $\{\M E_{\mathrm{cm},m}\}_{m = 1}^{100}$ using principal component analysis (PCA), and visualize the resulting two-dimensional embeddings in Figure~\ref{real:embed}. We observe that the domains of the reduced embeddings for the two topics are closely aligned, suggesting that the domains of these two embedding distributions are similar. This observation implies that $\mathcal{Q}_{ij}(\Omega_{\mathrm{ck}} \mid \kappa_{\mathrm{cm}:\mathrm{ck}})$ may serve as a reasonable approximation of $\mathcal{Q}_{ij}(\Omega_{\mathrm{cm}})$, where $\kappa_{\mathrm{cm}:\mathrm{ck}}(\M E)$ is the density ratio between the embedding distributions of college medicine and clinical knowledge. We apply our distribution-shift-aware method to estimate $\mathcal{Q}_{ij}(\Omega_{\mathrm{ck}} \mid \kappa_{\mathrm{cm}:\mathrm{ck}})$ and compute p-values for the corresponding hypothesis tests, as formalized in \eqref{testing:kappa}. The results are shown in Figures~\ref{real:S1}(a)--(c). In particular, the inference results based on $\mathcal{Q}_{ij}(\Omega_{\mathrm{ck}} \mid \kappa_{\mathrm{cm}:\mathrm{ck}})$ (Figure~\ref{real:S1}(b)) closely match those based on $\mathcal{Q}_{ij}(\Omega_{\mathrm{cm}})$ using the original estimator (Figure~\ref{real:S1}(c)), and differ noticeably from those based on $\mathcal{Q}_{ij}(\Omega_{\mathrm{ck}})$ (Figure~\ref{real:S1}(a)). These findings are consistent with our theoretical intuition and validate the effectiveness of our method in Section~\ref{sec:distributionalshift}, in handling moderate distributional shifts. To estimate the density ratio $\kappa_{\mathrm{cm}:\mathrm{ck}}(\M E)$, we first reduce the dimensionality of the question embeddings via PCA and then apply the relative unconstrained least-squares importance fitting (RuLSIF) method \citep{yamada2013relative}. We set the number of principal components to 6 based on ten-fold cross-validation, following the procedure outlined in \citet{kanamori2009least}. We repeat a similar analysis with distributional shift for the topic pair of clinical knowledge and medical genetics. In sharp contrast, as shown in Figure~\ref{real:embed}, the reduced embeddings for these two topics exhibit more distinct distributions. Consequently, we do not expect $\mathcal{Q}_{ij}(\Omega_{\mathrm{ck}} \mid \kappa_{\mathrm{mg}:\mathrm{ck}})$   to be a reasonable approximate of $\mathcal{Q}_{ij}(\Omega_{\mathrm{mg}})$. Thus in Figure~\ref{real:S2}, we observe noticeable differences between the inference results based on $\mathcal{Q}_{ij}(\Omega_{\mathrm{ck}} \mid \kappa_{\mathrm{mg}:\mathrm{ck}})$ and those based on $\mathcal{Q}_{ij}(\Omega_{\mathrm{mg}})$.

\begin{figure}
\begin{subfigure}[b]{0.32\textwidth}
         \centering
\includegraphics[width=\textwidth]{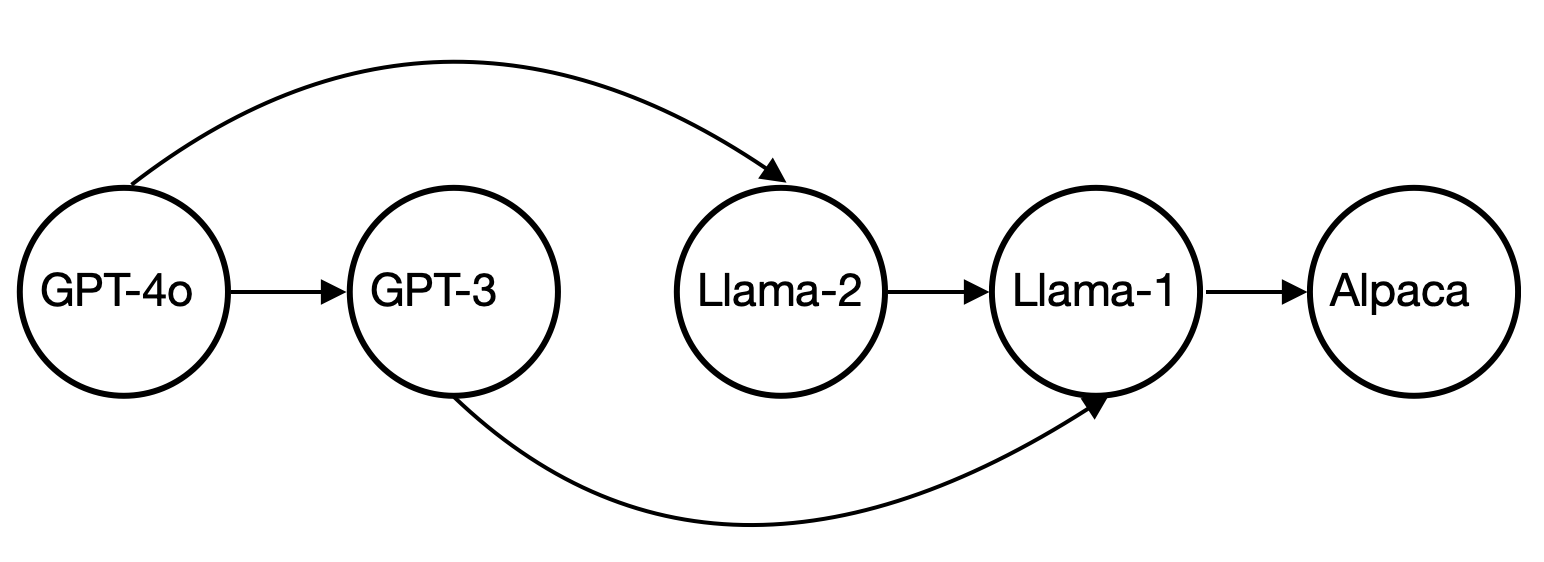} 
\includegraphics[width=\textwidth]{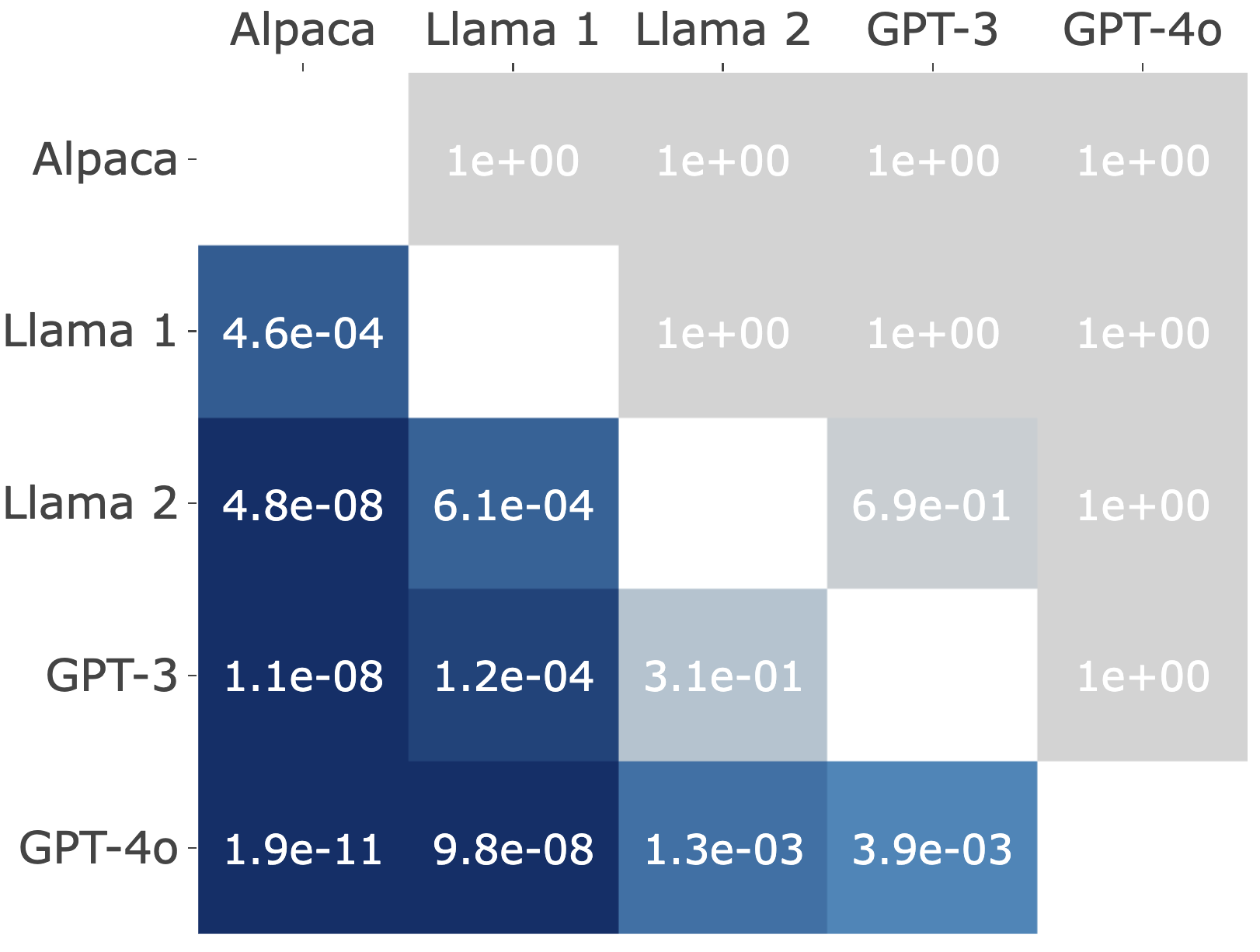}
\caption{Clinical knowledge}
     \end{subfigure} \hfill 
     \begin{subfigure}[b]{0.32\textwidth}
         \centering
 \includegraphics[width=\textwidth]{plot/compare_type1} \includegraphics[width=\textwidth]{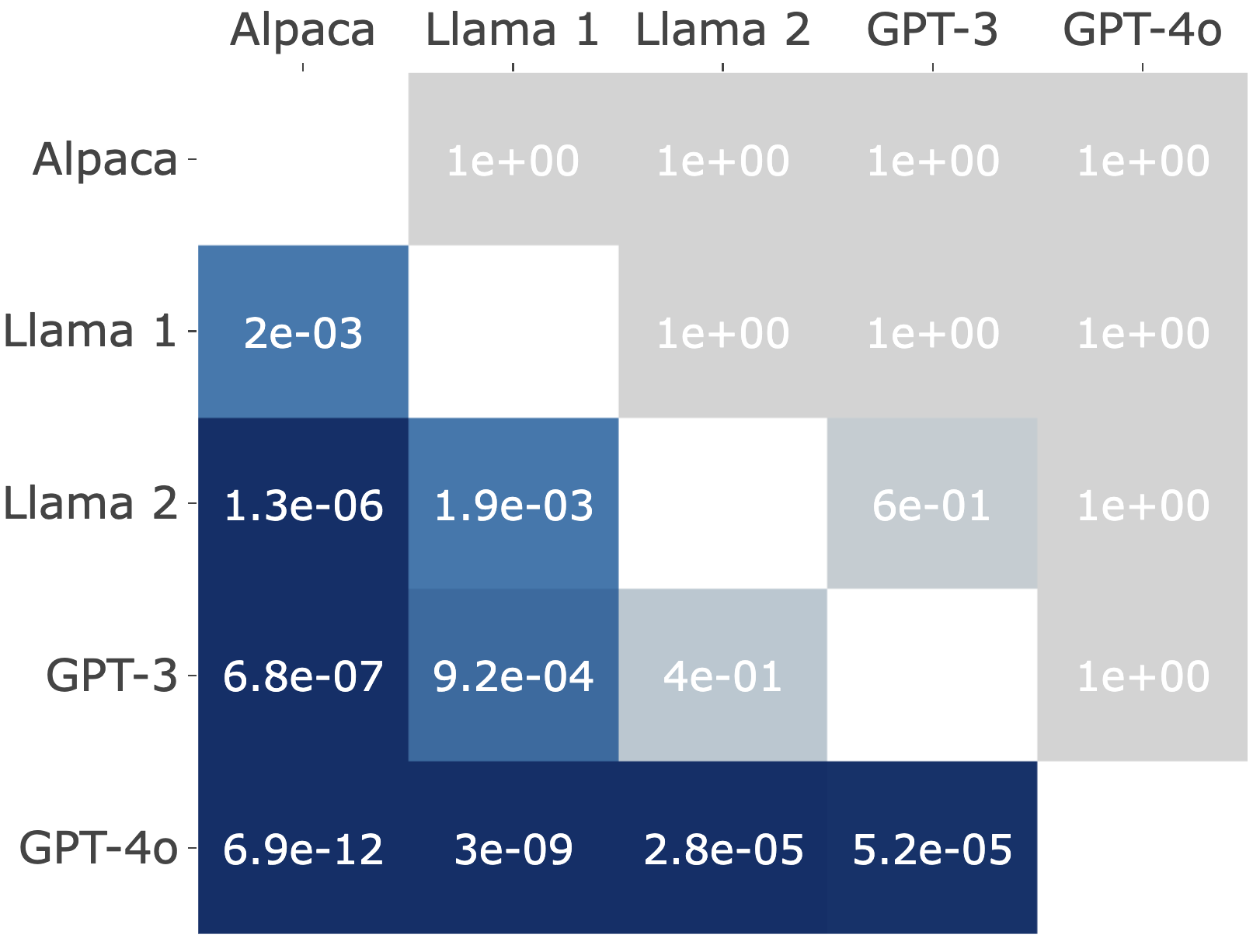}\caption{College biology}
     \end{subfigure}\hfill
     \begin{subfigure}[b]{0.32\textwidth}
         \centering
 \includegraphics[width=\textwidth]{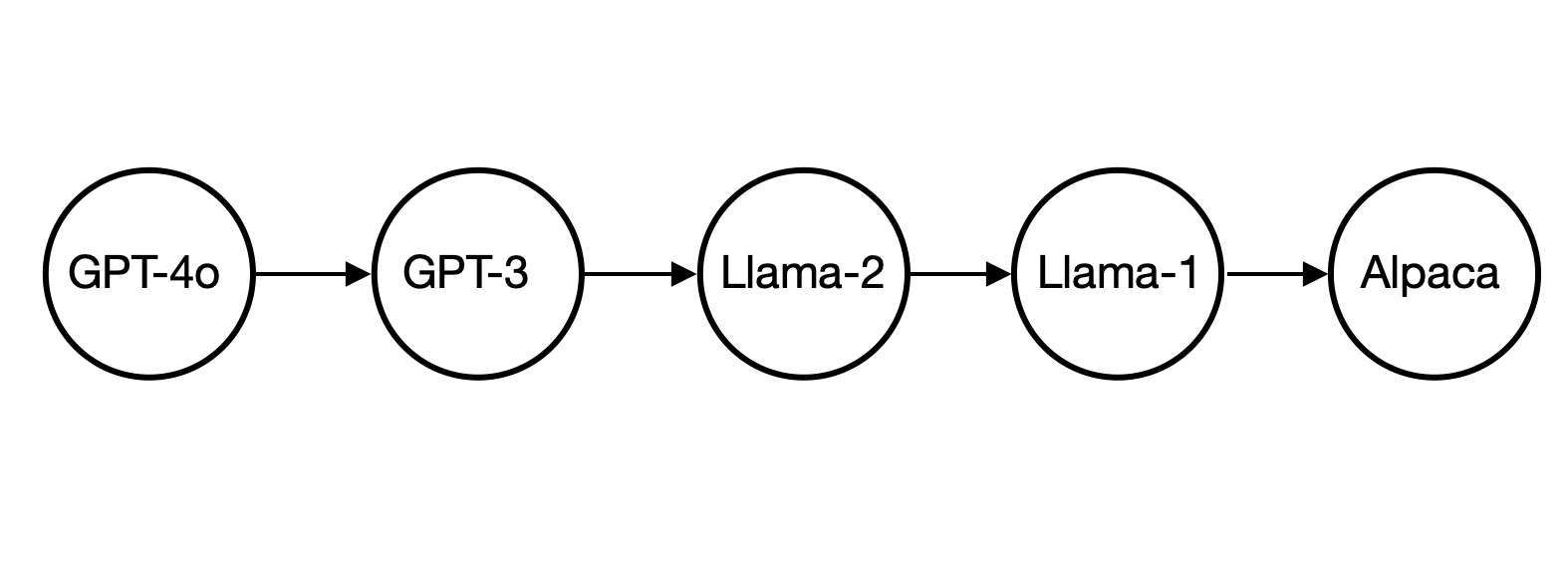}  \includegraphics[width=\textwidth]{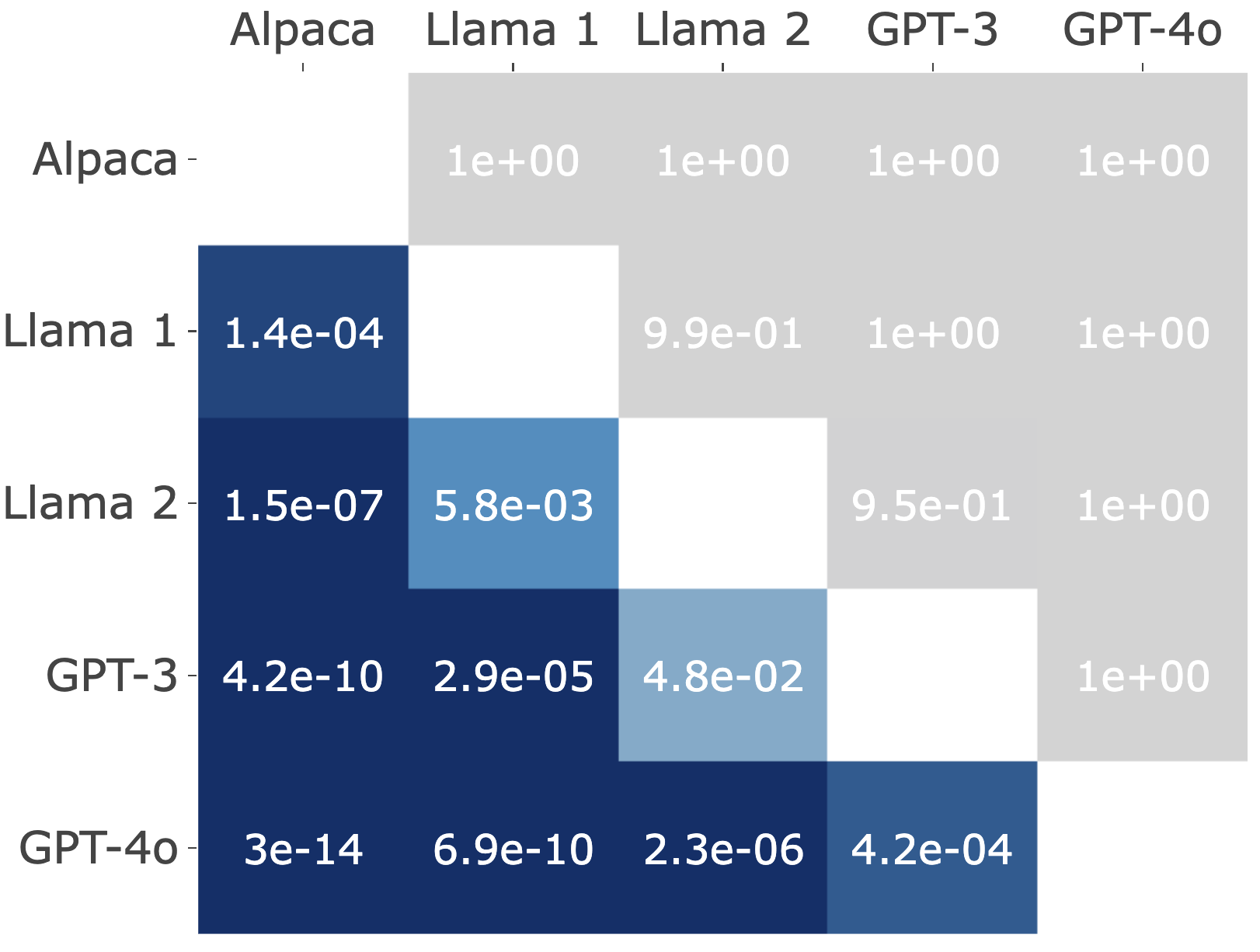}\caption{College medicine}
     \end{subfigure}
{\centering
     \begin{subfigure}[b]{0.32\textwidth}
         \centering
 \includegraphics[width=\textwidth]{plot/compare_type1} \includegraphics[width=\textwidth]{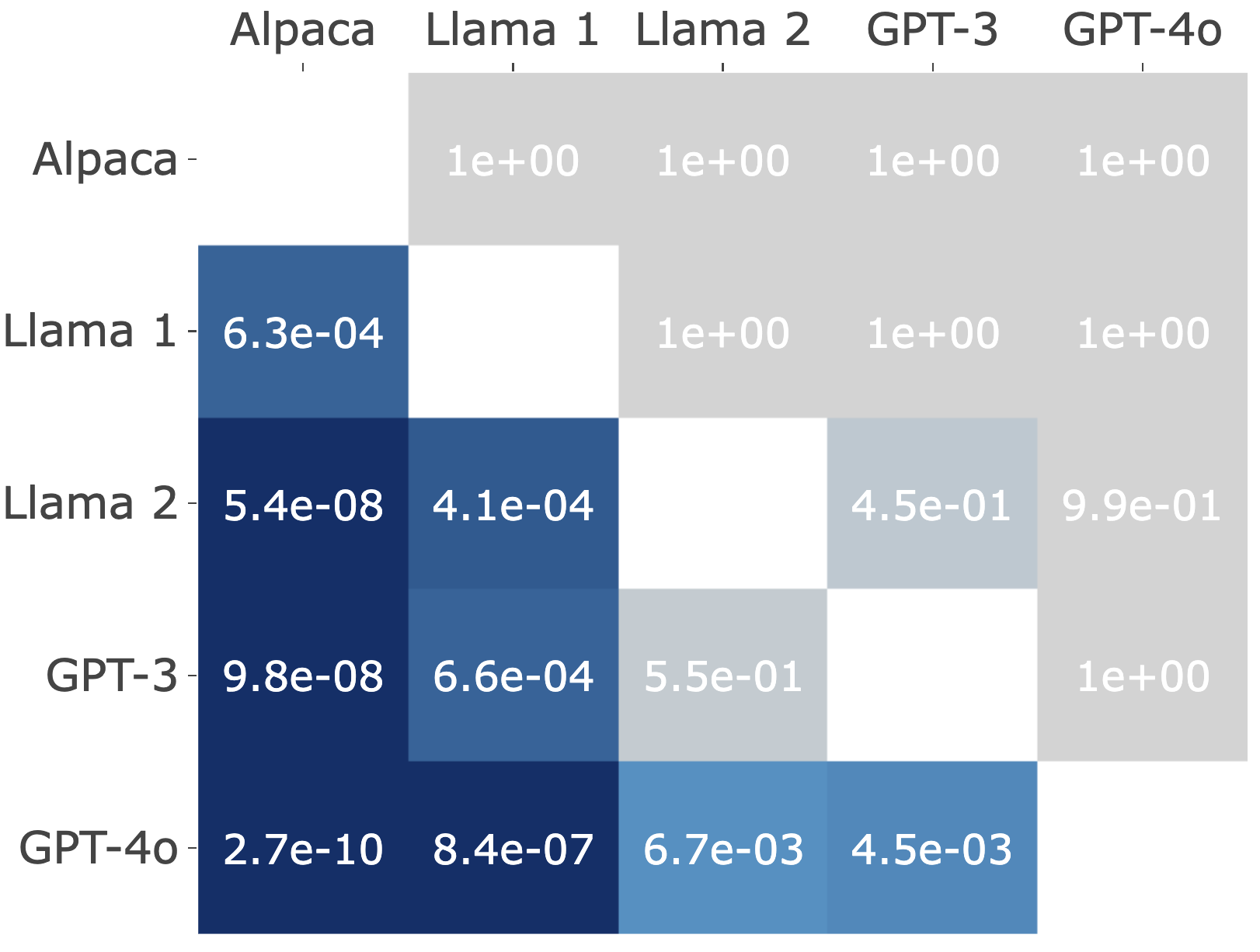}\caption{Medical genetics}
     \end{subfigure}}\centering \quad\quad\quad\quad\quad {\centering
\begin{subfigure}[b]{0.32\textwidth}
       
 \includegraphics[width=\textwidth]{plot/compare_type1} \includegraphics[width=\textwidth]{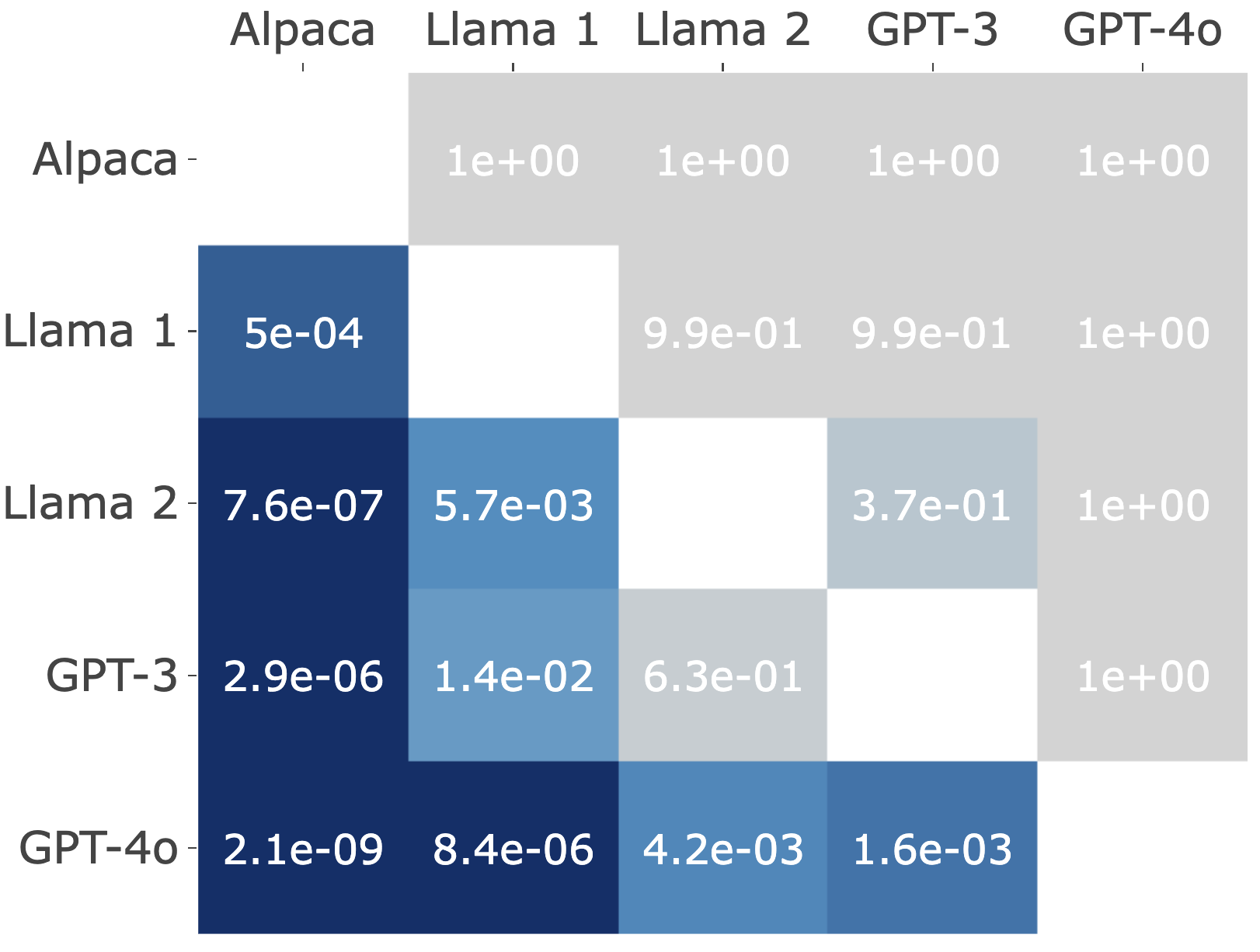}\caption{Anatomy}
     \end{subfigure}}
\caption{P-values diagrams and the subsequent Hasse diagrams  for the significant comparisons of LLMs under different topics with full dataset $ \mathcal{D}_{n,f}$ and $\alpha = 0.05$. For each subfigure, the upper panel is the Hasse diagram. The lower panel reports the p-value for the pairwise comparisons of two LLMs under the domains of different question topics. In Hasse diagram, if there is a direct path from LLM $i$ to LLM $j$, then LLM $i$ has a significant better performance than LLM $j$ under the corresponding topic.}\label{real:S1:f1}
\end{figure}

\begin{figure}
     \centering
     \begin{subfigure}[b]{0.32\textwidth}
         \centering
 \includegraphics[width=\textwidth]{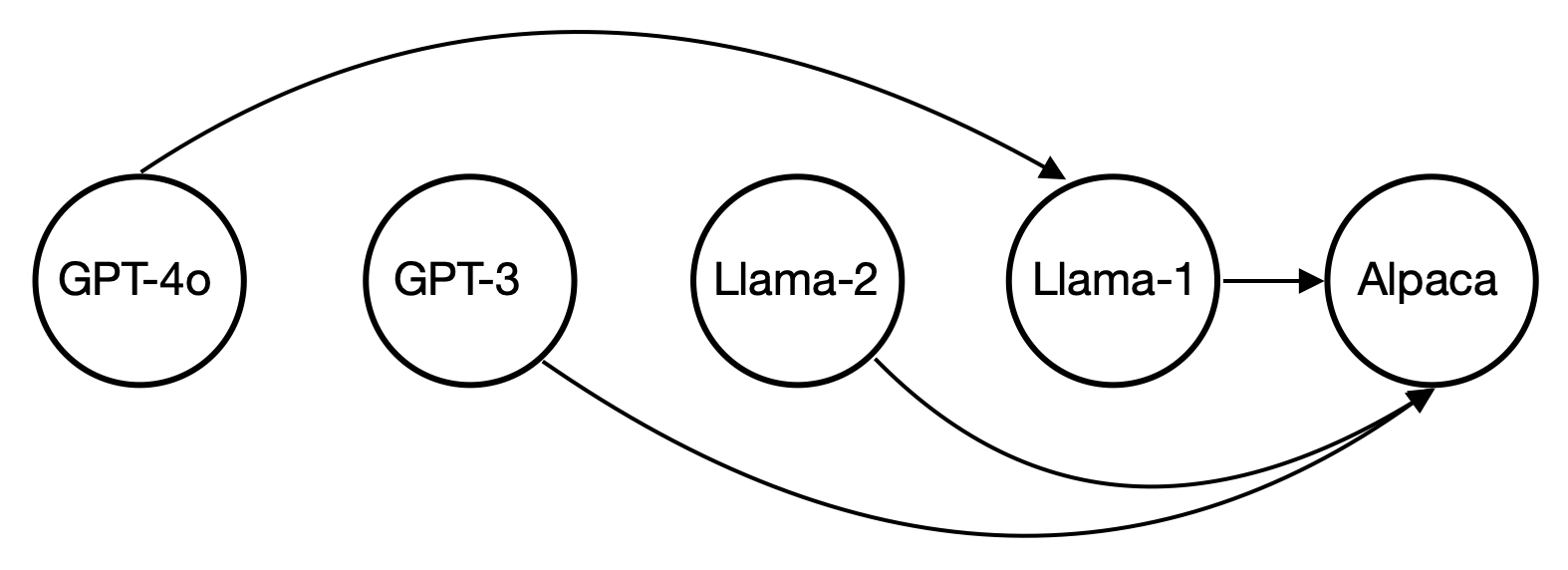} \includegraphics[width=\textwidth]{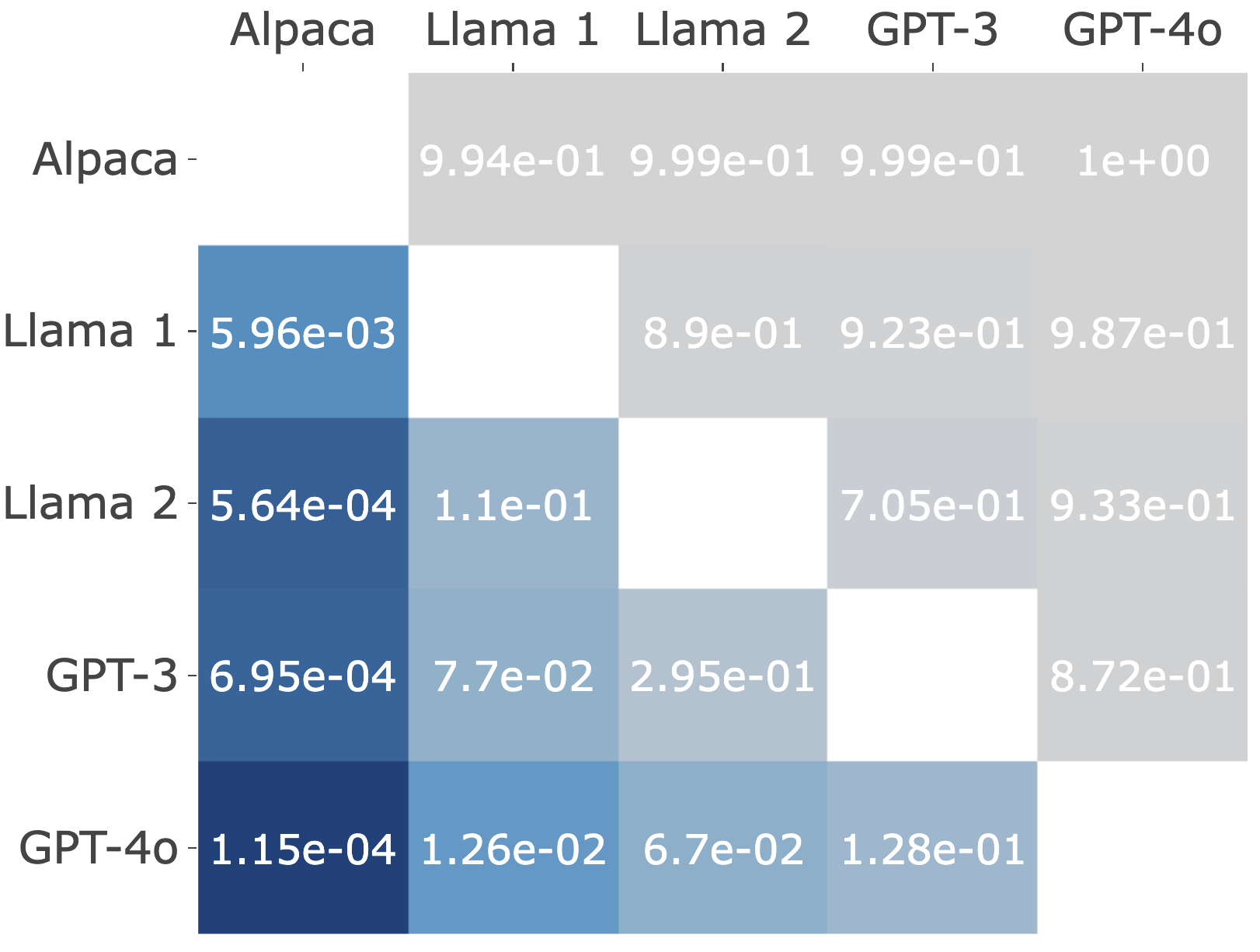}\caption{Clinical knowledge}
     \end{subfigure}
     \begin{subfigure}[b]{0.32\textwidth}
         \centering
 \includegraphics[width=\textwidth]{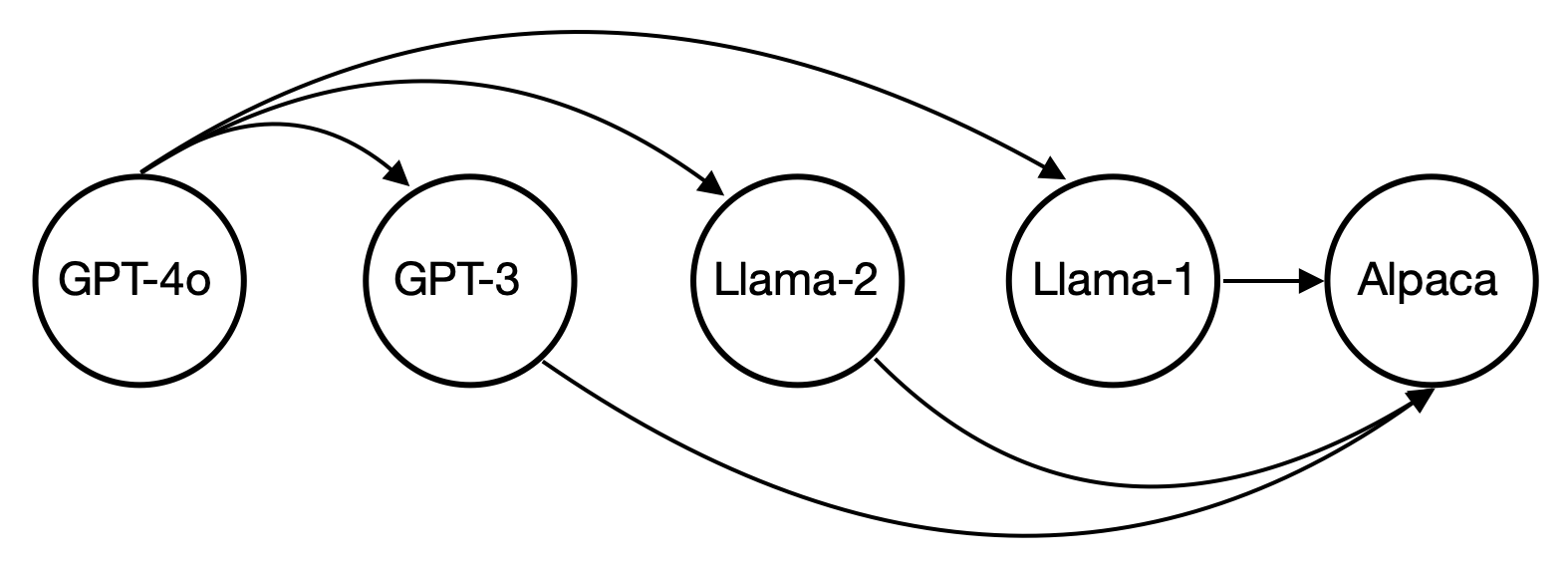} \includegraphics[width=\textwidth]{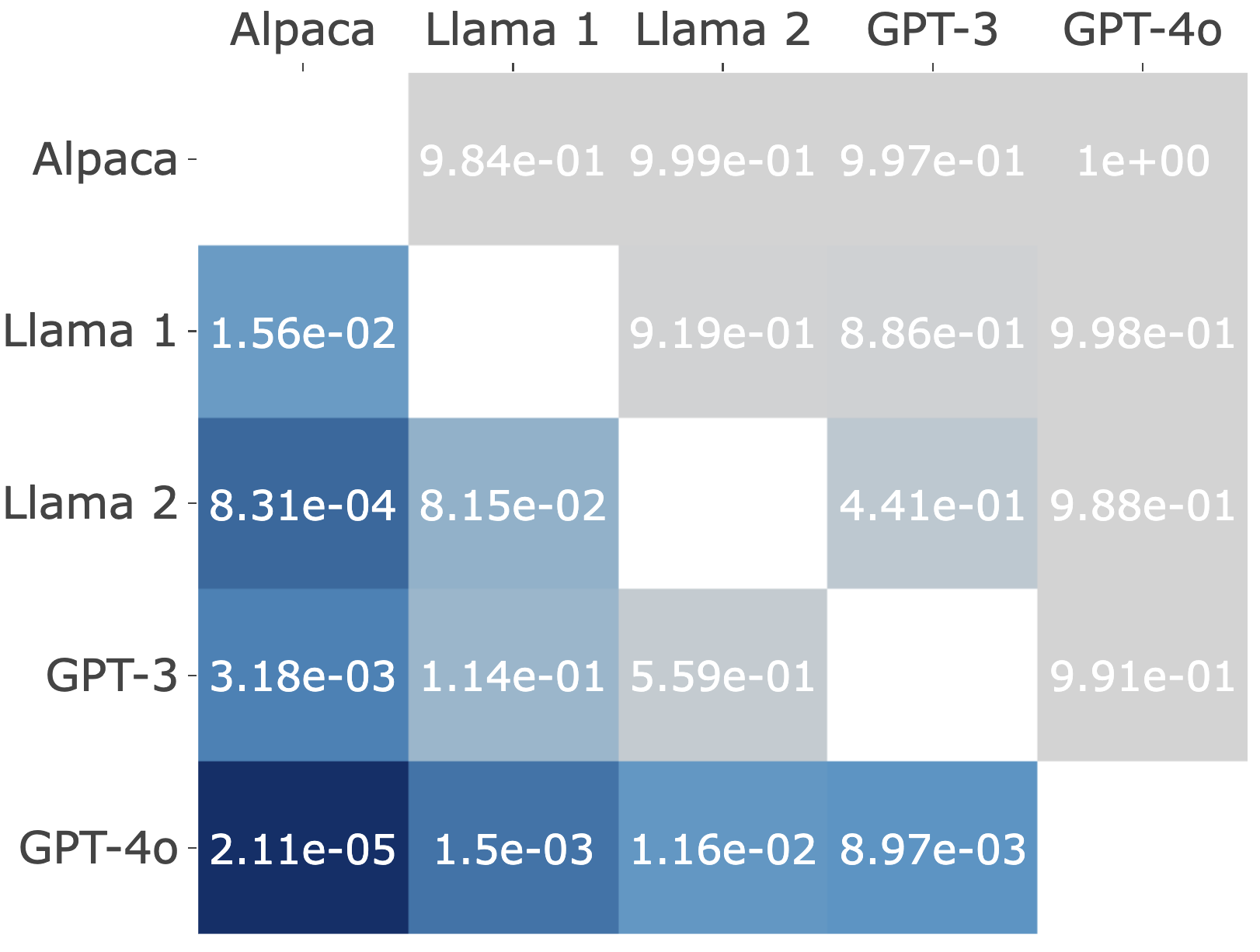}\caption{College biology}
     \end{subfigure}
\hfill
     \begin{subfigure}[b]{0.32\textwidth}
         \centering
 \includegraphics[width=\textwidth]{plot/type4} \includegraphics[width=\textwidth]{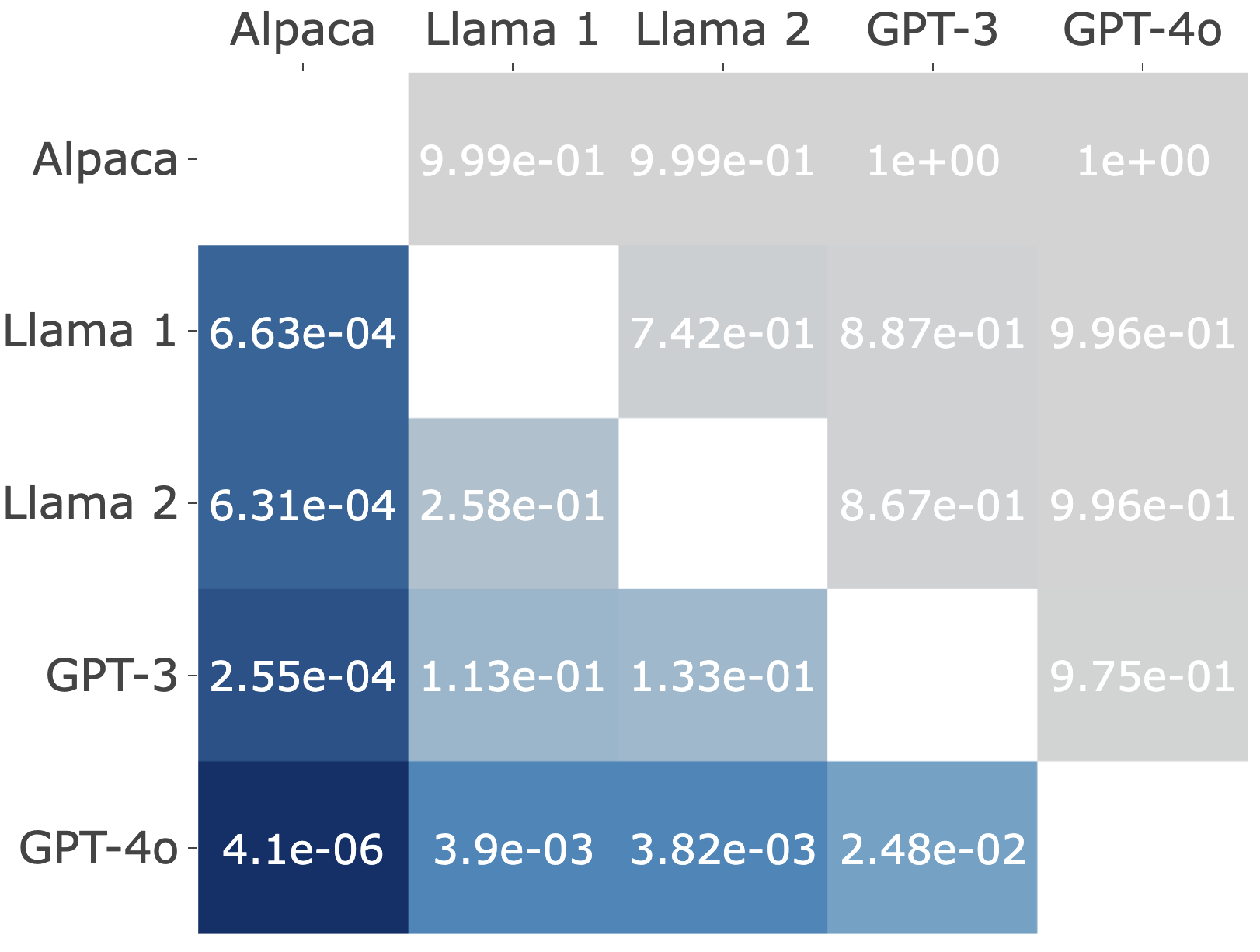}\caption{College medicine}
     \end{subfigure}
\\
     \begin{subfigure}[b]{0.32\textwidth}
         \centering
 \includegraphics[width=\textwidth]{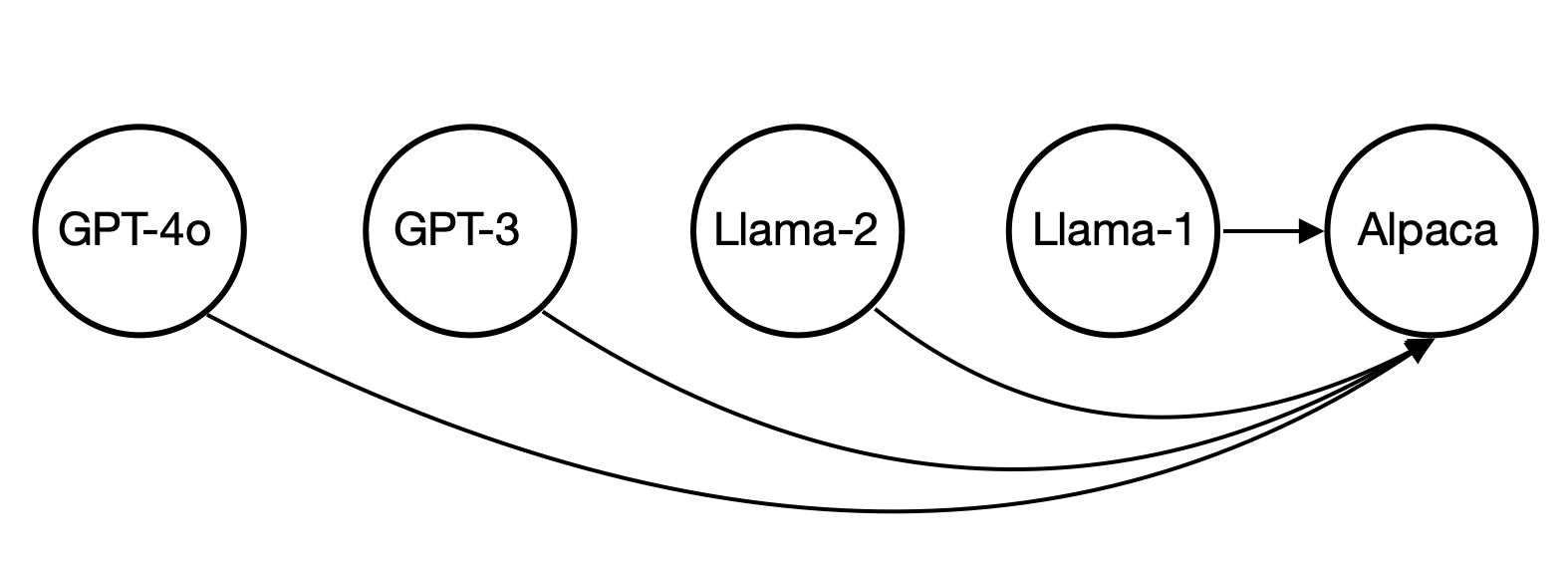} \includegraphics[width=\textwidth]{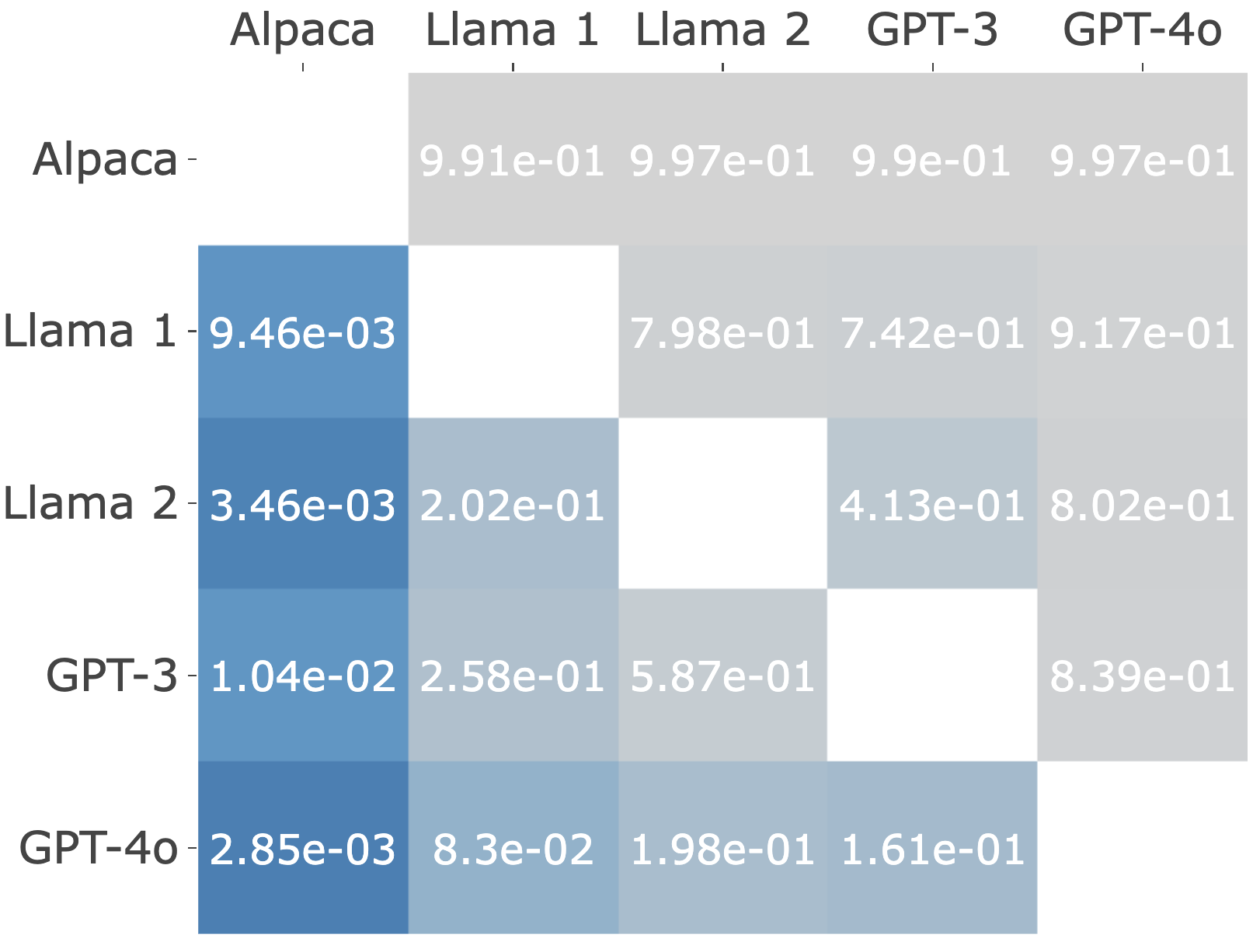}\caption{Medical genetics}
     \end{subfigure}\quad\quad\quad\quad\quad
     \begin{subfigure}[b]{0.32\textwidth}
         \centering
 \includegraphics[width=\textwidth]{plot/type5}
\includegraphics[width=\textwidth]{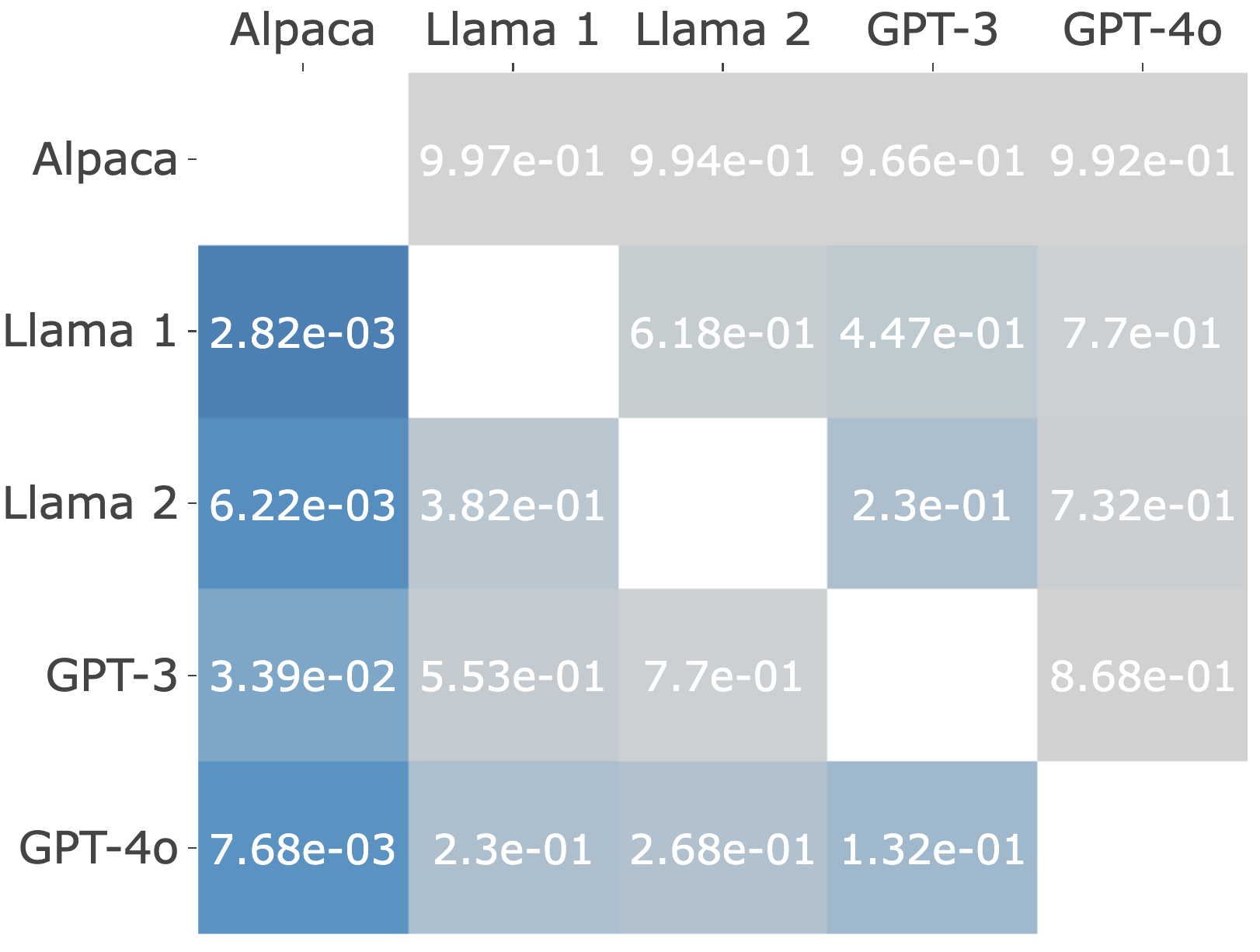}\caption{Anatomy}
     \end{subfigure}
\caption{P-values diagrams and the subsequent Hasse diagrams  for the significant comparisons of LLMs under different topics with partially observed dataset $ \mathcal{D}_{n,p}$ and $\alpha = 0.05$. For each subfigure, the upper panel is the Hasse diagram. The lower panel reports the p-value for the pairwise comparisons of two LLMs under the domains of different question topics. In Hasse diagram, if there is a direct path from LLM $i$ to LLM $j$, then LLM $i$ has a significant better performance than LLM $j$ under the corresponding topic.}\label{real:S1:f2}
\end{figure}

\begin{figure}
\centering
   \includegraphics[width=0.5\textwidth]{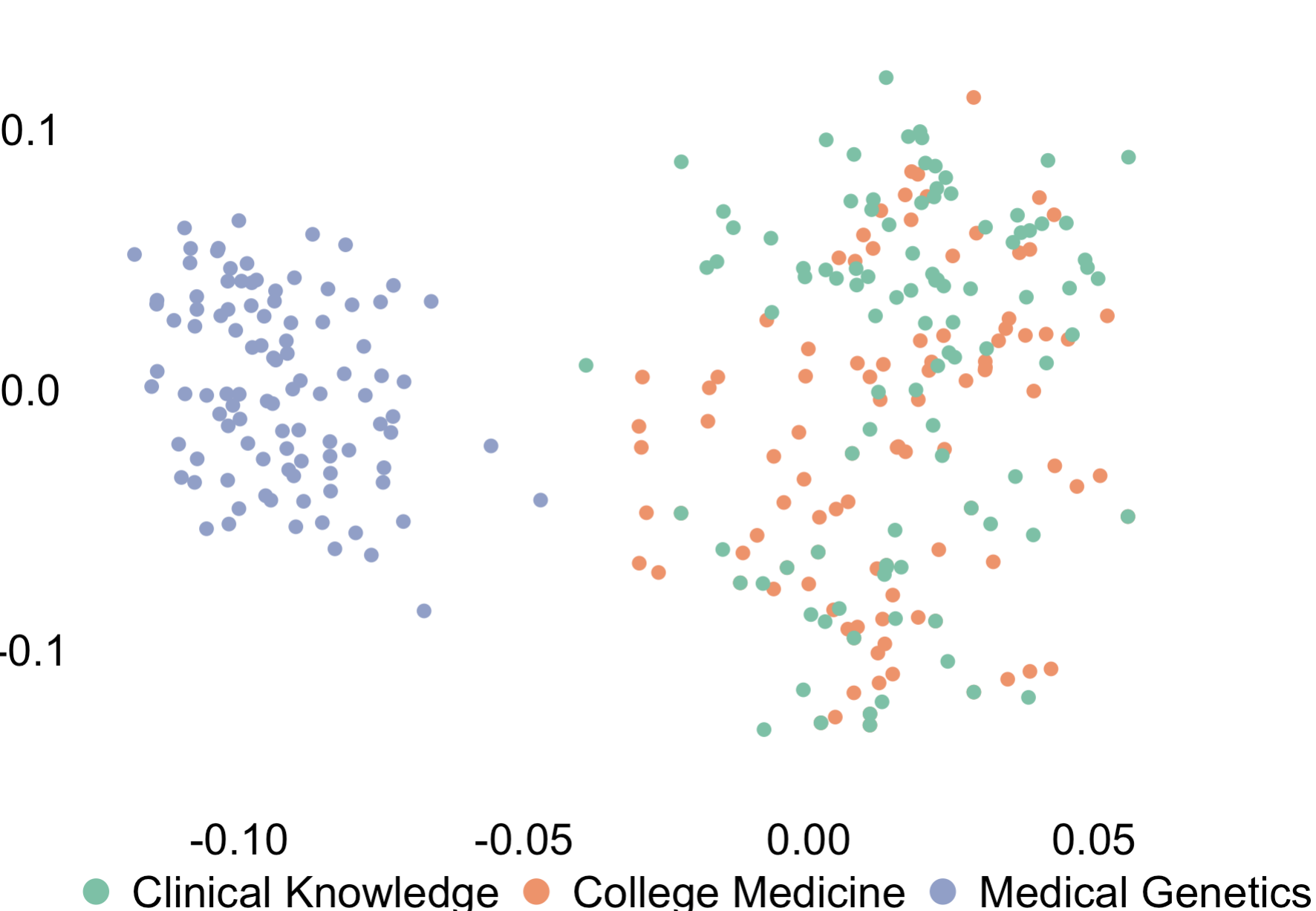}  
\caption{Visualization of the question distributions under different topics through 2-D PCA embeddings.}\label{real:embed}
\end{figure}

\begin{figure}
      \begin{subfigure}[b]{0.32\textwidth}
         \centering
\includegraphics[width=\textwidth]{plot/type3} \includegraphics[width=\textwidth]{plot/real_clinical_knowledge_22}  \caption{}
     \end{subfigure}    \begin{subfigure}[b]{0.32\textwidth}
         \centering
\includegraphics[width=\textwidth]{plot/type4} 
\includegraphics[width=\textwidth]{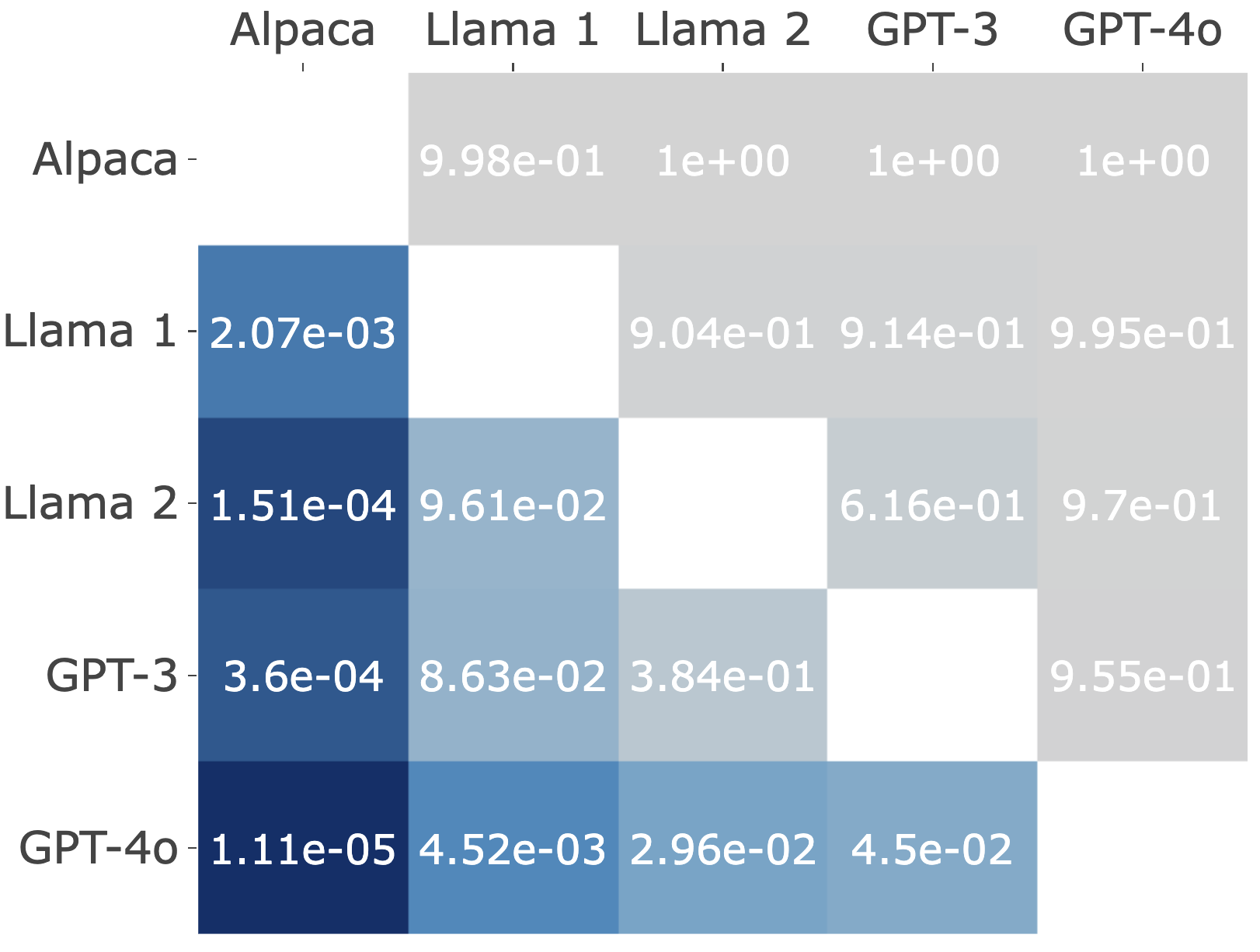} \caption{}
     \end{subfigure}     \begin{subfigure}[b]{0.32\textwidth}
         \centering
\includegraphics[width=\textwidth]{plot/type4} \includegraphics[width=\textwidth]{plot/real_college_medicine_22}  \caption{}
     \end{subfigure}
\caption{(a) Inference results based on \( \mathcal{Q}_{ij}(\Omega_{\mathrm{ck}}) \). (b) Inference results based on \( \mathcal{Q}_{ij}(\Omega_{\mathrm{ck}} \mid \kappa_{\mathrm{cm}:\mathrm{ck}}) \). (c) Inference results based on \( \mathcal{Q}_{ij}(\Omega_{\mathrm{cm}}) \).}\label{real:S1}
\end{figure} 
\begin{figure}
     \centering
      \begin{subfigure}[b]{0.32\textwidth}
         \centering
 \includegraphics[width=\textwidth]{plot/type3} \includegraphics[width=\textwidth]{plot/real_clinical_knowledge_22} \caption{}
     \end{subfigure}    \begin{subfigure}[b]{0.32\textwidth}
         \centering
\includegraphics[width=\textwidth]{plot/type4} 
\includegraphics[width=\textwidth]{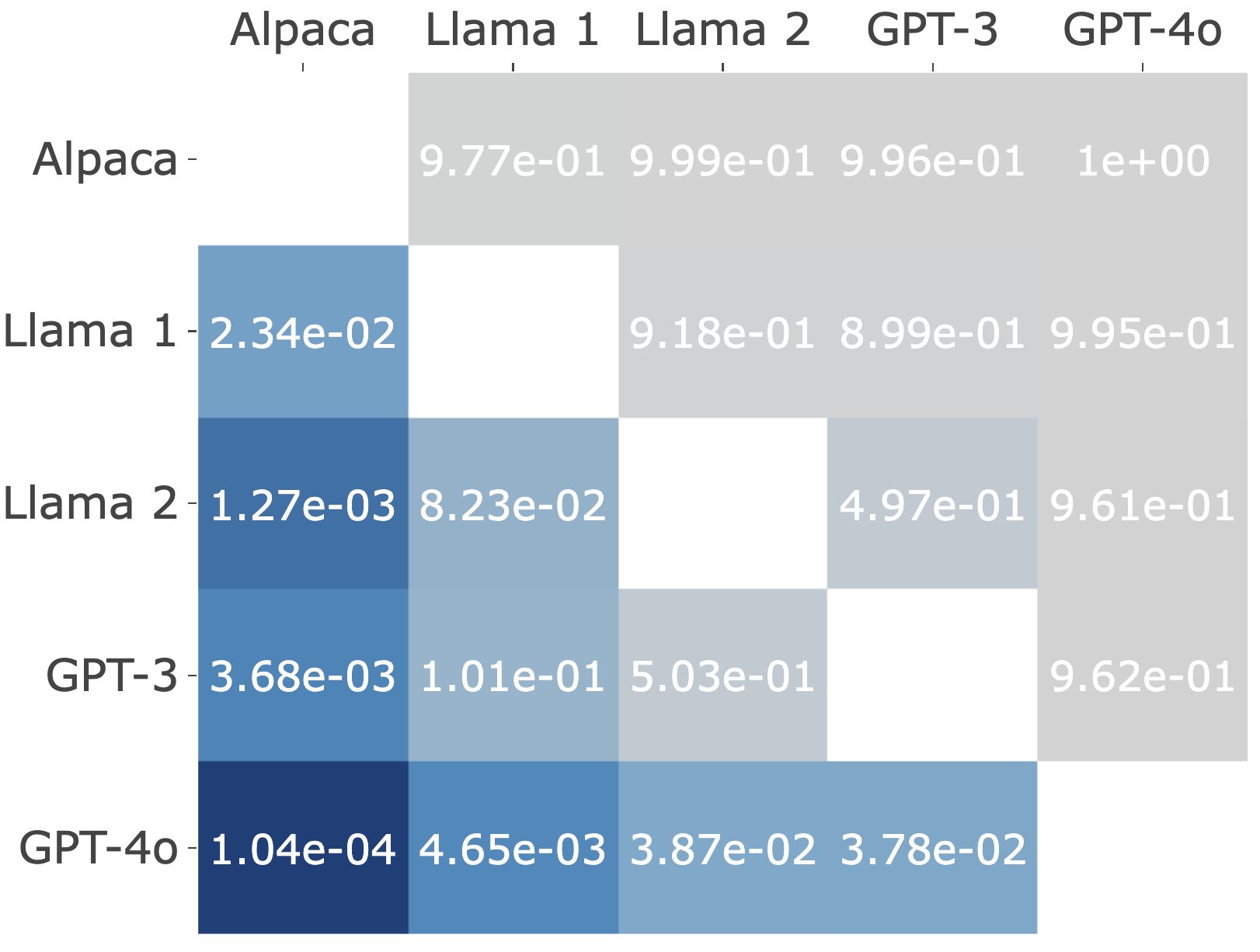} \caption{}
     \end{subfigure}     \begin{subfigure}[b]{0.32\textwidth}
         \centering
 \includegraphics[width=\textwidth]{plot/type5} \includegraphics[width=\textwidth]{plot/real_medical_genetics_22} \caption{} 
     \end{subfigure}
\caption{(a) Inference results based on \( \mathcal{Q}_{ij}(\Omega_{\mathrm{ck}}) \). (b) Inference results based on \( \mathcal{Q}_{ij}(\Omega_{\mathrm{ck}} \mid \kappa_{\mathrm{mg}:\mathrm{ck}}) \). (c) Inference results based on \( \mathcal{Q}_{ij}(\Omega_{\mathrm{mg}}) \).}\label{real:S2}
\end{figure}

\section{Discussion}
Motivated by human preference-based evaluation of large language models, we develop a statistical inference framework for the contextual BTL model. We employ a nonparametric maximum likelihood approach to estimate the preference score functions, using ReLU-DNNs for flexible function approximation. Building on this, we propose a random-walk debiased estimator for context-dependent comparisons, which leverages the full dataset and asymptotically achieves semiparametric efficiency under mild conditions. Our method accommodates high-dimensional contextual features and allows the use of general machine learning estimators for nuisance functions, with theoretical guarantees that ensure valid inference. Beyond individual comparisons, we extend the framework to simultaneous inference across multiple tasks and to settings with distributional shifts, enabling broad applicability in real-world scenarios. Notably, although our work is motivated by LLM evaluation, the proposed approach naturally generalizes to other applications involving context-dependent pairwise comparison data.

Several directions for future research remain open. First, it would be valuable to extend our contextual inference procedure to settings involving multiway comparisons, such as the contextual Plackett–Luce model. Second, studying dynamic ranking models in which the preference score functions evolve over time presents an interesting avenue for further investigation. Third, it would be worthwhile to explore the applicability of our random-walk-based debiased inference approach to broader data structures and models that involve a growing number of nuisance functions.

\bibliographystyle{chicago}
\bibliography{references.bib}
\newpage
 \appendix
\renewcommand\thesection{S\arabic{section}}
\renewcommand{\thefigure}{S\arabic{figure}}
\renewcommand{\thetable}{S\arabic{table}}
\renewcommand{\theequation}{S\arabic{equation}}
\renewcommand{\thelemma}{S\arabic{lemma}}
\renewcommand{\thetheorem}{S\arabic{theorem}}
 
\renewcommand{\theremark}{S\arabic{remark}}
\renewcommand{\theproposition}{S\arabic{proposition}}
\renewcommand{\theassumption}{S\arabic{assumption}}
 
\pagenumbering{arabic} %
\renewcommand*{\thepage}{S\arabic{page}}

\begin{center}
  \textit{\large Supplementary material to}
  \end{center}
  \begin{center}
  \title{\large Fisher Random Walk: Automatic Debiasing Contextual Preference Inference for Large Language Model Evaluation}
  \vskip10pt
  \end{center}

  % \setcounter{section}{0}
  % \renewcommand{\thesection}{\Alph{section}}
  
  % \maketitle
  
    This document contains the supplementary material to the paper
    ``Fisher Random Walk: Automatic Debiasing Contextual Preference Inference for Large Language Model Evaluation". Section \ref{sec:pf:thm2} proves Theorem \ref{po:solution}. Section \ref{sec:pfpo:pierror} proves Proposition \ref{po:pierror}. Section \ref{sec:pf:theorem:LD:new} proves Theorem \ref{theorem:LD:new} and presents the asymptotic normality and inference results. Section \ref{sec:pf:efficiency} proves Theorems \ref{thm:efficiency} and \ref{thm:oracle} and presents the semiparametric efficiency lower bound and variance equivalence proofs. Section \ref{proofthm:L2error:main} proves Theorem \ref{thm:L2error:main} and presents the estimation error bound. Section \ref{sec:pf:multiple} proves Theorems \ref{thm:bootstrap} and \ref{thm:disshift} for multiple hypotheses and domain shift. Section \ref{sec:algorithm} presents algorithms for estimating the semiparametric efficient variance. Section \ref{sec:connect} presents the electrical-network connections and a physics proof of Theorem \ref{po:solution}. Section \ref{sec:proof:graph} collects technical results for the comparison graph. The final section provides additional technical lemmas.

\section{Proof of Theorem \ref{po:solution}}\label{sec:pf:thm2}

Here we provide the formal proof for Theorem \ref{po:solution}. We also provide a heuristic physics proof of Theorem \ref{po:solution} in the Supplementary Material~\ref{sec:physics-proof}.
For notation simplicity, in this proof, we omit the fixed context $\M x$, indices $\iz,\jz$, and parameter $\bds \theta$, and write
\begin{align*}
  {\M \Xi}  = {\M \Xi}(\M x),  {\pmb{\mathscr{L}}} = {\pmb{\mathscr{L}}}(\M x \mid {\bds \theta}), \text{ and } {\bds \pi}  ={\bds \pi}(\M x\mid {\bds \theta}, \iz,\jz).
  \end{align*}

  Let $\mathrm{R}(i,j)$ be the path that starts at \(\iz\) and stops upon its first visit to  \(\jz\) following the transition probability normalized by the weight matrix $\M \Xi$. Recall the definition of $W_{ij}$ in \eqref{def:W}:
\bee\nonumber
{W}_{ij}(\M x\mid {\bds \theta}, \iz,\jz)&=|\M A|\mathbb{I}(\M x\in\Omega)\frac{\EE[\mathrm{R}(i,j)]}{\psi'(\theta_i(\M x) - \theta_j(\M x))} 
\ee
On the other side, denote ${\bds u} =  {\pmb{\mathscr{L}}}^{\dagger}(\M e_{\iz} - \M e_{\jz})$ and thus ${\bds \pi} = |\M A|\mathbb{I}(\M x\in\Omega){\bds u}$ by definition in \eqref{iota0}.
Then potential representation $W_{ij} = \pi_i-\pi_j$ in  \eqref{def:alpha:main} holds and we complete the proof of the  theorem if the following claim is true.

\noindent{\it Claim.} We have
$
  \EE[\mathrm{R}(i,j)] = \Xi_{ij}(u_i - u_j) = \psi'(\theta_i(\M x) - \theta_j(\M x)) (u_i - u_j).
$

The rest of the proof is to prove the claim above.
Denote the degrees $d_i=\sum_j \Xi_{ij}$, transition matrix $\M P=\M D^{-1}\M \Xi$ where $\M D=\mathrm{diag}(d_i)$, and thus we have the graph Laplacian $\pmb{\mathscr{L}}=\M D-\M \Xi$. Denote $(X_t)_{t\ge0}$ as the random walk with transition $\M P$, started at $\iz$ and stopped at its first visit to $\jz$, with hitting time $\tau_{\jz}=\inf\{t\ge0: X_t=\jz\}$. For any oriented edge $(i,j)$, recall that we define the directed edge-counts up to $\tau_{\jz}$ by
\[
|\{(i,j)\in \mathrm{R}\}|:=\sum_{t=0}^{\tau_{\jz}-1}\mathbb{I}\{X_t=i,\ X_{t+1}=j\},\qquad \mathrm{R}(i,j)=|\{(i,j)\in \mathrm{R}\}|-|\{(j,i)\in \mathrm{R}\}|.
\]

Define $v_i$ as the expected counts visiting $i$ before $\tau_{\jz}$ when started from $\iz$, such that 
$
v_i =\EE\!\big[\sum_{t=0}^{\tau_{\jz}-1}\mathbb{I}\{X_t=i\}\big],
$ for any $i \in [n]$. 
Writing the count as a sum over time and conditioning on $X_t$, we have
\begin{align*}
\EE[|\{(i,j)\in \mathrm{R}\}|]
&=\EE\Big[\sum_{t=0}^{\tau_{\jz}-1} \mathbb{I}\{X_t=i,\ X_{t+1}=j\}\Big]
=\sum_{t\ge0}\EE\Big[\mathbb{I}\{t<\tau_{\jz},\ X_t=i\}\,\PP(X_{t+1}=j\mid X_t, X_{t-1},\ldots)\Big]\\
&=\sum_{t\ge0}\EE\big[\mathbb{I}\{t<\tau_{\jz},\ X_t=i\}\big]\,P_{ij}
=P_{ij}\sum_{t\ge0}\PP(t<\tau_{\jz},\ X_t=i)
=v_i\,P_{ij},
\end{align*}
where we used the Markov property to replace $\PP(X_{t+1}=j\mid X_t, X_{t-1},\ldots)$ by $\PP(X_{t+1}=j\mid X_t=i)=P_{ij}$.
Hence the expected net crossings of oriented edge $(i,j)$ equal
\begin{align*}
\EE[\mathrm{R}(i,j)]&=\EE[|\{(i,j)\in \mathrm{R}\}|]-\EE[|\{(j,i)\in \mathrm{R}\}|]=\Xi_{ij}\left(\frac{v_i}{d_i}-\frac{v_j}{d_j}\right).
\end{align*}
Denote the vector ${\bds g}$ such that $g_i = v_i/d_i$ for each $i \in [n]$. As $d_k=\sum_j \Xi_{kj}$, we have
$$\textstyle
(\pmb{\mathscr{L}}{\bds g})_k= ((\M D-\M \Xi){\bds g})_k=d_k g_k-\sum_j \Xi_{kj} g_j = \sum_j \Xi_{kj}(g_k-g_j), \text{ for any $k \in [n]$}.
$$
On each sample path $\{X_t:0\le t<\tau_{\jz}\}$, the number of departures from $k$ minus the number of arrivals to $k$ equals $1$ if $k=\iz$, equals $-1$ if $k=\jz$, and equals $0$ otherwise. This argument gives us
\[
\sum_{j}\EE[\mathrm{R}(k,j)]=\sum_{j}\Xi_{kj}(g_k-g_j)=(\pmb{\mathscr{L}}{\bds g})_k=\begin{cases}1,&k=\iz,\\-1,&k=\jz,\\0,&\text{otherwise.}\end{cases}
\]
Thus $\pmb{\mathscr{L}}{\bds g}=\M e_{\iz}-\M e_{\jz}$.
By definition of the Moore--Penrose pseudoinverse, ${\bds u}$ is the centered solution to $\pmb{\mathscr{L}}{\bds x}=\M e_{\iz}-\M e_{\jz}$. Since ${\bds g}$ is also a solution, we have ${\bds g}={\bds u}+c\,\mathbf{1}$ for some constant $c$, and therefore $g_i-g_j=u_i-u_j$ for all $i,j$. Therefore, for all oriented edges $(i,j)$, we have
\[
\EE[\mathrm{R}(i,j)]=\Xi_{ij}\,(g_i-g_j) = \Xi_{ij}\,(u_i-u_j),
\]
which proves the claim.

\section{\color{black}Proof of Proposition~\ref{po:pierror}}\label{sec:pfpo:pierror}

We will prove the proposition given $\M A$ is fixed as satisfies $\mathcal{E}_{\mathrm{good}}$
By Lemma~\ref{lm:good-ER}, we have for Erd\H{o}s--R\'enyi graph, 
$
\pr(\mathcal{E}_{\mathrm{good}}) \geq 1 - 2n^{-\gamma} 
$ and thus the proposition applies.

Under $\mathcal{E}_{\mathrm{good}}$, it suffices to prove the proposition if we have  for all  $\M x\in\mathbb{X}$ and $\Omega\subseteq\mathbb{X}$,
\bee\label{target:p2:all}
\sup_{(\iz,\jz)\in[n]^2}\|\hat{\bds\pi}(\M x\mid \Omega,\iz,\jz) - {\bds\pi}^*(\M x\mid \Omega,\iz,\jz) \|  \leq \mathbb{I}(\M x\in\Omega)  C\cdot  n\| \hat{\bds\theta}(\M x) - {\bds\theta}^*(\M x)   \|_{\infty}.
\ee

In what follows, we prove \eqref{target:p2:all}. 
 Throughout our discussion, we denote by $C > 0$ a generic constant that only depends on $C$ and $F$, but may vary from place to place. We denote
\bee\nonumber
\pmb{\mathscr{L}}(\M x\mid \M A) = \pmb{\mathscr{L}}(\M x\mid \M A,\bds\theta^*),\quad \hat{\pmb{\mathscr{L}}}(\M x\mid \M A) = \pmb{\mathscr{L}}(\M x\mid \M A,\hat{\bds\theta}),
\ee
 to highlight the dependence. Then by Lemma~\ref{lemma:eigen:lower:bound}, if  $\mathcal{E}_{\mathrm{good}}$ holds, we have that, for any $\M x\in \mathbb{X}$,
\bee\label{bound:L}
\sup_{\M x\in\mathbb{R}}\left\|\hat{\pmb{\mathscr{L}}}(\M x\mid \M A)\right\|  \leq \frac{8\exp(2C)}{np}, \quad \sup_{\M x\in\mathbb{R}} \left\|\pmb{\mathscr{L}}^\dagger(\M x\mid \M A)\right\|  \leq  \frac{8\exp(2C)}{np},
\ee
as $\bds\theta^*(\M x)$ and $\hat{\bds\theta}(\M x)$ are uniformly bounded away from infinity, respectively. Thus, under $\mathcal{E}_{\mathrm{good}}$ we have that 
$
\mathrm{Rank}(\pmb{\mathscr{L}}(\M x\mid \M A)) = \mathrm{Rank}(\hat{\pmb{\mathscr{L}}}(\M x\mid \M A)) = n - 1.
$
We then then apply the perturbation inequality for pseduo-inverse in Lemma~\ref{lm:perb:wedin} such that
\bee\nonumber
&\|\hat{\bds\pi}(\M x\mid \Omega,\iz,\jz) - {\bds\pi}^*(\M x\mid \Omega,\iz,\jz) \|
= \Big\|\mathbb{I}(\M x\in\Omega)|\M A| \Big(\hat{\pmb{\mathscr{L}}}^{\dagger}(\M x\mid \M A) - \pmb{\mathscr{L}}^{\dagger}(\M x\mid \M A) \Big) (\M e_{\iz} - \M e_{\jz})  \Big\|
\\
&\leq   \mathbb{I}(\M x\in\Omega)\sqrt{2}|\M A|\Big\|  \hat{\pmb{\mathscr{L}}}^{\dagger}(\M x\mid \M A) - \pmb{\mathscr{L}}^{\dagger}(\M x\mid \M A)  \Big\| 
\leq   \mathbb{I}(\M x\in\Omega)Cn^2p\cdot\Big\|  \hat{\pmb{\mathscr{L}}}^{\dagger}(\M x\mid \M A)    \Big\| \Big\|   \pmb{\mathscr{L}}^{\dagger}(\M x\mid \M A)  \Big\| \Big\|  \hat{\pmb{\mathscr{L}}} (\M x\mid \M A) - \pmb{\mathscr{L}} (\M x\mid \M A)  \Big\|,
\ee
where we use $|\M A|\leq Cn^2p$ by \eqref{good:bound} in the last inequality.
Combining with \eqref{bound:L}, we have
\bee\nonumber
\|\hat{\bds\pi}(\M x\mid \Omega,\iz,\jz) - {\bds\pi}^*(\M x\mid \Omega,\iz,\jz)  \| \leq C  \mathbb{I}(\M x\in\Omega)p^{-1}\left\|  \hat{\pmb{\mathscr{L}}} (\M x\mid \M A) - \pmb{\mathscr{L}} (\M x\mid \M A)  \right\|.
\ee
Next, by definition of graph Laplacian in \eqref{def:H}, %
we have
\bee\nonumber
\left\|  \hat{\pmb{\mathscr{L}}} (\M x\mid \M A) - \pmb{\mathscr{L}} (\M x\mid \M A)  \right\| 
\leq \underbrace{\left\|\mathrm{diag}\Big(\hat{\bds \Xi}(\M x\mid  \M A )\M 1\Big)  - \mathrm{diag}\Big({\bds \Xi}(\M x\mid  \M A )\M 1\Big)  \right\|}_{I_1(\M x)} + \underbrace{\left\|\hat{\bds \Xi}\Big(\M x\mid  \M A \Big) - {\bds \Xi}\Big(\M x\mid  \M A \Big)\right\|}_{I_2(\M x)}.
\ee
For $I_1(\M x)$, we have for any $\M x\in\mathbb{X}$,
\begin{align}\nonumber
I_1^2(\M x)  &= \max_{i \in[n]}\left|\sum_{j=1}^n A_{ij}\left[\psi'\left(\hat{\theta}_i(\M x) - \hat{\theta}_j(\M x)\right) - \psi'\Big({\theta}^*_i(\M x) - {\theta}^*_j(\M x)\Big)\right]\right|^2
\\\nonumber
&\leq \max_{i \in [n]}\left| \sum_{j=1}^n A_{ij} \psi''\left(\xi_{ij}(\M x)\right)\left(\hat{\theta}_i(\M x) - {\theta}^*_i(\M x)- \hat{\theta}_j(\M x) + {\theta}^*_j(\M x)\right)  \right|^2 \leq Cn^2p^2\cdot \|  \hat{\bds\theta}(\M x) - {\bds\theta}^*(\M x)   \|_{\infty}^2  ,
\end{align}
where the first inequality is by mean-value theorem for $\xi_{ij}(\M x)$ is between $\hat{\theta}_i(\M x) - \hat{\theta}_j(\M x)$ and ${\theta}^*_i(\M x) - {\theta}^*_j(\M x)$, thus $|\psi''(\xi_{ij}(\M x))|\le C$ as ${\bds\theta}^*, \hat{\bds\theta} \in \Theta$, and the last inequality is by Lemma~\ref{lm:Cauchy-Schwarz}.
\par
For $I_2(\M x)$, we have that, for any $\M x\in\mathbb{X}$ and $\M u = (u_1,\ldots,u_n)\in \R^n$,
\begin{align}\nonumber
I_2(\M x) &\leq \sup_{\|\M u \| = 1}\left| \M u^\T \left(\hat{\bds \Xi}(\M x\mid  \M A ) - {\bds \Xi}(\M x\mid  \M A )\right) \M u\right|
\\\nonumber
 & \leq \sup_{\|\M u \| = 1} \sum_{(i,j)\in[n]^2}A_{ij}u_iu_j \left|\psi'\left(\hat{\theta}_i(\M x) - \hat{\theta}_j(\M x)\right) - \psi'\Big({\theta}^*_i(\M x) - {\theta}^*_j(\M x)\Big)\right|
\\\nonumber
 & \leq C \|  \hat{\bds\theta}(\M x) - {\bds\theta}^*(\M x)    \|_{\infty}\cdot \sup_{\|\M u \| = 1} \sum_{(i,j)\in[n]^2}A_{ij}u_iu_j  
  = C \| \hat{\bds\theta}(\M x) - {\bds\theta}^*(\M x) \|_{\infty} \|\M A\| \le Cnp \| \hat{\bds\theta}(\M x) - {\bds\theta}^*(\M x) \|_{\infty},
\end{align} 
where in last inequality we let $\M D = \mathrm{diag}(\M A\M1)$ and have $\|\M A\| \le \|\M D\|+\|L(\M A)\| \le Cnp$ by $\mathcal{E}_{\mathrm{good}}$.
Summarizing the above results, we conclude \eqref{target:p2:all}.
\qed

\section{Proof of Theorem~\ref{theorem:LD:new}}\label{sec:pf:theorem:LD:new}

In order to clarify how on the proof depends on the Erd\H{o}s--R\'enyi graph assumption, we will first fix $\M A$ as a deterministic graph by assuming it satisfies some good properties and then show that Erd\H{o}s--R\'enyi graph satiesfies these properties in high probability. 
 In particular, recalling that we define the event $\mathcal{E}_{\mathrm{good}}$ in \eqref{upperbound:A} that
\begin{align}\label{upperbound:A:supp}
 \nonumber \mathcal{E}_{\mathrm{good}} = \Big\{\M A \in \{0,1\}^{n\times n} ~\Big|~ & 0.5np \leq \min_{i\in[n]}\sum_{j\in[n]\backslash\{i\}} A_{ij} \leq \max_{i\in[n]}\sum_{j\in[n]\backslash\{i\}}A_{ij}\leq 2np;\\&\qquad
\qquad
 \frac{np}{2} \le \lambda_{\min,\perp}(\M L(\M A)) \le  \lambda_{\max}(\M L(\M A)) \le 2{np}\Big\},
\end{align}
where $L(\M A)$ is the graph Laplacian of $\M A$ and we replace $q$ by $p$ to align with the setting of Erd\H{o}s--R\'enyi graph. By Lemma~\ref{lm:good-ER}, if $np \ge 40 \log n$, we have $\PP(\mathcal{E}_{\mathrm{good}}) \ge 1- 2/n$. Therefore, in the rest we will prove the theorem for deterministic $\M A$ under $\mathcal{E}_{\mathrm{good}}$ and then the theorem is true for the Erd\H{o}s--R\'enyi graph with high probability when $np \ge 40 \log n$.

To simplify the presentation of the proof, we assume the nuisance estimator $\hat{\bds\theta}$ is computed using an independent copy of $\mathcal{D}_n$ so we do not need to split the data and cross-fit as in Algorithm~\ref{alg:ci:addition}. We assume the estimator $\hat{\fav}_{\iz\jz}$ uses all samples in $\mathcal{D}_n$ with $S=1$. Our results can be easily generalized to the cross-fitting estimator by simply taking average. 

We first introduce some notations. 
Define an empirical average operator
\bee\label{eq:Pn}
\mathbb{P}_n\big(g(\itt,\jtt,\M X,y)\big) = \frac{1}{|\M A|L}\sum_{(i,j)\in\mathcal{E}(\M A)} \sum_{\ell=1}^L g(i,j,\M X_{ij \ell},Y_{ij \ell}),  
\ee
for any $g:\mathcal{E}(\M A)\times\mathbb{X}\times (0,1)\mapsto \RR$. We denote
\begin{align}\label{eq:alpha-gamma}
  \gamma^*_{ij}(\cdot) = \theta_i^{*}(\cdot) -  \theta_j^{*}(\cdot), \hat{\gamma}_{ij}(\cdot) = \hat\theta_i(\cdot) -  \hat\theta_j(\cdot), 
  {\alpha}^{*}_{ij}(\cdot) = \pi_i^{*}(\cdot) - \pi_j^{*}(\cdot), \hat{\alpha}_{ij}(\cdot) = \hat\pi_i (\cdot) - \hat\pi_j (\cdot).
\end{align}
Denote $\bar{\fav}_{\iz\jz}$ as the estimate using truth, then we can write $\hat{\fav}_{\iz\jz}$ and $\bar{\fav}_{\iz\jz}$ as
\begin{align}\label{eq:Qhat}
  \hat{\fav}_{\iz\jz} (\Omega) &= \mathbb{P}_n\left( \mathbb{I}(\M X  \in \Omega)  \hat\gamma_{\iz\jz}(\M X)    +   \hat{\alpha}_{\itt\jtt}(\M X)  \left(Y - \psi\left(\hat\gamma_{\itt\jtt}(\M X )\right)\right) \right),\\
  \bar{\fav}_{\iz\jz} (\Omega) &= \mathbb{P}_n\left( \mathbb{I}(\M X  \in \Omega)  \gamma^{*}_{\iz\jz}(\M X)    +  {\alpha}^{*}_{\itt\jtt}(\M X)  \left(Y - \psi\left(\gamma^{*}_{\itt\jtt}(\M X )\right)\right) \right).\label{eq:Qbar}
\end{align}

We next present several theorectical results needed for the proof. First, we can show the orcale estimator $\bar{\fav}_{\iz\jz} (\Omega)$ is asymptotically normal.

\begin{theorem}\label{theorem:CLT:barQ}
  Assume ${\bds \theta}^* \in \Theta$ and we fix $\M A$ which  is connected and statisfies $\mathcal{E}_{\mathrm{good}}$ in \eqref{upperbound:A:supp} for some $p\ge C \log n/n$. We have
 \bee\label{terms:berryessen:barQ}
 &\sup_{z\in\R}\left| \pr\left[\frac{\bar{\fav}_{\iz\jz} (\Omega)  - \fav_{\iz\jz}(\Omega)}{\sqrt{V_{\iz\jz}(\Omega)}}  \leq z\right ] - \Phi(z)\right| \leq \frac{C}{\sqrt{npL}},
 \ee
where the variance $V_{\iz\jz}(\Omega)$ defined in  \eqref{asymptotic:variance} has $V_{\iz\jz}(\Omega) \asymp 1/(npL)$.
\end{theorem}
The proof of Theorems~\ref{theorem:CLT:barQ} will be presented in Sections~\ref{theorem:CLT:barQ:pf}.

Next, we will show the difference between the oracle estimator $\bar{\fav}_{\iz\jz} (\Omega)$ and our Fisher random walk debiased estimator is of higher order than $1/\sqrt{npL}$ where $npL$ is the effective sample size for the preference inference. We first need to show  our estimating equation satisfies the Neyman orthogonality condition \citep{chernozhukov2018double}. Neyman orthogonality condition imposes that the derivative of the estimating equation with respect to the nuisance is zero, which is verified for our estimator by the following lemma.

\begin{lemma}[Neyman orthogonality]\label{lm:neyman}  Consider a G\^ateaux path that perturbs the nuisance functions nodewise:
  for any vector-valued function $\bds h(\cdot)=(h_1(\cdot),\ldots,h_n(\cdot))$, set
  $\theta_i^{(t)} =\theta_i^* +t\,h_i $ so that
  $\gamma_{ij}^{(t)} =\gamma_{ij}^* +t\,\delta_{ij} $ with
  $\delta_{ij} =h_i -h_j $.
  Differentiating the map along this path and evaluating at $t=0$, we have
\begin{align*}
  &\frac{d}{dt}\,\E\Bigg[ \mathbb{I}(\M X\in\Omega)\gamma^{(t)}_{\iz\jz}(\M X)
  + \frac{1}{|\M A|}\sum_{(i,j)\in\mathcal{E}(\M A)}\alpha^*_{ij}(\M X)\Big\{Y-\psi\big(\gamma^{(t)}_{ij}(\M X)\big)\Big\}\Bigg]\Bigg|_{t=0}\\
 &\quad = \E\Bigg[\mathbb{I}(\M X\in\Omega)\,\delta_{\iz\jz}(\M X)
  - \frac{1}{|\M A|}\sum_{(i,j)\in\mathcal{E}(\M A)}\alpha^*_{ij}(\M X)\,\psi'\big(\gamma^*_{ij}(\M X)\big)\,\delta_{ij}(\M X)\Bigg] = 0.
\end{align*}
\end{lemma}

Then we can have the following theorem showing the rate of bias   $\hat{\fav}_{\iz\jz} (\Omega)-\bar{\fav}_{\iz\jz} (\Omega)$.

\begin{theorem}\label{theorem:barQ-hatQ}
  Assume ${\bds \theta}^* \in \Theta$ and we fix $\M A$ which  is connected and statisfies $\mathcal{E}_{\mathrm{good}}$ in \eqref{upperbound:A:supp} for some $p\ge 40 \log n/n$. For any $\epsilon \in (0,1)$, we have with probability $1-\epsilon$,
 \begin{align}\nonumber
  |\hat{\fav}_{\iz\jz} (\Omega) - \bar{\fav}_{\iz\jz} (\Omega)| \leq &C\sqrt{\log(2/\epsilon)}\left(n \cdot  L^{-1/2}{\mathcal{E}^{1/2}_{2}( \hat{\bds\theta},\bds\theta^*  \mid \M A)} +(n^2pL)^{-1/2}  {\mathcal{E}^{1/2}_2\left( \hat{\bds \pi},\bds \pi^* \mid \M A\right)}\right)+ {C\log(2/\epsilon) \cdot nL^{-1}}\\
  &\quad + C\left(n  \cdot\mathcal{E}_{2}( \hat{\bds\theta},\bds\theta^*  \mid \M A) +{\mathcal{E}^{1/2}_{2}( \hat{\bds\theta},\bds\theta^*  \mid \M A)\mathcal{E}^{1/2}_2\left( \hat{\bds \pi},\bds \pi^* \mid \M A\right) }\right)  \label{terms:barQ-hatQ}
 \end{align}
\end{theorem}
The proofs of Lemma~\ref{lm:neyman} and Theorem~\ref{theorem:barQ-hatQ} will be presented in Section~\ref{theorem:barQ-hatQ:pf}. 
Combining Theorem~\ref{theorem:barQ-hatQ} with Proposition~\ref{po:pierror} and Assumption~\ref{ass:theta}, we can bound the bias by the rate of smaller order comparing to the variance, i.e.,  
$
  |\hat{\fav}_{\iz\jz} (\Omega) - \bar{\fav}_{\iz\jz} (\Omega)| = o_P(1/\sqrt{npL}).
$ 
Therefore, combining with Theorems~\ref{theorem:CLT:barQ}, we can prove the first part of Theorem~\ref{theorem:LD:new} that
\bee\label{eq:Q:clt}
\frac{\hat{\fav}_{\iz\jz} (\Omega)  - \fav_{\iz\jz}(\Omega)}{\sqrt{V_{\iz\jz}(\Omega )}} \rightsquigarrow \normal(0,1).
\ee

Next we need the following theorem to show the variance estimator $\hat{V}_{\iz\jz}(\Omega)$ is consistent.

\begin{theorem}\label{theorem:sigmahat:rate}
  Under Assumption~\ref{ass:theta}, assume ${\bds \theta}^* \in \Theta$ and we fix $\M A$ which  is connected and statisfies $\mathcal{E}_{\mathrm{good}}$ in \eqref{upperbound:A:supp} for some $p\ge C \log n/n$, then we have
  \begin{align}\nonumber
    &\frac{|\hat{V}_{\iz\jz}(\Omega%
    ) - {V}_{\iz\jz}(\Omega%
    )|}{{V}_{\iz\jz}(\Omega%
    )} =o_P(1).
  \end{align}
\end{theorem}
The proof of Theorem~\ref{theorem:sigmahat:rate} will be presented in Section \ref{theorem:sigmahat:rate:pf}.
Combining \eqref{eq:Q:clt} and Theorem~\ref{theorem:sigmahat:rate}, by  Slutsky's theorem, we prove the second part of Theorem~\ref{theorem:LD:new}
\bee\nonumber
\frac{\hat{\fav}_{\iz\jz} (\Omega)  - \fav_{\iz\jz}(\Omega)}{\sqrt{\hat{V}_{\iz\jz}(\Omega )}} \rightsquigarrow \normal(0,1),
\ee
which completes the proof.  

\subsection{Technical results for Theorem~\ref{theorem:LD:new}}

Here we will prove those theorems supporting the proof of Theorem~\ref{theorem:LD:new}. The proofs of Theorems~\ref{theorem:CLT:barQ}, \ref{theorem:barQ-hatQ}, and \ref{theorem:sigmahat:rate} will be presented in Sections~\ref{theorem:CLT:barQ:pf}, \ref{theorem:barQ-hatQ:pf}, and \ref{theorem:sigmahat:rate:pf} respectively.

\subsubsection{Proof of Theorem~\ref{theorem:CLT:barQ}}\label{theorem:CLT:barQ:pf}

The strategy of the proof is to leverage the Berry-Esseen bound in Lemma~\ref{lm:BE}.
Recall that we define in $\gamma^*,\alpha^*$ in  \eqref{eq:alpha-gamma}, $\bar{\fav}_{\iz\jz}(\Omega)$ in \eqref{eq:Qbar}, and $\PP_n$ in \eqref{eq:Pn}, thus we have
\bee\nonumber
 &\bar{\fav}_{\iz\jz}(\Omega) - \fav_{\iz\jz}(\Omega) = \mathbb{P}_n\left( \mathbb{I}(\M X  \in \Omega)  \gamma^{*}_{\iz\jz}(\M X)  - \fav_{\iz\jz}(\Omega)    +  {\alpha}^{*}_{\itt\jtt}(\M X)  \left(Y - \psi\left(\gamma^{*}_{\itt\jtt}(\M X )\right)\right) \right)
 =  \mathbb{P}_n Z_{\itt\jtt},
\ee
where we denote $Z_{ij \ell} = \mathbb{I}(\M X_{ij \ell}  \in \Omega)  \gamma^{*}_{\iz\jz}(\M X_{ij \ell} ) - \fav_{\iz\jz}(\Omega)    +   {\alpha}^{*}_{ij}(\M X_{ij \ell} )  \left(Y_{ij \ell}  - \psi\left(\gamma^{*}_{ij}(\M X_{ij \ell} )\right)\right)$. 

Note that 
given $\M A$, all $\{Z_{ij \ell}\mid (i,j)\in\mathcal{E}(\M A),\ell\in [L]\}$ are independent and of mean zero. To apply Berry-Esseen bound in Lemma~\ref{lm:BE}, what remains is to bound the variance and third moment of $Z_{ij \ell}$.

We first  derive the variance of $\PP_n(Z_{\itt\jtt})$. Denote $v_{\iz\jz}(\Omega)=\E\left(\mathbb{I}(\M X  \in \Omega)  \gamma^{*}_{\iz\jz}(\M X ) - \fav_{\iz\jz}(\Omega)\right)^2$ and we have 
\begin{align}\nonumber
&\mathrm{Var}(\PP_n(Z_{\itt\jtt}))
\\\nonumber
& = \sum_{(i,j)\in\mathcal{E}(\M A)}\sum_{\ell = 1}^L\frac{1}{{|\M A|^2L^2}}
\E\Big(\mathbb{I}(\M X_{ij \ell}  \in \Omega)  \gamma^{*}_{\iz\jz}(\M X_{ij \ell} ) - \fav_{\iz\jz}(\Omega)    +   {\alpha}^{*}_{ij}(\M X_{ij \ell} )  \left(Y_{ij \ell}  - \psi\left(\gamma^{*}_{ij}(\M X_{ij \ell} )\right)\right)\Big) ^2 
\\\nonumber
& \eqi\sum_{(i,j)\in\mathcal{E}(\M A)}\sum_{\ell = 1}^L\frac{1}{{|\M A|^2L^2}}
  \E\Big(\mathbb{I}(\M X_{ij \ell}  \in \Omega)  \gamma^{*}_{\iz\jz}(\M X_{ij \ell} ) - \fav_{\iz\jz}(\Omega)  \Big)^2  +\E\Big(   {\alpha}^{*}_{ij}(\M X_{ij \ell} )  \left(Y_{ij \ell}  - \psi\left(\gamma^{*}_{ij}(\M X_{ij \ell} )\right)\right)\Big) ^2  
\\\nonumber
& = \frac{1}{{|\M A|L}}v_{\iz\jz}(\Omega) + \sum_{(i,j)\in\mathcal{E}(\M A)}\sum_{\ell = 1}^L\frac{1}{{|\M A|^2L^2}}\E\Big(   {\alpha}^{*}_{ij}(\M X_{ij \ell} )  \left(Y_{ij \ell}  - \psi\left(\gamma^{*}_{ij}(\M X_{ij \ell} )\right)\right)\Big) ^2  
\\\nonumber
& \eqii \frac{1}{{|\M A|L}} v_{\iz\jz}(\Omega)
   + \sum_{(i,j)\in\mathcal{E}(\M A)} \frac{1}{L} \E_{\M X}\Big(\mathbb{I}(\M X \in\Omega)\left( (\M e_{i} - \M e_{j})^\T\pmb{\mathscr{L}}^{\dagger}(\M X \mid  \M A) (\M e_{\iz} - \M e_{\jz}) \right )^2\psi\left(\gamma^{*}_{ij}(\M X  )\right) \left(1  - \psi\left(\gamma^{*}_{ij}(\M X  )\right)\right)\Big) 
\\\label{varzijl}
& \eqiii  \frac{1}{{|\M A|L}}v_{\iz\jz}(\Omega)
   + \sum_{(i,j)\in\mathcal{E}(\M A)} \frac{1}{L}\E_{\M X}\Big(\mathbb{I}(\M X \in\Omega)\left( (\M e_{i} - \M e_{j})^\T\pmb{\mathscr{L}}^{\dagger}(\M X \mid  \M A) (\M e_{\iz} - \M e_{\jz}) \right )^2\psi'\left(\gamma^{*}_{ij}(\M X  )\right)\Big),
\end{align}
where (i) holds as $\E[Y_{ij\ell}| \M X_{ij\ell}] = \psi(\gamma^{*}_{ij}(\M X_{ij \ell} ))$, (ii) holds by the definition in \eqref{iota0} that ${\bds \pi}^*(\M x) = \mathbb{I}(\M x\in\Omega)\,|\M A|\,\pmb{\mathscr{L}}^{\dagger}(\M x\mid {\bds \theta}^*)(\M e_{\iz}-\M e_{\jz})$, and (iii) holds by identity $\psi'(t) =\psi(t) (1 - \psi(t) )$.
By   \eqref{lm:eclap2}, we  can  further simplify  \eqref{varzijl} as  
\bee\label{ec:target:identity}
\sum_{(i,j)\in\mathcal{E}(\M A)}  \left( (\M e_{i} - \M e_{j})^\T\pmb{\mathscr{L}}^{\dagger}(\M x \mid  \M A) (\M e_{\iz} - \M e_{\jz}) \right)^2\psi'\left(\gamma^{*}_{ij}(\M x )\right)  = (\M e_{\iz} - \M e_{\jz})^\T\pmb{\mathscr{L}}^{\dagger}(\M x \mid  \M A) (\M e_{\iz} - \M e_{\jz}).
\ee
 Plugging the above into \eqref{varzijl}, we have
\bee\label{varzijl2}
\frac{1}{|\M A|^2L^2}\sum_{(i,j)\in\mathcal{E}(\M A)}\sum_{\ell = 1}^L \mathrm{Var}(Z_{ij \ell})= \frac{1}{{|\M A|L}}v_{\iz\jz}(\Omega)
   +  \frac{1}{L}\sigma({\M A}) = V_{\iz\jz}(\Omega),
\ee
which is the variance in \eqref{asymptotic:variance}. By \eqref{V:bound}, we also have $V_{\iz\jz}(\Omega) \asymp 1/(npL)$ under $\mathcal{E}_{\mathrm{good}}$.

Next, we bound the third moment. Since ${\bds \theta} \in \Theta$, then similar to \eqref{varzijl} and \eqref{varzijl2}, we have
\begin{align}\nonumber
&\frac{1}{{|\M A|^3L^3}}\sum_{(i,j)\in\mathcal{E}(\M A)}\sum_{\ell = 1}^L \E |Z_{ij \ell}|^3 
\\\nonumber
& =\sum_{(i,j)\in\mathcal{E}(\M A)}\sum_{\ell = 1}^L   \frac{1}{{|\M A|^3L^3}}\E \Big|\mathbb{I}(\M X_{ij \ell}  \in \Omega)  \gamma^{*}_{\iz\jz}(\M X_{ij \ell} ) - \fav_{\iz\jz}(\Omega)    +   {\alpha}^{*}_{ij}(\M X_{ij \ell} )  \left(Y_{ij \ell}  - \psi\left(\gamma^{*}_{ij}(\M X_{ij \ell} )\right)\right)\Big|^3
\\\nonumber
&\leq\frac{C}{{|\M A|^2L^2}} +  \frac{C}{{|\M A|^3L^2}}\sum_{(i,j)\in\mathcal{E}(\M A)}   \E  \Big|    {\alpha}^{*}_{ij}(\M X  )  \left(Y_{ij}  - \psi\left(\gamma^{*}_{ij}(\M X )\right)\right)\Big|^3
\\\nonumber
 & \leq \frac{C}{{|\M A|^2L^2}} +  \frac{C}{{ L^2}}\sum_{(i,j)\in\mathcal{E}(\M A)}   \E \left( \Big|    \mathbb{I}(\M X\in\Omega) (\M e_{i} - \M e_{j})^\T\pmb{\mathscr{L}}^{\dagger}(\M X\mid  \M A) (\M e_{\iz} - \M e_{\jz})  \Big|^3 \psi'\left(\gamma^{*}_{ij}(\M X )\right)\right)
\\\nonumber
&\leq {\frac{C\sup_{\M x\in\mathbb{X}}\|\pmb{\mathscr{L}}^{\dagger}(\M x\mid \M A)\|_{\infty}}{{ L^2}}\sum_{(i,j)\in\mathcal{E}(\M A)}   \E _{\M X}\left(  \mathbb{I}(\M X\in\Omega)  \Big|   (\M e_{i} - \M e_{j})^\T\pmb{\mathscr{L}}^{\dagger}(\M X\mid  \M A) (\M e_{\iz} - \M e_{\jz})  \Big|^2 \psi'\left(\gamma^{*}_{ij}(\M X )\right)\right)
  + \frac{C}{{|\M A|^2L^2}}} 
\\\label{eq:Qbar:third:moment:bound}
&\leq \frac{C\sup_{\M x\in\mathbb{X}}\|\pmb{\mathscr{L}}^{\dagger}(\M x\mid  \M A)\| \sigma({\M A})}{{ L^2}}  
  + \frac{C}{{|\M A|^2L^2}}
\leq \frac{C  \lp  \sigma({\M A})}{{ L^2}}  
  + \frac{C}{{|\M A|^2L^2}},   
\end{align} 
where the second inequality holds by the definition of $ \alpha_{ij}^*(\M X) =  \pi_i^*(\M X) -  \pi_j^*(\M X)$ in \eqref{eq:alpha-gamma} and the law of total expectation and the last inequality holds by Lemma~\ref{lemma:eigen:lower:bound}. 
Now by Lemma~\ref{lm:BE}, we have 
\bee\label{main:term:concentration}
&\sup_{z\in\R}\left|\pr\left(\frac{1}{\sqrt{V_{\iz\jz}(\Omega)}}\left(\bar{\fav}_{\iz\jz}(\Omega)- \fav_{\iz\jz}(\Omega)\right) \leq z\right) - \Phi(z)\right|
\\\nonumber
&\leq \frac{C}{V^{3/2}_{\iz\jz}(\Omega)}\sum_{(i,j)\in\mathcal{E}(\M A)}\sum_{\ell = 1}^L \E |Z_{ij \ell}|^3  
\leq \frac{C}{V^{3/2}_{\iz\jz}(\Omega)L^2|\M A|^2}\left( |\M A|^2 \lp \sigma({\M A})  
  + 1\right) \leq \frac{C}{\sqrt{npL}},
\ee
where the last inequality is by Lemma~\ref{lm:boundsparty} under $\mathcal{E}_{\mathrm{good}}$.

\subsubsection{Proof of Lemma~\ref{lm:neyman} and Theorem~\ref{theorem:barQ-hatQ}}\label{theorem:barQ-hatQ:pf} 

We will first present the proof of Lemma~\ref{lm:neyman}.
\begin{proof}[Proof of Lemma~\ref{lm:neyman}]

    By \eqref{eq:alpha-gamma} and \eqref{iota0}, $\alpha^*_{ij}(\M x) = \pi^*_i(\M x)-\pi^*_j(\M x)$ with
    $\bds\pi^*(\M x) = \mathbb{I}(\M x\in\Omega)\,|\M A|\,\pmb{\mathscr{L}}^{\dagger}(\M x\mid {\bds \theta}^*)(\M e_{\iz}-\M e_{\jz})$,
    and recall that in \eqref{iota0}, $\psi'(\gamma^*_{ij}(\M x))$ is the edge weight induced $\pmb{\mathscr{L}}^{\dagger}(\M x\mid {\bds \theta}^*)$.
    By the definition of graph Laplacian, we have for any two vectors $\M u, \M v$, 
    \[
    \M u^\T\pmb{\mathscr{L}}(\M x\mid {\bds \theta}^*)\M v = \sum_{(i,j)\in\mathcal{E}(\M A)} \psi'(\gamma^*_{ij}(\M x))(u_i-u_j)(v_i-v_j).
    \]
    Applying the above identity for $\M u = {\bds \pi}^*(\M x), \M v = {\bds h}(\M x)$, we conclude the proof with the following identity
\begin{align*}
  &\frac{1}{|\M A|}\sum_{(i,j)}\alpha^*_{ij}(\M x)\,\psi'(\gamma^*_{ij}(\M x))\,\delta_{ij}(\M x)
    = \frac{1}{|\M A|}\,\bds\pi^*(\M x)^\T\pmb{\mathscr{L}}(\M x\mid \M A)\,{\bds h}(\M x)\\
    &\qquad = \mathbb{I}(\M x\in\Omega)\,(\M e_{\iz}-\M e_{\jz})^\T\big(\pmb{\mathscr{L}}^{\dagger}\pmb{\mathscr{L}}\big)\,{\bds h}(\M x)
    = \mathbb{I}(\M x\in\Omega)\,(\M e_{\iz}-\M e_{\jz})^\T {\bds h}(\M x)
    = \mathbb{I}(\M x\in\Omega)\delta_{\iz\jz}(\M x),
\end{align*}
    where we used that $\pmb{\mathscr{L}}^{\dagger}\pmb{\mathscr{L}}$ is the orthogonal projector onto $\{\M 1\}^\perp$
    and $(\M e_{\iz}-\M e_{\jz})\in\{\M 1\}^\perp$.
\end{proof}

Now we are ready to prove Theorem~\ref{theorem:barQ-hatQ}.
We first introduce some preliminaries and notations to facilitate our discussion. Throughout the proof, we assume $\M A$ is given and satisfies $\mathcal{E}_{\mathrm{good}}$. For simplicity of notation, we take $S=1$ and assume both $\hat{\bds\theta}$ and $\hat{\bds\pi}$ are estimated using an independent copy of $\mathcal{D}_n$ with the same $\M A$; the proof extends to any fixed $S$. Let $\epsilon \in (0,1)$ be some prespecified constant.

Recall that we define $\bar{\fav}_{\iz\jz}(\Omega)$ in \eqref{eq:Qbar} and $\PP_n$ in \eqref{eq:Pn}.
We decompose the Fisher random walk debiased estimator $\hat{\fav}_{\iz\jz}(\Omega)$ as
\bee\nonumber
 \hat{\fav}_{\iz\jz}(\Omega) 
& = \mathbb{P}_n\left( \mathbb{I}(\M X  \in \Omega)  \hat\gamma_{\iz\jz}(\M X)    +   \hat{\alpha}_{\itt\jtt}(\M X)  \left(Y - \psi\left(\hat\gamma_{\itt\jtt}(\M X )\right)\right) \right)
\\
&= \bar{\fav}_{\iz\jz}(\Omega)  + \Delta_1 + \Delta_2 + \Delta_3, \text{ where }
\ee
\vspace{-5pt}
\begin{align}\label{eq:Delta1:Delta2:Delta3}
  \Delta_1 &= \mathbb{P}_n\left(\mathbb{I}(\M X \in \Omega)\left(\hat\gamma_{\iz\jz}(\M X) - \gamma^{*}_{\iz\jz}(\M X)\right) -  {\alpha}^*_{\itt\jtt}(\M X) \left(  \psi\left(\hat\gamma_{\itt\jtt}\left(\M X\right)\right) -  \psi\left(\gamma^{*}_{\itt\jtt}\left(\M X\right)\right)\right) \right)
\\
 \Delta_2 &= \mathbb{P}_n\left(\left(\hat{\alpha}_{\itt\jtt}(\M X) - {\alpha}^*_{\itt\jtt}(\M X)\right)\left(Y - \psi\left(\gamma^*_{\itt\jtt}\left(\M X\right)\right)\right)\right)
\\
\Delta_3 &= \mathbb{P}_n\left( \left(\hat{\alpha}_{\itt\jtt}(\M X) - {\alpha}^*_{\itt\jtt}(\M X)\right)\left(\psi\left(\gamma^*_{\itt\jtt}\left(\M X\right)\right) - \psi\left(\hat\gamma_{\itt\jtt}\left(\M X\right)\right) \right)\right)
\end{align}

We will bound the above three terms seperately.

\noindent{\bf Bound of $\Delta_1$.} 
By Taylor expansion, we have
\begin{align}\nonumber
\Delta_1  &= \Delta_{1,1} + \Delta_{1,2} \text{ where }
\\\nonumber
\Delta_{1,1}& = \mathbb{P}_n\left( \mathbb{I}(\M X \in \Omega)\left(\hat\gamma_{\iz\jz}(\M X) - \gamma^{*}_{\iz\jz}(\M X)\right) -  {\alpha}^*_{\itt\jtt}(\M X) \left(  \psi'\left(\gamma^{*}_{\itt\jtt}\left(\M X\right)\right) \left(\hat\gamma_{\itt\jtt}\left(\M X\right) - \gamma^{*}_{\itt\jtt}\left(\M X\right) \right) \right) \right)
\\\nonumber
\Delta_{1,2}&= \mathbb{P}_n\left(  -   {\alpha}^*_{\itt\jtt}(\M X)    \psi''\left(\tilde{\gamma}_{\itt\jtt}\left(\M X\right)\right) \left(\hat\gamma_{\itt\jtt}\left(\M X\right) - \gamma^{*}_{\itt\jtt}\left(\M X\right) \right)^2   \right)
\end{align}
where $\tilde{\gamma}_{ij}(\M x)$ is between $\hat{\gamma}_{ij}(\M x)$ and ${\gamma}^*_{ij}(\M x)$ for any $(i,j)\in\mathcal{E}(\M A)$ and $\M x\in\mathbb{X}$. 

By Lemma~\ref{lm:alpha:bound}, we have both $\max_{i,j \in [n]}|\hat{\alpha}_{ij}(\M x)| \le C|\M A|\lambda^{-1}_{\min,\perp}(\M L(\M A))$ and $\max_{i,j \in [n]}|\alpha^*_{ij}(\M x)| \le C|\M A|\lambda^{-1}_{\min,\perp}(\M L(\M A))$. As ${\bds \theta}^*, \hat{\bds\theta} \in \Theta$, we have $|\hat{\gamma}_{ij}|$ and $|\gamma^*_{ij}|$ are uniformly bounded away from infinity and thus all the following four terms above are bounded as follows
\bee\label{eq:Deltas:upperbound}
\Delta_{1,1}, \Delta_{1,2}, \Delta_{2}, \Delta_{3}  \le C|\M A|\lambda^{-1}_{\min,\perp}(\M L(\M A)).
\ee

For $\Delta_{1,1}$, by Lemma~\ref{lm:neyman}, we have $\E[\Delta_{1,1}\mid \hat{\bds\theta}] =0$. Its  second moment can be bounded as follows:
\begin{align}\nonumber
  & \E\left[ \left\{\mathbb{I}(\M X \in \Omega)\left(\hat\gamma_{\iz\jz}(\M X) - \gamma^{*}_{\iz\jz}(\M X)\right) -  \frac{1}{|\M A|}\textstyle\sum_{(i,j)\in\mathcal{E}(\M A)}{\alpha}^*_{ij}(\M X)   \psi'\left(\gamma^{*}_{ij}\left(\M X\right)\right) \left(\hat\gamma_{ij}\left(\M X\right) - \gamma^{*}_{ij}\left(\M X\right) \right) \right\}^2 \right]
  \\\nonumber
 &\leq {\mathcal{E}_{2}( \hat{\gamma}_{\iz\jz},\gamma^*_{\iz\jz}  \mid \M A)}+  \frac{1}{|\M A|^2}\E\Big[\E\Big[\big( \textstyle \sum_{(i,j)\in\mathcal{E}(\M A)}{\alpha}^*_{ij}(\M X )    \psi'\left(\gamma^{*}_{ij}\left(\M X \right)\right) \left(\hat\gamma_{ij}\left(\M X \right) - \gamma^{*}_{ij}\left(\M X \right) \right) \big)^2\mid \hat{\bds\theta} \Big]\Big]
  \\ \nonumber
  &\le C {\mathcal{E}_{2}(\hat{\bds\theta},\bds\theta^*\mid \M A)}+   C |\M A|\Big(\max_{i \in [n]}\sum_{j=1}^n\E|\alpha^*_{ij}(\M X)|^2\Big) {\mathcal{E}_{2}( \hat{\gamma}_{\iz\jz},\gamma^*_{\iz\jz}  \mid \M A)}
  \\\label{eq:Delta11:var}
  &\le C {\mathcal{E}_{2}(\hat{\bds\theta},\bds\theta^*\mid \M A)}+   C |\M A|\lambda^{-2}_{\min,\perp}(\M L(\M A)){\mathcal{E}_{2}(\hat{\bds\theta},\bds\theta^*\mid \M A)},
  \end{align}
where the first inequlity is by triangle inequality and the second inequality is because $\psi'\left(\gamma^{*}_{ij}\left(\M x \right)\right) \le C$ for any $\M x$ as ${\bds \theta}^* \in \Theta$, $\max_{i,j \in [n]}{\mathcal{E}_{2}( \hat{\gamma}_{ij},\gamma^*_{ij}  \mid \M A)} \le C {\mathcal{E}_{2}(\hat{\bds\theta},\bds\theta^*\mid \M A)}$ by triangle inequality  and applying the Cauchy-Schwarz inequality to the second term. The last inequality is by Lemma~\ref{lm:alpha:bound}.
As $\Delta_{1,1} \le C|\M A|\lambda^{-1}_{\min,\perp}$ by \eqref{eq:Deltas:upperbound}, applying the Bernstein's inequality in Lemma~\ref{lm:bern}(II), we have with probability $1-C\epsilon$,
\bee\label{eq:Delta11:subG}
\left|\Delta_{1,1}\right| &\leq  \sqrt{\frac{C\log(2/\epsilon)}{L}}\sqrt{ {\mathcal{E}_{2}(\hat{\bds\theta},\bds\theta^*\mid \M A)}+   |\M A|\lambda^{-2}_{\min,\perp}(\M L(\M A)){\mathcal{E}_{2}(\hat{\bds\theta},\bds\theta^*\mid \M A)}} + \frac{C|\M A|\lambda^{-1}_{\min,\perp}(\M L(\M A))\log(2/\epsilon)}{L}.
\ee

For $\Delta_{1,2}$, as $|\psi''(t)| \le C$ when $t$ is away from infinity and  Lemma~\ref{lm:alpha:bound}, we have
$$
\left|\Delta_{1,2}\right| \le {C}\lp  \sum_{(i,j)\in\mathcal{E}(\M A)}   \left(\hat\gamma_{ij}\left(\M X\right) - \gamma^{*}_{ij}\left(\M X\right) \right)^2.
$$ 
We then can bound the first and second moments as
\begin{align}\nonumber
  \E\left|\Delta_{1,2}\right| &\le {C}\lambda^{-1}_{\min,\perp}(\M L(\M A))  \sum_{(i,j)\in\mathcal{E}(\M A)}  \E \left(\hat\gamma_{ij}\left(\M X\right) - \gamma^{*}_{ij}\left(\M X\right) \right)^2
  \le C \lambda^{-1}_{\min,\perp}(\M L(\M A)) |\M A|\cdot {\mathcal{E}_{2}(\hat{\bds\theta},\bds\theta^*\mid \M A)}
  \\\label{eq:Delta12:var}
  \E\left|\Delta_{1,2}\right|^2 &\le {C}\lambda^{-2}_{\min,\perp}(\M L(\M A))|\M A|  \sum_{(i,j)\in\mathcal{E}(\M A)}  \E \left(\hat\gamma_{ij}\left(\M X\right) - \gamma^{*}_{ij}\left(\M X\right) \right)^4
  \le C \lambda^{-2}_{\min,\perp}(\M L(\M A)) |\M A|^2\cdot {\mathcal{E}_{2}(\hat{\bds\theta},\bds\theta^*\mid \M A)},  
  \end{align}
where the last inequality is by the uniform boundedness of $|\hat\gamma_{ij}|$ and $|\gamma^*_{ij}|$ as ${\bds \theta}^*, \hat{\bds\theta} \in \Theta$.
 By Bernstein's inequality and \eqref{eq:Deltas:upperbound}, we have with probability $1-C\epsilon$,
 \bee\label{eq:Delta12:subG}
 &\left|\Delta_{1,2} - \E \Delta_{1,2}\right|\\
 &\le \sqrt{\frac{C\log(2/\epsilon)}{L}}\sqrt{ |\M A|^2\lambda^{-2}_{\min,\perp}(\M L(\M A)){\mathcal{E}_{2}(\hat{\bds\theta},\bds\theta^*\mid \M A)}} + \frac{C|\M A|\lambda^{-1}_{\min,\perp}(\M L(\M A))\log(2/\epsilon)}{L}.
 \ee
Combining the bounds of $\Delta_{1,1}$ and $\Delta_{1,2}$ above, we have with probability $1-C\epsilon$,
\bee\label{eq:Delta1:bound}
\left|\Delta_{1}\right| &\le \sqrt{\frac{C\log(2/\epsilon)}{L}}\sqrt{ {\mathcal{E}_{2}(\hat{\bds\theta},\bds\theta^*\mid \M A)}+   |\M A|^2\lambda^{-2}_{\min,\perp}(\M L(\M A)){\mathcal{E}_{2}(\hat{\bds\theta},\bds\theta^*\mid \M A)}} 
\\
&\qquad + \frac{C|\M A|\lambda^{-1}_{\min,\perp}(\M L(\M A))\log(2/\epsilon)}{L}+ C \lambda^{-1}_{\min,\perp}(\M L(\M A)) |\M A|\cdot {\mathcal{E}_{2}(\hat{\bds\theta},\bds\theta^*\mid \M A)}.
\ee

\noindent{\bf Bound of $\Delta_2$.}  We have $\E[\Delta_{2}  \mid \hat{\bds\pi}] = 0$ as $\E[Y_{ij \ell}\mid \M X_{ij \ell}] =  \psi\left(\gamma^*_{ij}\left(\M X_{ij \ell}\right) \right)$. Similar to the analysis of $\Delta_1$, we bound the second moment by
\bee\label{eq:Delta2:var}
  \E[\Delta_2^2] &\leq \frac{C}{|\M A|^2 L^2 } \E\left[ \sum_{(i,j)\in\mathcal{E}(\M A)}\sum_{\ell\in[L]}\E\left(\left(\hat{\alpha}_{ij}(\M X_{ij \ell}) - {\alpha}^*_{ij}(\M X_{ij \ell})\right)^2 \mid \hat{\bds\pi}\right)\right]\\
  &\leq \frac{C}{|\M A|^2 L } \sum_{(i,j)\in\mathcal{E}(\M A)}\E\left[\hat{\alpha}_{ij}(\M X_{ij \ell}) - {\alpha}^*_{ij}(\M X_{ij \ell})\right]^2 \le \frac{C}{|\M A|L} \mathcal{E}_{2}\left(\hat{\bds\pi},\bds\pi^*\mid \M A\right),
\ee
where the first inequality is by the uniform constant upper bounds of $Y_{ij \ell}$ and $\psi(\cdot)$ as ${\bds \theta} \in \Theta$, the second inequality is the tower property of conditional expectation and the third inequality is by Lemma~\ref{lm:Cauchy-Schwarz}. Therefore, by Bernstein's inequality and \eqref{eq:Deltas:upperbound}, we have with probability $1-C\epsilon$,
\bee\label{eq:Delta2:bound}
 \left|\Delta_{2}\right| \leq  \sqrt{\frac{C\log(2/\epsilon)}{|\M A|L}} \sqrt{\mathcal{E}_{2}\left(\hat{\bds\pi},\bds\pi^*\mid \M A\right)} + \frac{C|\M A|\lambda^{-1}_{\min,\perp}(\M L(\M A))\log(2/\epsilon)}{L}.
 \ee

\noindent{\bf Bound of $\Delta_3$.} 
Then, we bound  $ \Delta_3$. 
By Taylor expansion, we have
\bee\label{eq:Delta3:var}
 \Delta_3 &=   \mathbb{P}_n\left( \left(\hat{\alpha}_{\itt\jtt}(\M X) - {\alpha}^*_{\itt\jtt}(\M X)\right)\left(\psi\left(\gamma^*_{\itt\jtt}\left(\M X\right)\right) - \psi\left(\hat\gamma_{\itt\jtt}\left(\M X\right)\right) \right)\right)
 = \mathbb{P}_n\left(\psi'\left(\bar{\gamma}_{\itt\jtt}(\M X)\right)\left(\hat{\alpha}_{\itt\jtt}(\M X) - {\alpha}^*_{\itt\jtt}(\M X)\right)\left( \gamma^*_{\itt\jtt}\left(\M X\right)  -  \hat\gamma_{\itt\jtt}\left(\M X\right) \right)\right),
\ee
where $\bar{\gamma}_{ij}(\M x)$ is between $\hat{\gamma}_{ij}(\M x)$ and ${\gamma}^{*}_{ij}(\M x)$ and thus $|\psi'\left(\bar{\gamma}_{ij}(\M x)\right)| \le C$ for any $i,j \in [n]$ and $\M x \in \mathbb{X}$ as ${\bds \theta} \in \Theta$.  Similar to the analysis of $\Delta_{1,2}$, 
 we can bound the second moment of $\Delta_3$ by
\bee\label{eq:Delta3:var}
  \E| \Delta_3|^2 
& \leq \frac{C }{|\M A| L} \E\left(\sum_{i,j}\sum_{\ell=1}^L A_{ij}\left|\left(\hat{\alpha}_{ij}(\M X_{ij \ell}) - {\alpha}^*_{ij}(\M X_{ij \ell})\right)\left( \gamma^*_{ij}\left(\M X_{ij \ell}\right)  -  \hat\gamma_{ij}\left(\M X_{ij \ell}\right) \right)\right|^2\right)\\
&\leq {C }  {\mathcal{E}_{2}\left(\hat{\bds\theta},\bds\theta^*\mid \M A\right)\mathcal{E}_{2}\left(\hat{\bds\pi},\bds\pi^*\mid \M A\right)},
\ee
where the last inequality is by Lemma~\ref{lm:Cauchy-Schwarz}.
Similarly, we have $\E| \Delta_3|  \leq {C }  {\mathcal{E}^{1/2}_{2}\left(\hat{\bds\theta},\bds\theta^*\mid \M A\right)\mathcal{E}^{1/2}_{2}\left(\hat{\bds\pi},\bds\pi^*\mid \M A\right)}.$
Therefore, by Bernstein's inequality and \eqref{eq:Deltas:upperbound}, we have with probability $1-C\epsilon$,
\bee\label{eq:Delta3:bound}
 \left|\Delta_{3}\right| &\leq  \sqrt{\frac{C\log(2/\epsilon)}{L}}  {\mathcal{E}^{1/2}_{2}\left(\hat{\bds\theta},\bds\theta^*\mid \M A\right)\mathcal{E}^{1/2}_{2}\left(\hat{\bds\pi},\bds\pi^*\mid \M A\right)} \\
 &\qquad+ \frac{C|\M A|\lambda^{-1}_{\min,\perp}(\M L(\M A))\log(2/\epsilon)}{L} + {C }  {\mathcal{E}^{1/2}_{2}\left(\hat{\bds\theta},\bds\theta^*\mid \M A\right)\mathcal{E}^{1/2}_{2}\left(\hat{\bds\pi},\bds\pi^*\mid \M A\right)}.
 \ee

Combining \eqref{eq:Delta11:var}, \eqref{eq:Delta12:var}, \eqref{eq:Delta2:var}, and \eqref{eq:Delta3:var}, we have
\bee\label{eq:Ztl:var}
\E|\hat{\fav}_{\iz\jz}(\Omega) - \bar{\fav}_{\iz\jz}(\Omega)|^2 &\le C {\mathcal{E}_{2}(\hat{\bds\theta},\bds\theta^*\mid \M A)}+   C |\M A|^2\lambda^{-2}_{\min,\perp}(\M L(\M A)){\mathcal{E}_{2}(\hat{\bds\theta},\bds\theta^*\mid \M A)} \\
&\qquad +\frac{C}{|\M A|L} \mathcal{E}_{2}\left(\hat{\bds\pi},\bds\pi^*\mid \M A\right)+{C }  {\mathcal{E}_{2}\left(\hat{\bds\theta},\bds\theta^*\mid \M A\right)\mathcal{E}_{2}\left(\hat{\bds\pi},\bds\pi^*\mid \M A\right)}
\ee

Combining \eqref{eq:Delta1:bound}, \eqref{eq:Delta2:bound} and \eqref{eq:Delta3:bound}, we have with probability $1-C\epsilon$,
\bee\nonumber
&|\hat{\fav}_{\iz\jz}(\Omega) - \bar{\fav}_{\iz\jz}(\Omega)| \leq |\Delta_1| + |\Delta_2| + |\Delta_3| \\
&\le \sqrt{\frac{C\log(2/\epsilon)}{L}}\sqrt{ {\mathcal{E}_{2}(\hat{\bds\theta},\bds\theta^*\mid \M A)}+   |\M A|^2\lambda^{-2}_{\min,\perp}(\M L(\M A)){\mathcal{E}_{2}(\hat{\bds\theta},\bds\theta^*\mid \M A)}} + \sqrt{\frac{C\log(2/\epsilon)}{|\M A|L}} \sqrt{\mathcal{E}_{2}\left(\hat{\bds\pi},\bds\pi^*\mid \M A\right)} 
\\
&\quad + \frac{C|\M A|\lambda^{-1}_{\min,\perp}(\M L(\M A))\log(2/\epsilon)}{L}+ C \lambda^{-1}_{\min,\perp}(\M L(\M A)) |\M A|\cdot {\mathcal{E}_{2}(\hat{\bds\theta},\bds\theta^*\mid \M A)}\\
&\qquad +C\sqrt{\frac{\log(2/\epsilon)}{L}+1} \cdot {\mathcal{E}^{1/2}_{2}\left(\hat{\bds\theta},\bds\theta^*\mid \M A\right)\mathcal{E}^{1/2}_{2}\left(\hat{\bds\pi},\bds\pi^*\mid \M A\right)}.\\
&\leq C\sqrt{\log(2/\epsilon)}\left(n \cdot  L^{-1/2}{\mathcal{E}^{1/2}_{2}( \hat{\bds\theta},\bds\theta^*  \mid \M A)} +(n^2pL)^{-1/2}  {\mathcal{E}^{1/2}_2\left( \hat{\bds \pi},\bds \pi^* \mid \M A\right)}\right)+ {C\log(2/\epsilon) \cdot nL^{-1}}\\
  &\quad + C\left(n  \cdot\mathcal{E}_{2}( \hat{\bds\theta},\bds\theta^*  \mid \M A) +{\mathcal{E}^{1/2}_{2}( \hat{\bds\theta},\bds\theta^*  \mid \M A)\mathcal{E}^{1/2}_2\left( \hat{\bds \pi},\bds \pi^* \mid \M A\right) }\right) 
\ee
where the last inequality is by the assumption of $\mathcal{E}_{{\rm good}}$.

\subsubsection{Proof of Theorem~\ref{theorem:sigmahat:rate}}\label{theorem:sigmahat:rate:pf}

For the simplicity of notation, we assume that $S = 1$, and both $\hat{\bds\theta}$ and $\hat{\bds\pi}$ are estimated  through an independent copy of $\mathcal{D}_n$ using the same $\M A$. Our proof  holds in general if $S$ is fixed. Recall we define $\sigma(\M A)$ in \eqref{eq:asymptotic:variance2} and estimate it as $\sigma({\M A})$ in Algorithm~\ref{alg:ci:addition}. We denote
$
  \tilde\sigma({\M A})   = \E\big[\hat\pi_{\iz}(\M X) - \hat\pi_{\jz}(\M X)\big]
$, and then have
  \bee\nonumber
   \hat{\sigma}(\M A) - \tilde{\sigma}(\M A)=\frac{1}{|\M A|^2L}\sum_{(i,j)\in\mathcal{E}(\M A)\atop \ell\in[L]}\Big(\left(\hat{\pi}_{\iz}(\M X_{ij \ell}) - \hat{\pi}_{\jz}(\M X_{ij \ell}) \right)-  
  \E\left(\hat\pi_{\iz}(\M X ) - \hat\pi_{\jz}(\M X )\right)\Big).
  \ee
  By Lemma~\ref{lm:good-ER}, we have $|\M A| \geq Cn^2p$ with probability at least $1-2/n$. Combining it with \eqref{bound:L}, we have
  \bee\label{eq:pi:inf-rate}
  \sup_{\M x\in\mathbb{X}}|\hat\pi_{\iz}(\M x ) - \hat\pi_{\jz}(\M x )| =  \sup_{\M x\in\mathbb{X}}\left||\M A|(\M e_{\iz} - \M e_{\jz})^\T\hat{\pmb{\mathscr{L}}}^\dagger(\M x\mid \M A)(\M e_{\iz} - \M e_{\jz})\right| = O_P(n).
  \ee 
  Therefore, applying the Markov inequality, as $n \rightarrow \infty$, we have
  \bee\nonumber
  | \hat{\sigma}(\M A) - \tilde{\sigma}(\M A)| = O_P\Big(\frac{\sup_{\M x\in\mathbb{X}}|\hat\pi_{\iz}(\M x ) - \hat\pi_{\jz}(\M x )|}{|\M A|^{3/2}L^{1/2}}\Big) = o_P( (np)^{-1}),
  \ee
  where the last equality is due to \eqref{eq:pi:inf-rate} and $|\M A| \geq Cn^2p$ with high probability.
   On the other hand, we have
  \begin{align}\nonumber
  |\tilde{\sigma}(\M A) - \sigma(\M A)| &= \frac{1}{|\M A|} \left|
  \E\left(\hat\pi_{\iz}(\M X ) - \hat\pi_{\jz}(\M X )\right) -   
  \E\left(\pi^*_{\iz}(\M X ) - \pi^*_{\jz}(\M X )\right)\right|
  \\\nonumber
  &\leq \frac{C}{|\M A|}\sqrt{n\mathcal{E}_2(\hat{\bds\pi},\bds\pi^*\mid \M A )}
  \leq \frac{C}{|\M A|}\sqrt{n^3\mathcal{E}_2(\hat{\bds\theta},\bds\theta^*\mid \M A)}
   =  o_P((np)^{-1}),
  \end{align}
  where the first inequality holds by Proposition~\ref{po:pierror}, and the last equality holds by Assumption~\ref{ass:theta}.
  %   $$
  % \frac{C}{|\M A|}\sqrt{n^3\mathcal{E}_2(\hat{\bds\theta},\bds\theta^*\mid \M A)}\prec \frac{\sqrt{n^3(n^{-3}p^{-1} \wedge n^{-3/2}p^{-1/2}L^{-1/2})}}{n^2p} = \frac{\sqrt{(n^{-2}p^{-1} \wedge n^{-1/2}p^{-1/2}L^{-1/2})}}{np}\prec (np)^{-1}.
  % $$ 
  We thus have  $$L^{-1}|\hat{\sigma}(\M A) - {\sigma}(\M A)|  = o_P((npL)^{-1}).$$

  Following the similar proof above, as $  |\hat{\mathcal{Q}}_{\iz\jz}(\Omega) -  {\mathcal{Q}}_{\iz\jz}(\Omega) | = o_P(1)$ by \eqref{eq:Q:clt},  we can also have
  \begin{align*} 
    \frac{1}{{|\M A|^2L^2 }}\sum_{(i,j)\in\mathcal{E}(\M A) \!\!\!\ell\in[L]} \big\{\mathbb{I}(\M X_{ij \ell}  \in \Omega)  \hat\gamma_{\iz\jz}(\M X_{ij \ell} ) - \hat{\fav}_{\iz\jz}(\Omega)\big\}^2
- \frac{1}{|\M A|L}\E \big[\mathbb{I}(\M X  \in \Omega)  \gamma^{*}_{\iz\jz}(\M X ) - \fav_{\iz\jz}(\Omega)\big]^2  
  =o_P((n^2pL)^{-1}).
  \end{align*}
Finally, by \eqref{V:bound}, we also have $V_{\iz\jz}(\Omega) \asymp 1/(npL)$ under $\mathcal{E}_{\mathrm{good}}$.
 Summarizing above,  we have that with probability that goes to 1 as $n\rightarrow\infty$, 
  \begin{align}\nonumber
  &\frac{|\hat{V}_{\iz\jz}(\Omega%
  ) - {V}_{\iz\jz}(\Omega%
  )|}{{V}_{\iz\jz}(\Omega%
  )}  = o_P(1).
  \end{align}

\section{Technical Proofs for Semiparametric Efficiency  in Section~\ref{sec:eff}}\label{sec:pf:efficiency}

In this section, we provide the technical proofs for Theorem~\ref{thm:efficiency} in Section~\ref{sec:eff} and Theorem~\ref{thm:oracle} in Section~\ref{pfthm:oracle}.

\subsection{Proof of Theorem~\ref{thm:efficiency}}\label{pf:thm:efficiency}
We work under an information-enriched case $\mathcal{S}_{n}$ as specified below. Under $\mathcal{S}_{n}$, in addition to assuming that $\M A$ is given and fixed with $L\rightarrow\infty$ as in Theorem~\ref{thm:efficiency},  we further assume that other than $\theta^*_{\iz}(\M x)$ and $\theta^*_{\jz}(\M x)$, all $\{\theta^*_{i}(\M x)\}_{i\neq \iz,\jz}$ are known. We derive the semiparametric efficiency bound under this information-enriched condition $\tilde{V}_{\iz\jz}(\Omega)$, which serves as a lower bound of the exact semiparametric efficiency bound $V_{\iz\jz}^*(\Omega)$ without assuming $\{\theta^*_{i}(\M x)\}_{i\neq \iz,\jz}$ are known.

% It is left to derive $\tilde{V}_{\iz\jz}(\Omega)$. We recall that $\tilde{\sigma}(\M A)$ is 
% \begin{align}\nonumber
% &\tilde{\sigma}(\M A) =   \E_{\M X}\left(\mathbb{I}(\M X\in\Omega)\left(\frac{(1 + \Delta_{\iz}(\M X))^2}{\sum_{j}A_{j\iz}\psi'(\theta^*_{\iz}(\M X) - \theta^*_{j}(\M X))} + \frac{(1 + \Delta_{\jz}(\M X))^2}{\sum_{i}A_{i\jz}\psi'(\theta^*_{\jz}(\M X) - \theta^*_{i}(\M X))}\right)\right) 
% \\\nonumber
% &\quad\quad\quad\quad \quad\quad    + \mathbb{I}((\iz,\jz)\in\mathcal{E}(\M A))\cdot  \E_{\M X}\left(\left(\frac{2\mathbb{I}(\M X\in\Omega) \psi'(\theta_{\iz}^*(\M X) - \theta^*_{\jz}(\M X))}{\sum_{j\in\delta(\iz\mid {\M A} )\atop j\neq \jz}\psi'(\theta^*_{\iz}(\M X) - \theta^*_{j}(\M X)) \cdot \sum_{i\in\delta(\jz\mid {\M A} )\atop i\neq \iz}\psi'(\theta^*_{\jz}(\M X) - \theta^*_{i}(\M X))} \right)\right) ,
% \\\nonumber
% &\Delta_{\iz}(\M X)   =   A_{\iz\jz} \frac{  \psi'(\theta_{\iz}^*(\M X ) - \theta_{\jz}^*(\M X ))}{\sum_{j'\in \delta(\iz\mid \M A)\atop j'\neq \jz}\psi'(\theta_{\iz}^*(\M X ) - \theta_{j'}^*(\M X ))}, \quad \Delta_{\jz}(\M X)   =  A_{\iz\jz} \frac{  \psi'(\theta_{\iz}^*(\M X ) - \theta_{\jz}^*(\M X ))}{\sum_{i'\in \delta(\jz\mid \M A)\atop i'\neq \iz}\psi'(\theta_{\jz}^*(\M X ) - \theta_{i'}^*(\M X ))}.
% \end{align}
Under $\mathcal{S}_{n}$, with given $\M A$ and fixed $n$, we  rearrange the data as $\mathcal{D}_n =   \cup_{\ell = 1}^L\mathcal{D}_{n,\ell} $, where 
$$
\mathcal{D}_{n,\ell} = \{(\M X_{ij \ell},Y_{ij \ell})\mid (i,j)\in\mathcal{E}(\M A)\}.
$$ Clearly,  $\mathcal{D}_{n,1},\ldots,\mathcal{D}_{n,\ell}$ follow a same distribution. Moreover, since $\theta_i(\M x)$ for $i\in[n]\setminus\{\iz,\jz\}$ is known, all  $ Y_{ij \ell} $ with $(i,j) \neq (\iz,\jz)$ are non-informative, and thus the full samples can be summarized by 
\bee\nonumber
&\mathcal{D}_{n}^{*}  =   \cup_{\ell = 1}^L \mathcal{D}_{n,\ell}^{*} =  \cup_{\ell = 1}^L\big\{ \M X_{ij \ell},Y_{i'j'\ell} \mid A_{ij} = 1, A_{i'j'} = 1,i' = \iz\text{ and/or }j' = \jz  \big\}  .
\ee
Thus under $\mathcal{S}_{n}$, it is equivalent to derive  the efficiency bound of   ${\fav}_{\iz\jz}(\Omega)$   with data in $\mathcal{D}_{n}^{*}$. For a clear order of all nodes in the following analysis, we assume  $\iz = 1$ and $\jz = n$ without loss of generality.

To derive the semiparametric efficiency bound, we focus on deriving the semiparametric efficient influence function (EIF), denoted as $e(\mathcal{D}_n^{*})$. Then the corresponding semiparametric efficiency bound is $\E \left(e^2(\mathcal{D}_n^{*})\right)$ \citep{bickel1993efficient}. We assume that the distribution of $\mathcal{D}_n^{*}$ follows some parametric distribution $p(\cdot\mid \zeta) $ with parameter $\zeta\in \R$ such that $p (\mathcal{D}_n^{*}\mid 0) = p(\mathcal{D}_n^{*})$. That is, density $p(\mathcal{D}_n^{*}\mid  \zeta)$ coincides with the true distribution density $p(\mathcal{D}_n^{*})$ when $\zeta = 0$.   We denote the corresponding parametric models for $\theta^*_{\iz}(\M x)$ and $\theta_{\jz}^*(\M x)$ as $\theta_{\iz}(\M x\mid \zeta)$ and $\theta_{\jz}(\M x\mid \zeta)$,  such that $\theta_{\iz}(\M x\mid \zeta =  0) = \theta^*_{\iz}(\M x)$ and $\theta_{\jz}(\M x\mid \zeta =  0) = \theta^*_{\jz}(\M x)$. 

Under the information-enriched condition, we impose a constrain on the distribution family $p(\cdot\mid \zeta)$ such that for all $i\notin \{\iz,\jz\}$, we have $\theta_i(\cdot\mid \zeta) = \theta_i^*(\cdot )$ for any $\zeta\in\R$, which is assumed to be known. Then the target estimand with distribution $p(\cdot\mid {\zeta})$ becomes 
\bee\nonumber
{\fav}_{\iz\jz}(\Omega\mid \zeta)= \int_{\Omega} \left(\theta_{i_0}(\bx\mid \zeta) - \theta_{j_0}(\bx\mid \zeta)\right) p(\bx\mid \zeta)d\bx.
\ee
We also denote by $\E_\zeta(\cdot)$ the expectation taken with respect to $\mathcal{D}_n^*\sim  p(\cdot\mid \zeta)$.

 Recall $\mathcal{D}_{n,1}^*,\ldots,\mathcal{D}_{n,\ell}^*$ are i.i.dgenerated. Under information-enriced condition with parameter $\zeta$, the likelihood of $\mathcal{D}_{n,\ell}^*$ for any $\ell\in[L]$ is
\bee\nonumber
 p(\mathcal{D}_{n,\ell}^*\mid \zeta) 
 = & \prod_{(i,j)\in\mathcal{E}(\M A) }p(\M X_{ij \ell}\mid \zeta)\\
& \quad  \cdot \prod_{(i,j)\in\mathcal{E}_{\iz,\jz}(\M A) }\Big(\psi\left(\theta_{i} (\M X_{ij \ell}\mid\zeta) - \theta_{j} (\M X_{ij \ell}\mid\zeta)\right)\Big)^{Y_{ij \ell}}\Big(1 - \psi\left(\theta_{i} (\M X_{ij \ell}\mid\zeta) - \theta_{j} (\M X_{ij \ell}\mid\zeta)\right)\Big)^{1 - Y_{ij \ell}},
\ee
where $\mathcal{E}_{\iz,\jz}(\M A) = \big\{(i,j)\in\mathcal{E}(\M A)\mid i  = \iz \text{ and/or }j =  \jz \big\}$. Thus, the preference score function for our parametric model is
\begin{align}\nonumber
s(\mathcal{D}_{n,\ell}^*\mid \zeta) = &\sum_{(i,j)\in\mathcal{E}(\M A) }\frac{d}{d\zeta}\log p(\M X_{ij \ell}\mid \zeta)
 +  \sum_{(i,j)\in\mathcal{E}_{\iz,\jz}(\M A) }Y_{ij \ell}\frac{d}{d\zeta}\log\Big(\psi\left(\theta_{i} (\M X_{ij \ell}\mid\zeta) - \theta_{j} (\M X_{ij \ell}\mid\zeta)\right)\Big) 
\\\label{scorefunction:form}
& +\sum_{(i,j)\in\mathcal{E}_{\iz,\jz}(\M A) }(1 - Y_{ij \ell})\frac{d}{d\zeta}\log\Big(1 - \psi\left(\theta_{i} (\M X_{ij \ell}\mid\zeta) - \theta_{j} (\M X_{ij \ell}\mid\zeta)\right)\Big).
\end{align}
We further have
\begin{align}\nonumber
s(\mathcal{D}_{n,\ell}^*\mid \zeta) &= \sum_{(i,j)\in\mathcal{E}(\M A) } \frac{p^{(\zeta)}(\M X_{ij \ell}\mid \zeta)}{p(\M X_{ij \ell}\mid \zeta)}
  +  \sum_{(i,j)\in\mathcal{E}_{\iz,\jz}(\M A) } s_{ij}(\M X_{ij \ell},Y_{ij \ell}\mid \zeta), \text{ where }
\\\nonumber
s_{ij}(\M X_{ij \ell},Y_{ij \ell}\mid \zeta)  &=\left( \mathbb{I}(i = \iz)\theta^{(\zeta)}_{i}(\M X_{ij \ell}\mid \zeta) - \mathbb{I}(j = \jz)\theta^{(\zeta)}_{j}(\M X_{ij \ell}\mid \zeta)\right) \psi'\big(\theta_{i} (\M X_{ij \ell}\mid\zeta) - \theta_{j} (\M X_{ij \ell}\mid\zeta)\big) 
\\\nonumber
& \quad \cdot \Bigg( \frac{Y_{ij \ell}}{\psi\left(\theta_{i} (\M X_{ij \ell}\mid\zeta) - \theta_{j} (\M X_{ij \ell}\mid\zeta)\right)} -  \frac{(1 - Y_{ij \ell})}{ 1 - \psi\left(\theta_{i} (\M X_{ij \ell}\mid\zeta) - \theta_{j} (\M X_{ij \ell}\mid\zeta)\right) } \Bigg),
\end{align}
where we use the fact that if $i \neq \iz $ or $j\neq \jz$, $\theta^{(\zeta)}_{i}(\M x\mid \zeta) = 0$ and $\theta^{(\zeta)}_{j}(\M x\mid \zeta) = 0$, respectively. Note that as preference   score functions, we   have  at the point of  $\zeta  = 0$, 
\bee\label{scorefund:meanzero}
\E_{ 0 }\left( \frac{p^{(\zeta)} (\M X_{ij \ell}\mid 0)}{p(\M X_{ij \ell}\mid0)}\right) = 0,\quad \E_{0}\left(s_{ij}(\M X_{ij \ell},Y_{ij \ell}\mid 0)\mid \M X_{ij \ell}\right) = 0, \quad \E_{0}\left(s(\mathcal{D}_{n,\ell}^*\mid 0)\right) = 0,
\ee
where $\E_0(\cdot)$ represents the expectation taken with respect to $\mathcal{D}_n^*\sim  p(\cdot\mid 0)$
\par 
The classic semiparametric theory requires that for any influence function (IF) of ${\fav}_{\iz\jz}(\Omega)$, namely $e(\mathcal{D}_{n,\ell}^*)$, it  satisfies 
\bee\label{EIF:condition}
\frac{d}{d\zeta} {\fav}_{\iz\jz}(\Omega\mid \zeta)\,\,\Big|_{\zeta = 0} = \E_{0}\left(e(\mathcal{D}_{n,\ell}^*)s (\mathcal{D}_{n,\ell}^*\mid \zeta = 0)\right);
\ee 
see e.g., \citet{bickel1993efficient}. In what follows, we first derive an IF for \eqref{EIF:condition}. Then we show that such IF is actually an EIF by showing that such IF belongs to the semiparametric tangent space. 

To an IF for \eqref{EIF:condition},  considering the mean-square closure of all score functions in the form of \eqref{scorefunction:form} when $\zeta = 0$, over all possible submodels, we   characterize the semiparametric tangent space that
\bee\nonumber
\mathcal{T} = \mathcal{T}_{\M X_{\mathcal{E}(\M A)}} \oplus \mathcal{T}_{(\M X,Y)_{\mathcal{E}_{\iz,\jz}(\M A)}},
\ee
where we let $X_{\mathcal{E}(\M A)} = (\M X_{ij}\mid (i,j) \in\mathcal{E}(\M A))$, $(\M X,Y)_{\mathcal{E}_{\iz,\jz}(\M A)} = ((\M X_{ij},Y_{ij})\mid (i,j)\in\mathcal{E}_{\iz,\jz}(\M A))$,  
\begin{align}\nonumber
 \mathcal{T}_{\M X_{\mathcal{E}(\M A)}} &=\left(h_1(\M X_{\mathcal{E}(\M A)}) = \sum_{(i,j)\in\mathcal{E}(\M A)}h(\M X_{ij})\mid \E_{\M X}(h(\M X)) = 0  \right),
\\\nonumber
\mathcal{T}_{(\M X,Y)_{\mathcal{E}_{\iz,\jz}(\M A)}} &=\Bigg(h_1\left((\M X,Y)_{\mathcal{E}_{\iz,\jz}(\M A)}\right) = \sum_{(i,j)\in\mathcal{E}_{\iz,\jz}(\M A)}h_{ij}(\M X_{ij},Y_{ij}) \mid \text{ for any }h_{\iz}(\M X)\text{ and }h_{\jz}(\M X),
\\\nonumber
&  \quad \quad\quad  h_{ij}(\M X_{ij},Y_{ij}) = \Big(\mathbb{I}(i = \iz)h_{\iz}(\M X_{ij}) +  \mathbb{I}(j = \jz)h_{\jz}(\M X_{ij}) \Big)\psi'\left(\theta^*_{i} (\M X_{ij}) - \theta^*_{j} (\M X_{ij})\right) 
\\\nonumber
&\quad\quad\quad\quad\quad\quad\quad\quad\quad\quad\cdot \left(  \frac{Y_{ij}}{\psi\left(\theta^*_{i} (\M X_{ij}) - \theta^*_{j} (\M X_{ij})\right)}    -  \frac{1 - Y_{ij} }{ 1 - \psi\left(\theta^*_{i} (\M X_{ij}) - \theta^*_{j} (\M X_{ij})\right)} \right) \Bigg) .
\end{align}

 By some calculation, we   have
\begin{align}\nonumber
&\frac{d}{d\zeta}{\fav}_{\iz\jz}(\Omega\mid \zeta)\,\,\Big|_{\zeta = 0}
\\\nonumber
&= \int_{\Omega} \left(\theta^{(\zeta)}_{i_0}(\bx\mid \zeta) - \theta^{(\zeta)}_{j_0}(\bx\mid \zeta)\right) p(\bx\mid \zeta)d\bx\,\,\Big|_{\zeta = 0} + \int_{\Omega} \left(\theta_{i_0}(\bx\mid \zeta) - \theta_{j_0}(\bx\mid \zeta)\right) p^{(\zeta)}(\bx\mid \zeta)d\bx\,\,\Big|_{\zeta = 0}
\\\label{derive:Fij}
&= \E_0\left(\mathbb{I}(\M X\in\Omega)\left(\theta^{(\zeta)}_{i_0}(\M X\mid \zeta = 0) - \theta^{(\zeta)}_{j_0}(\M X\mid \zeta = 0)\right)\right)    + \int_{\Omega} \left(\theta^*_{i_0}(\M x) - \theta^*_{j_0}(\M x)\right) p^{(\zeta)}(\bx\mid \zeta = 0)d\bx ,
\end{align}
where $\theta^{(\zeta)}_{i_0}(\M x\mid \zeta)$  represents derivative of the function with respect to $\zeta$, and similarly for $p^{(\zeta)}(\M x\mid \zeta)$. 

We  represent the first term on the right-hand side of \eqref{derive:Fij} in the form of \eqref{EIF:condition}.  Specifically, we denote  $e_1(\mathcal{D}_{n,\ell}^*)$ and $e_{1,ij}(\M X_{ij \ell},Y_{ij \ell})$ as
\begin{align}\nonumber
& e_1(\mathcal{D}_{n,\ell}^*)
 = \sum_{(i,j)\in \mathcal{E}_{\iz\jz}(\M A)} e_{1,ij}(\M X_{ij \ell},Y_{ij \ell})
\\\nonumber
& =\sum_{(i,j)\in \mathcal{E}_{\iz\jz}(\M A)}\left( \frac{\mathbb{I}(i = \iz)\mathbb{I} (\M X_{ij \ell}\in\Omega)(1+\Delta_{\iz}(\M X_{ij \ell}))} {\sum_{j'}A_{j'\iz}\psi'(\theta_{\iz}^*(\M X_{i j \ell}) - \theta_{j'}^*(\M X_{i j \ell}))}+ \frac{\mathbb{I}(j = \jz)\mathbb{I}(\M X_{ij \ell}\in\Omega) (1+\Delta_{\jz}(\M X_{ij \ell}))}{\sum_{i'}A_{i'\jz}\psi'(\theta_{\jz}^*(\M X_{i j \ell}) - \theta_{i'}^*(\M X_{i j \ell}))} \right)
\\\nonumber
&  \quad \quad \quad \quad \quad \quad\quad\cdot\psi'\left(\theta^*_{i} (\M X_{ij \ell}) - \theta^*_{j} (\M X_{ij \ell})\right)  \left(  \frac{Y_{ij \ell}}{\psi\left(\theta^*_{i} (\M X_{ij \ell}) - \theta^*_{j} (\M X_{ij \ell})\right)}    -  \frac{1 - Y_{ij \ell} }{ 1 - \psi\left(\theta^*_{i} (\M X_{ij \ell}) - \theta^*_{j} (\M X_{ij \ell})\right)} \right),
\end{align}
where $\Delta_{\iz}, \Delta_{\jz}$ are defined in \eqref{def:tildesigma:main}.
% \begin{align}\nonumber
% \Delta_{\iz}(\M X) & =  A_{\iz\jz} \frac{  \psi'(\theta_{\iz}^*(\M X ) - \theta_{\jz}^*(\M X ))}{\sum_{j'\in \delta(\iz\mid \M A)\atop j'\neq \jz}\psi'(\theta_{\iz}^*(\M X ) - \theta_{j'}^*(\M X ))},
% \\\label{correction:term}
% \Delta_{\jz}(\M X) & =  A_{\iz\jz} \frac{  \psi'(\theta_{\iz}^*(\M X ) - \theta_{\jz}^*(\M X ))}{\sum_{i'\in \delta(\jz\mid \M A)\atop i'\neq \iz}\psi'(\theta_{\jz}^*(\M X ) - \theta_{i'}^*(\M X ))}.
% \end{align}
One can verify that $\E\left(e_{1,ij}(\M X_{ij \ell},Y_{ij \ell}) \mid \M X_{i'j'l},(i',j')\in\mathcal{E}(\M A)\right) = 0$ for any $(i,j)\in\mathcal{E}(\M A)$.  Furthermore, by~\eqref{scorefund:meanzero}, we have 
\begin{align}\nonumber
&\E_{0}\left(e_1(\mathcal{D}_{n,\ell}^*)s (\mathcal{D}_{n,\ell}^*\mid 0)\right) =  \sum_{(i,j)\in\mathcal{E}_{\iz,\jz}(\M A) }\E_0\Big(  s_{ij}(\M X_{ij \ell},Y_{ij \ell}\mid 0)e_{1,ij}(\M X_{ij \ell},Y_{ij \ell}) \Big) 
\\\nonumber
& = \sum_{(i,j)\in\mathcal{E}_{\iz,\jz}(\M A)/\{(\iz,\jz)\}\atop i = \iz }  \E_0\left(  \frac{ \mathbb{I} (\M X \in\Omega)(1+\Delta_{\iz}(\M X )) } {\sum_{j'}A_{j'\iz}\psi'(\theta_{\iz}^*(\M X ) - \theta_{j'}^*(\M X ))}\psi'\left(\theta^*_{\iz} (\M X ) - \theta^*_{j} (\M X )\right)\theta^{(\zeta)}_{\iz}(\M X \mid  0)  \right) 
\\\nonumber
& \quad\quad - \sum_{(i,j)\in\mathcal{E}_{\iz,\jz}(\M A)/\{(\iz,\jz)\}\atop j = \jz }  \E_0\left(  \frac{ \mathbb{I} (\M X \in\Omega)(1+\Delta_{\jz}(\M X )) } {\sum_{i'}A_{i'\jz}\psi'(\theta_{\jz}^*(\M X ) - \theta_{i'}^*(\M X ))}\psi'\left(\theta^*_{i} (\M X ) - \theta^*_{\jz} (\M X )\right)\theta^{(\zeta)}_{\jz}(\M X \mid  0)  \right)  
\\\nonumber
&\quad \quad +\mathbb{I}( (\iz,\jz) \in \mathcal{E}(\M A))\cdot  \E\Bigg(\psi'\left(\theta^*_{\iz} (\M X ) - \theta^*_{\jz} (\M X )\right)
\\\nonumber
&\quad\quad\quad\quad\quad\quad\quad\quad\quad\quad\quad\quad \cdot \left(\frac{ \mathbb{I} (\M X \in\Omega)(1+\Delta_{\iz}(\M X ))\theta^{(\zeta)}_{\iz}(\M X \mid  0) } {\sum_{i'}A_{i'\iz}\psi'(\theta_{\iz}^*(\M X ) - \theta_{i'}^*(\M X ))} - \frac{ \mathbb{I}(\M X \in\Omega) (1+\Delta_{\jz}(\M X ))\theta^{(\zeta)}_{\jz}(\M X \mid  0) }{\sum_{j'}A_{j'\jz}\psi'(\theta_{\jz}^*(\M X ) - \theta_{j'}^*(\M X ))}\right)\Bigg)
\\\nonumber
&\quad \quad +\mathbb{I}( (\iz,\jz) \in \mathcal{E}(\M A)) \cdot  \E\Bigg(\psi'\left(\theta^*_{\iz} (\M X ) - \theta^*_{\jz} (\M X )\right)
\\\nonumber
&\quad\quad\quad\quad\quad\quad\quad\quad\quad\quad\quad\quad \cdot \left(-\frac{ \mathbb{I} (\M X \in\Omega)(1+\Delta_{\iz}(\M X ))\theta^{(\zeta)}_{\jz}(\M X \mid  0) } {\sum_{i'}A_{i'\iz}\psi'(\theta_{\iz}^*(\M X ) - \theta_{i'}^*(\M X ))} + \frac{ \mathbb{I}(\M X \in\Omega) (1+\Delta_{\jz}(\M X ))\theta^{(\zeta)}_{\iz}(\M X \mid  0) }{\sum_{j'}A_{j'\jz}\psi'(\theta_{\jz}^*(\M X ) - \theta_{j'}^*(\M X ))}\right)\Bigg)
\\\nonumber
& =  \E_0\left( \mathbb{I} (\M X \in\Omega) \theta^{(\zeta)}_{\iz}(\M X \mid  0) -  \mathbb{I} (\M X \in\Omega) \theta^{(\zeta)}_{\jz}(\M X \mid  0) \right)  
 +  \E_0\left( \mathbb{I} (\M X \in\Omega) \Delta_{\iz}(\M X )  \theta^{(\zeta)}_{\iz}(\M X \mid  0) -  \mathbb{I} (\M X \in\Omega) \Delta_{\jz}(\M X )   \theta^{(\zeta)}_{\jz}(\M X \mid  0) \right)
\\\nonumber
&\quad +\mathbb{I}( (\iz,\jz) \in \mathcal{E}(\M A)) \cdot  \E\Bigg(\psi'\left(\theta^*_{\iz} (\M X ) - \theta^*_{\jz} (\M X )\right)
\\\nonumber
&\quad\quad\quad\quad\quad\quad\quad\quad\quad\quad\quad\quad \cdot \left(-\frac{ \mathbb{I} (\M X \in\Omega)(1+\Delta_{\iz}(\M X ))\theta^{(\zeta)}_{\jz}(\M X \mid  0) } {\sum_{j'}A_{j'\iz}\psi'(\theta_{\iz}^*(\M X ) - \theta_{j'}^*(\M X ))} + \frac{ \mathbb{I}(\M X \in\Omega) (1+\Delta_{\jz}(\M X ))\theta^{(\zeta)}_{\iz}(\M X \mid  0) }{\sum_{i'}A_{i'\jz}\psi'(\theta_{\jz}^*(\M X ) - \theta_{i'}^*(\M X ))}\right)\Bigg)
\\\label{e1s}
& =  \E_0\left( \mathbb{I} (\M X \in\Omega) \theta^{(\zeta)}_{\iz}(\M X \mid  0) -  \mathbb{I} (\M X \in\Omega) \theta^{(\zeta)}_{\jz}(\M X \mid  0) \right).
\end{align}

To  represent the second term on the right-hand side of \eqref{derive:Fij} in the form of \eqref{EIF:condition}, we consider
\bee\nonumber
e_2(\mathcal{D}_{n,\ell}^{*}) = \frac{1}{|\M A|}\sum_{(i,j)\in\mathcal{E}(\M A)}\mathbb{I}(\M X_{ij \ell}\in \Omega)\left(\theta^*_{i_0}(\M X_{ij \ell}) - \theta^*_{j_0}(\M X_{ij \ell})\right) - {\fav}_{\iz\jz}(\Omega).
\ee
Then we have
\begin{align}\nonumber
&\E_0\left(e_2(\mathcal{D}_{n,\ell}^{*})s (\mathcal{D}_{n,\ell}^*\mid 0)\right)  = \sum_{(i,j)\in\mathcal{E}(\M A) } \E_0\left(\frac{p^{(\zeta)}(\M X_{ij \ell}\mid 0)}{p(\M X_{ij \ell}\mid 0)}e_2(\mathcal{D}_{n,\ell}^{*})\right)
  +  \sum_{(i,j)\in\mathcal{E}_{\iz,\jz}(\M A) } \E_0\Big( s_{ij}(\M X_{ij \ell},Y_{ij \ell}\mid 0)e_2(\mathcal{D}_{n,\ell}^{*})\Big)
% \\\nonumber
% & =\frac{1}{|\M A|}   \sum_{(i,j)\in\mathcal{E}(\M A) } \E_0\left(\frac{p^{(\zeta)}(\M X_{ij \ell}\mid 0)}{p(\M X_{ij \ell}\mid 0)}\mathbb{I}(\M X_{ij \ell}\in \Omega)\left(\theta^*_{i_0}(\M X_{ij \ell}) - \theta^*_{j_0}(\M X_{ij \ell})\right) \right)
% \\\nonumber
% & \quad - \frac{1}{|\M A|}   \sum_{(i,j)\in\mathcal{E}(\M A) } {\fav}_{\iz\jz}(\Omega)\cdot \E_0\left(\frac{p^{(\zeta)}(\M X_{ij \ell}\mid 0)}{p(\M X_{ij \ell}\mid 0)}  \right)
\\\label{e2s}
& =\frac{1}{|\M A|}   \sum_{(i,j)\in\mathcal{E}(\M A) } \E_0\left(\frac{p^{(\zeta)}(\M X \mid 0)}{p(\M X )}\mathbb{I}(\M X \in \Omega)\left(\theta^*_{i_0}(\M X ) - \theta^*_{j_0}(\M X )\right) \right)
 = \int_{\Omega} \left(\theta^*_{i_0}(\M x) - \theta^*_{j_0}(\M x)\right) p^{(\zeta)}(\bx\mid \zeta = 0)d\bx,
\end{align} 
where the second and third equality hold by the law of total expectation with \eqref{scorefund:meanzero}.

Now combing \eqref{e1s} and \eqref{e2s}, we  let $e(\mathcal{D}_{n,\ell}^*) = e_1(\mathcal{D}_{n,\ell}^*) + e_2(\mathcal{D}_{n,\ell}^*)$, and we have
\bee\nonumber
\E_0(e(\mathcal{D}_{n,\ell}^*) s(\mathcal{D}^*_{n,\ell}\mid 0)) = \frac{d}{d\zeta}{\fav}_{\iz\jz}(\Omega\mid \zeta)\,\,\Big|_{\zeta = 0},
\ee
which implies that $e(\mathcal{D}_{n,\ell}^*)$ is the IF of ${\fav}_{\iz\jz}(\Omega)$. It is easy to see that $e_1(\mathcal{D}_{n,\ell}^*)\in \mathcal{T}_{\M X_{\mathcal{E}(\M A)}}$ and $e_2(\mathcal{D}_{n,\ell}^*) \in \mathcal{T}_{(\M X,Y)_{\mathcal{E}_{\iz,\jz}(\M A)}}$, and thus $e(\mathcal{D}_{n,\ell}^*)$ is also the EIF of ${\fav}_{\iz\jz}(\Omega)$. Then by \eqref{def:tildesigma:main}, we derive and simplify the semiparametric efficiency bound that
\begin{align}\nonumber
&\frac{1}{L}\E\left(e^2(\mathcal{D}_{n,\ell}^*)\right)   = \frac{1}{L}\E\left(e_2^2(\mathcal{D}_{n,\ell}^*)\right) + \frac{1}{L}\E\left(e_2^2(\mathcal{D}_{n,\ell}^*)\right)
 = \frac{1}{L}\E_{\M X}\left(\frac{1}{{|\M A|}}\Big(\mathbb{I}(\M X  \in \Omega)  \gamma^{*}_{\iz\jz}(\M X ) - \fav_{\iz\jz}(\Omega)\Big)^2  \right) + \frac{1}{L}\tilde{\sigma}(\M A),
\\\nonumber
&\tilde{\sigma}(\M A) = \E_{\M X}\left(\mathbb{I}(\M X\in\Omega)\left(\frac{(1 + \Delta_{\iz}(\M X))^2}{\sum_{j}A_{j\iz}\psi'(\theta^*_{\iz}(\M X) - \theta^*_{j}(\M X))} + \frac{(1 + \Delta_{\jz}(\M X))^2}{\sum_{i}A_{i\jz}\psi'(\theta^*_{\jz}(\M X) - \theta^*_{i}(\M X))}\right)\right)
\\\nonumber
&\quad\quad + \mathbb{I}((\iz,\jz)\in\mathcal{E}(\M A)) \E_{\M X}\left(\left(\frac{2\mathbb{I}(\M X\in\Omega)   \psi'(\theta_{\iz}^*(\M X) - \theta^*_{\jz}(\M X))}{\sum_{j\neq \jz}A_{ij}\psi'(\theta^*_{\iz}(\M X) - \theta^*_{j}(\M X)) \cdot \sum_{i\neq \iz}A_{ij}\psi'(\theta^*_{\jz}(\M X) - \theta^*_{i}(\M X))} \right)\right).
\end{align} 
If $(\iz,\jz)\notin \mathcal{E}(\M A)$, we directly have $\Delta_{\iz}(\M x) = \Delta_{\jz}(\M x) = 0$, and thus $\tilde{\sigma}(\M A)$ degenerates to \bee\label{tildeAdef}
\tilde{\sigma}(\M A)& = \E \left(\mathbb{I}(\M X\in\Omega)\left\{\frac{1}{\sum_{j\in[n]}A_{\iz j}\psi'(\theta^*_{\iz}(\M X) - \theta^*_{j}(\M X))} + \frac{1}{\sum_{i\in[n]}A_{i\jz}\psi'(\theta^*_{\jz}(\M X) - \theta^*_{i}(\M X))}\right\}\right),
\ee
which concludes the proof.
\qed

\subsection{Proof of Theorem~\ref{thm:oracle}}\label{pfthm:oracle}
We  focus on the case that $\iz$ and $\jz$ is not connected over $\M A$ for simplicity, and thus $\tilde{\sigma}(\M A)$ can be simplified to \eqref{tildeAdef}. This is an event with probability approaching 1 as $n\rightarrow \infty$ whenever $p=o(1)$. On the other hand, for a dense graph with $np = \Theta( n)$ and $\bds\theta^* \in \Theta$, we have $\Delta_{\iz}(\M x) = {O}( n^{-1})$ and $\Delta_{\iz}(\M x) = {O}( n^{-1})$ with high probability as $n\rightarrow \infty$, which are also negligible.

By $\pr(\Omega) > c_\Omega$ and \eqref{upperbound:A:supp}, we have
\bee\nonumber
\tilde{\sigma}(\M A) &= \E_{\M X}\left(\mathbb{I}(\M X\in\Omega)\left(\frac{1}{\sum_{j}A_{j\iz}\psi'(\theta^*_{\iz}(\M X) - \theta^*_{j}(\M X))} + \frac{1}{\sum_{i}A_{i\jz}\psi'(\theta^*_{\jz}(\M X) - \theta^*_{i}(\M X))}\right)\right)
\geq C(np)^{-1}.
\ee
On the other hand, by Lemma~\ref{lm:good-ER} we have
$
V_{\iz\jz}(\Omega)   \geq    L^{-1}\sigma(\M A) \geq C (npL)^{-1}.
$
We thus have,
\bee\label{oracle:target}
\frac{ |V_{\iz\jz}(\Omega) - \tilde{V}_{\iz\jz}(\Omega)| }{{V}_{\iz\jz}(\Omega)}
  = \frac{|\sigma({\M A}) - \tilde{\sigma}({\M A})|}{L\cdot  {V}_{\iz\jz}(\Omega)}  \leq    Cnp|\sigma({\M A}) - \tilde{\sigma}({\M A})|.
\ee
Thus we only need to bound $|\sigma({\M A}) - \tilde{\sigma}({\M A})|$. %
 We further have
\begin{align}\nonumber
 |\sigma({\M A}) - \tilde{\sigma}({\M A})|
 & =\left|(\M e_{\iz} - \M e_{\jz})^\T\E_{\M X}\Big(\mathbb{I}(\M X\in\Omega) \left( \pmb{\mathscr{L}}^\dagger(\M X) - \M D(\bds \Xi(\M X))^{-1}    \right)\Big)(\M e_{\iz} - \M e_{\jz})\right|
\\\label{sigmadiff:1}
&\leq 4\sup_{\M x\in\mathbb{X}}\|\pmb{\mathscr{L}}^\dagger(\M x) - \M D^{-1}(\bds \Xi(\M x))\|_{\infty} = O_P\left(C (np)^{-1}\left(\frac{1}{n^{1/3}} + \sqrt{\frac{\log n}{np}}\right)\right).
\end{align}
where the last rate is by applying Theorem~\ref{theorem:ER:Ldagger}. Thus the theorem is prove by combining \eqref{sigmadiff:1} with \eqref{oracle:target}.

\section{Proof of Theorem~\ref{thm:L2error:main}}\label{proofthm:L2error:main}

We first introduce some  preliminaries to facilitate our discussion. For our general theorem, we allow $\hat{\bds\theta}(\M x)$ to have some optimization error with respect to the exact minimum of $\mathcal{L}_n(\bds\theta)$, and we denote the marginal and conditional optimization errors for any function class $\mathcal{F}$ as
\bee\label{eq:DeltaA}
\Delta(\hat{\bds\theta}) = \mathbb{E} \Big(\mathcal{L}_n(\hat{\bds\theta}  ) - \inf_{\boldsymbol{\theta}  \in  \mathcal{G}({\mathcal{F}}_{}^n) }  \mathcal{L}_n({\boldsymbol{\theta}} )\Big),\quad \Delta_{\M A}(\hat{\bds\theta}) = \mathbb{E} \Big(\mathcal{L}_n(\hat{\bds\theta}  ) - \inf_{\boldsymbol{\theta}  \in  \mathcal{G}({\mathcal{F}}^n) }  \mathcal{L}_n({\boldsymbol{\theta}} )\mid \M A\Big).
\ee 
Here, the optimization errors measure how close the estimator achieves the global minimizer of the loss function $\mathcal{L}_n$. We consider the scenarios that   $\M A$  is given or randomly generated from Erd\H{o}s--R\'enyi graph in the first and second parts of Proposition~\ref{thm:L2error}, respectively, whose proof is given in Section \ref{sec:oracle}. The proof of Theorem~\ref{thm:L2error:main} is given in Section~\ref{proofthm:L2error:main}.

The rates of our estimator will depends on the complexity of the function class characterized by the covering number. 
 In particular, for any $\vartheta(\M x) \in \mathcal{F}$, there exists some $\bar{\vartheta}(\M x)\in\mathcal{F}_{\delta}$ such that $\|{\vartheta}(\M x) - \bar{\vartheta}(\M x)\|_{\infty}\leq \delta$. Among all function class having such property, we denote $\mathcal{F}_{\delta}$ as one of them that have the smallest size. We call $\mathcal{F}_{\delta}$  the  minimal $\delta$-covering of $\mathcal{F}$ with respect to the $\mathcal{L}_{\infty}$ metric. We also denote by ${\mathcal{N}_{\delta}(\mathcal{F})} = |\mathcal{F}_\delta|$ the $\delta$-covering number of $\mathcal{F}$.   
 
 Let $\bds \Psi(\M x ) = (\psi(\M x),  1- \psi(\M x))^\T$. In this section, we will use a compact notation that
 \[
  \log\left(\frac{\bds \Psi\big(\theta_i(\M x)\big)}{\bds \Psi\big(\theta'_i(\M x)\big)}\right) = \left(\log\left(\frac{  \psi\big(\theta_i(\M x)\big)}{ \psi\big(\theta'_i(\M x)\big)}\right), - \log\left(\frac{  \psi\big(\theta_i(\M x)\big)}{ \psi\big(\theta'_i(\M x)\big)}\right)\right)^\T.
 \]

The following proposition shows an oracle inequality of our estimator and its proof is deferred to Section~\ref{sec:oracle}.

\begin{proposition}\label{thm:L2error}
Assume  $\bds\theta^*(\M x)\in\Theta$ and $\M A$ is fixed satisfing $\mathcal{E}_{\mathrm{good}}$ in  \eqref{upperbound:A:supp}.  Denote the Kullback–Leibler divergence
\bee\label{eq:RA}
R_{\M A} \big({\bds\theta}, \bds\theta^*\big) = (n^2p)^{-1}\sum_{i > j}A_{ij}\E\left(\big\{\bds\Psi\big(\theta^*_i(\M X) - \theta^*_j(\M X)\big)\big\}^\T\log\left(\frac{\bds \Psi\big(\theta^*_i(\M X) - \theta^*_j(\M X)\big)}{\bds \Psi\big({\theta}_i(\M X) - {\theta}_j(\M X)\big)}\right)\right).
\ee
If ${\mathcal{N}_{\delta}(\mathcal{F})} \geq 2$, we have the following oracle inequality
\bee\label{oracle:E2:con}
 \mathcal{E}_2(\hat{\bds\theta},\bds\theta^*\mid \M A) \leq C \Bigg(\inf_{\bds \theta\in \mathcal{G}(\mathcal{F}^n)} R_{\M A}(\bds\theta,\bds\theta^*) + \Delta_{\M A} (\hat{\bds\theta}) + \delta + \frac{{n\log ({\mathcal{N}_{\delta}(\mathcal{F})})}   + \sqrt{n\delta L\log({\mathcal{N}_{\delta}(\mathcal{F})})}}{n^2pL}\Bigg).
\ee
% On the other hand, let $\gamma > 1$ be any pre-specified  constant, and $\M A$ is generated from  an Erd\H{o}s--R\'enyi graph with sampling probability $p$. When $n\geq \max\{3,(3/2)^\gamma\}$ and $np  \geq 52(\gamma + 1)(\log n )/3$,
% \bee\label{oracle:E2}
%  \mathcal{E}_2(\hat{\bds\theta},\bds\theta^*) 
%  \leq C \Bigg\{ \inf_{\bds \theta\in \mathcal{G}({\mathcal{F}}^n)} R_{\M A}(\bds\theta,\bds\theta^*)  + \Delta(\hat{\bds\theta})
% + \delta + \frac{{n\log ({\mathcal{N}_{\delta}(\mathcal{F})})}   + \sqrt{n\delta L\log({\mathcal{N}_{\delta}(\mathcal{F})})}}{n^2pL} + \frac{n^{-\gamma + 1}}{\log n}\Bigg\}.
% \ee
\end{proposition}
  
  We   upper bound $\inf_{\bds \theta\in \mathcal{G}(\mathcal{F}^n)} R_{\M A}(\bds\theta,\bds\theta^*)$ under $\mathcal{E}_{\text{good}}$.  Since $\bds\theta^*\in \Theta$, by the mean-value theorem, we have under $\mathcal{E}_{\text{good}}$,
\begin{align}\nonumber
\inf_{\bds \theta\in \mathcal{G}(\mathcal{F}^n)} R_{\M A}(\bds\theta,\bds\theta^*) &= (n^2p)^{-1} \inf_{\bds \theta\in \mathcal{G}(\mathcal{F}^n)} \sum_{i > j}A_{ij}\E_{\M X}\left(\bds\Psi^\T\big(\theta^*_i(\M X) - \theta^*_j(\M X)\big)\log\left(\frac{\bds \Psi\big(\theta^*_i(\M X) - \theta^*_j(\M X)\big)}{\bds \Psi\big({\theta}_i(\M X) - {\theta}_j(\M X)\big)}\right)\right)
\\\nonumber
&\leq C   (n^2p)^{-1}   \inf_{\bds \theta\in \mathcal{G}(\mathcal{F}^n)} \sum_{i > j}A_{ij} \max_{i\in[n]}\|\theta_i - \theta_i^*\|_{\infty}
\\\label{approx:thm1}
&\leq   C   \inf_{\bds \theta\in \mathcal{G}(\mathcal{F}^n)}   \max_{i\in[n]}\|\theta_i - \theta_i^*\|_{\infty}
=   C \cdot \mathrm{UAE}(\mathcal{F},\bds\theta^*),
\end{align}
where   the second inequality holds by \eqref{upperbound:A:supp}.  Since $\hat{\bds\theta}$ achieves the minimum of $\mathcal{L}_n$, we have $\Delta_{\M A}(\hat{\bds\theta}) = 0$. Following \eqref{oracle:E2:con},   we have
\bee
\label{oracle:E2:con:addition}
 \mathcal{E}_2(\hat{\bds\theta},\bds\theta^*\mid \M A) \leq C \Bigg(\mathrm{UAE}(\mathcal{F},\bds\theta^*) + \delta + \frac{{n\log ({\mathcal{N}_{\delta}(\mathcal{F})})}   + \sqrt{n\delta L\log({\mathcal{N}_{\delta}(\mathcal{F})})}}{n^2pL}\Bigg).
\ee
Following \eqref{oracle:E2:con:addition}, Assumption~\ref{ass:theta} holds under \eqref{asymptotic:c}, and thus the asymptotic normality follows from Theorem~\ref{theorem:LD:new} as desired. By Lemma~\ref{lm:good-ER}, $\mathcal{E}_{\mathrm{good}}$ happens with probability at least $1 - 2n^{-2}$, and thus the theorem is true with high probability for Erd\H{o}s--R\'enyi graph with sampling probability $p$.  \qed

\subsection{Proof of Proposition~\ref{thm:L2error}}\label{sec:oracle}
%

% Recalling that we define $\mathcal{G}(\mathcal{F}^n)$ in \eqref{proposed:estimator}, we define the normalized $\delta$-covering for $\mathcal{G}(\mathcal{F}^n)$:
% \bee\label{def:G}
%   \mathcal{G}({\mathcal{F}_{\delta}^n}) = \Big\{\bar{\bds\theta}(\M x) = \bar{{\bds\vartheta}}(\M x) - \bds 1 \bds 1^\T\bar{{\bds\vartheta}}(\M x)/n \mid  \bar{{\bds\vartheta}}(\M x)\in \mathcal{F}^n_\delta\Big\}.
% \ee 
%  It is easy to see $|\mathcal{G}({\mathcal{F}_{\delta}^n})| \leq {\mathcal{N}^n_{\delta}(\mathcal{F})}$. We also write all elements in $\mathcal{G}({\mathcal{F}_{\delta}^n})$ as $\mathcal{G}({\mathcal{F}_{\delta}^n}) = \big\{{\bds\theta}^{(1)}(\M x),\ldots,{\bds\theta}^{(|\mathcal{G}({\mathcal{F}_{\delta}^n})|)}(\M x)\big\}$.
Our definition of $\mathcal{F}^n_{\delta}$  implies that for any $\hat{\bds\vartheta}(\M x)\in\mathcal{F}^n$, there exists a   ${\bds\vartheta}^{(\nu)}(\M x)\in\mathcal{F}^n_{\delta}$, such that $\|\hat{\vartheta}_i  - {{\vartheta}}^{(\nu)}_i \|_{\infty} \le \delta$ for any $i\in[n]$. In particular, for $\bds\theta^{(\nu)}(\M x) = {\bds\vartheta}^{(\nu)}(\M x) - \bds 1 \bds 1^\T \bds\vartheta^{(\nu)}(\M x)/n$, the mean value theorem  implies for any $(i,j)$ and ${\M x}\in\mathbb{X}$,
\begin{align}\nonumber
&\Big|\log\Big(\psi\big(\hat{\theta}_i({\M x}) - \hat{\theta}_j({\M x})\big)\Big) - \log\Big(\psi\big({\theta}^{(\nu)}_i({\M x}) - {\theta}^{(\nu)}_j({\M x})\big)\Big) \Big|
 =\Big|\log\Big(\psi\big(\hat{\vartheta}_i({\M x}) - \hat{\vartheta}_j({\M x})\big)\Big) - \log\Big(\psi\big({{\vartheta}}^{(\nu)}_i({\M x}) - {{\vartheta}}^{(\nu)}_j({\M x})\big)\Big) \Big|
\\\nonumber
&=\frac{1}{\exp(h_{ij}({\M x})) + 1}\Big|\hat{\vartheta}_i({\M x}) - \hat{\vartheta}_j({\M x}) - {{\vartheta}}^{(\nu)}_i({\M x}) + {{\vartheta}}^{(\nu)}_j({\M x})\Big|
% \\\nonumber
% & \leq \frac{1}{\exp(h_{ij}({\M x})) + 1}\Big\{\Big|\hat{\vartheta}_i({\M x})  - {{\vartheta}}^{(\nu)}_i({\M x}) \Big| +  \Big|\hat{\vartheta}_j({\M x}) - {{\vartheta}}^{(\nu)}_j({\M x})\Big|\Big\} 
 \leq C\delta,
\end{align}
where $h_{ij}({\M x})$ is a function between $\hat{\vartheta}_i({\M x}) - \hat{\vartheta}_j({\M x})$ and ${{\vartheta}}^{(\nu)}_i({\M x}) - {{\vartheta}}^{(\nu)}_j({\M x})$, and since $\hat{\vartheta}_i(\M x),\hat{\vartheta}_j(\M x),{{\vartheta}}^{(\nu)}_i(\M x),{{\vartheta}}^{(\nu)}_j(\M x)\in\mathcal{F}$ are  functions uniformly bounded, we have $({\exp[h_{ij}({\M x})] + 1})^{-1}$ is upper bounded by some constant. Applying similar argument to $\log(1-t)$, we have for any $(i,j)$, 
\bee\label{coveringerror}
\left\|\log\Big({\bds \Psi\big(\hat{\theta}_i(\M x)  - \hat{\theta}_j(\M x) \big)}\Big) - \log \Big({\bds \Psi\big({\theta}^{(\nu)}_i  (\M x)- {\theta}^{(\nu)}_j(\M x) \big)}\Big)\right\|_{\infty} \leq C \delta.
\ee 
\par
Now let   
$
    \tilde{\mathcal{D}}_{n} = \big\{\tilde{\M X}_{ij{\ell}}, \tilde{y}_{ij{\ell}} = \tilde{y}_{ij}(\tilde{\M X}_{ij{\ell}}) \mid (i,j)\in[n]^2, i > j,  {\ell} \in [L]\big\}
$
 be an independent copy of $\mathcal{D}_{n}$, and $\tilde{\M Y}_{ij{\ell}} = (\tilde{y}_{ij{\ell}}, 1 - \tilde{y}_{ij{\ell}})^\T$. Define
 \bee\nonumber
 &D_{\M A} \big({\bds\theta}, \bds\theta'\big) = \sum_{i > j}  {A}_{ij} \left( \sum_{{\ell} =1}^L  {\M Y}^\T_{ij{\ell}} \log\Bigg(\frac{\bds \Psi\big({\theta}'_i( {\M X}_{ij{\ell}}) - {\theta}'_j( {\M X}_{ij{\ell}})\big)}{\bds \Psi\big({\theta}_i( {\M X}_{ij{\ell}}) - {\theta}_j( {\M X}_{ij{\ell}})\big)} \Bigg)\right),\\
&\tilde{D}_{\M A} \big({\bds\theta}, \bds\theta'\big) = \sum_{i > j}  {A}_{ij} \left( \sum_{{\ell} =1}^L \tilde{\M Y}^\T_{ij{\ell}} \log\Bigg(\frac{\bds \Psi\big({\theta}'_i(\tilde{\M X}_{ij{\ell}}) - {\theta}'_j(\tilde{\M X}_{ij{\ell}})\big)}{\bds \Psi\big({\theta}_i(\tilde{\M X}_{ij{\ell}}) - {\theta}_j(\tilde{\M X}_{ij{\ell}})\big)} \Bigg)\right).
\ee
Denote $R_{n,\M A}\big(\hat{\bds\theta}, \bds\theta^*\big) = ({n^2pL})^{-1} \E\tilde{D}_{\M A} \big(\hat{\bds\theta}, \bds\theta\big)$ as the empirical KL-divergence, we aim to bound the difference between the empirical and population KL-divergences as
\begin{align}\nonumber
  &\left|R_{\M A}\big(\hat{\bds\theta}, \bds\theta^*\big) - R_{n,\M A}\big(\hat{\bds\theta}, \bds\theta^*\big)\right|
= \frac{1}{n^2pL}\left|\E[\tilde{D}_{\M A}\big(\hat{\bds\theta}, \bds\theta^*\big)] - \E[D_{\M A}\big(\hat{\bds\theta}, \bds\theta^*\big)]\right|
  \\\nonumber
  &\leq \frac{1}{n^2pL}\left|\E[\tilde{D}_{\M A}\big(\hat{\bds\theta}, {\bds\theta}^{(\nu)}\big)]\right|
  + \frac{1}{n^2pL}\left|\E[D_{\M A}\big(\hat{\bds\theta}, {\bds\theta}^{(\nu)}\big)]\right|
 + \frac{1}{n^2pL}\left|\E[\tilde{D}_{\M A}\big({\bds\theta}^{(\nu)}, \bds\theta^*\big)] - \E[D_{\M A}\big({\bds\theta}^{(\nu)}, \bds\theta^*\big)]\right|
  \\\label{eq:R-R}
  &\leq C\delta
  + \frac{1}{n^2pL}\left|\E[\tilde{D}_{\M A}\big({\bds\theta}^{(\nu)}, \bds\theta^*\big)] - \E[D_{\M A}\big({\bds\theta}^{(\nu)}, \bds\theta^*\big)]\right|
  % \\\nonumber
  % &= C\delta + \left|\E \Bigg( \frac{1}{n^2pL}\sum_{i > j}  {A}_{ij}\sum_{{\ell} =1}^L \left(\tilde{\Gamma}_{ij{\ell}}^{(\nu)} - \Gamma_{ij{\ell}}^{(\nu)}\right)\Bigg)\right|
  % \\\label{main:decompose:oracle}
  % &= C\delta + \left|\E \left( \frac{1}{n^2pL}\sum_{i > j}  D_{ij}^{(\nu)}\right)\right|,
  \end{align}
  where the second inequality is by \eqref{coveringerror}. Define $\eta_{\M A}^{(k)} = {R^{1/2}_{\M A}\Big({\bds\theta}^{(k)},\bds\theta^*\Big) }  + \sqrt{\log ({\mathcal{N}_{\delta}(\mathcal{F}^n)})/ L}$ and
$$
D_1 = \frac{1}{n^2pL} \E \left[\max_{k} |\tilde{D}_{\M A}\big({\bds\theta}^{(k)}, \bds\theta^*\big)-D_{\M A}\big({\bds\theta}^{(k)}, \bds\theta^*\big)|/\eta_{\M A}^{(k)}\right],
D_2 = \frac{1}{n^2pL} \E \left[\max_{k} |\big(\tilde{D}_{\M A}\big({\bds\theta}^{(k)}, \bds\theta^*\big)-D_{\M A}\big({\bds\theta}^{(k)}, \bds\theta^*\big)\big)/\eta_{\M A}^{(k)}|^2\right].
$$
Then we have
\bee\label{bound:R2}
&\frac{1}{n^2pL}\left|\E[\tilde{D}_{\M A}\big({\bds\theta}^{(\nu)}, \bds\theta^*\big)] - \E[D_{\M A}\big({\bds\theta}^{(\nu)}, \bds\theta^*\big)]\right|  
\le \frac{1}{n^2pL} \E \left[ \max_{k} \left|\frac{\tilde{D}_{\M A}\big({\bds\theta}^{(k)}, \bds\theta^*\big)-D_{\M A}\big({\bds\theta}^{(k)}, \bds\theta^*\big)}{\eta_{\M A}^{(k)}}\right| \eta_{\M A}^{(\nu)}\right]  \\
&\le \frac{1}{n^2pL} \E \left[ \max_{k} \left|\frac{\tilde{D}_{\M A}\big({\bds\theta}^{(k)}, \bds\theta^*\big)-D_{\M A}\big({\bds\theta}^{(k)}, \bds\theta^*\big)}{\eta_{\M A}^{(k)}}\right| {{R}^{1/2}_{\M A}\left(\hat{\bds\theta},\bds\theta^*\right) }\right]+ \left(\sqrt{\frac{\log ({\mathcal{N}_{\delta}(\mathcal{F}^n)})}{L}} + \sqrt{2C\delta}\right)D_1\\
&\leq \frac{1}{n\sqrt{pL}} \sqrt{D_2{{R}_{\M A}\left(\hat{\bds\theta},\bds\theta^*\right) }} + \left(\sqrt{\frac{\log ({\mathcal{N}_{\delta}(\mathcal{F}^n)})}{L}} + \sqrt{2C\delta}\right)D_1, 
\ee
where the second inequality is by
$
{R^{1/2}_{\M A}({\bds\theta}^{(\nu)},\bds\theta^* ) }  
  \leq  { R^{1/2}_{\M A}(\hat{\bds\theta},\bds\theta^* )} + \sqrt{2C\delta }
  $
and last inequality is by Cauchy-Schawrz inequality. The following lemma provides the upper bounds of $D_1, D_2$ whose proof is deferred to Section~\ref{pf:lm:D1D2}.

\begin{lemma}\label{lm:D1D2}
Under the same conditions as Proposition~\ref{thm:L2error}, we have
\[
D_1 \le \frac{C}{n^2p\sqrt{L}}\left\{ {\sqrt{  \log({\mathcal{N}_{\delta}(\mathcal{F}^n)})}} +  \frac{2}{\sqrt{\log({\mathcal{N}_{\delta}(\mathcal{F}^n)})}}\right\}, D_2 \le C \frac{\log({\mathcal{N}_{\delta}(\mathcal{F}^n)})+1}{n^2p}.
\]
\end{lemma}

Applying Lemma~\ref{lm:D1D2} to \eqref{bound:R2} and \eqref{eq:R-R}, there exists a sufficiently large constant $C$ such that
\bee\nonumber
&\frac{1}{C}\left|R_{\M A}\big(\hat{\bds\theta}, \bds\theta^*\big) - R_{n,\M A}\big(\hat{\bds\theta}, \bds\theta^*\big)\right| \leq   \frac{\sqrt{\log({\mathcal{N}_{\delta}(\mathcal{F}^n)}){{R}_{\M A}\left(\hat{\bds\theta},\bds\theta^*\right) }}}{n^2p\sqrt{L}}  
 + \delta + \frac{{\log ({\mathcal{N}_{\delta}(\mathcal{F}^n)})}   + \sqrt{\delta L\log({\mathcal{N}_{\delta}(\mathcal{F}^n)})}}{n^2pL}.
 \ee

 By Lemma~\ref{lm:inequality}, 
%  with 
%  \bee\nonumber
%  \epsilon &= 0.5,\, a = R_{\M A}\big(\hat{\bds\theta},\, \bds\theta^*\big), b = R_{n,\M A}(\hat{\bds\theta}, \bds\theta^*), \,c = 2^{-1}C \log({\mathcal{N}_{\delta}(\mathcal{F}^n)})/(n^4p^2L),
%  \\
%  d &=C\left\{ \delta +  {{\log ({\mathcal{N}_{\delta}(\mathcal{F}^n)})/({n^2pL})}   + \sqrt{\delta L\log({\mathcal{N}_{\delta}(\mathcal{F}^n)})}}/({n^2pL})\right\},
%  \ee 
 we have
 \bee\nonumber
 \frac{1}{C}\left|R_{\M A}\big(\hat{\bds\theta}, \bds\theta^*\big) \right|\leq  R_{n,\M A}(\hat{\bds\theta}, \bds\theta^*) + \delta + \frac{{\log ({\mathcal{N}_{\delta}(\mathcal{F}^n)})}   + \sqrt{\delta L\log({\mathcal{N}_{\delta}(\mathcal{F}^n)})}}{n^2pL} + \frac{\log({\mathcal{N}_{\delta}(\mathcal{F}^n)})}{n^4p^2{L}},
 \ee
 for some sufficiently large constant ${C} > 0$. Let $\bds\theta^\dagger\in \mathcal{G}(\mathcal{F}^n)$ be some fixed function  such that 
 $
 R_{\M A}(\bds\theta^\dagger,\bds\theta^*) = \inf_{\bds \theta\in \mathcal{G}(\mathcal{F}^n)} R_{\M A}(\bds\theta,\bds\theta^*).
 $ Recall that $\Delta_{\M A}$ defined in \eqref{eq:DeltaA}, we have $\Delta_{\M A}(\bds\theta^\dagger) \ge 0$, and thus
 \begin{align}\nonumber
 R_{n,\M A}\big(\hat{\bds\theta}, \bds\theta^*\big)  &= \E\big(\mathcal{L}_n(\hat{\bds\theta}) - \mathcal{L}_n(\bds\theta^*)\mid\M A\big)
 \leq \E\big(\mathcal{L}_n(\hat{\bds\theta})\mid \M A\big) + \Delta_{\M A}(\bds\theta^\dagger) - \E\big(\mathcal{L}_n(\bds\theta^*)\mid\M A\big)
 \\\nonumber
 & = \E\big(\mathcal{L}_n({\bds\theta}^\dagger)\mid\M A\big) + \Delta_{\M A}(\hat{\bds\theta}) - \E\big(\mathcal{L}_n(\bds\theta^*)\mid\M A\big)
=R_{\M A}(\bds\theta^\dagger,\bds\theta^*) + \Delta_{\M A}(\hat{\bds\theta}) 
=  \inf_{\bds \theta\in \mathcal{G}(\mathcal{F}^n)} R_{\M A}(\bds\theta,\bds\theta^*) + \Delta_{\M A}(\hat{\bds\theta}).
 \end{align}
 Summarizing the above results, we conclude that there exists some constant $C$,
 \bee\label{oracle:in:RA}
 \frac{1}{C}\left|R_{\M A}\big(\hat{\bds\theta}, \bds\theta^*\big)\right| \leq   \inf_{\bds \theta\in \mathcal{G}(\mathcal{F}^n)} R_{\M A}(\bds\theta,\bds\theta^*) + \Delta_{\M A} (\hat{\bds\theta}) + \delta + \frac{{n\log ({\mathcal{N}_{\delta}(\mathcal{F})})}   + \sqrt{n\delta L\log({\mathcal{N}_{\delta}(\mathcal{F})})}}{n^2pL},
 \ee
 noting here we use the fact that ${\log({\mathcal{N}_{\delta}(\mathcal{F}^n)})}   \leq n\log ({\mathcal{N}_{\delta}(\mathcal{F})})$ as $np  \geq (\gamma + 1)52(\log n)/3$ and $n\geq 3$.
 On the other side, by mean value theorem, there exists $\tilde{\bds \theta}$ between $\bds\theta^* $ and $\bds\theta$ such that  
\bee\nonumber
R_{\M A}\big(\hat{\bds\theta}, \bds\theta^*\big)= \mathcal{L}_{\M A}(\hat{\bds\theta}) - \mathcal{L}_{\M A}(\bds\theta^*) & =    \frac{1}{2n^2p}\E_{\M X}\Big(\big(\hat{\bds\theta}(\M X) - \bds\theta^*(\M X)\big)^\T \pmb{\mathscr{L}}(\M x\mid \tilde{\bds\theta}) \big(\hat{\bds\theta}(\M X) - \bds\theta^*(\M X)\big)\Big)
\geq c\mathcal{E}_2(\hat{\bds\theta},\bds\theta^*\mid \M A).
\ee
where the last inequality is by Lemma~\ref{lemma:eigen:lower:bound}. Combining the above inequality with \eqref{oracle:in:RA}, we obtain \eqref{oracle:E2:con}.

\subsubsection{Proof of Lemma~\ref{lm:D1D2}}\label{pf:lm:D1D2}
Define the centered random variable for $i,j \in [n]$, $k \in [\mathcal{N}^n_{\delta}(\mathcal{F}^n)]$, 
\bee\nonumber
 &D_{ij}^{(k)}=   {A}_{ij} \left( \sum_{{\ell} =1}^L  {\M Y}^\T_{ij{\ell}} \log\Bigg(\frac{\bds \Psi\big({\theta}^*_i( {\M X}_{ij{\ell}}) - {\theta}^*_j( {\M X}_{ij{\ell}})\big)}{\bds \Psi\big({\theta}^{(k)}_i( {\M X}_{ij{\ell}}) - {\theta}^{(k)}_j( {\M X}_{ij{\ell}})\big)} \Bigg)\right) - {A}_{ij} \left( \sum_{{\ell} =1}^L  {\tilde{\M Y}}^\T_{ij{\ell}} \log\Bigg(\frac{\bds \Psi\big({\theta}^*_i( \tilde{\M X}_{ij{\ell}}) - {\theta}^*_j( \tilde{\M X}_{ij{\ell}})\big)}{\bds \Psi\big({\theta}^{(k)}_i( \tilde{\M X}_{ij{\ell}}) - {\theta}^{(k)}_j( \tilde{\M X}_{ij{\ell}})\big)} \Bigg)\right).
\ee

% Since $\theta^*_i,\theta^*_j,{\theta}^{(k)}_i,{\theta}^{(k)}_j$ all belong to some bounded functional class, we have that there exists some  constant $C$ such that
% \bee\label{def:B'}
% C \geq \sup_{{\M x} \in \mathbb{X}} \left\|\log\Bigg(\frac{\bds \Psi\big({\theta}^*_i({\M x}) - {\theta}^*_j({\M x})\big)}{\bds \Psi\big({\theta}^{(k)}_i({\M x}) - {\theta}^{(k)}_j({\M x})\big)} \Bigg)\right\|_{\infty}\vee 2.
% \ee 
To apply the Bernstein's inequality in Lemma~\ref{lm:bern}(I), we bound the moment as
\begin{align}\nonumber
&\E  \left|D_{ij}^{(k)}\right|^m  
\leq 2^{m} \E_{\M X}\left(\big\{\bds \Psi\big({\theta}^*_i({\M X}) - {\theta}^*_j({\M X})\big)\big\}^\T\left| \log\bigg(\frac{\bds \Psi\big({\theta}^*_i({\M X}) - {\theta}^*_j({\M X})\big)}{\bds \Psi\big({\theta}^{(k)}_i({\M X}) - {\theta}^{(k)}_j({\M X})\big)} \Bigg) \right|^m\right)
\leq C^{m} m!  R_{ij}\left({\bds\theta}^{(k)},\bds \theta^*\right),
\end{align}
where the second inequality holds by Lemma~\ref{lm:KL} and
\bee\nonumber
R_{ij} \big({\bds\theta}, \bds\theta^*\big) =  \E\left(\big\{\bds\Psi\big(\theta^*_i(\M X) - \theta^*_j(\M X)\big)\big\}^\T\log\left(\frac{\bds \Psi\big(\theta^*_i(\M X) - \theta^*_j(\M X)\big)}{\bds \Psi\big({\theta}_i(\M X) - {\theta}_j(\M X)\big)}\right)\right).
\ee
Then by Lemma~\ref{lm:bern}(I), we have, 
\begin{align}\nonumber
  \textstyle  \pr\left(\left|\sum_{i > j}D^{(k)}_{ij}\Big/ \eta_{\M A}^{(k)}\right| \geq t\right) &\leq2\exp\left(-\frac{t^2}{C LR_{\M A}\left({\bds\theta}^{(k)},\bds\theta^*\right)/\big (\eta_{\M A}^{(k)}
\big)^2 + C/\eta_{\M A}^{(k)}}\right)\leq 2\exp\left(-\frac{t^2}{C L + C t/\eta_{\M A}^{(k)}}\right).
\end{align}
By union bound and $\E[Z] = \int_0^\infty \pr(Z \ge t)dt$ for non-negative $Z$, we have
\bee\nonumber
 D_1 & \textstyle  \le  {\mathcal{N}^n_{\delta}(\mathcal{F}^n)} \int_0^{\infty}\pr\left(\left|\sum_{i > j}D^{(k)}_{ij}\Big/ \eta_{\M A}^{(k)}\right| \geq t\right)dt
\\
& \leq \frac{1}{n^2pL}\sqrt{C\log({\mathcal{N}_{\delta}(\mathcal{F}^n)})L}   + \frac{2{\mathcal{N}^n_{\delta}(\mathcal{F}^n)} }{n^2pL}\int_{\sqrt{C\log({\mathcal{N}_{\delta}(\mathcal{F}^n)})L}}^\infty\exp\left(- {t\sqrt{\log({\mathcal{N}_{\delta}(\mathcal{F}^n)})/L}}\right) dt
\\
&\leq \frac{C}{n^2p\sqrt{L}}\left( {\sqrt{  \log({\mathcal{N}_{\delta}(\mathcal{F}^n)})}} +   {2}/{\sqrt{\log({\mathcal{N}_{\delta}(\mathcal{F}^n)})}}\right).
\ee
We can also bound $D_2$ following the similar method
\begin{align}\nonumber
  D_2 &\le  \textstyle  {\mathcal{N}^n_{\delta}(\mathcal{F}^n)} \int_0^{\infty}\pr\left(\left|\sum_{i > j}D^{(k)}_{ij}\Big/ \eta_{\M A}^{(k)}\right| \geq \sqrt{t}\right)dt
 \\\nonumber
&\leq \frac{1}{n^2pL}\bigg(\log({\mathcal{N}_{\delta}(\mathcal{F}^n)}) C L + 2{\mathcal{N}^n_{\delta}(\mathcal{F}^n)}\int_{C\log({\mathcal{N}_{\delta}(\mathcal{F}^n)})L}^\infty \exp\left(- {\sqrt{\log({\mathcal{N}_{\delta}(\mathcal{F}^n)})/Lt}} \right)dt\bigg)
\\\nonumber
&\leq \frac{C}{n^2p}\left(\log({\mathcal{N}_{\delta}(\mathcal{F}^n)})   +1  \right),
\end{align}
which holds by ${\mathcal{N}_{\delta}(\mathcal{F}^n)} > e$ and $\int_{a}^\infty\exp(-b\sqrt{t})dt = 2(\sqrt{a}b + 1)\exp(-\sqrt{a}b)/b^2$.

\subsection{Technical lemmas for ReLU-DNN}\label{sec:DNN}
In this section, we formally introduce the ReLU-DNN function class and its properties helping us to derive \eqref{uaelog} and \eqref{oracle:E2:3} in Section~\ref{sec:uaelog}. The ReLU activation function is 
$\sigma(v) = \max(v,0)$ for   $v\in\R$. Then for any $a\in\mathbb{Z}^+$ and $\M z,\M v \in \R^{a}$,  Let  the shifted activation function be
$
\sigma_{\M v}\left(\M z\right) =(
\sigma(z_1 - v_1),\ldots,
\sigma(z_a - v_a)
)^\T
$ 
where $z_i$ and $v_i$ represent the $i$-th entries of $\M v$ and $\M z$, respectively, and the ReLU neural network is
\bee\label{dnn:structure}
g({\M x})  = \M J_{M}\sigma_{\M v_M}\circ \M J_{M-1}\sigma_{\M v_{M -1}}\circ\cdots\M J_1\sigma_{\M v_1}\circ \M J_0 {\M x}.
\ee 
Here $M$ is the number of hidden layers, $\M v_m\in\R^{p_i}$ is the shift vector, $\M J_m\in \R^{p_{i + 1}\times p_i}$ is the weight matrix, and $\M p = (p_0,\ldots,p_{M + 1})$  denotes the dimensions of  weight matrices.  All parameters defining the ReLU neural network in \eqref{dnn:structure} are summarized as 
$$
\M\Theta_{\mathcal{D}} = \left\{\M J_0,\M J_{m},\M v_m\text{ for all }m\in[M]\right\},
$$
 and its $\ell_0$ and $\ell_{\infty}$ norms are defined as
$$
|\M\Theta_{\mathcal{D}}|_{\infty} =\max_{m\in\{0\}\cup[M],\, m'\in[M]} \{|\M J_{m}|_{\infty},|\M v_{m'}|_\infty\},\quad |\M\Theta_{\mathcal{D}}|_{0}  = \sum_{m = 0}^M|\M J_m|_0 + \sum_{m' = 1}^M|\M v_{m'}|_0.
$$
For any matrix $\M M$, we let $|\M M|_{\infty}$ be its elementwise max norm, and $|\M M|_{0}$ be the number of non-zero entries in $\M M$. In summary,  the function class of ReLU-DNNs with sparsity  $s > 0$ and magnitude $F > 0$ is 
\bee\label{def:FD}
\mathcal{F}_{\mathcal{D}}(M,\M p, s,F) = \Big\{ f:\RR^{d} \mapsto \RR \,\Big|\,&  f({\M x})\text{ in the form  \eqref{dnn:structure}}, |\M\Theta_{\mathcal{D}}|_{\infty} < 1, |\M\Theta_{\mathcal{D}}|_0 < s,\|f\|_{\infty} < F \Big\}.
\ee
For simplicity, we sometimes write $\mathcal{F}_{\mathcal{D}} = \mathcal{F}_{\mathcal{D}}(M,\M p, s,F) $ when it does not cause  any confusion.

Lemma~\ref{complex:DNN} bounds the complexity of classes $\mathcal{F}_{\mathcal{D}}$ \citep[Remark~5]{schmidt2020nonparametric}. Lemma~\ref{approx:DNN} quantifies the approximation power of ReLU-DNN towards the function class $\mathscr{C}$    in Definition~\ref{def:commodel}. One can follow the same argument as in  \citet[Proof of Theorem~1]{schmidt2020nonparametric} to show Lemma~\ref{approx:DNN}. Note that the $L$ and $n$ in \citet{schmidt2020nonparametric} are our $M$ and $N$, respectively.

\begin{lemma}\label{complex:DNN}
Recall $\mathcal{F}_{\mathcal{D}}$ is the ReLU-DNN function class as defined in Section~\ref{sec:DNN}. For any $\delta > 0$, $$\log \big(\mathcal{N}_{\delta}(\mathcal{F}_{\mathcal{D}})) \leq (s + 1)\log (2^{2M + 5}\delta^{-1}(M + 1)p_0^2p_{M + 1}^2s^{2M} \big).$$
\end{lemma} 
\begin{lemma}\label{approx:DNN}
Let $f^*(\M x)\in \mathscr{C}(q,\M d,\M t,\bds \beta,C)$, and denote:
$$
(\beta^*,t^*) = \argmin_{u  = 0,\ldots,q}\tilde{\beta}_u / t_u, \quad \psi_{N} =  N ^{-\frac{2\beta^*}{2\beta^* + t^*}},
$$ 
for any $m\in\mathbb{N}_+$. Now suppose that, for some universal constants $C_1$-$C_5 > 0$, the parameters for $\mathcal{F}_{D} $ satisfy the following conditions:
\begin{itemize}
\item[(i)] $F \geq \max\{C_1,1\}$,
\item[(ii)] $\sum_{u = 0}^q\log_2(\max(4t_{u} ,4\beta_{u}))\log_2(N)\leq M \leq C_2 N\psi_{N} $,
\item[(iii)] $ N\psi_{N}\leq C_3 \min_{m\in[M]}p_m$,
\item[(iv)] $s\in[C_4 N\psi_{N}\log(Nd),\, C_5 N\psi_{N}\log(Nd)]$.
\end{itemize}
We then have
\bee\nonumber
\inf_{f (\M x)\in \mathcal{F}_{D} }\|f(\M x) - f^*(\M x)\|_{\infty} \leq  C' \psi_N,
\ee
where $C' > 0$ is a constant depending on $C_1$-$C_5$ and $q,\M d, \M t, \bds\beta,F.$  
\end{lemma}

\subsection{Derivations of \eqref{uaelog} and \eqref{oracle:E2:3}}\label{sec:uaelog}
{\color{black}
Combining \eqref{oracle:E2:con:addition} with Lemma \ref{complex:DNN}, and taking $\delta = (npL)^{-1}$,  %
  we have
$$
\log(\mathcal{N}_{\delta}(\mathcal{F}_{\mathcal{D}})) \lesssim %
\log(npL)\Big\{ \log^2(npL) + \log(\delta^{-1}) \Big\}(npL)^{ \frac{t^*}{2\beta^* +  t^*}},
$$
 \bee\label{oracle:E2:1}
 \mathcal{E}_2(\hat{\bds\theta},\bds\theta^*\mid \M A) 
  \leq C \left\{\mathrm{UAE}\left(\mathcal{F}_{\mathcal{D}} ,\bds \theta^*\right)    +  \frac{s\log(npLd) + sM\log(s + 1)}{npL}\right\},
 \ee
with probability at least $  1 - 3 n^{-{2}}$.  Next, we impose the following assumptions following \citet[Theorem 1]{schmidt2020nonparametric} with some universal constants $c_1,C_1 > 0$:\begin{equation}\label{condition:thm:theta}
\begin{split}
\text{(i).}&\,\,  F \geq \max(C_1,1),\\
\text{(ii).}&\,\,  \sum_{u = 1}^q\log_2(4t_{u} + 4\beta_{u})\log_2(npL)\leq M \leq C_1 \min\left(  (npL)^{1-\frac{2\beta^*}{2\beta^* + t^*}}, \sum_{u = 1}^q\log_2(4t_{u} + 4\beta_{u})\log_2(npL) \right) ,\\
\text{(iii).}&\,\,  (npL)^{1-\frac{2\beta^*}{2\beta^* + t^*}}\leq C_1 \min_{m\in[M]}p_m,\\
\text{(iv).}&\,\, s\in\Big(c_1 (npL)^{1-\frac{2\beta^*}{2\beta^* + t^*}}\log(npLd),\,C_1 (npL)^{1-\frac{2\beta^*}{2\beta^* + t^*}}\log(npLd)\Big),\\
\text{(v).} &\,\,\Delta(\hat{\bds\theta}) \leq C_1 (npL)^{-\frac{2\beta^*}{2\beta^* + t^*}}\log^2(npLd),
\\
\text{(v)$^*$.}&\,\,  \Delta_{\M A}(\hat{\bds\theta}) \leq C_1(npL)^{-\frac{2\beta^*}{2\beta^* + t^*}} \log^3(npL),\text{ with probability approaching $1$ as $n\rightarrow \infty$}.
\end{split}
\end{equation}
When all conditions in \eqref{condition:thm:theta} hold, we have for all $\theta_1^*(\M x),\ldots,\theta_n^*(\M x) \in \mathscr{C}(q,\M d,\M t,\bds \beta,\bar C)$, there exist $\tilde\vartheta_1(\M x),\ldots,\tilde\vartheta_n(\M x) \in \mathcal{F}_{\mathcal{D}}$ such that 
\bee\nonumber
\|\theta_i^* - \tilde\vartheta_i\|_{\infty} \leq C (npL)^{-\frac{2\beta^*}{2\beta^* + t^*}},\quad \text{for all }i\in[n],
\ee 
by Lemma~\ref{approx:DNN} when we take $N = npL$ therein. Let $\tilde{\bds\theta} = \tilde{\bds\vartheta} - \M 1 \M 1^\T\tilde{\bds\vartheta}/n \in \mathcal{G}(\mathcal{F}_{\mathcal{D}}^n)$. Recall $\M 1^\T \bds\theta^* = \M 0$. We  have
\begin{align}\nonumber
\left\|\theta_i^* - \tilde\theta_i\right\|_{\infty}  &= \left\| \left(\theta_i^* - \tilde\vartheta_i\right) -  \left(\sum_{j = 1}^n\theta_j^*/n -\sum_{j = 1}^n \tilde\vartheta_j/n\right)\right\|_{\infty}
\leq\left\|  \theta_i^* - \tilde\vartheta_i \right\|_{\infty}  + \sum_{j = 1}^n\left\| \theta_j^*  -  \tilde\vartheta_j  \right\|_{\infty} /n
 \le C (npL)^{-\frac{2\beta^*}{2\beta^* + t^*}}. 
\end{align}
Thus we conclude 
\bee\label{aebound:theta}
\mathrm{UAE}(\mathcal{F}_{\mathcal{D}},\bds\theta^*) = \inf_{\bds \theta\in \mathcal{G}(\mathcal{F}_{\mathcal{D}}^n)}   \max_{i\in[n]}\|\theta_i - \theta_i^*\|_{\infty} 
\leq \max_{i\in[n]}\|\tilde{\theta}_i - \theta_i^*\|_{\infty}
\leq 2 C (npL)^{-\frac{2\beta^*}{2\beta^* + t^*}}.
\ee
Then, by Conditions (ii) and (iv) in \eqref{condition:thm:theta} we have
\begin{align}\nonumber
\frac{s\log(npLd) + sM\log(s + 1)}{npL} \leq C \log^3(npL)  \cdot (npL)^{-\frac{2\beta^*}{2\beta^* + t^*}}.
\end{align} 
Combining  the above result with \eqref{oracle:E2:1}, we can have the bound in \eqref{oracle:E2:3}.

\section{Technical Proofs for Section~\ref{sec:extend}}\label{sec:pf:multiple}

In this section, we provide proofs for theorems in Section~\ref{sec:extend}. In specific, we will prove Theorem~\ref{thm:bootstrap} in Section~\ref{pf:thm:bootstrap} and prove Theorem~\ref{thm:disshift} in Section~\ref{pf:thm:disshift}.

\subsection{Proof of Theorem~\ref{thm:bootstrap}}\label{pf:thm:bootstrap}
Recalling in Section~\ref{sec:uniform}, we denote $\mathcal{W}_T = \{(i_1,j_1,\Omega_1),\ldots,(i_T,j_T,\Omega_T)\}$, $\hat{\fav}_{\mathcal{W}_T}=(\hat{\fav}_{i_1j_1}(\Omega_1),\ldots,\hat{\fav}_{i_Tj_T}(\Omega_T))$, where $\hat{\fav}_{i_tj_t}(\Omega_t)$ is given in \eqref{eq:ql}. We had defined the test statistic $T_{\mathcal{W}_T} = \max_{t\in[T]} \sqrt{npL}\left\{\hat{\fav}_{i_tj_t}(\Omega_t) - {\fav}_{i_tj_t}(\Omega_t)\right\}$ in \eqref{eq:TWT}.
In addition, we define the summand  
\[
\textstyle  \hat{\fav}_{\mathcal{W}_T,\ell}=(\hat{\fav}_{i_1j_1\ell}(\Omega_1),\ldots,\hat{\fav}_{i_Tj_T\ell}(\Omega_T)) \text{ such that }\hat{\fav}_{\mathcal{W}_T}= L^{-1}\sum_{\ell = 1}^L\hat{\fav}_{\mathcal{W}_T,\ell}.
\] 
Similar to \eqref{eq:Qbar}, let  $\bar{\fav}_{\mathcal{W}_T}$ and $\bar{\fav}_{\mathcal{W}_T,\ell}$ denote the counterparts of $\hat{\fav}_{\mathcal{W}_T}$ and $\hat{\fav}_{\mathcal{W}_T,\ell}$ obtained by replacing the nuisance estimator $\hat{\bds \theta}$ with their true values ${\bds \theta}^*$. 
Note that our test statistics $T_{\mathcal{W}_T}$ can be viewed as a normalized elementwise maximum of $\sqrt{npL}({\hat{\fav}}_{\mathcal{W}_T} - {\fav}_{\mathcal{W}_T})$. To derive its distribution, we use the linear expansion
\bee\label{highd:linear}
\sqrt{npL}({\hat{\fav}}_{\mathcal{W}_T} - {\fav}_{\mathcal{W}_T}) = \frac{1}{\sqrt{L}}\sum_{\ell = 1}^L(\sqrt{np}\,\bar{\fav}_{\mathcal{W}_T,\ell}- \sqrt{np} \,{\fav}_{\mathcal{W}_T})+ r_{n,L},
\ee
where the summands are i.i.d. mean-zero influence functions and $r_{n,L} = \sqrt{npL}({\hat{\fav}}_{\mathcal{W}_T} - {\bar{\fav}}_{\mathcal{W}_T})$ is the linearization error. Define the empirical influence function approximation error 
\[
\mathcal{E}_{\mathrm{IF},t}=\frac{1}{L} \sum_{\ell = 1}^L\Big(\hat{Z}_{t\ell} - Z_{t\ell}\Big)^2, \text{ where } Z_{t\ell}=\sqrt{np}\,\bar{\fav}_{i_tj_t,\ell}(\Omega_t)-\sqrt{np}\,{\fav}_{i_tj_t}(\Omega_t), \hat{Z}_{t\ell}=\sqrt{np}\,\hat{\fav}_{i_tj_t,\ell}(\Omega_t)-\sqrt{np}\,\hat{\fav}_{i_tj_t}(\Omega_t).
\]

We now state the result on the rate of the apprximation whose proof is deferred to Section~\ref{sec:pf:bootstrap:tech}.

\begin{proposition}\label{po:bootstrap:tail}
Assume the conditions of Theorem~\ref{thm:bootstrap}. Then there exists sequences  $\zeta_{n,L} \to0$ and $\varepsilon_n \to0$, we have, for any $\M A$ satisfying $\mathcal{E}_{\mathrm{good}}$,
\bee\label{eq:EIF:bound}
\PP\Big(\max_{t\in[T]}|\mathcal{E}_{\mathrm{IF},t}| \ge \zeta_{n,L}^2/(\log TL)^2\mid \M A\Big)\le \varepsilon_n,\quad 
\PP\big(\|r_{n,L}\|_\infty\ge \zeta_{n,L}/\sqrt{\log TL}\mid \M A\big)\le \varepsilon_n.
\ee
\end{proposition}

% Second, we clarify some one-to-one correspondence between our notations and \citet{belloni2018high}. Our $L$ is their $n$, our $T$ is their $p$, our $\sqrt{np}{\hat{\fav}}_{\mathcal{W}_T}$ and $\sqrt{np}{{\fav}}_{\mathcal{W}_T}$ are their $\hat{\theta}$ and $\theta_0$, our $\sqrt{np}{\bar{\fav}}_{\mathcal{W}_T,\ell}- \sqrt{np}{{\fav}}_{\mathcal{W}_T}$ is their $Z_i$, our $\M r_i$ is their $r_n$, our ${\mathcal{E}}_{\mathrm{IF}}$ is their $ \E_n((\hat{Z}_{i} - Z_{i})^2)$, and our $\varepsilon_n$ is their $\beta_n$, which are presented in Proposition~\ref{po:bootstrap:tail}. Third, since under the condition of Theorem~\ref{thm:bootstrap}, both $n$ and $T$ are growing polynomially with respect to $L$, and $L$ is also growing polynomially with respect to $n$, %
% we do not distinguish the main growing term among $n,L,T$ in what follows. %

Our next step is to apply Theorems 2.1 and 2.3 in \citet{belloni2018high} to prove the theorem. 
For completeness, we restate the conditions needed for these theorems in our notation. We need to find a sequence $B_{n}>0$ such that the following conditions are satisfied.

\noindent{\bf Condition M} \citep{belloni2018high} There exist constants $c,C>0$ such that for any $t\in[T]$ and $\ell\in[L]$,
\[
\E[Z_{t\ell}^2\mid \M A]\ge c,\qquad \E|Z_{t\ell}|^3\le C\,B_{n},\qquad \E Z_{t\ell}^4\le C\,B_{n}^2.
\]
Here the lower bound of variance can be standardized as 1 if we normalize the statistic properly. 

\noindent{\bf Condition E.1} \citep {belloni2018high} For any $t\in[T]$ and $\ell\in[L]$, we have
\[
 \E[\exp(|Z_{t\ell}|/B_{n})] \le 2 \qquad \delta_{n,L}=B_{n}^2\log^7(nT)/L  \to0.
\]

\noindent{\bf Condition A} \citep{belloni2018high}
We have \eqref{eq:EIF:bound}.

By Proposition~\ref{po:bootstrap:tail} and Lemma~\ref{lm:good-ER}, 
Condition A is satisfied with high probability. We will 
verify Conditions M and E.1 in the following lemmas.

\begin{lemma}[Condition M]\label{lem:M}
There exists $B_{n}>0$ such that uniformly over $t\in[T]$ and $\ell\in[L]$:
\[
\E\big[Z_{t\ell}^2\mid \M A\big]\ge c>0,\quad \E\big|Z_{t\ell}\big|^3\le C\Big(n^{1/2}+ (\tfrac{\log n}{p})^{3/2}\Big),\quad \E Z_{t\ell}^4\le C\Big(n^{2/3}+ ({\log n}/{p})^{2}\Big),
\]
and Condition M holds with $B_{n}=C \,\big(n^{1/2}+ (\log n/p)^{3/2}\big)$.
\end{lemma}

\begin{lemma}[Conditions E.1]\label{lem:E1A}
For $B_{n}=C \,\big(n^{1/2}+ (\log n/p)^{3/2}\big)$, we have  for any $t\in[T]$ and $\ell\in[L]$
\[
 \E[\exp(|Z_{t\ell}|/B_{n})] \to0 
\]
and $\delta_{n,L}=B_{n}^2\log^7 (nT)/L \to0$. 
\end{lemma}

The proofs of Lemmas~\ref{lem:M} and \ref{lem:E1A} are deferred to Sections~\ref{sec:pf:lem:M} and \ref{sec:pf:lem:E1A}.
Let 
$$
\bds N(\M A) = (N_1(\M A),\ldots,N_T(\M A))\sim \mathcal{N}(\M 0 , \M V(\M A)) \text{ with } \M V(\M A) = np\cdot \E\big[(\bar{\fav}_{\mathcal{W}_T,\ell}- {\fav}_{\mathcal{W}_T})(\bar{\fav}_{\mathcal{W}_T,\ell}- {\fav}_{\mathcal{W}_T})^\top\mid \M A\big].
$$
By Lemma~\ref{lem:M}, Lemma~\ref{lem:E1A}, and Proposition~\ref{po:bootstrap:tail}, the assumptions of Theorems~2.1 and 2.3 in \citet{belloni2018high} are met.  As the set $\{\M u \in \mathbb{R}^T \mid \|\M u\|_{\infty} \le x\}$ is a rectangle, Theorem~2.1 in \citet{belloni2018high} implies
\bee\label{thm2.1:imply}
\textstyle \sup_{x\in \R}\left|\mathbb{P}\left(T_{\mathcal{W}_T} \leq x\mid \M A\right) - \pr(\max_{t\in[T]}  N_t(\M A) \leq x\mid \M A)\right| \leq C(\delta_{n,L} + \varepsilon_n)\rightarrow 0.
\ee
 Theorem~2.3 in \citet{belloni2018high} implies that 
\bee\label{thm2.3:imply}
\textstyle\sup_{x\in \R}\left|\pr_{\bds\xi} \left(T_{\mathcal{W}_T}^{\star} \leq x \mid \mathcal{D}_n\right) - \pr(\max_{t\in[T]}  N_t(\M A) \leq x\mid\M A)\right| \leq C\delta_{n,L}  \rightarrow 0,
\ee
with probability at least $1 - 2\varepsilon_n -  n^{-1}$ as $n\rightarrow \infty$. Combining \eqref{thm2.1:imply} and \eqref{thm2.3:imply}, with probability at least $1 - 2\varepsilon_n -  n^{-1}$ as $n\rightarrow \infty$, we have
\bee\label{thm:gb:final}
\sup_{x\in \R}\left|\mathbb{P}\left(T_{\mathcal{W}_T} \leq x\mid \M A\right) - \pr_{\bds\xi} \left(T_{\mathcal{W}_T}^{\star} \leq x \mid \big\{\hat{\fav}_{i_tj_t{\ell} }(\Omega_t)\mid t\in[T], \ell\in[L]\big\}\right) \right| \leq C(\delta_{n,L} + \varepsilon_n)\rightarrow 0,
\ee
which finishes the proof.

\subsection{Proof of Theorem~\ref{thm:disshift}}\label{pf:thm:disshift}
The proof follows the same steps as in Section~\ref{sec:pf:theorem:LD:new}.  We can simply replace $\mathbb{I}(\M X \in \Omega)$ by its weighted version $\mathbb{I}(\M X \in \Omega) \kappa(\M X)$.  Conditioning on $\M A$ under $\mathcal{E}_{\mathrm{good}}$, we apply the Berry--Esseen argument for the oracle estimator as in Theorem~\ref{theorem:CLT:barQ}, but with the summands multiplied by the density ratio weight $\kappa(\M X)$. The Neyman orthogonality (Lemma~\ref{lm:neyman}) continues to hold since $\kappa(\M X)$ enters as a multiplicative factor in the moment map and thus preserves the cancellation in the Gateaux derivative. Consequently, the decomposition $\Delta_1,\Delta_2,\Delta_3$ and their high-probability bounds in Section~\ref{theorem:barQ-hatQ:pf} remain valid after replacing every occurrence of $\mathbb{I}(\M X\in\Omega)$ by $\mathbb{I}(\M X\in\Omega)\kappa(\M X)$, provided $\sup_{\M x}\kappa(\M x)\le C$. The variance expression and its estimator are the $\kappa$-weighted analogs of \eqref{asymptotic:variance}, namely $V_{\iz\jz}(\Omega\mid \kappa)$ and $\hat V_{\iz\jz}(\Omega\mid \kappa)$ defined in Theorem~\ref{thm:disshift}; their consistency is verified by the same arguments as in Theorem~\ref{theorem:sigmahat:rate} with the same replacement. Combining the weighted CLT with Slutsky's theorem yields the asserted normal limits for both $V_{\iz\jz}(\Omega\mid \kappa)$ and $\hat V_{\iz\jz}(\Omega\mid \kappa)$.

\subsection{Technical results for Theorem~\ref{thm:bootstrap}}\label{sec:pf:bootstrap:tech}
We prove Lemma~\ref{lem:M} in Section~\ref{sec:pf:lem:M}, Lemma~\ref{lem:E1A} in Section~\ref{sec:pf:lem:E1A}, and Proposition~\ref{po:bootstrap:tail} in Section~\ref{sec:proof:prop3}.

\subsubsection{Proof of Lemma~\ref{lem:M}}\label{sec:pf:lem:M}

We have studied the remainder term for the single hypothesis in the proof of Theorem~\ref{theorem:CLT:barQ} in Section~\ref{theorem:CLT:barQ:pf}. The analysis here is similar. First, for simplicity, we denote 
\bee\label{def:pi}
\hat{\pi}_i(\M x\mid t) = \hat{\pi}_i(\M x\mid \Omega_t,i_t,j_t),\quad \hat{\alpha}_{ij}(\M x\mid t) = \hat{\pi}_i(\M x\mid \Omega_t,i_t,j_t) -\hat{\pi}_j(\M x\mid \Omega_t,i_t,j_t).
\ee 

Following the same analysis as \eqref{varzijl} and \eqref{varzijl2}, we can have for any $t \in [T]$, 
\begin{align}\label{decom:barFF:t}
&\E\big[\big(\sqrt{np}\, \bar{\fav}_{i_tj_t,\ell}(\Omega_t)- \sqrt{np}\,{\fav}_{i_tj_t}(\Omega_t)\big)^2\mid \M A\big] = V_{i_tj_t}(\Omega_t)/(npL) \ge C >0,
\end{align}
where $V_{i_tj_t}(\Omega_t)$ is defined in \eqref{asymptotic:variance} by replacing $\iz,\jz, \Omega$ to $i_t,j_t, \Omega_t$ and the constant lower bound is by \eqref{V:bound} under $\mathcal{E}_{\mathrm{good}}$.

We have also studied the third moment of $Z_{t\ell}$ in \eqref{eq:Qbar:third:moment:bound} and have for any $t\in [T]$,
\begin{align}\label{upperboundFijt:3}
\!\!\!\! \E\left(\left|\sqrt{np} \bar{\fav}_{i_tj_t,\ell}(\Omega_t)- \sqrt{np}{\fav}_{i_tj_t}(\Omega_t)\right|^{3}\mid \M A\right) &\leq \frac{C}{(npL)^{3/2}}  \left(\frac{  \lp  \sigma_{ij}({\M A})}{{ L^2}}  
  + \frac{1}{{|\M A|^2L^2}}\right)
\leq \frac{C}{\sqrt{npL}},
\end{align}
where $\sigma_{ij}({\M A})$ in the first inequality is to replace $\iz,\jz$ with $i_t,j_t$  in \eqref{eq:asymptotic:variance2} and the last inequality is by Lemma~\ref{lm:boundsparty} under $\mathcal{E}_{\mathrm{good}}$.

Lastly, we bound the 4-th moment.   We define the notations as follows
\begin{align}\nonumber
 &\bar{\fav}_{i_tj_t\ell }(\Omega_t) - {\fav}_{i_tj_t\ell }(\Omega_t)  = \frac{1}{|\M A| }\sum_{(i,j)\in\mathcal{E}(\M A)}\Delta^*(i,j,\ell,t), \text{ where }
\\\nonumber
&\Delta^*(i,j,\ell,t) = \mathbb{I}(\M X_{ij \ell} \in \Omega_t) \left(\theta^*_{i_t}(\M X_{ij \ell})  - \theta^*_{j_t}(\M X_{ij \ell})\right)  - {\fav}_{i_tj_t\ell} \\\nonumber
&  \quad\quad\quad\quad\quad\quad\quad\quad  + \left({\pi}^*_{i}(\M X_{ij \ell}\mid \Omega_t,i_t,j_t) - {\pi}^*_{j}(\M X_{ij \ell}\mid \Omega_t,i_t,j_t)\right)\left(Y_{ij \ell} - \psi\left(\theta^*_{i}(\M X_{ij \ell})  - \theta^*_{j}(\M X_{ij \ell})\right)\right).
\end{align}
By \eqref{eq:bound:maxdiff}, we have that for any $\M x\in \mathbb{X}$,
\begin{align}\nonumber
&|{\pi}^*_{i}(\M x \mid t) - {\pi}^*_{j}(\M x \mid t)| 
= \big| |\M A|(\M e_{i} - \M e_{j})^\T\pmb{\mathscr{L}}^{\dagger}(\M x\mid \M A) (\M e_{i_t} - \M e_{j_t})\big|
\\\nonumber
&\leq C(n^2p)\Big(\big| (\M e_{i} - \M e_{j})^\T \M D(\M\Xi(\M x\mid \M A))^{-1}(\M e_{i_t} - \M e_{j_t})\big|+ \big|(\M e_{i} - \M e_{j})^\T \left(\M D(\M\Xi(\M x\mid \M A))^{-1} - \pmb{\mathscr{L}}^{\dagger}\right)(\M e_{i_t} - \M e_{j_t})\big|\Big)
\\\label{alpha:upper}
&\leq\begin{cases}
Cn & i= i_t\text{ or } j= j_t,
\\
Cn\left(\frac{1}{n^{1/3}} + \sqrt{\frac{\log n}{np}}\right) & \text{otherwise}.
\end{cases}
\end{align}
As ${\bds \theta}^* \in \Theta$, then for any $i,j,t$, we have $\Delta^*(i,j,\ell,t) \le C |{\pi}^*_{i}(\M x \mid t) - {\pi}^*_{j}(\M x \mid t)|$ and thus it has the same upper bound as \eqref{alpha:upper}. 
Building upon the previous results and observations, we can now proceed to deduce that 
\begin{align}\nonumber
& \E\left(\left|\sqrt{np} \bar{\fav}_{i_tj_t,\ell}(\Omega_t)- \sqrt{np}{\fav}_{i_tj_t}(\Omega_t)\right|^{4}\mid \M A\right)   
 =  \frac{(np)^{2}}{|\M A|^4}\E\left(\left|\sum_{(i,j)\in\mathcal{E}(\M A)}\Delta^*(i,j,\ell,t) \right|^{4}\mid \M A\right)
\\\nonumber
& \leq C    \frac{(np)^{2}}{|\M A|^4}\E\left(\sum_{(i,j)\in\mathcal{E}(\M A)\atop (i',j')\in\mathcal{E}(\M A)}\big(\Delta^*(i,j,\ell,t) \big)^{2}\big(\Delta^*(i',j',\ell,t)  \big)^{2}\mid \M A\right)
\\\nonumber
&\leq C \frac{1}{n^6p^2}\E\left(\left(\sum_{\mathcal{E}\cap\mathcal{E}'} +\sum_{\mathcal{E}^c\cap\mathcal{E}'^c} + \sum_{\mathcal{E}^c\cap\mathcal{E}'} + \sum_{\mathcal{E}\cap\mathcal{E}'^c}  \right)\big(\Delta^*(i,j,\ell,t) \big)^{2}\big(\Delta^*(i',j',\ell,t)  \big)^{2}\mid \M A\right)
\\\label{upperboundFijt:4}
&\leq C \frac{1}{n^6p^2} \left(n^6p^2 + n^8p^2\left(\frac{1}{n^{1/3}} + \sqrt{\frac{\log n}{np}}\right)^4 + n^5p^2\left(\frac{1}{n^{1/3}} + \sqrt{\frac{\log n}{np}}\right)^2 \right)
\leq C  n^{2/3} + C \frac{(\log n)^2}{p^2}  ,
\end{align}
where the third equality is summing over $\mathcal{E} = \{(i,j)\in\mathcal{E}(\M A)\mid i = i_t \text{ or }j = j_t\}$ and $\mathcal{E}' = \{(i',j')\in\mathcal{E}(\M A)\mid i' = i_t \text{ or }j = j'_t\}$. 
Thus, by taking 
$
B_{n} = C  n^{1/2} +  C(\frac{\log n}{p})^{3/2}
$, 
Conditions~M is statisfied.

\subsubsection{Proof of Lemma~\ref{lem:E1A}}\label{sec:pf:lem:E1A}
We adopt the sub-Gaussian and sub-exponential norms: $\|\cdot\|_{\psi_2}$ and $\|\cdot\|_{\psi_1}$  as $\E[\exp(Z)] \le \|Z\|_{\psi_1}$ (c.f. \citet{vershynin2018high}).
 By~\eqref{eq:Deltas:upperbound}, we have    
\begin{align}\nonumber
&\|\sqrt{np} \bar{\fav}_{i_tj_t,\ell}(\Omega_t)- \sqrt{np}{\fav}_{i_tj_t}(\Omega_t)\|^2_{\psi_2}
\\\nonumber
&\leq \frac{np}{|\M A|^2 L ^2}\left\|\sum_{(i,j)\in\mathcal{E}(\M A)}\mathbb{I}(\M X_{ij \ell} \in \Omega_t) \left( \gamma^*_{i_tj_t}(\M X_{ij \ell})\right) +  {\alpha}^*_{ij}(\M X_{ij \ell}\mid t) \left(Y_{ij \ell} - \psi\left(\gamma^*_{ij}(\M X_{ij \ell})\right)\right) - {\fav}_{i_tj_t }( \Omega_t)\right\|_{\psi_2}^2
\\\nonumber
&\leq \frac{np}{|\M A|^2 L ^2}\sum_{(i,j)\in\mathcal{E}(\M A)}\left\|\mathbb{I}(\M X_{ij \ell} \in \Omega_t) \left( \gamma^*_{i_tj_t}(\M X_{ij \ell})\right) +  {\alpha}^*_{ij}(\M X_{ij \ell}\mid t) \left(Y_{ij \ell} - \psi\left(\gamma^*_{ij}(\M X_{ij \ell})\right)\right) - {\fav}_{i_tj_t }( \Omega_t)\right\|_{\psi_2}^2
\\\label{eq:Ztl:subgaussian}
&\leq \frac{np}{|\M A|^2 L ^2}\left(Cn^2 p \cdot n^2\left(\frac{1}{n^{1/3}} + \sqrt{\frac{\log n}{np}}\right)^2 + C n^2 \cdot np\right)
\leq C L^{-2}\left(p^{-1}\log n+ n^{1/3}\right),
\end{align}
where the third inequality holds by \eqref{alpha:upper}. We thus have for any $t\in[T]$,
\begin{align}\nonumber
\|\sqrt{np} \bar{\fav}_{i_tj_t,\ell}(\Omega_t)- \sqrt{np}{\fav}_{i_tj_t}(\Omega_t)\|_{\psi_1}/B_n &\leq \|\sqrt{np} \bar{\fav}_{i_tj_t,\ell}(\Omega_t)- \sqrt{np}{\fav}_{i_tj_t}(\Omega_t)\|_{\psi_2}/B_n
 \to0,
\end{align}
which directly implies that the first part of Condition~E.1 holds. Finally, $
\delta_{n,L}\rightarrow 0,
$   
as $L = \omega(n^2)$, $\log (nT) \asymp \log n$, $(\log n)^{10}/(p^3 L )\rightarrow 0$ and the second part of Condition~E.1  holds.

\subsubsection{Proof of Proposition~\ref{po:bootstrap:tail}}\label{sec:proof:prop3}

The proposition is directly implied by Theorem~\ref{theorem:barQ-hatQ}. Applying Theorem~\ref{theorem:barQ-hatQ} with $\Omega = \Omega_t$, $\iz = i_t$, $\jz = j_t$, for any $\epsilon \in (0,1)$, for any $t\in[T]$, we have with probability $1-\epsilon$,
\bee\nonumber
  |\hat{\fav}_{i_t j_t} (\Omega_t) - \bar{\fav}_{i_t j_t} (\Omega_t)| \leq &C\sqrt{\log(2/\epsilon)}\left(n \cdot  L^{-1/2}{\mathcal{E}^{1/2}_{2}( \hat{\bds\theta},\bds\theta^*  \mid \M A)} +(n^2pL)^{-1/2}  {\mathcal{E}^{1/2}_2\left( \hat{\bds \pi},\bds \pi^* \mid \M A\right)}\right)+ {C\log(2/\epsilon) \cdot nL^{-1}}\\
  &\quad + C\left(n  \cdot\mathcal{E}_{2}( \hat{\bds\theta},\bds\theta^*  \mid \M A) +{\mathcal{E}^{1/2}_{2}( \hat{\bds\theta},\bds\theta^*  \mid \M A)\mathcal{E}^{1/2}_2\left( \hat{\bds \pi},\bds \pi^* \mid \M A\right) }\right).  
 \ee

By Assumption~\ref{ass:theta} and Proposition~\ref{po:pierror}, there exists a sequence $\zeta_{n,L} =o(1)$ such that for any $t \in [T]$, we have with probability $1 - 1/(TL)$,
\[
\sqrt{npL}| \hat{Q}_{i_t j_t} (\Omega_t) - \bar{Q}_{i_t j_t} (\Omega_t)| = o( 1/\sqrt{\log(TL)})..
\]
By union bound, we have with probability $1 - 1/L$,
\[
\|r_{n,L}\|_{\infty} = \sup_{t \in [T]} \sqrt{npL}| \hat{Q}_{i_t j_t} (\Omega_t) - \bar{Q}_{i_t j_t} (\Omega_t)| =  o(1/\sqrt{\log(TL)}).
\]

To study the rate of $\max_{t\in[T]} \mathcal{E}_{\mathrm{IF},t}$, we reformulate $\mathcal{E}_{\mathrm{IF},t}$ as
\bee\nonumber
&\mathcal{E}_{\mathrm{IF},t}=\frac{1}{L} \sum_{\ell = 1}^L\Big({Z}'_{t\ell} - \bar{Z}'_{t} - \bar{Z}_{t}\Big)^2, \text{ where }\\
&\quad Z'_{t\ell}= \sqrt{np}\,\bar{\fav}_{i_tj_t,\ell}(\Omega_t)-\sqrt{np}\,\hat{\fav}_{i_tj_t,\ell}(\Omega_t), \bar{Z}'_{t} = \frac{1}{L} \sum_{\ell = 1}^L Z'_{t\ell}, \bar{Z}_{t} = \sqrt{np}\,\bar{\fav}_{i_tj_t}(\Omega_t)-\sqrt{np}\,{\fav}_{i_tj_t}(\Omega_t).
\ee
We consider to apply the Cauchy-Schwarz inequality to get
\bee\label{eq:EIF:bound:t}
\mathcal{E}_{\mathrm{IF},t}=\frac{1}{L} \sum_{\ell = 1}^L\Big({Z}'_{t\ell} - \bar{Z}'_{t} - \bar{Z}_{t}\Big)^2 \le \frac{2}{L} \sum_{\ell = 1}^L (Z_{t\ell} - \bar{Z}'_{t})^2 + 2\bar{Z}_{t}^2.
\ee

For the second term in \eqref{eq:EIF:bound:t}, by
 \eqref{eq:Ztl:subgaussian}, $Z_{t\ell}$ is sub-Gaussian with $\max_{t,\ell}\|Z_{t\ell}\|_{\psi_2} \le B_n$. Combining union bound and Bernstein's inequality (c.f. Corollary 2.8.3 in \citet{vershynin2018high}), we have with probability $1-1/L$,
\bee\label{eq:EIF:bound:t:2}
\max_{t\in[T]} |\bar{Z}_{t}| = \max_{t\in[T]} \Big|\frac{1}{L} \sum_{\ell = 1}^L ( Z_{t\ell} - \E[ Z_{t\ell}])\Big| \le C B_n \cdot \sqrt{\log(TL)/L} = o( 1/{\log^2(TL)}).
\ee

For the first term in \eqref{eq:EIF:bound:t}, the proof of Theorem~\ref{theorem:barQ-hatQ} in Section~\ref{theorem:barQ-hatQ:pf} essentially proves $Z'_{t\ell}$'s are sub-Gaussian. In specific, by \eqref{eq:Deltas:upperbound} and \eqref{eq:Ztl:var} we have for any $t \in [T], \ell \in [L]$,
\bee
Z'_{tl}/\sqrt{np} & \le C|\M A|\lambda^{-1}_{\min,\perp}(\M L(\M A)) \le Cn,\\
\E[Z'_{tl}/\sqrt{np}]^2 &\le C {\mathcal{E}_{2}(\hat{\bds\theta},\bds\theta^*\mid \M A)}+   C |\M A|^2\lambda^{-2}_{\min,\perp}(\M L(\M A)){\mathcal{E}_{2}(\hat{\bds\theta},\bds\theta^*\mid \M A)} \\
&\qquad +\frac{C}{|\M A|} \mathcal{E}_{2}\left(\hat{\bds\pi},\bds\pi^*\mid \M A\right)+{C }  {\mathcal{E}_{2}\left(\hat{\bds\theta},\bds\theta^*\mid \M A\right)\mathcal{E}_{2}\left(\hat{\bds\pi},\bds\pi^*\mid \M A\right)} = o(1/(np)),
\ee
where the last inequality is by Assumption~\ref{ass:theta}, Proposition~\ref{po:pierror} and $\mathcal{E}_{\mathrm{good}}$.
Applying union bound and Theorem 10 of \citet{maurer2009empirical} on the rate of sample variance estimator, we have with probability $1-1/L$, 
$$
\max_{t \in [T]}\frac{1}{L} \sum_{\ell = 1}^L (Z_{t\ell} - \bar{Z}'_{t})^2  \le \max_{t \in [T]}\E[Z'_{t1}]^2 + C\frac{\log(TL)}{L} = o( 1/{\log^2(TL)}).
$$
In summary, we prove with probability $1-1/L$, $\max_{t \in [T]}|\mathcal{E}_{\mathrm{IF},t}| =  o( 1/{\log^2(TL)})$.

\section{Estimation of Semiparametric Efficient Variance}\label{sec:algorithm}

By the asymptotic equivalence of \( V_{\iz\jz}(\Omega) \) and \( \tilde{V}_{\iz\jz}(\Omega) \) established in Section~\ref{sec:eff}, we propose an alternative valid confidence intervals for \( \mathcal{Q}_{\iz\jz}(\Omega) \), analogous to Algorithm~\ref{alg:ci:addition}, which is presented in Algorithm~\ref{alg:ci:2}. We show its asymptotic normality in Theorem~\ref{thm:ci2}, with the proof deferred to Section~\ref{proof:thm:ci2}. Furthermore, we summarize the proposed estimator and CI under distributional shift, as introduced in Section~\ref{sec:distributionalshift}, in Algorithm~\ref{alg:pde:ds} and Algorithm~\ref{alg:pde:ds:ci}, respectively.
\begin{theorem}\label{thm:ci2}
Suppose the conditions of Theorems~\ref{theorem:LD:new} and \ref{thm:oracle} hold, and $\hat{\tilde{V}}_{\iz\jz}(\Omega)$ is generated from Algorithm~\ref{alg:ci:2}.
Then we have as $n\rightarrow \infty$,
$$
\frac{ {\hat{\fav}_{\iz\jz} (\Omega)  - \fav_{\iz\jz}(\Omega)} }{ \sqrt{\hat{\tilde V}_{\iz\jz}(\Omega) }}  \rightsquigarrow \normal(0,1).
$$
\end{theorem}
\begin{algorithm}[t]
\caption{CI $\tilde{\mathcal{C}}_{\iz\jz,1-\alpha}(\Omega)$ for $\fav_{\iz\jz}(\Omega)$}\label{alg:ci:2}
\KwIn{Comparison graph $\M A$, samples $\mathcal{D}_n$, number of cross-fittings $S$, nuisance estimator  class $\mathcal{F}_{\theta}$, 
    the covariates domain of interests $\Omega$, confidence level $1 - \alpha$.}
\KwOut{$(1 - \alpha)$-CI: $\tilde{\mathcal{C}}_{\iz\jz,1-\alpha}( \Omega)$. }
\begin{itemize}
\item[1.] Run   Algorithm~\ref{alg:pde}.
\item[2.] For each $s\in[S]$, obtain $\hat{\sigma}^{(s)}(\M A)$ as
\bee
\frac{1}{|\M A|L/S}\sum_{(i,j)\in\mathcal{E}(\M A)\atop \ell\in\mathcal{J}^{(s)}_n} &\mathbb{I}(\M X_{ij \ell}\in\Omega)\Bigg(\sum_{k\in\{\iz,\jz\}}\frac{1}{\sum_{i'}A_{ik}\psi'(\hat{\theta}^{(s)}_{k}(\M X_{ij \ell}) - \hat{\theta}^{(s)}_{i'}(\M X_{ij \ell}))}\Bigg) .
\ee
\item[3.] For each $s\in[S]$, obtain $ \hat{\tilde{V}}^{(s)}_{\iz\jz}(\Omega)$ as
\bee
\frac{1}{{|\M A|^2L^2/S^2}}\sum_{(i,j)\in\mathcal{E}(\M A)\atop \ell\in\mathcal{J}^{(s)}_n} \Big(\mathbb{I}(\M X_{ij \ell}  \in \Omega)  (\hat{\theta}^{(s)}_{\iz}(\M X_{ij \ell} ) - \hat{\theta}^{(s)}_{\jz}(\M X_{ij \ell} )) - \hat{\fav}_{\iz\jz}(\Omega)\Big)^2    + \frac{1}{L}\hat{\sigma}^{(s)}(\M A).
\ee
\item[4.] Obtain $\hat{\tilde{V}}_{\iz\jz}(\Omega) =S^{-1} \sum_{s = 1}^S \hat{\tilde{V}}_{\iz\jz}^{(s)}(\Omega)$, and build $(1-\alpha)$ CI:
\bee\label{CI:closedform}
\tilde{\mathcal{C}}_{\iz\jz,1 - \alpha}(  \Omega) = \left(\hat{\fav}_{\iz\jz} (\Omega)  - z_{1 - \alpha/2}\sqrt{\hat{\tilde{V}}_{\iz\jz}(\Omega)}, \,\hat{\fav}_{\iz\jz} (\Omega)  + z_{1 - \alpha/2}\sqrt{\hat{\tilde{V}}_{\iz\jz}(\Omega)}\right),
\ee 
where $z_{1 - \alpha/2} = \Phi^{-1}(1 - \alpha/2)$.
\end{itemize}
\end{algorithm}

\begin{algorithm} 
\caption{Fisher random walk estimator  with distributional shift}\label{alg:pde:ds}
\KwIn{Comparison graph $\M A$, samples $\mathcal{D}_n$, number of cross-fittings $S$, nuisance function estimator classes $\mathcal{F}_{\theta}$ 
  and domain $\Omega$, density ratio $\kappa(\cdot)$.}
\KwOut{$\hat{\fav}_{\iz\jz}(\Omega\mid \kappa) = \sum_{s = 1}^S \hat{\fav}_{\iz\jz}^{(s)}(\Omega\mid \kappa) /S$. }
\begin{itemize}
\item[1.] Split $[L]$ equally into $S$ subsets, namely, $\mathcal{J}^{(1)}_n,\ldots,\mathcal{J}^{(S)}_n$, such that $\mathcal{J}_n^{(1)}\cup\cdots\cup\mathcal{J}^{(S)}_n = [L]$. We then split $\mathcal{D}_n$ into $\mathcal{D}^{(1)}_n,\ldots,\mathcal{D}^{(S)}_n$ as $\mathcal{D}^{(s)}_n = \{(\M X_{ij \ell},Y_{ij \ell})\mid (i,j) \in \mathcal{E}(\M A), \ell\in\mathcal{J}^{(s)}_n\}$\;
\item[2.] Use  $\mathcal{D}^{(-s)}_n = \mathcal{D}_n/\mathcal{D}_n^{(s)}$ to train $\hat{\bds\theta}^{(s)}$ following \eqref{equ:likelihood_func}--\eqref{proposed:estimator} with $\mathcal{F} = \mathcal{F}_{\theta}$, and $\hat{\bds\pi}^{(s)}$ following \eqref{iota0}  with  $\hat{\bds\theta} = \hat{\bds\theta}^{(s)}$\;
\item[3.] For each $s\in[S]$, obtain the test statistics $\hat{\fav}_{\iz\jz}^{(s)}(\Omega\mid \kappa)$ as
\bee\nonumber
\frac{1}{|\M A|L/S}\sum_{\scalebox{0.65}{${(\M X_{ij \ell},Y_{ij \ell})\in\mathcal{D}_n^{(s)}}$}} &\kappa(\M X_{ij \ell})\Bigg(\mathbb{I}(\M X_{ij \ell} \in \Omega) \left(\hat\theta^{(s)}_{\iz}(\M X_{ij \ell})  - \hat\theta^{(s)}_{\jz}(\M X_{ij \ell})\right)  \\&
+ \left(\hat{\pi}^{(s)}_{i}(\M X_{ij \ell}) - \hat{\pi}^{(s)}_{j}(\M X_{ij \ell})\right)\left(Y_{ij \ell} - \psi\left(\hat\theta^{(s)}_{i}(\M X_{ij \ell})  - \hat\theta^{(s)}_{j}(\M X_{ij \ell})\right)\right)\Bigg).
\ee 
\end{itemize}
\end{algorithm}

\begin{algorithm}[t]
\caption{CI $\hat{\mathcal{C}}_{\iz\jz,1-\alpha}(\Omega\mid\kappa)$ for ${\fav}_{\iz\jz}(\Omega\mid \kappa) $}\label{alg:pde:ds:ci}
\KwIn{Comparison graph $\M A$, samples $\mathcal{D}_n$, number of cross-fittings $S$, nuisance estimator  class $\mathcal{F}_{\theta}$, 
    the covariates domain of interests $\Omega$, confidence level $1 - \alpha$, density ratio $\kappa(\cdot)$.}
\KwOut{$(1 - \alpha)$-CI: $\hat{\mathcal{C}}_{\iz\jz,1-\alpha}(\Omega\mid \kappa)$. }
\begin{itemize}
\item[1.] Run   Algorithm~\ref{alg:pde}.
\item[2.] For each $s\in[S]$, obtain $\hat{\sigma}^{(s)}(\M A\mid\kappa)$ as
$
\frac{1}{|\M A|^2L/S}\sum_{(i,j)\in\mathcal{E}(\M A)\atop \ell\in\mathcal{J}^{(s)}_n}\Big( \hat{\pi}_{\iz}^{(s)}(\M X_{ij \ell}) - \hat{\pi}_{\jz}^{(s)}(\M X_{ij \ell}) \Big)\kappa(\M X_{ij \ell}).
$
\item[3.] For each $s\in[S]$, obtain $ \hat{V}^{(s)}_{\iz\jz}(\Omega\mid \M A,\ell,\kappa)$ as
\bee\nonumber
\frac{1}{{|\M A|^2L^2/S^2}}\sum_{(i,j)\in\mathcal{E}(\M A)\atop \ell\in\mathcal{J}^{(s)}_n} \mathbb{I}(\M X_{ij \ell}  \in \Omega)\Big(  \left(\hat{\theta}^{(s)}_{\iz}(\M X_{ij \ell} ) - \hat{\theta}^{(s)}_{\jz}(\M X_{ij \ell} )\right) - \hat{\fav}_{\iz\jz}(\Omega)\Big)^2\kappa(\M X_{ij \ell})    + \frac{1}{L}\hat{\sigma}^{(s)}(\M A\mid \kappa).
\ee
\item[4.] Obtain 
\bee\label{def:Vestimate:ds}
\hat{V}_{\iz\jz}(\Omega\mid \M A,L,\kappa) =S^{-1} \sum_{s = 1}^S \hat{V}_{\iz\jz}^{(s)}(\Omega\mid \M A,L,\kappa),
\ee and build $(1-\alpha)$ CI:
\bee\nonumber
\hat{\mathcal{C}}_{\iz\jz,1-\alpha}( \Omega\mid\kappa) = \left(\hat{\fav}_{\iz\jz} (\Omega\mid\kappa)  - z_{1 - \alpha/2}\sqrt{\hat{V}_{\iz\jz}(\Omega\mid \M A,L,\kappa)}, \,\hat{\fav}_{\iz\jz} (\Omega\mid \kappa)  + z_{1 - \alpha/2}\sqrt{\hat{V}_{\iz\jz}(\Omega\mid \M A,L,\kappa)}\right),
\ee 
where $z_{1 - \alpha/2} = \Phi^{-1}(1 - \alpha/2)$.
\end{itemize}
\end{algorithm}

\subsection{Proof of Theorem~\ref{thm:ci2}}\label{proof:thm:ci2}
  We assume $S = 1$ and $\hat{\bds\theta}$ is trained through an independent copy of $\mathcal{D}_n$ with the same $\M A$. Such assumption is only for the simplicity of the proof, and our results will continuously hold as long as $S$ is fixed. We will write $\hat{\sigma}^{(1)}(\M A)$ as $\hat{\sigma}(\M A)$, etc. We define $\bar{\sigma}(\M A)$ as a counterpart of $\hat{\sigma}(\M A)$, with $\hat{\bds\theta}$ replaced by ${\bds\theta}^*$, i.e.,
\bee\nonumber
\bar{\sigma}(\M A) = \frac{1}{|\M A|L}\sum_{(i,j)\in\mathcal{E}(\M A)\atop \ell\in[L]} &\mathbb{I}(\M X_{ij \ell}\in\Omega)\Bigg(\sum_{k\in\{\iz,\jz\}}\frac{1}{\sum_{i'\in[n]}A_{ki'}\psi'({\theta}^{*}_{k}(\M X_{ij \ell}) - {\theta}^{*}_{i'}(\M X_{ij \ell}))}\Bigg) .
\ee 
Similar to the Proof of Theorem~\ref{pfthm:oracle} in Section~\ref{pfthm:oracle}, we focus on the case that there is no edge between $\iz$ and $\jz$ in $\M A$ for simplicity. Then $\tilde{\sigma}(\M A)$ in \eqref{def:tildesigma:main} reads

$$
\tilde{\sigma}(\M A)  =   \E_{\M X} \left(\frac{\mathbb{I}(\M X\in\Omega)(1 + \Delta_{\iz}(\M X))^2}{\sum_{j\in[n]} A_{\iz j}\psi'(\theta^*_{\iz}(\M X) - \theta^*_{j}(\M X))} + \frac{\mathbb{I}(\M X\in\Omega)(1 + \Delta_{\jz}(\M X))^2}{\sum_{i\in[n]} A_{\jz i}\psi'(\theta^*_{\jz}(\M X) - \theta^*_{i}(\M X))} \right), 
$$
as shown in \eqref{tildeAdef}.

We use $C$ to represent constants that may vary from place to place but only depending on $F$ and $C$. For each $n$, we condition on given $\M A$ which satisfies  $\mathcal{E}_{\mathrm{good}}$,  \eqref{A-P:hpb},  \eqref{good''happens}, and  \eqref{good'''happens}. 
By Taylor expansion, we have
\begin{align}\nonumber
&\hat{\sigma}(\M A) - \bar{\sigma} (\M A)
\\\nonumber
& =  \frac{1}{|\M A|L }\sum_{(i,j)\in\mathcal{E}(\M A)\atop \ell\in [L]}\Bigg(\mathbb{I}(\M X_{ij \ell}\in\Omega)\Bigg(\sum_{k\in\{\iz,\jz\}}\frac{\sum_{i'}A_{ik}\psi'({\theta}^*_{k}(\M X_{ij \ell}) - {\theta}^*_{i'}(\M X_{ij \ell})) - \psi'(\hat{\theta}_{k}(\M X_{ij \ell}) - \hat{\theta}_{i'}(\M X_{ij \ell}))}{\sum_{i'}A_{ik}\psi'(\hat{\theta}_{k}(\M X_{ij \ell}) - \hat{\theta}_{i'}(\M X_{ij \ell}))\sum_{i'}A_{ik}\psi'({\theta}^*_{k}(\M X_{ij \ell}) - {\theta}^*_{i'}(\M X_{ij \ell}))}  \Bigg)\Bigg)  
\\\nonumber
& = \frac{1}{|\M A|L }\sum_{(i,j)\in\mathcal{E}(\M A)\atop \ell\in [L]}\Bigg(\underbrace{\mathbb{I}(\M X_{ij \ell}\in\Omega)\Bigg(\sum_{k\in\{\iz,\jz\}}\frac{\sum_{i'}A_{ik} \xi_{k,i'}(\M X_{ij \ell}\mid \hat{\bds\theta},\bds\theta^*)({\theta}^*_{k}(\M X_{ij \ell}) - {\theta}^*_{i'}(\M X_{ij \ell})  -  \hat{\theta}_{k}(\M X_{ij \ell}) +  \hat{\theta}_{i'}(\M X_{ij \ell}))}{\sum_{i'}A_{ik}\psi'(\hat{\theta}_{k}(\M X_{ij \ell}) - \hat{\theta}_{i'}(\M X_{ij \ell}))\sum_{i'}A_{ik}\psi'({\theta}^*_{k}(\M X_{ij \ell}) - {\theta}^*_{i'}(\M X_{ij \ell}))}  \Bigg)}_{\zeta(\M X_{ij \ell}\mid \hat{\bds\theta},\bds\theta^*)}\Bigg)  
\\\nonumber
& = \E\zeta(\M X\mid \hat{\bds\theta},\bds\theta^*) + \frac{1}{|\M A|L }\sum_{(i,j)\in\mathcal{E}(\M A)\atop l \in[L]}\zeta(\M X_{ij \ell}\mid \hat{\bds\theta},\bds\theta^*) - \E\zeta(\M X\mid \hat{\bds\theta},\bds\theta^*),
\end{align}
where $\xi_{k,j'}(\M X_{ij \ell}\mid \hat{\bds\theta},\bds\theta^*)\leq C < \infty$. By \eqref{upperbound:A:supp}, we have as $n\rightarrow \infty$,
\begin{align}\nonumber
\E|\zeta(\M X\mid \hat{\bds\theta},\bds\theta^*)| &\leq C(np)^{-2}\cdot \E\left(\sum_{k\in\{\iz,\jz\}}\sum_{i'}A_{i'k}|\theta_k^*(\M X) - \hat{\theta}_k(\M X)| + |\theta_{i'}^*(\M X) - \hat{\theta}_{i'}(\M X)|\right)
\\\nonumber
&\leq C(np)^{-2}\cdot \E\left(\sum_{k\in\{\iz,\jz\}}\left(np|\theta_k^*(\M X) - \hat{\theta}_k(\M X)| + \sum_{i'}A_{i'k}  |\theta_{i'}^*(\M X) - \hat{\theta}_{i'}(\M X)|\right)\right)
\\\nonumber
&\leq C(np)^{-1}\cdot\sup_{i\in[n]} \E\left( |\theta_i^*(\M X) - \hat{\theta}_i(\M X)|\right)
= O\left((np)^{-1}\sqrt{\frac{1}{n^2p} \wedge\frac{1}{\sqrt{npL}}}\right)
 = o((np)^{-1}),
\end{align}
where the last two equalities is by Assumption~\ref{ass:theta}. By \eqref{upperbound:A:supp} and $|\zeta(\M X_{ij \ell}\mid \hat{\bds\theta},\bds\theta^*)|\leq C(np)^{-1}$ uniformly for all $\M X_{ij \ell}$,   a standard concentration argument   by Markov's inequality shows that
\bee\label{var:est:concentrate:eta}
\left|\frac{1}{|\M A|L }\sum_{(i,j)\in\mathcal{E}(\M A)\atop \ell\in[L]}\zeta(\M X_{ij \ell}\mid \hat{\bds\theta},\bds\theta^*) - \E\zeta(\M X\mid \hat{\bds\theta},\bds\theta^*)\right| \leq C(np)^{-1}(n^2 p L )^{-1/2} = o( (np)^{-1}),
\ee
with probability approaching 1 as $n\rightarrow \infty$ under the conditioned events of $\M A$ at the beginning of the proof. In summary, we have
$
L^{-1}|\hat{\sigma}(\M A) - \bar{\sigma} (\M A)| = o((npL)^{-1})
$
with probability approaching 1 as $n\rightarrow \infty$. Similar to \eqref{var:est:concentrate:eta}, we also have $$L^{-1}|\bar{\sigma}(\M A) - \tilde{\sigma}(\M A)| \leq C (npL)^{-1} (n^2pL)^{-1/2} = o((npL)^{-1})$$ with probability approaching 1 as $n\rightarrow \infty$, by Markov's inequality. Overall, unconditioning the event of $\M A$ as stated in the beginning of the proof,  the above results  imply that with probability approaching 1, $$L^{-1}|\hat{\sigma}(\M A) - \tilde{\sigma}(\M A)|  = o((npL)^{-1}).$$ 
Thus, we have with probability approaching 1, $V_{\iz\jz}(\Omega ) = o(1)$ and,
$
|\hat{\mathcal{Q}}_{\iz\jz}(\Omega) -  {\mathcal{Q}}_{\iz\jz}(\Omega) | = o(1),
$
where we use the fact that convergence in distribution to a constant (i.e., $\mathcal{Q}_{\iz\jz}(\Omega)$) implies the convergence in probability, and by Lemma~\ref{lm:good-ER}, $V_{\iz\jz}(\Omega) = \Omega((npL)^{-1})$. Summarizing all results above and by Theorem~\ref{thm:oracle}, we   have, with probability approaching 1 as $n\rightarrow\infty$, 
\begin{align}\nonumber
&\frac{|\hat{V}_{\iz\jz}(\Omega%
) - {V}_{\iz\jz}(\Omega%
)|}{{V}_{\iz\jz}(\Omega%
)} 
\\\nonumber
&\leq \frac{|\hat{V}_{\iz\jz}(\Omega%
) - \tilde{V}_{\iz\jz}(\Omega%
)|}{{V}_{\iz\jz}(\Omega%
)}  + \frac{|\tilde{V}_{\iz\jz}(\Omega%
) - {V}_{\iz\jz}(\Omega%
) |}{{V}_{\iz\jz}(\Omega%
)}
\\\nonumber
&\leq  C(npL)\Bigg(L^{-1}|\hat{\sigma}(\M A) - \tilde{\sigma} (\M A)|+ \frac{1}{{|\M A|^2L^2 }}\sum_{(i,j)\in\mathcal{E}(\M A)\atop \ell\in[L]} \Big(\mathbb{I}(\M X_{ij \ell}  \in \Omega)  (\hat{\theta} _{\iz}(\M X_{ij \ell} ) - \hat{\theta} _{\jz}(\M X_{ij \ell} )) - \hat{\fav}_{\iz\jz}(\Omega)\Big)^2
\\\nonumber
&\quad\quad\quad\quad\quad + \frac{1}{{|\M A|^2L}}\sum_{(i,j)\in\mathcal{E}(\M A) } \E\Big(\mathbb{I}(\M X   \in \Omega)  ({\theta}^* _{\iz}(\M X  ) - {\theta}^*_{\jz}(\M X  )) - {\fav}_{\iz\jz}(\Omega)\Big)^2\Bigg) + o(1)
\\\nonumber
&=o\left(npL\cdot (npL)^{-1}\right) + {O}\left((npL)\cdot\frac{L\cdot n^2p}{n^4p^2L^2}\right)+ o(1)
 = o(1),
\end{align}
where for the last equality, we note that with probability approaching 1, $|\M A| = O(n^2p)$, and $\hat{\theta}_i(\M x)$ and $\hat{\mathcal{Q}}_{\iz\jz}(\Omega)$ are all bounded by constants. Finally, Slutsky's theorem, we have
\bee\nonumber
 \frac{{\hat{\fav}_{\iz\jz} (\Omega)  - \fav_{\iz\jz}(\Omega)}}{ \sqrt{\hat{V}_{\iz\jz}(\Omega%
)}}  =  \frac{{\hat{\fav}_{\iz\jz} (\Omega)  - \fav_{\iz\jz}(\Omega)}}{ \sqrt{{V}_{\iz\jz}(\Omega%
)}}\cdot \sqrt{\frac{1}{\frac{\hat{V}_{\iz\jz}(\Omega%
) - {V}_{\iz\jz}(\Omega%
)}{{V}_{\iz\jz}(\Omega%
)} + 1}}\leadsto\mathrm{N}(0,1).
\ee 
\qed

\section{Electrical Network and Physics Proof}\label{sec:connect}

In this section, we provide some electrical network analogies with the concepts in this paper. We will first summarize some theoretical results that build the connection between the graph Laplacian of a weighted undirected graph and the corresponding electric  network; see e.g., \citet{doyle1984random,mahoney2016lecture} for more details. We use such relationship to build an identity for graph Laplacian in  \eqref{lm:eclap2} and this identity helps us to simplify the form of the asymptotic variance in \eqref{ec:target:identity}. We will then present a heuristic physics proof of Theorem~\ref{po:solution} in the Supplementary Material~\ref{sec:physics-proof}.

In this section, we consider a  generic weighted, connected and undirected graph $$\mathcal{G}_n = (\mathcal{V}_n = [n],\mathcal{E}_n,(w_{ij})_{(i,j)\in\mathcal{E}_n}),$$ where $\mathcal{V}_n$ and $\mathcal{E}_n$ are the node and edge sets of $\mathcal{G}_n$, respectively, and $w_{ij} > 0$ is the edge weight between node $i$ and $j$. Let $\bar{\M A} = (\tilde{a}_{ij})_{ij\in}$ be the adjacency matrix of $\mathcal{G}_n$, such that   $\bar{a}_{ij} = \bar{a}_{ji} = w_{ij} > 0$  if $(i,j)\in\mathcal{E}_n$, and $\bar{a}_{ij} = \bar{a}_{ji} = 0$ otherwise, and let $$\pmb{\mathscr{{H}}}(\bar{\M A}) = \mathrm{diag}(\bar{\M A}\M 1) - \bar{\M A}$$ be the corresponding weighted graph Laplacian. 

We   define a   single-voltage-source electrical network over $\mathcal{G}_n$.  
\begin{definition}
We define the following single-voltage-source electric  network $\mathscr{C}_{\iz\jz}(\mathcal{G}_n)$ from $\mathcal{G}_n$: 
\begin{itemize}
\item[(i)]  $\mathscr{C}_{\iz\jz}(\mathcal{G}_n)$ has the same node and edge structures as $\mathcal{G}_n$.
\item[(ii)] The resistance between node $i$ and node $j$, namely, $R_{ij}$ is $w^{-1}_{ij}$.
\item[(iii)] 
There is a single  voltage source $\mathrm{VS}$ connected between $\iz$ and $\jz$. The current between node $i$ and node $j$ is denoted by $Y_{ij}$ such that $Y_{ij} = -Y_{ji}$. By Kirchhoff current law, we have
\bee\nonumber
\sum_{k\in\delta(i)}Y_{ik} = \begin{cases}
Y & i = \iz
\\
-Y &  i = \jz
\\
0 & \text{otherwise}
\end{cases},
\ee
for some $Y\in\R$. We  set $\mathrm{VS}$ in the way that $Y = 1$. In another word, $\mathrm{VS}$ is a one-ampere current source between $\iz$ and $\jz$. 
\item[(iv)] Ohm's Law states that, given $\mathscr{C}_{\iz\jz}(\mathcal{G}_n)$, we have unique node potentials $\M V_{\iz\jz} = (V_1,\ldots,V_n)$ such that $Y_{ij} = (V_i - V_j)/R_{ij} = w_{ij}(V_i - V_j),$ for any $(i,j)\in \mathcal{E}_n$.
\item[(v)] The effective resistance $R_{\iz\jz}$ between $\iz$ and $\jz$ is defined as 
$$
V_{\iz} - V_{\jz} = \frac{Y}{R_{\iz\jz}} \Longleftrightarrow R_{\iz\jz} = (V_{\iz} - V_{\jz})/Y = V_{\iz} - V_{\jz}.
$$
\end{itemize}
\end{definition}
We note that if $\iz$ and $\jz$ are connected, the definition of $R_{\iz\jz}$ complies with the edge resistance definition, i.e.,  $R_{\iz\jz} = w_{\iz\jz}^{-1}$. With   $\mathscr{C}_{\iz\jz}(\mathcal{G}_n)$ defined above, we can now formally state the relationships between the node potentials and effective resistance on $\mathscr{C}_{\iz\jz}(\mathcal{G}_n)$,  and  the graph Laplacian of $\mathcal{G}_n$. The following result can be found in e.g., \citet[Lecture~16]{mahoney2016lecture}.
\begin{lemma}\label{lm:eclap}
The following identities hold: 
\begin{enumerate}
\item[(i).] $
\M V_{\iz\jz} = \pmb{\mathscr{{H}}}(\bar{\M A})^\dagger(\M e_{\iz} - \M e_{\jz})
$.
\item[(ii).] $R_{\iz\jz} =  (\M e_{\iz} - \M e_{\jz})^\T\pmb{\mathscr{{H}}}(\bar{\M A})^\dagger(\M e_{\iz} - \M e_{\jz})$.
\end{enumerate}
\end{lemma}
\par
We now investigate the electronic power $\mathrm{EP}_n$ on $\mathscr{C}_{\iz\jz}(\mathcal{G}_n)$, which is the total electrical  energy transferred in  $\mathscr{C}_{\iz\jz}(\mathcal{G}_n)$. For each edge $(i,j)\in\mathcal{E}_n$, the electronic power on $(i,j)$, namely $\mathrm{EP}_{i,j}$, is the voltage between $i$ and $j$ times the current on $(i,j)$, i.e.,
\bee\nonumber
\mathrm{EP}_{ij} = (V_i - V_j)Y_{ij} = \frac{(V_i - V_j)^2}{R_{ij}} = w_{ij} \big((\M e_i - \M e_j)^\T\pmb{\mathscr{{H}}}(\bar{\M A})^\dagger(\M e_{\iz} - \M e_{\jz})\big)^2,
\ee
where the second equality uses the Ohm's Law and the last equality uses Lemma~\ref{lm:eclap}(i). Therefore, we can calculate $\mathrm{EP}_n$ from a edgewise perspective,
\bee\label{EPn1}
\mathrm{EP}_n = \sum_{(i,j)\in\mathcal{E}_n}  \mathrm{EP}_{ij} = \sum_{(i,j)\in\mathcal{E}_n} w_{ij} ((\M e_i - \M e_j)^\T\pmb{\mathscr{{H}}}(\bar{\M A})^\dagger(\M e_{\iz} - \M e_{\jz}))^2.
\ee
On the other hand, from the macro perspective, the electronic power can also be calculated as the total current between two voltage source nodes $\iz,\jz$, times the effective resistance of $\iz,\jz$; see e.g., \citet[Section~B]{ghosh2008minimizing}. Then we have,
\bee\label{EPn2}
\mathrm{EP}_n = Y R_{\iz\jz} = R_{\iz\jz} = (\M e_{\iz} - \M e_{\jz})^\T\pmb{\mathscr{{H}}}(\bar{\M A})^\dagger(\M e_{\iz} - \M e_{\jz}),
\ee
by Lemma~\ref{lm:eclap}(ii). Comparing \eqref{EPn1} and \eqref{EPn2}, we construct the following identity,
\bee\label{lm:eclap2}
\sum_{(i,j)\in\mathcal{E}_n} w_{ij} ((\M e_i - \M e_j)^\T\pmb{\mathscr{{H}}}(\bar{\M A})^\dagger(\M e_{\iz} - \M e_{\jz}))^2 = (\M e_{\iz} - \M e_{\jz})^\T\pmb{\mathscr{{H}}}(\bar{\M A})^\dagger(\M e_{\iz} - \M e_{\jz}).
\ee
A rigorous proof for \eqref{lm:eclap2} relies on the result in spectral graph theory that the leverage score matrix of any weighted graph is a projection matrix. Specifically, \eqref{lm:eclap2}  is also the direct corollary of \citet[Lemma 3(iv)]{spielman2008graph}.

\subsection{Physics proof of Theorem~\ref{po:solution}}\label{sec:physics-proof}

We follow the electrical network model as above, and show that \eqref{def:alpha:main} essentially follows Ohm's law. Consider
an electrical network over $\M A$, where $A_{ij} = 1$ implies that there is a resistor between nodes $i$ and $j$ with resistance $R_{ij} =  |\M A| \hat{\alpha}_{ij}(\M x)$. 
The probability that an electron moves from node $i$ to node $j$ is 	inversely proportional to the resistance $R_{ij}$; see e.g., \citep[$\mathsection$1.3]{doyle1984random}. Namely, then transition probability $P_{ij} = R^{-1}_{ij}/(\sum_{k=1}^n A_{ik}R^{-1}_{ik})$, which is exactly the same as in \eqref{trans:propose}. Therefore, the current from node $i$ to node $j$ is the expected net number of electron random walks crossing edge $(i,j)$:
\[
I_{ij} = \EE_{\mathrm{R} \sim \mathcal{R}_{\iz\jz}(\M x, \hat{\bds\theta})}|\{(i,j) \in \mathrm{R}\}| - \EE_{\mathrm{R} \sim \mathcal{R}_{\iz\jz}(\M x, \hat{\bds\theta})}|\{(j,i) \in \mathrm{R}\}|.
\]
By \eqref{def:W}, the residual balancing weight $\hat{W}_{ij} = I_{ij}R_{ij}$.
On the other side, from the macroscopic perspective, each node $i$ has the eletric potential (or voltage) $\hat \pi_i$. By Ohm's law, we have
$\hat \pi_i - \hat \pi_j = I_{ij}R_{ij} = \hat{W}_{ij}$, which is essentially \eqref{def:alpha:main}. The classical electrical network theory then ensures that the (normalized) node potential for a unit current injection on nodes $\iz$ and $\jz$, is given explicitly by
$
{\color{black}\mathbb{I}(\M x\in\Omega)}|\M A| \hat{\pmb{\mathscr{L}}}^{\dagger}(\M x)(\M e_{\iz} - \M e_{\jz}).
$
See e.g., \citet[Claim 10]{mahoney2016lecture}.

\section{Technical results for the comparison graph}
\label{sec:proof:graph}

This section assembles all the techinical results we need for the comparison graph especially for the graph Laplacian. The weight matrix might change with place.
We will first discuss the properties for fixed graph $\M A$ and then show the results when $\M A$ is generated from the Erd\H{o}s-R\'{e}nyi graph in Section~\ref{sec:proof:ER}. 
We begin with reviewing some general properties of graph Laplacian, see \citet{spielman2012spectral} for more general discussion. Let $\pmb{\mathscr{L}}$ be some generic graph Laplacian of some weighted matrix supported of $\M A$. We have
\bee\label{property:H}
&\mathrm{Rank}(\pmb{\mathscr{L}}) = n - 1,
\pmb{\mathscr{L}}\pmb{\mathscr{L}}^\dagger = \M I - n^{-1}\M 1\M 1^\T, 
 \pmb{\mathscr{L}} \M 1 = 0,
(\pmb{\mathscr{L}} +\lambda \M 1\M 1^\T)^{-1} = \pmb{\mathscr{L}}^\dagger  + \lambda^{-1}n^{-2}\M 1\M 1^\T,
\ee
for any $\lambda \neq 0$, where $ \dagger$ represents the Moore--Penrose inverse of a matrix.  

\begin{lemma}\label{lemma:eigen:lower:bound}
  For any $\bds\theta(\M x)$ satisfying, 
  \bee\label{lemma:theta:bound}
  \sup_{{\M x}\in\mathbb{X}}\Big(\max_{i\in[n]}\theta_i({\M x}) - \min_{i\in[n]}\theta_i({\M x})\Big) \leq M
  \ee
  with some constant $M > 0$, we have that, for any ${\M x}\in\mathbb{X}$,
  $$
  \lambda_{\min,\perp}\left(\pmb{\mathscr{L}}(\M x\mid \bds\theta)\right)\geq  c\exp(-M) \cdot \lambda_{\min,\perp}\left(\M L(\M A)\right), \lambda_{\max}\left(\pmb{\mathscr{L}}(\M x\mid \bds\theta)\right)\le  C\exp(M)  \lambda_{\max}\left(\M L(\M A)\right).
  $$
 If  $\M A$ satisfies $\mathcal{E}_{\mathrm{good}}$ in \eqref{upperbound:A:supp}, we further have
  \bee\nonumber
  \lambda_{\min,\perp}\left(\pmb{\mathscr{L}}(\M x\mid \bds\theta)\right)\geq c \exp\left( - M\right)\cdot np \text{ and } \lambda_{\max}\left(\pmb{\mathscr{L}}(\M x\mid \bds\theta)\right)\leq C \exp\left(M\right)\cdot np.
  \ee
  \end{lemma}
  \begin{proof}
    This lemma is an extension of \citet[Lemma~8.3]{chen2022partial} to the contextual setting. We still present the proof for completeness.
    By definition, we have for any ${\M x}\in\mathbb{X}$ and $\M u= (u_1,\ldots,u_n)\neq \bds 0$ such that $\M u^\T \bds 1 = 0$,
  \begin{align}\nonumber
  \M u^\T\pmb{\mathscr{L}}(\M x\mid \bds\theta)\M u &=  \sum_{i > j}A_{ij}\psi\Big(\theta_i({\M x}) - \theta_j({\M x})\Big)\psi\Big(\theta_j({\M x}) - \theta_i({\M x})\Big)\big(u_i - u_j\big)^2
  \\\nonumber
  &\geq \frac{1}{4}\exp(-M) \sum_{i > j}A_{ij}\big(u_i - u_j\big)^2
  \\\nonumber
  &= \frac{1}{4}\exp(-M) \left\{\sum_{i\in[n]}\left(\sum_{j\in[n]}A_{ij}\right)u_i^2 - \sum_{i\neq j}u_iA_{ij}u_j\right\}
  = \frac{1}{4}\exp(-M) \M u^\T\M L(\M A)\M u,
  \end{align}
  where the first inequality holds by \eqref{lemma:theta:bound}. Meanwhile, we have $\psi\big(\theta_i({\M x}) - \theta_j({\M x})\big)\psi\big(\theta_j({\M x}) - \theta_i({\M x})\big)\geq \exp(-M)/4$ for any ${\M x}\in\mathbb{X}$, which implies that
  \bee\nonumber
  \lambda_{\min,\perp}\left(\pmb{\mathscr{L}}(\M x\mid \bds\theta)\right) \geq \frac{1}{4}\exp(-M) \cdot \lambda_{\min,\perp}\left(\M L(\M A)\right).
  \ee
   If $\mathcal{E}_{\mathrm{good}}$ holds,  \eqref{upperbound:A:supp} further implies
  $
  \lambda_{\min,\perp}\left(\pmb{\mathscr{L}}(\M x\mid \bds\theta)\right)  \geq	 ({np}/{8})\exp\left( - M\right).
  $ The result for $\lambda_{\max}$ applies similarly.
  \end{proof}

  \begin{lemma}\label{lm:alpha:bound}
    Recall $\bds\pi(\M x \mid {\bds \theta}, \iz, \jz)$ is defined in \eqref{iota0} and define $\alpha_{ij}(\M x \mid {\bds \theta}) = \pi_i(\M x \mid {\bds \theta}, \iz, \jz) - \pi_j(\M x \mid {\bds \theta}, \iz, \jz)$ for any $i,j \in [n]$. For any ${\bds \theta}$, we have
      \bee\nonumber
    \sup_{\M x\in\mathbb{X}}\|\bds\pi(\M x \mid {\bds \theta}, \iz, \jz)\| \leq C|\M A|\lambda^{-1}_{\min,\perp}(\M L(\M A)) ,\text{ and } \sup_{\M x\in\mathbb{X}}\max_{i \in [n]}\Big(\sum_{j=1}^n |\alpha_{ij}(\M x\mid {\bds \theta})|^2\Big)^{1/2} \leq C|\M A|\lambda^{-1}_{\min,\perp}(\M L(\M A)).
    \ee
    In specific, under $\mathcal{E}_{\mathrm{good}}$, we further have
    \bee\nonumber
    \sup_{\M x\in\mathbb{X}}\|\bds\pi(\M x \mid {\bds \theta}, \iz, \jz)\| \leq Cn ,\text{ and } \sup_{\M x\in\mathbb{X}}\max_{i \in [n]}\Big(\sum_{j=1}^n |\alpha_{ij}(\M x\mid {\bds \theta})|^2\Big)^{1/2} \leq Cn.
    \ee
      \end{lemma}
    
    \begin{proof}
      By Lemma~\ref{lemma:eigen:lower:bound} and the facts that $\M A$ is fully connected, and $\hat{\bds\theta}^{(s)},\bds\theta^*$ are uniformly bounded, we have
      \bee\nonumber
      &\sup_{\M x\in\mathbb{X}}\|{\pmb{\mathscr{L}}}^\dagger(\M x\mid {\bds\theta})\| =  \sup_{\M x\in\mathbb{X}}\lambda^{-1}_{\min,\perp}(\pmb{\mathscr{L}}(\M x\mid {\bds\theta})) \leq C\lambda^{-1}_{\min,\perp}(\M L(\M A)).
      \ee
    Combining with \eqref{iota0}, we thus have
    \bee\nonumber
    \sup_{\M x\in\mathbb{X}}\|\bds\pi(\M x \mid {\bds \theta}, \iz, \jz)\| \leq 2|\M A|\sup_{\M x\in\mathbb{X}}\|{\pmb{\mathscr{L}}}^\dagger(\M x\mid {\bds \theta})\| \leq C|\M A|\lambda^{-1}_{\min,\perp}(\M L(\M A)),
    \ee
    and the upper bound of $\alpha_{ij}$ applies by triangle inequality.
    Under $\mathcal{E}_{\mathrm{good}}$, combining with \eqref{good:bound}, we further have $|\M A|\lambda^{-1}_{\min,\perp}(\M L(\M A)) \le Cn$.
    \end{proof}

\begin{lemma}\label{lm:Cauchy-Schwarz}
  Under the condition of $\mathcal{E}_{{\rm good}}$, for any $a_{ij}, b_{ij}$, we have
  \bee\nonumber
  &\frac{1}{|\M A|^2}\left(\sum_{(i,j)\in\mathcal{E}(\M A)} a_{ij} b_{ij}\right)^2 \leq \frac{C}{n^2} \max_{i \in [n]} \sum_{j=1}^n a_{ij}^2 \cdot \max_{i \in [n]} \sum_{j=1}^n b_{ij}^2 \text{ and } \\
  &\frac{1}{|\M A|}\max_{i \in [n]}\sum_{j=1}^n A_{ij}a_{ij}   \leq \frac{C}{n} \max_{i \in [n]} \sum_{j=1}^n |a_{ij}|.
  \ee
\end{lemma}
\begin{proof}
  By Cauchy-Schwarz inequality, we have
  \bee\nonumber
  \frac{1}{|\M A|}\left(\sum_{i,j} A_{ij} a_{ij} b_{ij}\right)^2 &\leq \frac{1}{|\M A|}\sum_{i,j} A_{ij} a_{ij}^2 \cdot \sum_{i,j} A_{ij} b_{ij}^2  \\
  &
  \le \frac{\max_{i \in [n]} \sum_{j\neq i} A_{ij}}{|\M A|}\max_{i \in [n]} \sum_{j=1}^n a_{ij}^2 \cdot \max_{i \in [n]} \sum_{j=1}^n b_{ij}^2 \le \frac{C}{n^2} \max_{i \in [n]} \sum_{j=1}^n a_{ij}^2 \cdot \max_{i \in [n]} \sum_{j=1}^n b_{ij}^2.
  \ee
  where the last inequality holds by the condition of $\mathcal{E}_{{\rm good}}$. The second inequality can be proved similarly.
  
\end{proof}

\begin{lemma}\label{lm:boundsparty}
  Under $\mathcal{E}_{\mathrm{good}}$ in \eqref{upperbound:A:supp}, there exists some constants $0<c<C$ such that we have
  \begin{align}\label{good:bound}
  &|\M A| \in(cn^2p,Cn^2p)
  \\
  \label{sigma:bound}
  &\sigma(\M A) \in \left(c(np)^{-1},C(np)^{-1}\right), 
  \\\label{V:bound}
  &V_{\iz\jz}(\Omega)%
  \in \left(c(npL)^{-1},C(npL)^{-1}\right).
  \end{align}
  \end{lemma}

\begin{proof}
The range of the total number of edges in \eqref{good:bound} follows immediately from the degree bound in \eqref{upperbound:A:supp}.
Under $\mathcal{E}_{\mathrm{good}}$, we have $\M A$ is connected. Then, as the   weighted graph Laplacian, we have for any $\M x\in\mathbb{X}$,
\bee\nonumber
\lambda_1\left(\pmb{\mathscr{L}}^\dagger(\M x\mid  \M A)\right) = \lambda^{-1}_{\min,\perp}(\pmb{\mathscr{L}}(\M x\mid  \M A)), \ 
\lambda_{\min,\perp}\left(\pmb{\mathscr{L}}^\dagger(\M x\mid  \M A)\right) = \lambda^{-1}_{1}\left(\pmb{\mathscr{L}} (\M x\mid  \M A)\right),\ 
\M 1^\T\pmb{\mathscr{L}}^\dagger(\M x\mid  \M A)  = 0.
\ee
Thus by Lemma~\ref{lemma:eigen:lower:bound}, we have
\begin{align}
\label{Ldaggerbound}
&\left\|\pmb{\mathscr{L}}^\dagger(\M x\mid  \M A)\right\| = \lambda^{-1}_{\min,\perp}(\pmb{\mathscr{L}}(\M x\mid  \M A))
 \leq C \lambda^{-1}_{\min,\perp}(\M L(\M A))
 \leq C(np)^{-1}\text{ and }\\
\nonumber
&\sigma({\M A})  = \E_{\M X}\Big(\mathbb{I}(\M X\in\Omega)(\M e_{\iz} - \M e_{\jz})^\T\pmb{\mathscr{L}}^\dagger(\M X\mid  \M A)(\M e_{\iz} - \M e_{\jz}) \Big)
\leq 2\E_{\M X}\Big(\mathbb{I}(\M X\in\Omega)\|\pmb{\mathscr{L}}^\dagger(\M X\mid  \M A)\|\Big)
\leq C(np)^{-1}.
\end{align}
The upper bound of $\sigma({\M A})$ follows similarly by applying the upper bound of $\lambda_{\max}(\M L(\M A))$ in Lemma~\ref{lemma:eigen:lower:bound}.
Finally, to prove \eqref{V:bound}, we have
\begin{align}\nonumber
V_{\iz\jz}(\Omega)%
&= \frac{1}{L}\E_{\M X}\left(\frac{1}{{|\M A|}}\Big(\mathbb{I}(\M X  \in \Omega)  \gamma^{*}_{\iz\jz}(\M X ) - \fav_{\iz\jz}(\Omega)\Big)^2 +  \sigma({\M A}) \right)
\leq C(n^2pL)^{-1} + C(npL)^{-1}
\leq C(npL)^{-1}.
\end{align}
On the other hand, we have $V_{\iz\jz}(\Omega)%
\geq L^{-1}\sigma(\M X) \geq c(npL)^{-1}$.
\end{proof}

\subsection{Technical results for Erd\H{o}s--R\'enyi graph}
\label{sec:proof:ER}

\begin{lemma}
\label{lm:good-ER}
Given any $\gamma >0$, let $c_\gamma = (\gamma + 1)52 /3$. If $np\geq c_\gamma\log n$, then for $\mathcal{E}_{\mathrm{good}}$ defined in \eqref{upperbound:A:supp}, we have
\begin{equation}
  \mathbb{P}\left(\mathcal{E}_{\mathrm{good}}\right) \geq 1 - 2n^{-\gamma}.
\end{equation}
\end{lemma}

\begin{proof}
  The result has been shown in \citet[Lemma~8.1]{chen2022partial} and \citet{tropp2015introduction}. For completeness, we provide the proof here. There are two events in $\mathcal{E}_{\mathrm{good}} = \mathcal{E}_{\mathrm{good},1} \cup \mathcal{E}_{\mathrm{good},2}$:

\begin{align}
  \mathcal{E}_{\mathrm{good},1} &= \Big\{\M A \in \{0,1\}^{n\times n} ~\Big|~ 0.5np \leq \min_{i\in[n]}\sum_{j\in[n]\backslash\{i\}} A_{ij} \leq \max_{i\in[n]}\sum_{j\in[n]\backslash\{i\}}A_{ij}\leq 2np\Big\},\\
  \mathcal{E}_{\mathrm{good},2} &= \Big\{\M A \in \{0,1\}^{n\times n} \Big|~\frac{np}{2} \le \lambda_{\min,\perp}(\M L(\M A)) \le  \lambda_{\max}(\M L(\M A)) \le 2{np}\Big\}.
\end{align}

For $\mathcal{E}_{\mathrm{good},1}$, by Bernstein's inequality~(see Lemma~\ref{lm:bern}(II)), we have when $np \geq (\gamma + 1)52(\log n)/3$,
$
\pr(\mathcal{E}_{\mathrm{good},1}) \geq1- 2n\exp\left(-\frac{3}{52}np\right) \geq 1 - 2n^{-\gamma}.
$

For $\mathcal{E}_{\mathrm{good},2}$, by Weyl's inequality \cite{Higham2021}, event $\mathcal{E}_{\mathrm{good},2}$  happens if and only if
\bee\label{A-P:hpb}
\|\M A - \M P\| \leq C_{\gamma}\sqrt{np}.
\ee 
Here, $C_{\gamma} > 0$ is the constant in    \citet[Theorem~5.2]{lei2015consistency}. By \citet[$\mathsection$5.3.3]{tropp2015introduction}, with $np\geq {52(\gamma + 1)}(\log n)/{3},$
we have 
\bee\nonumber
\pr(\mathcal{E}_{\mathrm{good},2}) \geq \pr(\|\M A - \M P\| \leq C_{\gamma}\sqrt{np}) \geq 1 - (n-1)\big(\sqrt{{2}/{e}}\big)^{pn/2} \geq 1-(n-1)n^{-\gamma -1}\geq 1 - n^{-\gamma}.
\ee

By the union bound, taking $c_\gamma = (\gamma + 1)52 /3$, we have
$
\pr(\mathcal{E}_{\mathrm{good}}) \geq \pr(\mathcal{E}_{\mathrm{good},1}) + \pr(\mathcal{E}_{\mathrm{good},2})-1 \geq 1 - 2n^{-\gamma}.
$

\end{proof}

We next show a novel results for the graph Laplacian of the Erd\H{o}s--R\'enyi graph. The following theorem shows that the pseudo-inverse of the graph Laplacian of the Erd\H{o}s--R\'enyi graph is asymptotically diagonal.

\begin{theorem}
  \label{theorem:ER:Ldagger}
  Let $\M A$ be a random matrix generated from the Erd\H{o}s--R\'enyi graph with probability $p$. Let $\M \Phi(\cdot)$ be the weight function supported on $\M A$ such that $\Phi_{ij}(\M x)$ is the weight of the edge $(i,j)$ given some $\M x\in\mathbb{X}$ such that $\Phi_{ij} = 0$ if $(i,j)\notin \mathcal{E}(\M A)$. Suppose $\M \Phi$ satisfies the   H\"older smooth condition, i.e., 
  $
  |\Phi_{ij}(\M x) - \Phi_{ij}(\M y)| \leq C\|\M x - \M y\|^{\beta}
  $
  for any $\M x,\M y \in\mathbb{X}\text{ and }i\in[n]$, where $C > 0,\beta \in (0,1]$ are  some fixed constants. Moreover, we assume the weight is bounded away from infinity and zero, i.e., $C \ge \max_{i\neq j} \Phi_{ij}(\M x) \ge \min_{i\neq j} \Phi_{ij}(\M x) \ge c_{\Phi} > 0$ for some fixed $c_{\Phi}$.
  Let $\pmb{\mathscr{L}}(\M x)$ be the graph Laplacian of the weight matrix $\M \Phi(\M x)$ and $\M D(\M x) = \mathrm{diag}(\M \Phi(\M x)\M1)$ be the degree matrix. If $np  \geq   (\log n)^{\xi}  $ for some fixed $\xi > 3$, then as $n\rightarrow \infty$, we have with probability $1-C/n^c$,
  \bee\label{eq:bound:maxdiff}
\sup_{\M x\in\mathbb{X}}\|\pmb{\mathscr{L}}^\dagger(\M x) - \M D^{-1}(\M x)\|_{\infty} \leq C (np)^{-1}\left(\frac{1}{n^{1/3}} + \sqrt{\frac{\log n}{np}}\right),
\ee
where $\|\cdot\|_{\infty}$ is the entriy-wise matrix max norm.
\end{theorem}

\begin{proof}
  Our proof depends on four key events on the graph $\M A$: $\mathcal{E}_{\mathrm{good}}$, \eqref{upperbound:A:supp}, \eqref{good''happens} and \eqref{good'''happens}. For Erd\H{o}s--R\'enyi graph, when $np \ge C \log^\xi n$ for some constant $\xi > 3$, we aim to prove these events happen with probability at least $1-C/n^c$ as $n\rightarrow \infty$. Actually, $\mathcal{E}_{\mathrm{good}}$ and \eqref{A-P:hpb}  have been studied in Lemma~\ref{lm:good-ER} and our proof will focus on the later two events. 
We let
  \begin{align}\nonumber
  &\M \Phi^*(\M x) =  \Big(  \mathbb{I}(i \neq j)\,p\,\Phi_{ij}(\M x) \Big)_{(i,j)\in[n]^2},
  \pmb{\mathscr{L}}^*(\M x) = \mathrm{diag}\left(\M \Phi^*(\M x)\M 1\right) - \M \Phi^*(\M x),\\\nonumber
  &\M\Delta(\M \Phi(\M x)) = \M \Phi(\M x) - \M \Phi^*(\M x),\quad \M\Delta(\lap(\M x)) = \lap(\M x) - \lap^*(\M x).
  \end{align}
 For simplicity, we sometimes write $\M \Phi(\M x)$ as $\M \Phi$, similar for other matrices, when it is clear from the context. 

We first bound $\|\pmb{\mathscr{L}}^\dagger(\M x) - \M D^{-1}(\M x)\|_{\infty}$ by conducting an entrywise pseudoinverse  analysis between $\pmb{\mathscr{L}}(\M x)$ and $\M D (\M x)$. First, by \eqref{property:H},  we have
\begin{align}\nonumber
\M D(\M x)\Big(\M D^{-1} (\M x)- \pmb{\mathscr{L}}^\dagger(\M x)\Big)\pmb{\mathscr{L}}(\M x) &= \pmb{\mathscr{L}}(\M x) - \M D(\M x)+ n^{-1}\M D(\M x)\M 1\M 1^\T = - \M \Phi(\M x)+ n^{-1}\M D(\M x)\M 1\M 1^\T,
\\\nonumber
\Big(\M D^{-1}(\M x) - \pmb{\mathscr{L}}^\dagger(\M x)\Big)(\M I - n^{-1}\M 1\M 1^\T)  &= -\M D(\M x)^{-1} \M \Phi(\M x) \pmb{\mathscr{L}}^\dagger(\M x),
\end{align}
which imply that
\begin{align}\nonumber
 &\pmb{\mathscr{L}}^\dagger(\M x) - \M D^{-1}(\bds \Xi(\M x))
  =  - \left(\M D^{-1}(\bds \Xi(\M x)) - \pmb{\mathscr{L}}^\dagger(\M x) \right)(\M I- n^{-1}\M 1 \M 1^\T)  - \left(\M D^{-1}(\bds \Xi(\M x)) - \pmb{\mathscr{L}}^\dagger(\M x)\right) n^{-1}\M 1 \M 1^\T 
\\\nonumber
& \quad =  \M D^{-1}(\bds \Xi(\M x)) \bds \Xi(\M x) \pmb{\mathscr{L}}^\dagger(\M x) -  n^{-1}\M D^{-1}(\bds \Xi(\M x))   \M 1 \M 1^\T  
\\\label{HD:perturb}
& \quad =   \M D^{-1}(\bds \Xi(\M x))  \bds \Xi^*(\M x)\pmb{\mathscr{L}}^\dagger(\M x) +\M D^{-1}(\bds \Xi(\M x)) \M\Delta(\bds \Xi(\M x)) \pmb{\mathscr{L}}^\dagger(\M x) -  n^{-1}\M D^{-1}(\bds \Xi(\M x))   \M 1 \M 1^\T.
\end{align} 
On the other hand, by $\pmb{\mathscr{L}}^\dagger(\M x)\pmb{\mathscr{L}}(\M x)  = \M I - n^{-1}\M 1\M 1^\T$ in \eqref{property:H}, we have
\bee\nonumber
\pmb{\mathscr{L}}^\dagger(\M x)\left(\M\Delta(\lap(\M x)) + \pmb{\mathscr{L}}^*(\M x)\right){\pmb{\mathscr{L}}^*}^\dagger(\M x) = (\M I - n^{-1}\M 1\M 1^\T){\pmb{\mathscr{L}}^*}^\dagger(\M x),
\ee
which implies
$
\pmb{\mathscr{L}}^\dagger(\M x)   =  {\pmb{\mathscr{L}}^*}^\dagger(\M x) - \pmb{\mathscr{L}}^\dagger (\M x)\M\Delta(\lap(\M x))  {\pmb{\mathscr{L}}^*}^\dagger(\M x).
$
 Together with \eqref{HD:perturb}, we  have
\begin{align}\nonumber
\pmb{\mathscr{L}}^\dagger(\M x) - \M D(\M x)^{-1} & =  \M D^{-1}(\M x)  \M \Phi^*(\M x)  {\pmb{\mathscr{L}}^*}^\dagger(\M x) 
- \M D^{-1}(\M x)  \M \Phi^*(\M x)\pmb{\mathscr{L}}^\dagger (\M x) \M\Delta(\lap(\M x))     {\pmb{\mathscr{L}}^*}^\dagger (\M x)
\\ \label{LD:decom}
& \quad +\M D^{-1}(\M x) \M\Delta(\M \Phi(\M x)) \pmb{\mathscr{L}}^\dagger(\M x) 
 -  n^{-1}\M D^{-1}(\M x )   \M 1 \M 1^\T.
\end{align} 
Combining \eqref{sigmadiff:1} and \eqref{LD:decom}, we have
\begin{align}\nonumber
&\sup_{\M x\in\mathbb{X}}\|\pmb{\mathscr{L}}^\dagger(\M x) - \M D^{-1}(\M x)\|_{\infty}
\leq  \sup_{\M x\in\mathbb{X}}\|\M D^{-1}(\M x)  \M \Phi^*(\M x)  {\pmb{\mathscr{L}}^*}^\dagger(\M x) \|_{\infty} +  \sup_{\M x\in\mathbb{X}}\|\M D^{-1}(\M x)  \M \Phi^*(\M x)\pmb{\mathscr{L}}^\dagger (\M x)\M\Delta(\lap(\M x)){\pmb{\mathscr{L}}^*}^\dagger (\M x)\|_{\infty}\\\nonumber
 &\quad \leq  \sup_{\M x\in\mathbb{X}}\|\M D^{-1}(\M x)  \M \Phi^*(\M x)  {\pmb{\mathscr{L}}^*}^\dagger(\M x) \|_{\infty} +  \sup_{\M x\in\mathbb{X}}\|\M D^{-1}(\M x)  \M \Phi^*(\M x)\pmb{\mathscr{L}}^\dagger (\M x)\M\Delta(\lap(\M x)){\pmb{\mathscr{L}}^*}^\dagger (\M x)\|_{\infty}
\\\label{sigmaA:bound:main}
& \qquad +   \sup_{\M x\in\mathbb{X}}\|\M D^{-1}(\M x) \M\Delta(\M \Phi(\M x)) \pmb{\mathscr{L}}^\dagger(\M x)\|_{\infty} +  4\sup_{\M x\in\mathbb{X}}\|n^{-1}\M D^{-1}(\M x )   \M 1 \M 1^\T\|_{\infty} .
\end{align}

We then bound all terms on the right-hand side of \eqref{sigmaA:bound:main} through a fixed-$\M x$ pointwise analysis.  Fixing    $\M x\in\mathbb{X}$, we derive uniform  upper bounds, which do not depend on $\M x$, for these four terms. We  further omit matrices' dependence    on $\M x$ for simplicity.  Using the facts that $\|\cdot\|_{\infty}\leq \|\cdot\|$ and $\|\M A\M B\|_{\infty}\leq \|\M A\|_{\infty}\|\M B\|_{\infty}$ when $\M A$ is a diagonal matrix, we have
\begin{align}\nonumber
\|n^{-1}\M D^{-1}(\M x  )   \M 1 \M 1^\T\|_{\infty} &\leq n^{-1} \|\M D^{-1}(\M x  ) \|_{\infty}\|  \M 1 \M 1^\T\|_{\infty} \leq Cn^{-1}(np)^{-1},\\\nonumber
\|\M D^{-1}(\M x)  \M \Phi^*  {\pmb{\mathscr{L}}^*}^\dagger\|_{\infty} &\leq \|\M D^{-1}(\M x)\|_{\infty} \| \M \Phi^*  {\pmb{\mathscr{L}}^*}^\dagger\|_{\infty}  \leq C(np)^{-1}\| \M \Phi^*  {\pmb{\mathscr{L}}^*}^\dagger\|_{\infty} ,
\\\nonumber
 \|\M D^{-1}(\M x)  \M \Phi^*\pmb{\mathscr{L}}^\dagger \M\Delta(\pmb{\mathscr{L}})  {\pmb{\mathscr{L}}^*}^\dagger \|_{\infty}  & \leq \|\M D^{-1}(\M x)\|_{\infty}  \|\M \Phi^*\pmb{\mathscr{L}}^\dagger \M\Delta(\pmb{\mathscr{L}})  {\pmb{\mathscr{L}}^*}^\dagger\| \leq C(np)^{-1} \|\M \Phi^*\pmb{\mathscr{L}}^\dagger \M\Delta(\pmb{\mathscr{L}})   {\pmb{\mathscr{L}}^*}^\dagger\|,
\\\label{max:bound:1}
\| \M D^{-1}(\M x) \M\Delta(\M \Phi) \pmb{\mathscr{L}}^\dagger\|_{\infty}  &\leq \|\M D^{-1}(\M x)\|_{\infty} \|\M\Delta(\M \Phi) \pmb{\mathscr{L}}^\dagger\| \leq C(np)^{-1}\|\M\Delta(\M \Phi) \pmb{\mathscr{L}}^\dagger\|.
\end{align}
We first bound $\|\M D(\M x)^{-1}  \M \Phi^*  {\pmb{\mathscr{L}}^*}^\dagger\|_{\infty}$ in \eqref{max:bound:1}.  We introduce the matrix inversion approximation~\citep{simons1998approximating}.
\begin{lemma}(\citet{simons1998approximating})\label{lm:approx}
Let $\M T = (t_{ij})_{(i,j)\in[n]^2}$ be a symmetric matrix with strict negative off-diagonal entries and  
\bee\nonumber
\tilde{t}_i= t_{ii} + \sum_{j\in[n]/(i)}t_{ij } > 0,\text{ for all }i\in[n].
\ee
Now denote $a = \max(|t_{ij}|,\tilde{t}_i)_{i\neq j}$, $b = \min(|t_{ij}|,\tilde{t}_i)_{i\neq j}$, and  a matrix $\M S = (S_{ij})_{(i,j)\in[n]^2}$ such that
\bee\nonumber
S_{ij} = \frac{\mathbb{I}(i  =  j)}{\tilde{t}_i + \sum_{j\in[n]/(i)}|t_{ij}|}+ \frac{1}{\sum_{i\in[n]}\tilde{t}_i}.
\ee
Then we have
\bee\nonumber
\|\M T^{-1} - \M S\|_{\infty}\leq (a/b^2 + a^2/b^3) n^{-2}. 
\ee
\end{lemma}
\par
By \eqref{property:H}, we have the following identity of ${\pmb{\mathscr{L}}^*}^\dagger$ hold for any $c_n\neq 0$
\bee\label{Hdagger:decom}
{\pmb{\mathscr{L}}^*}^\dagger & = \left(\pmb{\mathscr{L}}^* + \frac{c_n p }{n}\M 1\M 1^\T\right)^{-1}- \frac{1}{nc_np } \M 1\M 1^\T  =  \pmb{\mathscr{L}}^* (c_n) ^{-1}- \frac{1}{nc_np } \M 1\M 1^\T.
\ee
Here $c_n$ is some tuning variable to be specified later and we need to choose it such that $c_n  \rightarrow \infty$ when $n\rightarrow \infty$ and satisfies
\bee\label{condition:cn}
2{c_n }/{n } \leq   c_\Phi=\min_{i\neq j}\Phi_{ij}(\M x).
\ee
Then our goal is to find a $c_n$ to minimize the error bound in Lemma~\ref{lm:approx}. We need to first verify that $\pmb{\mathscr{L}}^* (c_n) $ can be considered as $\M T$ in Lemma~\ref{lm:approx} with $\M S = \M S(c_n)= (S_{ij}(c_n))_{(i,j)\in[n]^2}$, such that $\tilde{t}_i = c_np$ and
\begin{align}\nonumber
S_{ij}(c_n) &=  \frac{\mathbb{I}(i  =  j)}{c_np/n +  p\sum_{j\in[n]/(i)}\Phi_{ij}(\M x)}+ \frac{1}{c_nnp},
\\\nonumber
a &= \max_{i\neq j}(c_n p, \Phi_{ij}(\M x)p - c_np/ n ),
\\\label{def:S}
b &= \min_{i\neq j}(c_n p, \Phi_{ij}(\M x)p - c_np/ n )\geq p\min(c_n,c_\Phi/2).
\end{align}
When $n$ is sufficiently large, such that $c_n \geq c_\Phi/2$, we have
\begin{align}\label{Hcn}
\|\pmb{\mathscr{L}}^* (c_n) ^{-1} - \M S(c_n)\|_{\infty} &\leq C\left(\frac{c_n}{p c_\Phi^{2}} + \frac{c_n^2}{p c_\Phi^{3}} \right)n^{-2}
\leq   C\frac{c_n^2}{p c_\Phi^{3}}  n^{-2}.
\end{align}
 Combining \eqref{Hdagger:decom}, \eqref{def:S} and \eqref{Hcn}, we  have when $c_n \geq \max(0.25,0.5c_\Phi)$,
\begin{align}\nonumber
\| \M \Phi^*  {\pmb{\mathscr{L}}^*}^\dagger\|_{\infty} &= \left\|\M \Phi^*\left( \left(\pmb{\mathscr{L}}^* (c_n) ^{-1}-  \M S(c_n)\right) + \M S(c_n) - \frac{1}{nc_np } \M 1\M 1^\T\right)\right\|_{\infty}
\\\nonumber
&\leq n\|\M \Phi^*\|_{\infty}\|\pmb{\mathscr{L}}^* (c_n) ^{-1} - \M S(c_n)\|_{\infty} + \|\M \Phi^*\M S(c_n)\|_{\infty} + \frac{1}{c_n p}\|\M \Phi^*\|_{\infty}
\\\nonumber
&\leq\|\M \Phi^*\|_{\infty} \left(  C\frac{  c_n^2}{np c_\Phi^{3}}    + \frac{1}{c_np}  + \frac{ 1}{np c_\Phi + c_np/n}+ \frac{1}{c_n p }\right)
\leq C\left( \frac{c_n^2}{n} + \frac{1}{c_n}+\frac{1}{n} \right),
\end{align}
where we used that $\|\M \Phi^*\|_{\infty} \leq C p$ for a constant $C$ depending on the uniform bound of $\Phi_{ij}(\M x)$. To minimize the right hand side above, we take $c_n =  c_\Phi n^{1/3}/2 $, and can check that all the conditions on $c_n$ required above are satisfied. Then when $n \geq N$ for a constant $N$ only depending on $C$, we have $\| \M \Phi^*  {\pmb{\mathscr{L}}^*}^\dagger\|_{\infty} \leq C n^{-1/3}$, and thus in \eqref{max:bound:1},
\bee\nonumber
\|\M D^{-1}(\M x)  \M \Phi^*  {\pmb{\mathscr{L}}^*}^\dagger\|_{\infty} \leq C n^{-1/3}(np)^{-1},
\ee
   where the constant $C$ only depends on $C$.
\par
Now we bound $ \|\M D^{-1}(\M x)  \M \Phi^*\pmb{\mathscr{L}}^\dagger \M\Delta(\pmb{\mathscr{L}})  {\pmb{\mathscr{L}}^*}^\dagger \|_{\infty}$ and $\| \M D^{-1}(\M x) \M\Delta(\M \Phi) \pmb{\mathscr{L}}^\dagger\|_{\infty}$ in \eqref{max:bound:1}.   We first bound the second smallest eigenvalue of $ \lap^* $. Recall that  $\lap^* $ is the  Graph Laplacian of a fully connected  (complete) graph with $p\,\Phi_{ij}(\M x)$ being the weight between node $i$ and $j$. The second smallest eigenvalue of $ \lap^* $ is also called the algebraic connectivity of such a complete graph in the spectral graph theory, and has been studied in \citet{de2007old}. Specifically, by \citet[Proposition~2.1]{de2007old}, we have
\begin{align}\nonumber
 \lambda_{n - 1}({\pmb{\mathscr{L}}^*})  &=  n\inf_{\M a= (a_1,\ldots,a_n)^\T\atop \M a\neq c\M 1,\text{ for any }c\in\R} \frac{2\sum_{(i,j)\in\mathcal{E}(\M A)}p \psi'(\theta_i^*(\M x)-\theta_j^*(\M x))(a_i - a_j)^2}{\sum_{(i,j)\in\mathcal{E}(\M A)}(a_i - a_j)^2}
\\\nonumber
&\geq Cnp \inf_{\M a= (a_1,\ldots,a_n)^\T\atop \M a\neq c\M 1,\text{ for any }c\in\R} \frac{\sum_{(i,j)\in\mathcal{E}(\M A)}(a_i - a_j)^2}{\sum_{(i,j)\in\mathcal{E}(\M A)}(a_i - a_j)^2} = Cnp.
\end{align}
This implies that $\|{\pmb{\mathscr{L}}^*}^\dagger\| \leq C(np)^{-1}.$ On the other hand, under $\mathcal{E}_{\mathrm{good}}\cap\mathcal{E}_{\mathrm{good}}'$, we have
$
\|\pmb{\mathscr{L}}^\dagger \|\leq C(np)^{-1} 
$
by \eqref{Ldaggerbound}. We also have $\|\M \Phi^*\|\leq n\|\M \Phi^*\|_{\infty}\leq Cnp.$ 
\par We now bound  $\M\Delta(\M \Phi)$. We condition on some additional high-probability events of $\M A$. First, we have
\begin{align}\nonumber
\M\Delta(\M \Phi(\M x)) & = \left(\mathbb{I}(i\neq j)(A_{ij} - p) \, \Phi_{ij}(\M x)\right)_{(i,j)\in[n]^2},
\\\nonumber
\|\M\Delta(\M \Phi(\M x))\|&= \sup_{\|\M u \| = 1}\left| \M u^\T \Big(  {\M \Phi}(\M x ) - {\M \Phi^*}(\M x)\Big) \M u\right|
= \sup_{\|\M u \| = 1} \left|\sum_{(i,j)\in[n]^2}\mathbb{I}(i\neq j)(A_{ij} - p)u_iu_j  \, \Phi_{ij}(\M x)\right|.
\end{align}
Now consider a $  (n^{-1/\beta}) $-covering, namely $\mathcal{N}_n$ over a compact space $\mathbb{X}$, whose sizes equal to the  covering number,
\bee\label{N:bound}
N\left(\mathbb{X},\|\cdot\|,n^{-1/\beta}\right) \leq C\left(3n^{ 1/\beta}\right)^d \leq C'n^{d/ \beta},
\ee
where $d$ is the dimension of $\mathbb{X}$; see e.g., \citet{vershynin2018high} for the above bound. Then by the H\"older smooth condition and the mean-value theorem, we have for any $\M x$, there exists $\tilde{\M x}\in\mathcal{N}_n$ such that $\|\M x - \tilde{\M x}\| \leq n^{-1/\beta}$ and thus for any $(i,j)\in[n]^2$,
\bee\label{distance:x:xtilde}
\left|\Phi_{ij}(\M x) - \Phi_{ij}(\tilde{\M x})\right|&\leq 
 \leq C\|\M x - \tilde{\M x}\|^{\beta} \leq C n^{-1}.
\ee
Assume $\M A$ is randomly generated following Erd\H{o}s--R\'enyi graph with probability $p$. For any fixed $\tilde{\M x}\in \mathcal{N}_n$, we bound $ 
\|\M\Delta(\M \Xi(\tilde{\M x}))\|$ through  \citet{erdHos2013spectral} (see their    Lemma~4.3 and Remark~2.4). 
\begin{lemma}(\citet{erdHos2013spectral})\label{lemma:matrix:concentration}
Let $\M H = (h_{ij})_{(i,j)\in[n]^2}$ be a random matrix with symmetrically independent entries such that
\bee\nonumber
\E h_{ij} = 0,\quad \E |h_{ij}|^2 \leq C_H/n,\quad \E |h_{ij}|^p \leq C_H^p/(n q_n^{p-2}),
\ee
for any  $(i,j)\in[n]^2$ and $3 \leq p\leq (\log n)^{A_0\log\log n}$, where $A_0$ and $C_H$ are  positive constants. Moreover, assume $(\log n)^{3\xi} \leq q_n \leq C_q n^{1/2}$ for some fixed $C_q > 0$ and $\xi > 1$.
\par
Then there exists some fixed constants $\nu,C_{1},N_1 > 0$, such that when $n \geq N_1$, with probability at least $1 - \exp( -\nu(\log n)^\xi )$,
\bee\nonumber
\|\M H\|\leq C_1\big(2 + (\log n)^\xi q_n^{-1/2}\big).
\ee
\end{lemma}
For any fixed $\tilde{\M x}\in \mathcal{N}_n$, we let $\M H = \M\Delta(\M\Phi(\tilde{\M x}))/\sqrt{np}$ and $q_n = np$ in Lemma~\ref{lemma:matrix:concentration}. Then, we have that
\bee\nonumber
\E h_{ij} = 0,\quad \E |h_{ij}|^2 \leq C/n,\quad \E |h_{ij}|^p \leq C^p/(n q_n^{p-2}),
\ee 
for any $p > 2$, where we use the fact that $|h_{ij}| \leq C/\sqrt{np}$. Thus, taking appropriate positive constants $C_H, C_1, A_0$ in Lemma~\ref{lemma:matrix:concentration}, we conclude that 
$
\|\M\Delta(\M \Phi(\tilde{\M x}))\|\leq C\sqrt{np},
$
with probability at least $1 - \exp( -\nu(\log n)^\xi )$ when $n \geq N$. Thus, with probability at least  $1 - C'n^{d/\beta} \exp( -\nu(\log n)^\xi )$ as $n\rightarrow \infty$, by \eqref{N:bound}, we have
\bee\label{good''happens}
\sup_{\tilde{\M x}\in\mathcal{N}_n}\|\M\Delta(\M \Phi(\tilde{\M x}))\|\leq C\sqrt{np}.
\ee
with probability approaching 1 as $n\rightarrow\infty$.  Given \eqref{good''happens}, we then have for any $\M x\in\mathbb{X}$,  there exists $\tilde{\M x}\in\mathcal{N}_n$ such that, $\M x$ and $\tilde{\M x}$ satisfy \eqref{distance:x:xtilde} and thus
\begin{align}\nonumber
 &\|\M\Delta(\M \Phi )\|
 = \sup_{\|\M u \| = 1} \left|\sum_{(i,j)\in[n]^2}\mathbb{I}(i\neq j)  (A_{ij} - p)u_iu_j  \, \Phi_{ij}(\M x)\right|
\\\nonumber
&\leq\underbrace{\sup_{\|\M u \| = 1} \left|\sum_{(i,j)\in[n]^2}\mathbb{I}(i\neq j)(A_{ij} - p)u_iu_j  \, \Phi_{ij}(\tilde{\M x})\right|}_{ \|\M\Delta(\M\Phi(\tilde{\M x}))\|\leq C\sqrt{np}} 
 + \sup_{\|\M u \| = 1} \sum_{(i,j)\in[n]^2} \left| u_iu_j \left( \Phi_{ij}(\M x) -  \Phi_{ij}(\tilde{\M x})\right)\right| 
\\\label{Egood3}
&\leq C\sqrt{np} + Cn^{-1}\sup_{\|\M u \| = 1}\left(\sum_{i\in[n]}\left| u_i  \right|\right)^2 
\leq C\sqrt{np} + Cn\cdot n^{-1}\leq C' \sqrt{np},
\end{align} 
where the third inequality holds by the Cauchy--Schwarz inequality.
\par
We further bound $\M D(\M x ) - \M D^*(\M x )$. For any fixed $\tilde{\M x}\in\mathcal{N}_n$, we have
\bee\nonumber
\|\M D(\tilde{\M x} ) - \M D^*(\tilde{\M x} )\| = \sup_{i\in[n]}\left|\sum_{j\in[n]}  \mathbb{I}(i\neq j)  (A_{ij} - p)   \, \Phi_{ij}(\tilde{\M x}) 
\right|.
\ee 
By Bernstein's inequality (c.f., Lemma~\ref{lm:bern}(II)), we have
\bee\nonumber
\pr\left( \left|\sum_{j\in[n]}   \mathbb{I}(i\neq j)  (A_{ij} - p)   \, \Phi_{ij}(\tilde{\M x}) \right| \geq t \right)& \leq 2\exp\left(-\frac{t^2/2}{ C^2 np + Ct/3}\right).
\ee
By taking $t = C\sqrt{(\log n)np}$ with a sufficient large $C > 0$, we have
\bee\nonumber
\left|\sum_{j\in[n]}   \mathbb{I}(i\neq j)  (A_{ij} - p)   \, \Phi_{ij}(\tilde{\M x}) \right| \leq C\sqrt{(\log n)np},
\ee
with probability at least $1 - 2n^{-d/\beta-2}$. Then uniformly over $\tilde{\M x}\in\mathcal{N}_n$ and $i\in[n]$, we have
\bee\label{good'''happens}
\sup_{\tilde{\M x}\in\mathcal{N}_n\atop i\in[n]}\left|\sum_{j\in[n]}   \mathbb{I}(i\neq j)  (A_{ij} - p)   \, \Phi_{ij}(\tilde{\M x}) \right| \leq C\sqrt{(\log n)np},
\ee
with probability at least $1 - 2C'n^{-d/\beta - 2}\cdot n\cdot |\mathcal{N}_n| = 1- 2C'n^{-1} $. 
We have for any $\M x\in\mathbb{X}$,  there exists $\tilde{\M x}\in\mathcal{N}_n$ such that, $\M x$ and $\tilde{\M x}$ satisfy \eqref{distance:x:xtilde} and thus
\begin{align}\nonumber
\|\M D(\M x )- \M D^*(\M x)\| & = \sup_{i\in[n]}\left|\sum_{j\in[n]}  \mathbb{I}(i\neq j)  (A_{ij} - p)   \, \Phi_{ij}({\M x})\Big) 
\right|
\\\nonumber
&\leq \sup_{i\in[n]}\left|\sum_{j\in[n]}  \mathbb{I}(i\neq j)  (A_{ij} - p)   \, \Phi_{ij}(\tilde{\M x}) 
\right| + C\sup_{i\in[n]}\sum_{j\in[n]}  \mathbb{I}(i\neq j)n^{-1}
\leq C\sqrt{(\log n)np}.
\end{align}
Thus for $\M A$ satisfying \eqref{good''happens} and \eqref{good'''happens}, we have
\bee\nonumber
\|\M\Delta(\pmb{\mathscr{L}})\| \leq \|\M D(\M x )- \M D^*(\M x)\| +  \|\M\Delta(\M \Phi )\| \leq  C\sqrt{(\log n)np}.
\ee
\par
Summarizing all bounds above, we have
\begin{align}\nonumber
 \|\M D^{-1}(\M x)  \M \Phi^*\pmb{\mathscr{L}}^\dagger \M\Delta(\pmb{\mathscr{L}})  {\pmb{\mathscr{L}}^*}^\dagger \|_{\infty}  &\leq C(np)^{-1}\|\M \Phi^*\|\|\pmb{\mathscr{L}}^\dagger \|\|\M\Delta(\pmb{\mathscr{L}})  \|\| {\pmb{\mathscr{L}}^*}^\dagger\|
\leq C(np)^{-1} \sqrt{\frac{\log n}{np}}, %
\\\nonumber
\| \M D^{-1}(\M x) \M\Delta(\M \Phi) \pmb{\mathscr{L}}^\dagger\|_{\infty} & \leq C(np)^{-1}\|   \M\Delta(\M \Phi) \|\|\pmb{\mathscr{L}}^\dagger\|
\leq  C(np)^{-1}\sqrt{\frac{1}{np}}, %
\end{align}
where $np \geq (\log n)^{3\xi}$ for some $\xi > 1$. Finally, we have that when $\M A$ satisfies \eqref{upperbound:A}, \eqref{good''happens}, and \eqref{good'''happens}, all four terms on the right-hand side of \eqref{LD:decom} are bounded, and thus %
\bee\label{eq:bound:maxdiff:2}
\sup_{\M x\in\mathbb{X}}\|\pmb{\mathscr{L}}^\dagger(\M x) - \M D^{-1}(\M x)\|_{\infty} \leq C (np)^{-1}\left(\frac{1}{n^{1/3}} + \sqrt{\frac{\log n}{np}}\right).
\ee
Together with \eqref{oracle:target}, we conclude the proof.
\end{proof}

\section{Technical lemmas}

\begin{lemma}(Berry-Esseen bound, \citet[Theorem~3.7]{chen2010normal})\label{lm:BE}
  For random variables $(Z_i)_{i = 1}^n$ with zero mean. Then we have
  \bee\nonumber
  \sup_{z\in\R}\left|\pr\left(\frac{1}{\sqrt{\sum_{i = 1}^n\mathrm{Var}(Z_i)}}\sum_{i = 1}^n Z_i \leq z\right) - \Phi(z)\right|\leq 10\sum_{i = 1}^n\frac{\E |Z_i|^3}{\big(\sum_{i = 1}^n\mathrm{Var}(Z_i)\big)^{3/2}}.
  \ee
  \end{lemma}

\label{sec:proof:tech-misc}
\begin{lemma}[Bernstein's inequality \citep{van1996weak}]\label{lm:bern}  Let $\{Z_{i}\}_{i = 1}^n$ be independent random variables with zero mean.
  \begin{itemize}
  \item[(I).] For random variables $\{Z_{i}\}_{i = 1}^n$ with zero mean and moment bounds
  $$
  \E|Z_i|^m \leq 0.5m!C_U^{m - 2}v_i
  $$
  for $m = 2,3,\ldots$ and $i = 1,\ldots,n$ for some constants $C_U > 0$ and $v_i > 0$, we have
  \bee\nonumber
  \pr\left(\left|\sum_{ i =1}^n Z_i\right| > x\right) \leq 2\exp\left( - \frac{x^2}{2(\sum_{i = 1}^nv_i) + 2 C_U x}\right). 
  \ee
  \item[(II).] For random variables $\{Z_i\}_{i = 1}^n$ with zero mean, uniform bound $M > 0$ such that $|Z_i| \leq M$ and finite variances, we have
  \bee\nonumber
  \pr\left(\left|\sum_{ i =1}^n Z_i\right| > x\right) \leq 2\exp\left( - \frac{x^2}{2M x /3+ 2 \sum_{i = 1}^n\E(Z_i^2)}\right).
  \ee
  \end{itemize}
  \end{lemma}

The next lemma is a perturbation bound for the Moore--Penrose inverses.
\begin{lemma}(\citet{wedin1973perturbation})\label{lm:perb:wedin}
If $\M M_1,\M M_2\in\R^{a\times b}$ and $\mathrm{Rank}(\M M_1) = \mathrm{Rank}(\M M_2)$, we have
\bee\nonumber
\left\|\M M_1^\dagger - \M M_2^\dagger\right\| \leq \frac{1 + \sqrt{5}}{2}\left\|\M M_1^\dagger  \right\| \left\|\M M_2^\dagger  \right\| \left\| \M M_1 - \M M_2\right\|.
\ee
\end{lemma}
   
  Lemmas \ref{lm:KL}--\ref{lm:inequality} are from \citet{bos2022convergence}. 

  \begin{lemma}\label{lm:KL}
  For any $B > 1$, $m = 2,3,\ldots$, and any probability vectors $\bds p = (p_1,\ldots,p_K)$ and $\bds q = (q_1,\ldots,q_K)$, we have
  \bee\nonumber
  \sum_{ k = 1}^K p_k\left|\min\left\{B,\log\left(\frac{p_k}{q_k}\right)\right\}\right|^m \leq \Big(m!\vee\frac{B^m}{B - 1}\Big)\sum_{k = 1}^K p _k\left[\min\left\{B,
  \log\left(\frac{p_k}{q_k}\right)\right\}\right].
  \ee
  \end{lemma}
  % \begin{lemma}\label{lm:expect}
  % For a nonnegative random variable $Z$ with finite expectation,  $\E Z = \int_{0}^\infty\pr(Z \geq t) dt$.
  % \end{lemma}

  \begin{lemma}\label{lm:inequality}
  For $ c,d\in\R$ and $a \geq 0$ such that $|a - b|\leq 2\sqrt{a}c + d.$
  For any $\epsilon\in(0,1]$, we have
  $
  a \leq (1 +\epsilon)(b+d) + \frac{(1 + \epsilon)^2}{\epsilon}c^2.
  $
  \end{lemma}

  \begin{proof}[Proof of Lemma~\ref{lm:inequality}]
  From $|a-b| \leq 2\sqrt{a}\,c + d$ and $a\ge 0$ we have
  $a \leq b + d + 2\sqrt{a}\,c.$ Fix $\epsilon\in(0,1]$ and set $\eta = \epsilon/(1+\epsilon)\in(0,1)$. By Young's inequality $2xy \leq \eta x^2 + \eta^{-1}y^2$ with $x = \sqrt{a}$ and $y = c$, we obtain
  $2\sqrt{a}\,c \leq \eta a + \eta^{-1}c^2.$ Hence
  $a \leq b + d + \eta a + \eta^{-1}c^2$, which implies $(1-\eta)a \leq b + d + \eta^{-1}c^2$. Noting that $1-\eta = 1/(1+\epsilon)$ and $\eta^{-1} = (1+\epsilon)/\epsilon$, we conclude
  $a \leq (1+\epsilon)(b+d) + \frac{(1+\epsilon)^2}{\epsilon}c^2.$
  \end{proof}

\end{document}